\documentclass{article}

\usepackage{microtype}
\usepackage{graphicx}
\usepackage{subfigure}
\usepackage{booktabs} 
\usepackage{tablefootnote}
\usepackage{multirow}
\usepackage[font=small,skip=0pt]{caption}
\usepackage{enumitem}
\newenvironment{smallbmatrix}
  {\left[\begin{smallmatrix}}
  {\end{smallmatrix}\right]}

\usepackage{siunitx}

\usepackage{hyperref}

\usepackage[accepted]{icml2023}

\usepackage{amsmath}
\usepackage{amssymb}
\usepackage{mathtools}
\usepackage{amsthm}
\usepackage{bm}

\usepackage[capitalize,noabbrev]{cleveref}

\theoremstyle{plain}
\newtheorem{theorem}{Theorem}[section]

\newtheorem{lemma}[theorem]{Lemma}

\theoremstyle{definition}
\newtheorem{definition}[theorem]{Definition}

\theoremstyle{remark}
\newtheorem{remark}[theorem]{Remark}

\usepackage[textsize=tiny]{todonotes}

\usepackage{titletoc}
\newcommand\DoToC{%
  \startcontents
  \printcontents{}{1}{\textbf{Table of Contents}\vskip3pt\hrule\vskip5pt}
  \vskip3pt\hrule\vskip5pt
}

\icmltitlerunning{Neural Collapse in Deep Linear Networks: From Balanced to Imbalanced Data}

\allowdisplaybreaks[1]
\begin{document}

\twocolumn[
\icmltitle{Neural Collapse in Deep Linear Networks: From Balanced to Imbalanced Data}
\icmlsetsymbol{cofirst}{*}
\icmlsetsymbol{colast}{**}

\begin{icmlauthorlist}
\icmlauthor{Hien Dang}{cofirst,fpt}
\icmlauthor{Tho Tran}{cofirst,fpt}
\icmlauthor{Stanley Osher}{ucla}
\icmlauthor{Hung Tran-The}{deakin}
\icmlauthor{Nhat Ho}{colast,austin}
\icmlauthor{Tan Nguyen}{colast,nus}
\end{icmlauthorlist}

\icmlaffiliation{fpt}{FPT Software AI Center, Vietnam}
\icmlaffiliation{ucla}{Department of Mathematics, University of California, Los Angeles, USA}
\icmlaffiliation{deakin}{Applied Artificial Intelligence Institute, Deakin University, Victoria, Australia}
\icmlaffiliation{austin}{Department of Statistics and Data
Sciences, University of Texas at Austin, USA}
\icmlaffiliation{nus}{Department of Mathematics, National University of Singapore, Singapore}


\icmlcorrespondingauthor{Hien Dang}{danghoanghien1123@gmail.com}
\icmlcorrespondingauthor{Tho Tran}{thotranhuu99@gmail.com}

\icmlkeywords{Machine Learning, ICML}

\vskip 0.3in
]
\printAffiliationsAndNotice{\icmlEqualContribution}
\begin{abstract}
Modern deep neural networks have achieved impressive performance on tasks from image classification to natural language processing. Surprisingly, these complex systems with massive amounts of parameters exhibit the same structural properties in their last-layer features and classifiers across canonical datasets when training until convergence.  
In particular, it has been observed that the last-layer features collapse to their class-means, and those class-means are the vertices of a simplex Equiangular Tight Frame (ETF).
This phenomenon is known as Neural Collapse ($\mathcal{NC}$). Recent papers have theoretically shown that $\mathcal{NC}$ emerges in the global minimizers of training problems with the simplified ``unconstrained feature model''. In this context, we take a step further and prove the $\mathcal{NC}$ occurrences in deep linear networks for the popular mean squared error (MSE) and cross entropy (CE) losses, showing that global solutions exhibit $\mathcal{NC}$ properties across the linear layers.
Furthermore, we extend our study to imbalanced data for MSE loss and present the first geometric analysis of $\mathcal{NC}$ under bias-free setting. 
Our results demonstrate the convergence of the last-layer features and classifiers to a geometry consisting of orthogonal vectors, whose lengths depend on the amount of data in their corresponding classes. Finally, we empirically validate our theoretical analyses on synthetic and practical network architectures with both balanced and imbalanced scenarios.
\end{abstract}

\section{Introduction}
\label{sec:introduction}

Despite the impressive performance of deep neural networks (DNNs) across areas of machine learning and artificial intelligence~\cite{Krizhevsky12, Simonyan14, Good16, He15, Huang17, Brown20}, the highly non-convex nature of these systems, as well as their massive number of parameters, ranging from hundreds of millions to hundreds of billions, impose a significant barrier to having a concrete theoretical understanding of how they work. Additionally, a variety of optimization algorithms have been developed for training DNNs, which makes it more challenging to analyze the resulting trained networks and learned features~\cite{Ruder16}. In particular, the modern practice of training DNNs includes training the models far beyond \emph{zero error} to achieve \emph{zero loss} in the terminal phase of training (TPT)~\cite{Ma17, Belkin19, Belkin18}. A mathematical understanding of this training paradigm is important for studying the generalization and expressivity properties of DNNs~\cite{Papyan20, Han21}.

 Recently, \cite{Papyan20} has empirically discovered an  intriguing  phenomenon, named Neural Collapse ($\mathcal{NC}$), which reveals a common pattern of the learned deep representations across canonical datasets and architectures in image classification tasks. 
\cite{Papyan20} defined Neural Collapse as the existence of the following four properties:

$(\mathcal{NC}1)$ \textbf{Variability collapse:} features of the same class converge to a unique vector, as training progresses.
    
$(\mathcal{NC}2)$ \textbf{Convergence to simplex ETF:} the optimal class-means have the same length and are equally and maximally pairwise seperated, i.e., they form a simplex Equiangular Tight Frame (ETF).
    
$(\mathcal{NC}3)$ \textbf{Convergence to self-duality:} up to rescaling, the class-means and classifiers converge on each other. 
    
$(\mathcal{NC}4)$ \textbf{Simplification to nearest class-center:} given a feature, the classifier converges to choosing whichever class has the nearest class-mean to it.

Theoretically, it has been proven that $\mathcal{NC}$ emerges in the last layer of DNNs during TPT when the models belong to the class of ``unconstrained features model'' (UFM) \cite{Mixon20} and trained with cross-entropy (CE) loss or mean squared error (MSE) loss.
With regard to classification tasks, CE is undoubtedly the most popular loss function to train neural networks. However, MSE has recently been shown to be effective for classification tasks, with comparable or even better generalization performance than CE loss \cite{Hui20, Demirkaya20, Zhou22b}. 

\textbf{Contributions:} We provide a thorough analysis of the global solutions to the training deep linear network problem with MSE and CE losses under the unconstrained features model defined in Section~\ref{sec:UFM_setting}. Moreover, we study the geometric structure of the learned features and classifiers under a more practical setting where the dataset is imbalanced among classes. Our contributions are three-fold:

\textbf{1. UFM + MSE +  balanced + deep linear network:} We provide the \emph{first mathematical analysis of the global solutions for deep linear networks with arbitrary depths and widths under UFM setting}, showing that the global solutions exhibit $\mathcal{NC}$ properties and how adding the bias term can affect the collapsed structure, when training the model with the MSE loss and balanced data. 

\textbf{2. UFM + MSE +  imbalanced + plain/deep linear network:} We provide the \emph{first geometric analysis for the plain UFM}, which includes only one layer of weight after the unconstrained features, when training the model with the MSE loss and imbalanced data. This result for the plain UFM case sheds light on the geometry of the optimal last-layer classifier and last-layer features of deep non-linear networks, since this setting is consistent with practical overparameterized non-linear networks. Additionally, we also generalize this  setting to the deep linear network one. 

\textbf{3. UFM + CE +  balanced + deep linear network:} We study deep linear networks trained with CE loss
and demonstrate the existence of $\mathcal{NC}$ for any global minimizes in this setting.

\noindent
\textbf{Related works:} In recent years, there has been a rapid increase in interest in $\mathcal{NC}$, resulting in a decent amount of works in a short period of time. Under UFM, these works studied different training problems, proving ETF and $\mathcal{NC}$ properties are exhibited by any global solutions of the loss functions. In particular, a line of works use UFM with CE training to analyze theoretical abstractions of $\mathcal{NC}$ \cite{Zhu21, Fang21, Lu20, yaras23}. Other works study UFM with MSE loss \cite{Tirer22, Zhou22a, Ergen20, Rangamani22}. $\mathcal{NC}$ phenomenon has also been observed and analyzed for supervised contrastive loss \cite{graf23}. For MSE loss, recent extensions to account for additional layers in the analysis with non-linearity are studied in \cite{Tirer22, Rangamani22}, or with batch normalization \cite{Ergen20}. However, these works require strong assumptions on the global optimal solution or the network architecture/capability for their theoretical results to be hold (see Appendix \ref{sec:relatedworks} for more details). On the other hand, \cite{Zhu21, Zhou22a, Zhou22b} have shown the benign optimization landscape for several loss functions under the plain UFM setting, demonstrating that critical points can only be global minima or strict saddle points. Another line of work exploits the ETF structure to improve the network design by initially fixing the last-layer linear classifier as a simplex ETF and not performing any subsequent learning \cite{Zhu21, Yang22}.


Most recent papers study $\mathcal{NC}$ in a balanced setting, i.e., the number of training samples in every class is identical. This setting is vital for the existence of the ETF structure. To the best of our knowledge, $\mathcal{NC}$ with imbalanced data is studied in \cite{Fang21, Christos22, Yang22, Xie22}. In particular, \cite{Fang21} is the first to observe that for imbalanced setting, the collapse of features within the same class is preserved, but the geometry skew away from the ETF. \cite{Christos22} theoretically studies the SVM problem, whose global minima follows a more general geometry than the simplex ETF, called ``SELI''. However, this work also makes clear that the unregularized version of CE loss only converges to KKT points of the SVM problem, which are not necessarily global minima. \cite{Yang22} studies the imbalanced setting but with fixed last-layer linear classifiers initialized as a simplex ETF right at the beginning and proves the optimal last-layer features will converge to ETF structure.

Analyzing deep linear networks is an important step in studying deep nonlinear networks. The theoretical analysis of deep nonlinear networks is very challenging and, in fact, there has been no rigorous theory for deep nonlinear networks yet to the best of our knowledge. Thus, deep linear networks have been studied to provide insights into the behavior of deep nonlinear networks. \cite{Saxe14, Kawaguchi16, Laurent17, Hardt18} show that the optimization of deep linear models exhibits similar properties to those of the optimization of deep nonlinear models. As pointed out in \cite{Saxe14}, despite the linearity of their input-output map, deep linear networks have nonlinear gradient descent dynamics on weights that change with the addition of each new hidden layer. This nonlinear learning phenomenon is proven to be similar to those seen in deep nonlinear networks. On the other hand, in practice, deep linear networks can help improve the training and performance of deep nonlinear networks \cite{huh23, guo21, arora18}. For example, \cite{huh23} empirically proves that linear overparameterization in nonlinear networks improves generalization on classification tasks. In particular, \cite{huh23} expands each linear layer into a succession of multiple linear layers and does not include any non-linearities in between, which results in a considerable increase in performance.

Due to space considerations, we defer a full discussion of related works to Appendix \ref{sec:relatedworks}. A comparison of our results with existing works regarding the study of $\mathcal{NC}$ global optimality conditions is shown in Table \ref{table:1} in Appendix \ref{sec:relatedworks}.

\textbf{Notation:} For a weight matrix $\mathbf{W}$, we use $\mathbf{w}_{j}$ to denote its $j$-th row vector. $\| . \|_{F}$ denotes the Frobenius norm of a matrix and $\| . \|_{2}$ denotes $l_{2}$-norm of a vector.  $\otimes$ denotes the Kronecker product. The symbol ``$\propto$'' denotes proportional, i.e, equal up to a positive scalar. Moreover, we denote the best rank-$k$ approximation of a matrix $\mathbf{A}$ as $\mathcal{P}_{k}(\mathbf{A})$. We also use some common matrix notations: $\mathbf{1}_{n}$ is the all-ones vector,   $\operatorname{diag}\{a_{1},\ldots, a_{K}\}$ is a square diagonal matrix size $K \times K$ with diagonal entries $a_{1},\ldots, a_{K}$.

\section{Problem Setup}
\label{sec:problemsetup}
We consider the classification task with $K$ classes. Let $n_{k}$ denote the number of training samples of class $k$, $\forall \: k \in [K]$ and $N:= \sum_{k=1}^{K} n_{k}$.
A typical deep neural network $\mathbf{\psi} (\cdot): \mathbb{R}^{D} \to \mathbb{R}^{K}$ can be expressed as follows:
\begin{align}
    \mathbf{\psi} (\mathbf{x}) = \mathbf{W} \mathbf{\phi}(\mathbf{x}) + \mathbf{b}, \nonumber
\end{align}
where $\phi (\cdot): \mathbb{R}^{D} \to \mathbb{R}^{d}$ is the feature mapping, and $\mathbf{W} \in \mathbb{R}^{K \times d}$ and $\mathbf{b} \in \mathbb{R}^{K}$ are the last-layer linear classifiers and bias, respectively. Formally, the feature mapping $\phi(.)$ consists of a multilayer nonlinear compositional mapping, which can be written as:
\begin{align}
    \phi_{\theta}(\mathbf{x}) =
    \sigma(\mathbf{W}_{L} \ldots
    \sigma(\mathbf{W}_{1} \mathbf{x} + \mathbf{b}_{1}) + \mathbf{b}_{L} ), \nonumber
\end{align}
where $\mathbf{W}_{l}$ and $\mathbf{b}_{l}$, $l=1,\ldots,L$, are the weight matrix and bias at layer $l$, respectively. Here, $\sigma(\cdot)$ is a nonlinear activation function. Let $\theta:=\{\mathbf{W}_{l}, \mathbf{b}_{l}\}_{l=1}^{L}$ be the set of  parameters in the feature mapping and $\Theta := \left\{ \mathbf{W}, \mathbf{b}, \theta \right\}$ be the set of all network's parameters. We solve the following optimization problem to find the optimal values for $\Theta$:
\begin{align}
    \min_{\Theta} 
    \sum_{k=1}^{K}
    \sum_{i=1}^{n_{k}}
    \mathcal{L}(\psi(\mathbf{x}_{k,i}), \mathbf{y}_{k})
    + \frac{\lambda}{2} \| \Theta \|_F^2,
    \label{eq:loss}
\end{align}
where $\mathbf{x}_{k,i}\in \mathbb{R}^{D}$ is the $i$-th training sample in the $k$-th class, and   $\mathbf{y}_{k} \in \mathbb{R}^{K}$ denotes its corresponding label, which is a one-hot vector whose $k$-th entry is 1 and other entries are 0. Also, $\lambda > 0$ is the regularization hyperparameter that control the impact of the weight decay penalty, and $\mathcal{L}(\psi(\mathbf{x}_{k,i}), \mathbf{y}_{k})$ is the loss function that measures the difference between the output $\psi(\mathbf{x}_{k,i})$ and the target $\mathbf{y}_{k}$.

\begin{figure}[t]
    \centering    \includegraphics[width=.47\textwidth]{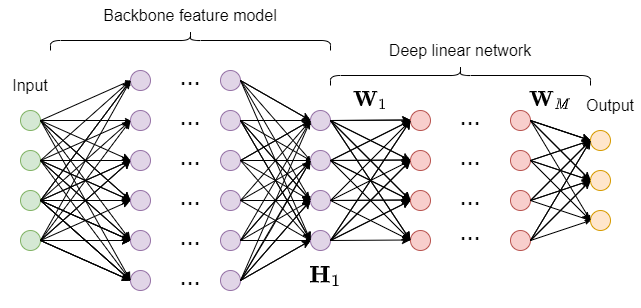}
    \vspace{-1em}
    \caption{Illustration of UFM, followed by linear layers.}
    \label{fig:DNN}
\end{figure}

\begin{figure}[!t]
\centering
\subfigure[OF (Thm. \ref{thm:bias-free})]{
\includegraphics[width=.18\textwidth]{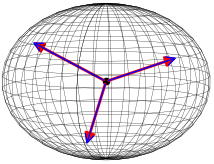}
}
\hspace{5mm}
\subfigure[ETF (Thm. \ref{thm:bias-free})]{
\includegraphics[width=.18\textwidth]{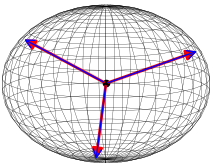}
}
\hspace{5mm}
\subfigure[GOF (Thm. \ref{thm:UFM_imbalance})]{
\includegraphics[width=.18\textwidth]{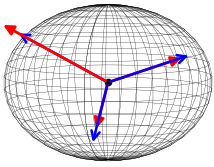}
}
\hspace{8mm}
\subfigure{
\includegraphics[width=.12\textwidth]{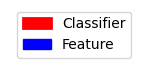}
}
\vspace{-0.75em}
\caption{Visualization of geometries of Frobenius-normalized classifiers and features with $K = 3$ classes. For imbalanced example, the number of samples for each class is 30, 10, and 5.}
\label{fig:geo}
\vspace{-0.1in}
\end{figure}

\subsection{Formulation under Unconstrained Features Model}
\label{sec:UFM_setting}
Following recent studies of the $\mathcal{NC}$ phenomenon, we adopt the \emph{unconstrained features model (UFM)} in our setting. UFM treats the last-layer features $\mathbf{h}=\phi(\mathbf{x}) \in \mathbb{R}^{d}$ as free optimization variables. This relaxation can be justified by the well-known result that an overparameterized deep neural network can  approximate any continuous function \cite{Hornik89, Hornik91, Zhou18, Yarotsky18}. Using the UFM, we consider the following slight variant of \eqref{eq:loss}: 
\begin{align}
    \min_{\mathbf{W}, \mathbf{H}, \mathbf{b}}
    f(\mathbf{W}, \mathbf{H}, \mathbf{b}) &:=
    \frac{1}{2 N}
    \sum_{k=1}^{K}
    \sum_{i=1}^{n_{k}}
    \mathcal{L}(\mathbf{W} \mathbf{h}_{k,i} + \mathbf{b}, \mathbf{y}_{k})  \nonumber\\
    + \frac{\lambda_{W}}{2} \| \mathbf{W} \|_F^2
    &+  \frac{\lambda_{H}}{2} \| \mathbf{H} \|_F^2
    + \frac{\lambda_{b}}{2} \| \mathbf{b} \|_2^2,    
\label{eq:UFM}
\end{align}
where $\mathbf{h}_{k,i}$ is the feature of the $i$-th training sample in the $k$-th class. We let $\mathbf{H}:=[\mathbf{h}_{1,1},\ldots, \mathbf{h}_{1,n_{1}},\mathbf{h}_{2,1},\ldots, \mathbf{h}_{K,n_{K}}] \in \mathbb{R}^{d \times N}$ be the matrix of unconstrained features. The feature class-means and global-mean are computed as $\mathbf{h}_{k} := n_{k}^{-1} \sum_{i=1}^{n_{k}} \mathbf{h}_{k,i}$ for $k=1,\ldots,K$ and $\mathbf{h_{G}} := N^{-1} \sum_{k=1}^{K} \sum_{i=1}^{n_{k}} \mathbf{h}_{k,i}$, respectively. In this paper, we also denote $\mathbf{H}$ by $\mathbf{H_1}$ and use these notations interchangeably.

\textbf{Extending UFM to the setting with $M$ linear layers:} $\mathcal{NC}$ phenomenon has been studied extensively for different loss functions under UFM but with only 1 to 2 layers of weights. In this work, we study $\mathcal{NC}$ under UFM in its significantly more general form with $M \ge 2$ linear layers by generalizing  \eqref{eq:UFM} to deep linear networks with arbitrary depths and widths (see Fig. \ref{fig:DNN} for an illustration). We consider the following generalization of \eqref{eq:UFM} in the $M$-linear-layer setting:
\begin{align}
    &\underset{\substack{\mathbf{W}_{M},\ldots, \mathbf{W}_{1} \\ \mathbf{H}_{1}, \mathbf{b}}}{\text{min}} \,\frac{1}{2N} \sum_{k=1}^{K}
    \sum_{i=1}^{n_{k}}
    \mathcal{L}(\mathbf{W}_{M} \mathbf{W}_{M-1} \ldots \mathbf{W}_{1} \mathbf{h}_{k,i}+ \mathbf{b}, \mathbf{y}_{k}) 
    \nonumber \\ 
    &+ \frac{\lambda_{W_{M}}}{2} \| \mathbf{W}_{M} \|^2_F + \frac{\lambda_{W_{M-1}}}{2} \| \mathbf{W}_{M-1} \|^2_F +
    \ldots  \nonumber \\
    &+
    \frac{\lambda_{W_{1}}}{2} \| \mathbf{W}_{1} \|^2_F
    +
    \frac{\lambda_{H_{1}}}{2} \| \mathbf{H}_{1} \|^2_F + \frac{\lambda_{b}}{2} \| \mathbf{b} \|_{2}^{2},
    \label{eq:UFM_general_linear}
\end{align}
where $M \geq 2$, $\lambda_{W_{M}}, \ldots,  \lambda_{W_{1}}, \lambda_{H_{1}}, \lambda_{b} > 0$ are regularization hyperparameters, 
and $\mathbf{W}_{M} \in \mathbb{R}^{K \times d_{M}}$, $\mathbf{W}_{M-1} \in \mathbb{R}^{d_{M} \times d_{M-1}}, \ldots, \mathbf{W}_{1}  \in \mathbb{R}^{d_{2} \times d_{1}}$ with $d_{M}, d_{M-1}, \ldots, d_{1}$ are arbitrary positive integers. In our setting, we do not consider the biases of intermediate hidden layers.

\noindent \textbf{Imbalanced data:} Without loss of generality, we assume $n_{1} \geq n_{2} \geq \ldots \geq n_{K}$. This setting is more general than those in previous works, where only two different class sizes are considered, i.e., the majority classes of $n_{A}$ training samples and the minority classes of $n_{B}$ samples with the imbalance ratio $R:= n_{A} /  n_{B} > 1 \label{eq:im_ratio}$~\cite{Fang21, Christos22}.

\noindent We now define the ``General Orthogonal Frame'' (GOF), which is the convergence geometry of the class-means and classifiers in imbalanced MSE training problem with no bias (see Section \ref{sec:imbalance}). 

\begin{definition}[General Orthogonal Frame]
\label{def:GOF}
    A standard general orthogonal frame (GOF) is a collection of points in $\mathbb{R}^{K}$ specified by the columns of:
    \begin{align}
        \mathbf{N} = \frac{1}{\sqrt{\sum_{k=1}^{K} a_{k}^{2}}} \operatorname{diag}(a_{1}, a_{2}, \ldots, a_{K}), \: a_{i} > 0 \: \: \forall \: i \in [K]. \nonumber
    \end{align}
    We also consider the general version of GOF as a collection of points in $
    \mathbb{R}^{d} \: (d \geq K)$ specified by the columns of $ \mathbf{P} \mathbf{N}$ where $\mathbf{P} \in \mathbb{R}^{d \times K}$ is an orthonormal matrix, i.e. $\mathbf{P}^{\top} \mathbf{P} = \mathbf{I}_{K}$. In the special case where $a_{1} = a_{2} = \ldots = a_{K}$, we have $\mathbf{N}$ follows OF structure in \cite{Tirer22}, i.e., $\mathbf{N}^{\top} \mathbf{N} \propto \mathbf{I}_{K}$. 
    Fig. \ref{fig:geo} shows a visualization for GOF versus OF and ETF in~\cite{Papyan20}.
\end{definition}

\section{Neural Collapse in Deep Linear Networks under the UFM Setting with Balanced Data}
\label{sec:mainresults}

\noindent In this section, we present our study on the global optimality conditions for the $M$-layer deep linear networks ($M \geq 2$), trained with the MSE loss under the balanced setting, i.e., $n_{1} = n_{2} = \ldots = n_{K} := n$, extending the prior results that consider only one or two hidden layers. We consider the following optimization problem for training the model:
\begin{align}
    &\underset{\substack{\mathbf{W}_{M},\ldots, \mathbf{W}_{1} \\ \mathbf{H}_{1}, \mathbf{b}}}{\text{min}} \,
    \frac{1}{2N} \| \mathbf{W}_{M} \mathbf{W}_{M-1} \ldots  \mathbf{W}_{1}
    \mathbf{H}_{1} + \mathbf{b} \mathbf{1}^{\top}_{n} - \mathbf{Y} \|_F^2  \nonumber \\
    &+ \frac{\lambda_{W_{M}}}{2} \| \mathbf{W}_{M} \|^2_F +
    \ldots +
    \frac{\lambda_{W_{1}}}{2} \| \mathbf{W}_{1} \|^2_F +
    \frac{\lambda_{H_{1}}}{2} \| \mathbf{H}_{1} \|^2_F, \label{eq:bias_free}
\end{align}
where $\mathbf{Y}=\mathbf{I}_{K} \otimes \mathbf{1}_{n}^{\top} \in \mathbb{R}^{K \times N}$ is the one-hot vectors matrix. Note that~\eqref{eq:bias_free} is a special case of~\eqref{eq:UFM_general_linear} when $\lambda_{b_{M}} = 0$.

We further consider two different settings from \eqref{eq:bias_free}: (i) bias-free, i.e., excluding $\mathbf{b}$, and (ii) last-layer unregularized bias, i.e., including $\mathbf{b}$. We now state the characteristics of the global solutions to these problems. 

\begin{theorem}
\label{thm:bias-free}
Let $R : = \min(K, d_{M}, d_{M-1}, \ldots, d_{2}, d_{1})$ and $\left(\mathbf{W}_{M}^*, \mathbf{W}_{M-1}^*, \ldots, \mathbf{W}_1^*, \mathbf{H}_1^*, \mathbf{b}^* \right)$ be any global minimizer of \eqref{eq:bias_free}. Denoting $a:= K \sqrt[M]{K n  \lambda_{W_{M}} \lambda_{W_{M-1}} \ldots \lambda_{W_{1}} \lambda_{H_{1}}}$, then the following results hold for both (i) bias-free setting with $\mathbf{b}^*$ excluded and (ii) last-layer unregularized bias setting with $\mathbf{b}^*$ included:
\begin{itemize}
    \item[(a)] If $a < \frac{(M-1)^{\frac{M-1}{M}}}{M^{2}}$, we have: \\
    
    ($\mathcal{NC}1$) $\mathbf{H}_1^*=\overline{\mathbf{H}}^{*} \otimes \mathbf{1}_n^{\top}$, where $\overline{\mathbf{H}}^{*} = [\mathbf{h}_{1}^{*}, \ldots, \mathbf{h}_{K}^{*}]  \in \mathbb{R}^{d \times K}$ and $\mathbf{b}^{*} = \frac{1}{K} \mathbf{1}_{K}$.
    \\
    
     ($\mathcal{NC}2$) $\forall \: j = 1,\ldots,M$ :
        \begin{align}
            \mathbf{W}_{M}^{\ast} \mathbf{W}_{M}^{\ast
            \top} \propto 
            \overline{\mathbf{H}}^{\ast \top} \overline{\mathbf{H}}^{\ast}
            \propto
            \mathbf{W}_{M}^{\ast} \mathbf{W}_{M-1}^{\ast} 
            \ldots 
            \overline{\mathbf{H}}^{*}  
            \nonumber \\ \propto
            (\mathbf{W}_{M}^{\ast} \mathbf{W}_{M-1}^{\ast} \ldots \mathbf{W}_{j}^{\ast})(\mathbf{W}_{M}^{\ast} \mathbf{W}_{M-1}^{\ast} \ldots \mathbf{W}_{j}^{\ast})^{\top}
             \nonumber    
        \end{align}
        and align to:
        \begin{itemize}
            \item[(i)] OF structure if \eqref{eq:bias_free} is bias-free:
            \begin{align}
                \left\{\begin{matrix}
                \mathbf{I}_{K} &\text{if } R \geq K \\
                \mathcal{P}_{R}(\mathbf{I}_{K}) & \text{if } R < K
                \end{matrix}\right. .
                \nonumber
            \end{align}
            \item[(ii)] ETF structure if \eqref{eq:bias_free} has last-layer bias $\mathbf{b}$:
            \begin{align}
                \left\{\begin{matrix}
                \mathbf{I}_{K} - \frac{1}{K} \mathbf{1}_{K} \mathbf{1}_{K}^{\top}&\text{if } R \geq K - 1 \\
                \mathcal{P}_{R} \left( \mathbf{I}_{K} - \frac{1}{K} \mathbf{1}_{K} \mathbf{1}_{K}^{\top} \right) & \text{if } R < K - 1
                \end{matrix}\right. .
                \nonumber
            \end{align}
        \end{itemize}

    ($\mathcal{NC}3$) $\forall \: j = 1,\ldots,M$:
    \begin{align}
    \begin{gathered}
        \mathbf{W}_{M}^{\ast} \mathbf{W}_{M-1}^{\ast} \ldots \mathbf{W}_{1}^{\ast} \propto \overline{\mathbf{H}}^{* \top}, \\
        \mathbf{W}_{M}^{*} \mathbf{W}_{M-1}^{*} \ldots \mathbf{W}_{j}^{*} \propto  (\mathbf{W}_{j-1}^{*} \ldots \mathbf{W}_{1}^{*} \overline{\mathbf{H}}^{*})^{\top}. \nonumber
    \end{gathered}
    \end{align}

    
\item[(b)] If $a > \frac{(M-1)^{\frac{M-1}{M}}}{M^{2}}$, \eqref{eq:bias_free} only has trivial global minima $(\mathbf{W}_{M}^*, \mathbf{W}_{M-1}^*, \ldots, \mathbf{W}_1^*, \mathbf{H}_1^{*}, \mathbf{b}^{*} ) = (\mathbf{0}, \mathbf{0},\ldots, \mathbf{0}, \mathbf{0}, \frac{1}{K} \mathbf{1}_{K})$.
\\

\item[(c)] If $a = \frac{(M-1)^{\frac{M-1}{M}}}{M^{2}}$, \eqref{eq:bias_free} has trivial global solution $(\mathbf{W}_{M}^*, \ldots, \mathbf{W}_1^*, \mathbf{H}_1^{*}, \mathbf{b}^{*} ) = (\mathbf{0}, .., \mathbf{0}, \mathbf{0}, \frac{1}{K} \mathbf{1}_{K})$ and nontrivial global solutions that have the same $(\mathcal{NC}1)$ and $(\mathcal{NC}3)$ properties as case (a).
\\

For $(\mathcal{NC}2)$ property, for $j = 1, \ldots, M$, we have:
    \begin{align}
         \mathbf{W}_{M}^{\ast} \mathbf{W}_{M}^{\ast \top} \propto 
            \overline{\mathbf{H}}^{\ast \top} \overline{\mathbf{H}}^{\ast}
            \propto
            \mathbf{W}_{M}^{\ast} \mathbf{W}_{M-1}^{\ast} 
            \ldots 
            \overline{\mathbf{H}}^{*}  \propto
            \nonumber \\
            (\mathbf{W}_{M}^{\ast} \mathbf{W}_{M-1}^{\ast} \ldots \mathbf{W}_{j}^{\ast})(\mathbf{W}_{M}^{\ast} \mathbf{W}_{M-1}^{\ast} \ldots \mathbf{W}_{j}^{\ast})^{\top}
             \nonumber    
        \end{align}
        and align to:
        \begin{align}
            \left\{\begin{matrix}
            \mathcal{P}_{r} \left( \mathbf{I}_{K} \right) &\text{if \eqref{eq:bias_free} is bias-free}  \\
            \mathcal{P}_{r} \left( \mathbf{I}_{K} - \frac{1}{K} \mathbf{1}_{K} \mathbf{1}_{K}^{\top} \right) &\text{if \eqref{eq:bias_free} has last-layer bias}
            \end{matrix}\right. , \nonumber
        \end{align}
    with $r$ is the number of positive singular value of $\overline{\mathbf{H}}^{*}$.
    \end{itemize}
\end{theorem}

Our proofs (in Appendix \ref{sec:proofs_balanced}) first characterize critical points of the loss function, showing that the weight matrices of the network have the same set of singular values, up to a factor depending on the weight decay. Then, we use the singular value decomposition on these weight matrices to
transform the loss function into a function of singular values of $\mathbf{W}_{1}$ and singular vectors of $\mathbf{W}_{M}$. Due to the separation of the singular values/vectors in the expression of the loss function, we can optimize each one individually. This method shares some similarities with the proof for bias-free case in \cite{Tirer22} where they transform a lower bound of the loss function into a function of singular values. Furthermore, the threshold $(M-1)^{\frac{M-1}{M}}/{M^{2}}$  of the constant $a$
is derived from the minimizer of the function $g(x) = 1/(x^{M} + 1) + bx$ for $x \geq 0$. For instance, if $b > (M-1)^{\frac{M-1}{M}}/M$, $g(x)$ is minimized at $x = 0$ and the optimal singular values will be $0$'s, leading to the stated solution.

The main difficulties and novelties of our proofs for deep linear networks are: 
i) we observe that the product of many matrices can be simplified by using SVD with identical orthonormal bases between consecutive weight matrices  (see Lemma \ref{lm:4}) and, thus, only the singular values of $\mathbf{W}_{1}$ and left singular vectors of $\mathbf{W}_{M}$ remain in the loss function, 
ii) optimal singular values are related to the minimizer of the function $g(x) = 1/(x^M + 1) + bx$ (see Appendix \ref{sec:study_g}), and iii) we study the properties of optimal singular vectors to derive the geometries of the global solutions. 

Theorem \ref{thm:bias-free} implies the following interesting results:
\begin{itemize}[leftmargin=1em]
\itemsep0em 
\vspace{-1em}
    \item \textbf{Features collapse}: For each $k \in [K]$, with class-means matrix $\overline{\mathbf{H}}^{*} = [\mathbf{h}^{*}_{1}, \ldots, \mathbf{h}^{*}_{K}] \in \mathbb{R}^{d \times K}$, we have $\mathbf{H}_{1}^{*} = \overline{\mathbf{H}}^{*} \otimes \mathbf{1}_{n}^{\top}$, implying the collapse of features within the same class to their class-mean. 
    
    \item \textbf{Convergence to OF/Simplex ETF:} The class-means matrix, the last-layer linear classifiers, or the product of consecutive weight matrices converge to OF in the case of bias-free and simplex ETF in the case of having last-layer bias. This result is consistent with the one and two-layer cases in \cite{Tirer22, Zhou22a}.
  
    \item \textbf{Convergence to self-duality:} If we separate the product $\mathbf{W}_{M}^{*} \ldots \mathbf{W}_{1}^{*} \overline{\mathbf{H}}^{*}$ (once) into any two components, they will be perfectly aligned to each other up to rescaling. This generalizes from the previous results which demonstrate that the last-layer linear classifiers are perfectly matched with the class-means after rescaling.
\end{itemize}

\begin{remark}
The convergence of the class-means matrix to OF/ETF happens when $d_{m} \geq K$ (or $K-1$) $\: \forall \: m \in [M]$, which often holds in practice~\cite{Krizhevsky12, He15}. Otherwise, they converge to the best rank-$R$ approximation of $\mathbf{I}_{K}$ or $\mathbf{I}_{K} - \frac{1}{K} \mathbf{1}_{K}\mathbf{1_{K}}^{\top}$, where the class-means neither have the equinorm nor the maximally pairwise separation properties. This result is consistent with the two-layer case observed in \cite{Zhou22a}.
\end{remark}
\vspace{0.5em}

\begin{remark} 
From the proofs, we can show that under the condition $d_{m} \geq K,$ $\: \forall \: m \in [M]$, the optimal value of the loss function is strictly smaller than when this condition does not hold. Our result is aligned with \cite{Zhu18}, where they empirically observe that a larger network (i.e., larger width) tends to exhibit severe $\mathcal{NC}$ and have smaller training errors.
\end{remark}
\vspace{0.5em}

\begin{remark}  
We study deep linear networks under UFM and balanced data for CE loss in Appendix \ref{sec:CEloss}. The result demonstrates $\mathcal{NC}$ properties of every global solutions, whose the matrices product $\mathbf{W}_{M} \times \mathbf{W}_{M-1} \times \ldots \times \mathbf{W}_{1}$ and $\mathbf{H}_{1}$ converge to the ETF structure when training progresses.
\end{remark}

\section{Neural Collapse in Deep Linear Networks under the UFM Setting with MSE Loss and Imbalanced Data}
\label{sec:imbalance}
The majority of theoretical results for $\mathcal{NC}$ only consider the balanced data setting, i.e., the same number of training samples for each class. This assumption plays a vital role in the existence of the well-structured ETF geometry. In this section, we instead consider the imbalanced data setting and derive the first geometry analysis under this setting for MSE loss. Furthermore, we extend our study from the plain UFM setting, which includes only one layer of weight after the unconstrained features, to the deep linear network one.

\subsection{Plain UFM Setting with No Bias}
The bias-free plain UFM  with MSE loss is given by:
\begin{align}
    \min_{\mathbf{W}, \mathbf{H}} \frac{1}{2N}
    \| \mathbf{W} \mathbf{H} - \mathbf{Y} \|_F^2 
    &+ \frac{\lambda_{W}}{2} \| \mathbf{W} \|_F^2  
    + \frac{\lambda_{H}}{2} \| \mathbf{H} \|_F^2 ,
    \label{eq:UFM_imbalance}
\end{align}
where $\mathbf{W} \in \mathbb{R}^{K \times d}$,
$\mathbf{H} \in \mathbb{R}^{d \times N}$, and $\mathbf{Y} \in \mathbb{R}^{K \times N}$ is the one-hot vectors matrix consisting $n_{k}$ one-hot vectors for each class $k$, $\forall \: k \in [K]$. We now state the $\mathcal{NC}$ properties of the global solutions of \eqref{eq:UFM_imbalance} under the imbalanced data setting when the feature dimension $d$ is at least the number of classes $K$.

\begin{theorem}
\label{thm:UFM_imbalance}
   Let $d \geq K$ and $(\mathbf{W}^{*}, \mathbf{H}^{*})$ be any global minimizer of problem \eqref{eq:UFM_imbalance}. Then, we have:
    \\
    
    $(\mathcal{NC}1) \quad
        \mathbf{H}^{*} = \overline{\mathbf{H}}^{*} \mathbf{Y} 
        \Leftrightarrow \mathbf{h}_{k,i}^{*} = \mathbf{h}_{k}^{*} \: \forall \: k \in [K], i \in [n_{k}], \nonumber$
    where $\overline{\mathbf{H}}^{*} = [\mathbf{h}_{1}^{*},\ldots,\mathbf{h}_{K}^{*} ] \in \mathbb{R}^{d \times K}$.
    \\
    
    $(\mathcal{NC}2)$ Let $a:= N^{2} \lambda_{W} \lambda_{H}$, we have:
    \begin{align}
    \begin{gathered}
             \mathbf{W}^{*} \mathbf{W}^{* \top}
             = \operatorname{diag}
             \left\{s_{k}^{2} \right\}_{k=1}^{K},
             \\
            \overline{\mathbf{H}}^{* \top}
            \overline{\mathbf{H}}^{*} 
            =
            \operatorname{diag}
            \left\{
           \frac{s_{k}^{2}}{(s_{k}^{2} + N \lambda_{H})^{2}}
            \right\}_{k=1}^{K}, \nonumber
     \end{gathered}    
    \end{align}
    \begin{align}
    \begin{gathered}
            \mathbf{W}^{*} \mathbf{H}^{*}
            = \operatorname{diag}
            \left\{
            \frac{s_{k}^{2}}{s_{k}^{2} + N \lambda_{H}}
            \right\}_{k=1}^{K}
            \mathbf{Y}
            \nonumber \\
            = \begin{bmatrix}
            \frac{s_{1}^{2}}{s_{1}^{2} + N \lambda_{H}} \mathbf{1}_{n_{1}}^{\top}  & \ldots & \mathbf{0} \\
            \vdots & \ddots & \vdots \\
            \mathbf{0}  & \ldots & \frac{s_{K}^{2}}{s_{K}^{2} + N \lambda_{H}} \mathbf{1}_{n_{K}}^{\top} \\
            \end{bmatrix}. \nonumber
    \end{gathered}    
    \end{align}
    where:
    \begin{itemize}
        \item If $\frac{a}{n_{1}} \leq \frac{a}{n_{2}} \leq \ldots \leq \frac{a}{n_{K}} \leq 1$:
        \begin{align}
        \begin{aligned}
            s_{k} = 
            \sqrt{ \sqrt{\frac{n_{k} \lambda_{H}}{\lambda_{W}} }
          - N \lambda_{H}} \quad &\forall \: k \in [K] \nonumber  
        \end{aligned}
        \end{align}
        
        \item If there exists a $j \in [K-1]$ s.t. $\frac{a}{n_{1}} \leq \frac{a}{n_{2}} \leq \ldots \leq \frac{a}{n_{j}} \leq 1 < \frac{a}{n_{j+1}} \leq \ldots \leq \frac{a}{n_{K}}$:
        \begin{align}
        \begin{aligned}
        s_{k} = \left\{\begin{matrix}
        \sqrt{ \sqrt{\frac{n_{k} \lambda_{H}}{\lambda_{W}} }
          - N \lambda_{H}} \quad &\forall \: k \leq j \\ 0 \quad &\forall \: k > j
        \end{matrix}\right. .
        \nonumber
        \end{aligned}
        \end{align}
        
        \item If $1 < \frac{a}{n_{1}} \leq \frac{a}{n_{2}} \leq \ldots \leq \frac{a}{n_{K}} $:
        \begin{align}
           (s_{1}, s_{2}, \ldots, s_{K} ) &= (0,0,\ldots,0), \nonumber
        \end{align}
        and $(\mathbf{W}^{*}, \mathbf{H}^{*}) = (\mathbf{0}, \mathbf{0})$ in this case.
    \end{itemize}
     For any $k$ such that $s_{k} = 0$, we have:
        \begin{align}
            \mathbf{w}_{k}^{*} = \mathbf{h}_{k}^{*} = \mathbf{0}. \nonumber
        \end{align}

    $(\mathcal{NC}3) \quad 
    \mathbf{w}_{k}^{*} = \sqrt{\frac{n_{k} \lambda_{H}}{\lambda_{W}}} \mathbf{h}_{k}^{*} \quad \forall \: k \in [K]. \nonumber$
    
\end{theorem}
\noindent The detailed proofs are provided in the Appendix \ref{sec:proofs_im}. We use the same approach as the proofs of Theorem \ref{thm:bias-free} to prove this result, with challenge arises in the process of lower bounding the loss function w.r.t. the singular vectors of $\mathbf{W}$. Interestingly, the left singular matrix of $\mathbf{W}^{*}$ consists multiple orthogonal blocks on its diagonal, with each block corresponds with a group of classes having the same number of training samples. This property creates the orthogonality of $(\mathcal{NC}2)$ geometries.

Theorem \ref{thm:UFM_imbalance} implies the following interesting results:

\begin{itemize}[leftmargin=1em]
\itemsep0em 
\vspace{-1em}
\item \textbf{Features collapse:} The features in the same class also converge to their class-mean, similar as balanced case.

 \item \textbf{Convergence to GOF:} When the condition $N^{2} \lambda_{W} \lambda_{H} / n_{K} < 1$ is hold, the class-means matrix and the last-layer classifiers converge to GOF (see Definition \ref{def:GOF}). This geometry includes orthogonal vectors, but their length depends on the number of training samples in the class. The above condition implies that the imbalance and the regularization level should not be too heavy to avoid trivial solutions that may harm the model performances. We will discuss more about this phenomenon in Section \ref{subsec:MC}.
    
\item \textbf{Alignment between linear classifiers and last-layer features:} The last-layer linear classifier is aligned with the class-mean of the same class, but with a different ratio across classes. These ratios are proportional to the square root of the number of training samples, and thus different compared to the balanced case where $\mathbf{W}^{*} / \| \mathbf{W}^{*} \|_F =  \overline{\mathbf{H}}^{* \top} / \| \overline{\mathbf{H}}^{* \top} \|_F$.

\end{itemize}

    

\begin{remark}
 We study the case $d < K$  in Theorem \ref{thm:im_UFM_bottleneck}. In this case, while $(\mathcal{NC}1)$ and $(\mathcal{NC}3)$ are exactly similar as the case $d \geq K$, the $(\mathcal{NC}2)$ geometries are different if $a/n_{d} < 1$ and $n_{d} = n_{d+1}$, where a square block on the diagonal is replaced by its low-rank approximation. This square block corresponds to classes with the number of training samples equal $n_{d}$. 
Also, we have $\mathbf{w}_{k}^{*} = \mathbf{h}^{*}_{k} = \mathbf{0}$ for any class $k$ with the amount of data is 
less than $n_{d}$.   
\end{remark} 


\subsection{GOF Structure with Different Imbalance Levels and Minority Collapse}
\label{subsec:MC}

Given the exact closed forms of the singular values of $\mathbf{W}^{*}$ stated in Theorem \ref{thm:UFM_imbalance}, we derive the norm ratios between the  classifiers and between features across classes as follows:

\begin{lemma}
\label{lm:ratio}
Suppose $(\mathbf{W}^{*}, \mathbf{H}^{*})$ is a global minimizer of problem \eqref{eq:UFM_imbalance} such that $d \geq K$ and  $N^{2} \lambda_{W} \lambda_{H} / n_{K} < 1$, so that all the $s_{k}$'s are positive. The following results hold:
\begin{align}
    \frac{\| \mathbf{w}_{i}^{*} \|^{2}}{\| \mathbf{w}_{j}^{*}  \|^{2}} = \frac{ \sqrt{\frac{n_{i} \lambda_{H}}{\lambda_{W}}} - N \lambda_{H} }{\sqrt{\frac{n_{j} \lambda_{H}}{\lambda_{W}}} - N \lambda_{H}},
    \frac{\| \mathbf{h}_{i}^{*}  \|^{2}}{\| \mathbf{h}_{j}^{*}  \|^{2}}
    = \frac{n_{j}}{n_{i}} 
    \frac{ \sqrt{\frac{n_{j} \lambda_{H}}{\lambda_{W}}} - N \lambda_{H} }{\sqrt{\frac{n_{i} \lambda_{H}}{\lambda_{W}}} - N \lambda_{H}}. \nonumber
\end{align}
If $n_{i} \geq n_{j}$, we have $\| \mathbf{w}_{i}^{*}  \| \geq \| \mathbf{w}_{j}^{*}  \|$ and $\| \mathbf{h}_{i}^{*}  \| \leq \| \mathbf{h}_{j}^{*}  \|$.
\end{lemma}

It has been empirically observed that the classifiers of the majority classes have greater norms~\cite{Kang19}. Our result is in agreement with this observation. Moreover, it has been shown that class imbalance impairs the model's accuracy on minority classes~\cite{Kang19, Cao19}. Recently, \cite{Fang21} discover the ``Minority Collapse'' phenomenon. In particular, they show that there exists a finite threshold for imbalance level beyond which all the minority classifiers collapse to a single vector, resulting in the model's poor performance on these classes. 
\emph{Theorem \ref{thm:UFM_imbalance} is not only aligned with the ``Minority Collapse'' phenomenon, but also provides the imbalance threshold for the  collapse of minority classes to vector $\mathbf{0}$, i.e., $N^{2} \lambda_{W} \lambda_{H} / n_{K} > 1$.}

\subsection{Bias-free Deep Linear Network under the UFM setting}

    We now generalize \eqref{eq:UFM_imbalance} to bias-free deep linear networks with $M \geq 2$ and arbitrary widths. We study the following optimization problem with imbalanced data:
    \begin{align}
    \begin{aligned}
        &\min _{\mathbf{W}_{M}, \mathbf{W}_{M-1},\ldots, \mathbf{W}_{1}, \mathbf{H}_{1}} 
        \frac{1}{2N} \| \mathbf{W}_{M} \mathbf{W}_{M-1} \ldots \mathbf{W}_{1}
        \mathbf{H}_{1} - \mathbf{Y} \|_F^2 
        \\
        &+ \frac{\lambda_{W_{M}}}{2} \| \mathbf{W}_{M} \|^2_F \
        + \ldots   + \frac{\lambda_{W_{1}}}{2} \| \mathbf{W}_{1} \|^2_F +
        \frac{\lambda_{H_{1}}}{2} \| \mathbf{H}_{1} \|^2_F,
        \label{eq:deep_imbalance}
    \end{aligned}
    \end{align}
where the target matrix $\mathbf{Y}$ is the one-hot vectors matrix defined in \eqref{eq:UFM_imbalance}. We now state the $\mathcal{NC}$ properties of the global solutions of \eqref{eq:deep_imbalance} when the dimensions of the hidden layers are at least the number of classes $K$.

\begin{theorem}
\label{thm:deep_imbalance}
     Let $d_{m} \geq K, \: \forall \: m \in [M]$, and $(\mathbf{W}_{M}^{*}, \mathbf{W}_{M-1}^{*},\ldots, \mathbf{W}_{1}^{*}, \mathbf{H}_{1}^{*})$ be any global minimizer of problem \eqref{eq:deep_imbalance}. We have the following results:
     \\
     
   ($\mathcal{NC}1) \quad
        \mathbf{H}_{1}^{*} = \overline{\mathbf{H}}^{*} \mathbf{Y} \nonumber
        \Leftrightarrow \mathbf{h}_{k,i}^{*} = \mathbf{h}_{k}^{*} \: \forall \: k \in [K], i \in [n_{k}], $
    where $\overline{\mathbf{H}}^{*} = [\mathbf{h}_{1}^{*},\ldots,\mathbf{h}_{K}^{*}]  \in \mathbb{R}^{d_{1} \times K}$.
    \\
    
    $(\mathcal{NC}2)$ Let $c:= \frac{\lambda_{W_{1}}^{M-1}}
    {\lambda_{W_{M}} \lambda_{W_{M-1}} \ldots \lambda_{W_{2}} }$, $a:= N \sqrt[M]{N  \lambda_{W_{M}} \lambda_{W_{M-1}} \ldots \lambda_{W_{1}} \lambda_{H_{1}}}$ and $\forall k \in [K]$,  $x^{*}_{k}$ is the largest positive solution of the equation $\frac{a}{n_{k}} - \frac{ x^{M-1}}{ (x^{M} + 1)^{2}} = 0$, we have the following:
    \begin{align}
    \begin{aligned}
             &\mathbf{W}^{*}_{M} \mathbf{W}^{* \top}_{M}
             = \frac{\lambda_{W_{1}}}{\lambda_{W_{M}}} \operatorname{diag}
             \left\{s_{k}^{2} \right\}_{k=1}^{K},
             \\
            &(\mathbf{W}_{M}^{*}  \ldots  \mathbf{W}_{1}^{*} )(\mathbf{W}_{M}^{*} \ldots  \mathbf{W}_{1}^{*} )^{\top} = 
            \operatorname{diag}
             \left\{c s_{k}^{2M} \right\}_{k=1}^{K},
             \\
           &\overline{\mathbf{H}}^{* \top}
            \overline{\mathbf{H}}^{*} 
            =
            \operatorname{diag}
            \left\{
            \frac{c s_{k}^{2M}}{(c s_{k}^{2M} + N \lambda_{H_{1}})^{2}}
            \right\}_{k=1}^{K}, \nonumber
            \\
             &\mathbf{W}_{M}^{*} \mathbf{W}_{M-1}^{*} \ldots  \mathbf{W}_{1}^{*}   \mathbf{H}_{1}^{*} 
             = 
             \left\{
            \frac{c s_{k}^{2M}}{c s_{k}^{2M} + N \lambda_{H_{1}}}
            \right\}_{k=1}^{K} \mathbf{Y},
    \end{aligned}    
    \end{align}
    
    ($\mathcal{NC}3$) We have, $\forall \: k \in [K]$:
    \begin{align}
        (\mathbf{W}_{M}^{*} \mathbf{W}_{M-1}^{*} \ldots \mathbf{W}_{1}^{*})_{k} = (c s_{k}^{2M} + N \lambda_{H_{1}}) \mathbf{h}_{k}^{*}, \nonumber
    \end{align}
    where:
    \begin{itemize}
        \item If $\frac{a}{n_{1}} \leq \frac{a}{n_{2}} \leq \ldots \leq \frac{a}{n_{K}} < \frac{(M-1)^{\frac{M-1}{M}}}{M^{2}}$, we have:
        \begin{align}
        \begin{aligned}
        s_{k} = 
        \sqrt[2M]{\frac{N \lambda_{H_{1}} x_{k}^{* M}}{c}} \quad \forall \: k \in [K]
        .
        \nonumber
        \end{aligned}
        \end{align}
     
        \item If there exists a $j \in [K-1]$ s.t. $\frac{a}{n_{1}} \leq \frac{a}{n_{2}} \leq \ldots \leq \frac{a}{n_{j}} < \frac{(M-1)^{\frac{M-1}{M}}}{M^{2}} < \frac{a}{n_{j+1}} \leq \ldots \leq \frac{a}{n_{K}}$, we have:
        \begin{align}
        \begin{aligned}
        s_{k} = \left\{\begin{matrix}
        \sqrt[2M]{\frac{N \lambda_{H_{1}} x_{k}^{* M}}{c}} \quad &\forall \: k \leq j \\ 0 \quad &\forall \: k > j
        \end{matrix}\right. .
        \nonumber
        \end{aligned}
        \end{align}
        For any $k$ such that $s_{k} = 0$, we have:
        \begin{align}
            (\mathbf{W}_{M}^{*})_{k} = \mathbf{h}_{k}^{*} = \mathbf{0}. \nonumber
        \end{align}  

        \item If $\frac{(M-1)^{\frac{M-1}{M}}}{M^{2}} < \frac{a}{n_{1}} \leq \frac{a}{n_{2}} \leq \ldots \leq \frac{a}{n_{K}}$, we have:
        \begin{align}
           (s_{1}, s_{2}, \ldots, s_{K} ) &= (0,0,\ldots,0), \nonumber
        \end{align}
        and $(\mathbf{W}_{M}^{*}, \ldots, \mathbf{W}_{1}^{*}, \mathbf{H}_{1}^{*}) = (\mathbf{0}, \ldots, \mathbf{0}, \mathbf{0})$ in this case.
    \end{itemize}
\end{theorem}

\noindent The detailed proofs of Theorem \ref{thm:deep_imbalance} and the remaining case where there are some $\frac{a}{n_{k}}$'s equal to  $\frac{(M-1)^{\frac{M-1}{M}}}{M^{2}}$ are provided in Appendix \ref{sec:proofs_im_deep}.
\vspace{0.5em}
\begin{remark}
The equation that solves for the optimal singular value, $\frac{a}{n} - \frac{x^{M-1}}{(x^{M}+1)^{2}} = 0$, has exactly two positive solutions when $a <  (M-1)^{\frac{M-1}{M}}/M^{2}$ (see Section \ref{sec:study_g}). Solving this equation leads to cumbersome solutions of a high-degree polynomial. Even without the exact closed-form formula for the solution, the $(\mathcal{NC}2)$ geometries can still be easily computed using numerical methods.
\end{remark}
\vspace{0.5em}

\begin{remark}
  We study the case $R := \min(d_{M}, \ldots, d_{1}, K) < K$ in Theorem \ref{thm:im_deep_bottleneck}. In this case, while $(\mathcal{NC}1)$ and $(\mathcal{NC}3)$ are exactly similar as the case $R = K$ in Theorem \ref{thm:deep_imbalance}, the $(\mathcal{NC}2)$ geometries are different if $a/n_{R} \leq 1$ and $n_{R} = n_{R+1}$, where a square block on the diagonal is replaced by its low-rank approximation. This square block corresponds to classes with the number of training samples equal $n_{R}$. Also, we have $(\mathbf{W}_{M})_{k}^{*} = \mathbf{h}^{*}_{k} = \mathbf{0}$ for any class $k$ with the amount of data is less than $n_{R}$.  
\end{remark}

\begin{figure}[t!]
    \centering
    \includegraphics[width = 1\columnwidth]{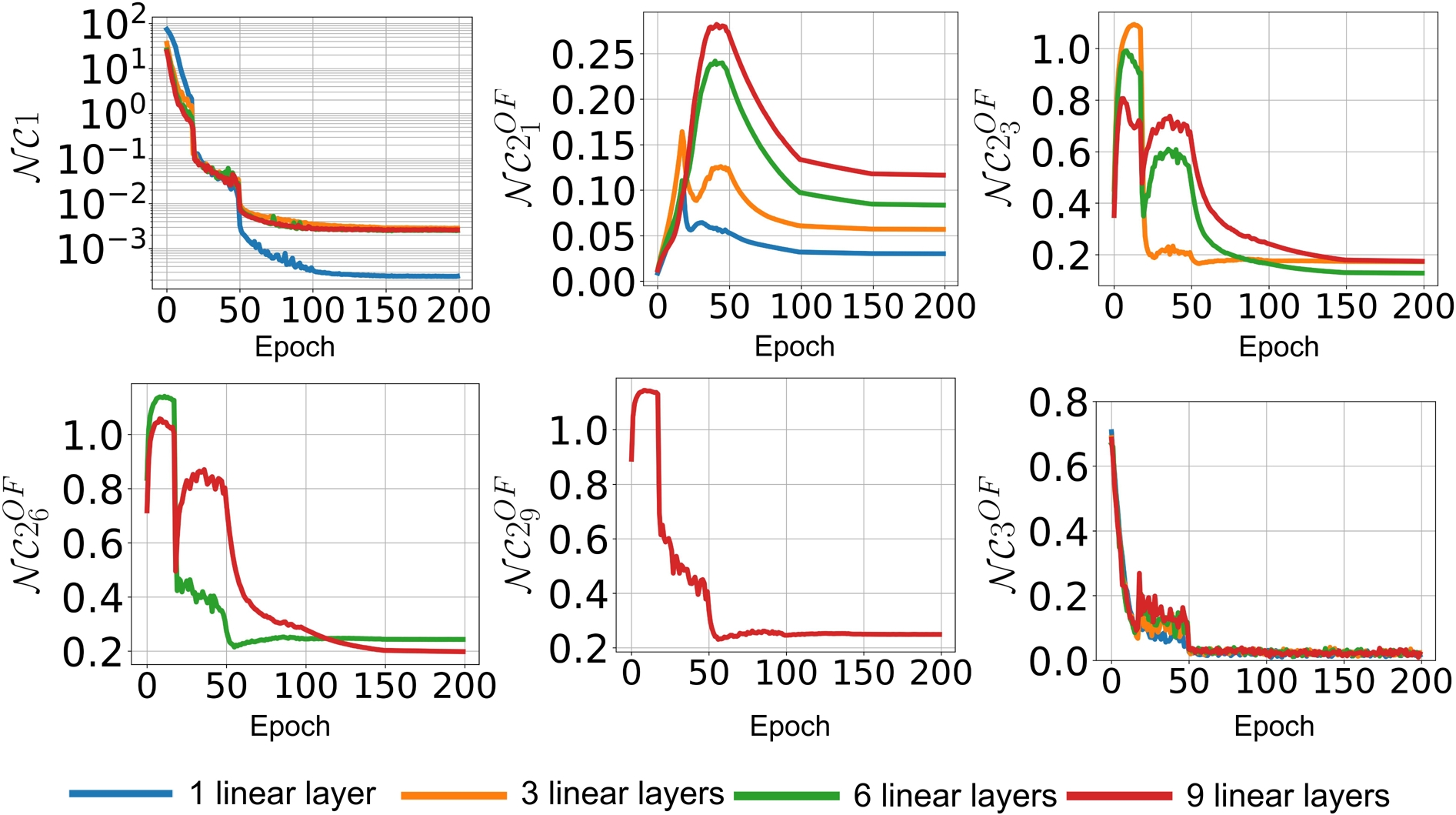}
    \vspace{-2em}
    \caption{\small \color{black}{Illustration of $\mathcal{NC}$ with 6-layer MLP backbone on CIFAR10 for MSE loss, balanced data and bias-free setting.}}
    \label{fig:MLP_nobias_balance}
\vspace{-0.1in}
\end{figure}

\begin{figure}[t!]
    \centering
    \includegraphics[width = 1\columnwidth]{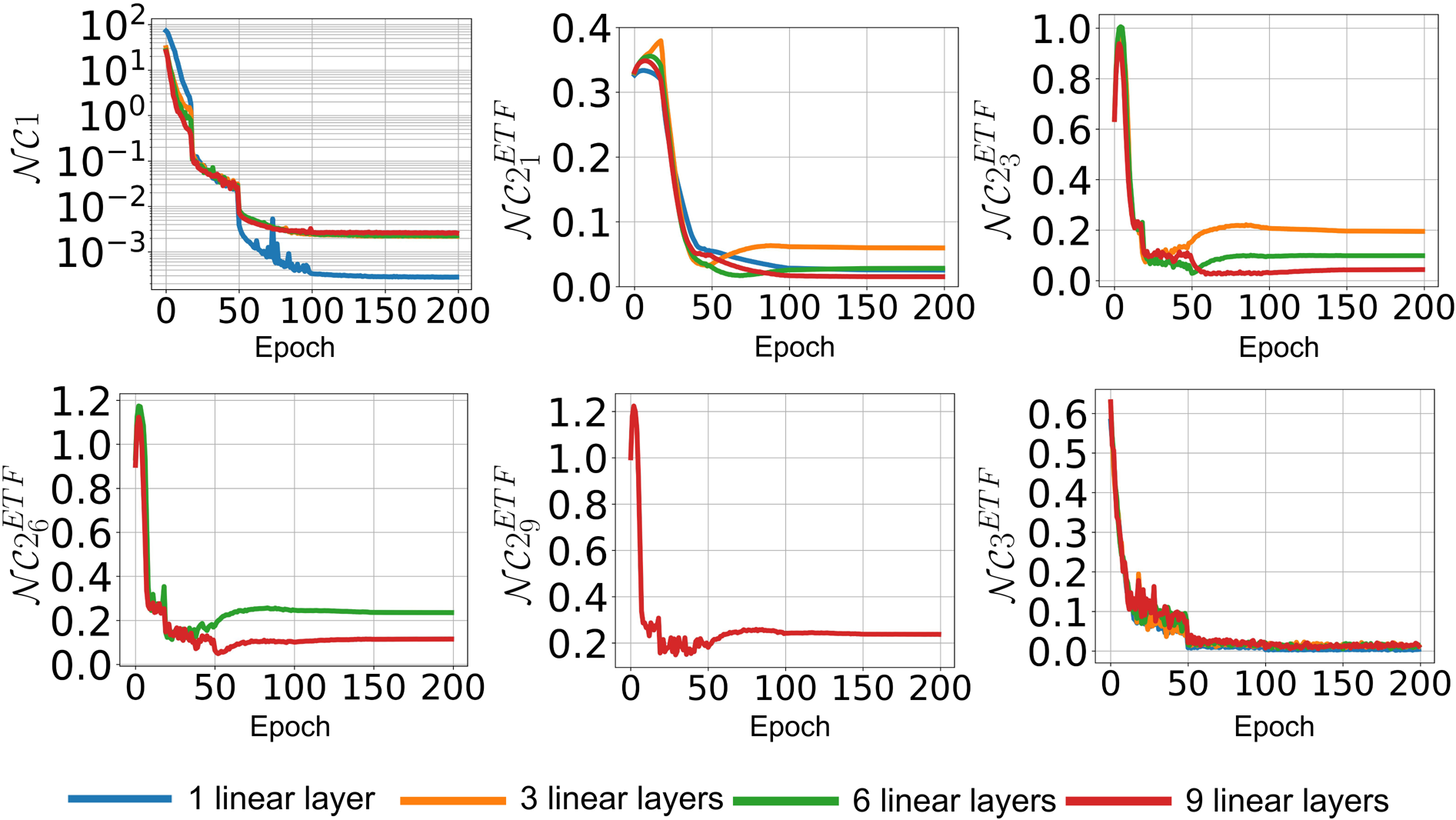}
    \vspace{-2em}
    \caption{\small \color{black}{Same setup as Fig. \ref{fig:MLP_nobias_balance} but having last-layer bias.}}
    \label{fig:MLP_bias_balance}
\vspace{-0.1in}
\end{figure}

\section{Experimental Results}
\label{sec:experimental_results} 

In this section, we empirically verify our theoretical results in multiple settings for both balanced and imbalanced data. In particular, we observe the evolution of $\mathcal{NC}$ properties in the training of deep linear networks with a prior backbone feature extractor (e.g., MLP, ResNet18) to create the ``unconstrained'' features (see Fig.~\ref{fig:DNN} for a sample visualization). The experiments are performed on CIFAR10~\citep{Krizhevsky09learningmultiple} dataset and EMNIST letter~\citep{cohen2017emnist} dataset for the image classification task. Moreover, we perform direct optimization experiments, which follows the setting in \eqref{eq:UFM_general_linear} to guarantee our theoretical analysis. To verify the results are consistent through different dataset, we also conduct experiments on text classification tasks in Appendix~\ref{subsubsec:addtional_balanced}.

The hyperparameters of the optimizers are tuned to reach the global optimizer in all experiments. The definitions of the $\mathcal{NC}$ metrics, hyperparameters details, and additional numerical results can be found in Appendix~\ref{sec:experiment_details_appendix}.



\begin{figure}[t!]
    \centering
    \includegraphics[width = 1\columnwidth]{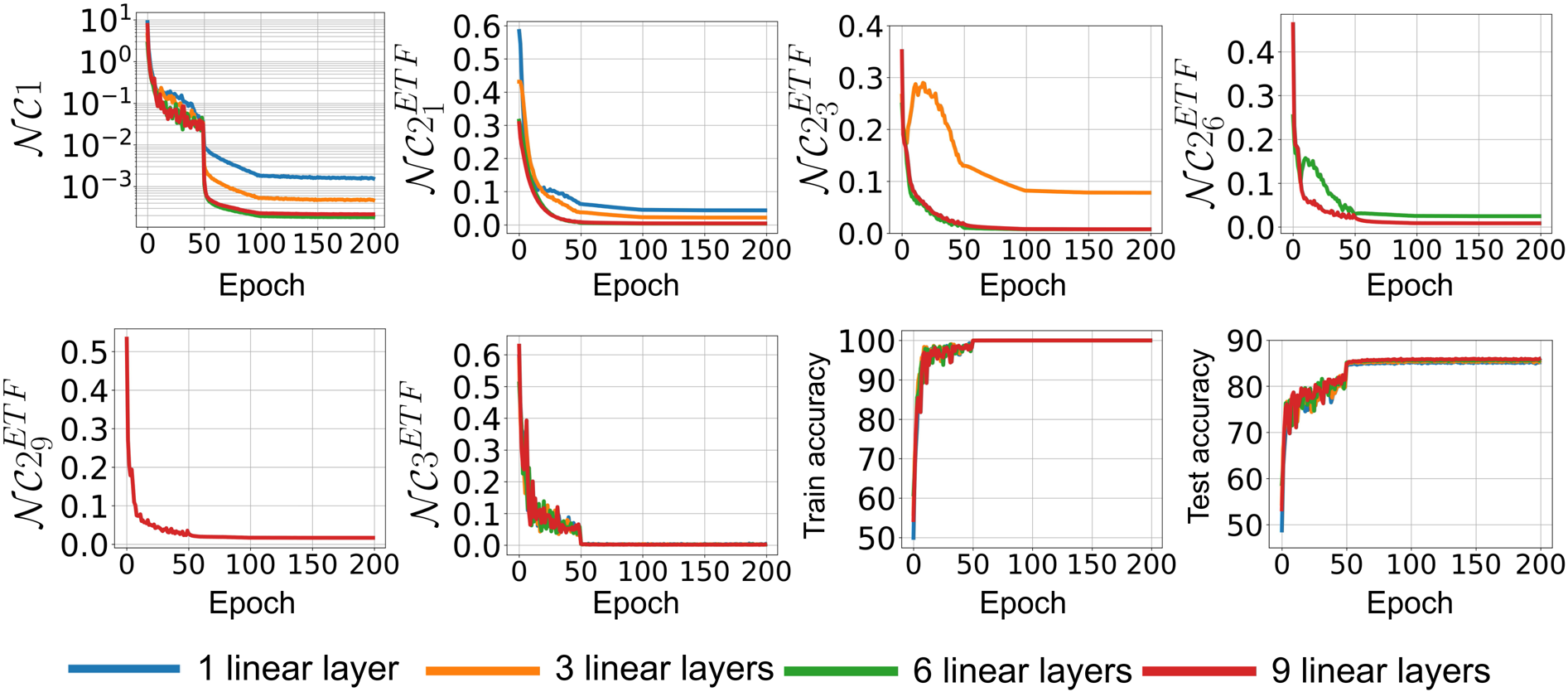}
    \vspace{-2em}
    \caption{\small \color{black}{Training results with ResNet18 backbone on CIFAR10 for MSE loss, balanced data and last-layer bias setting.}}
    \label{fig:ResNet_balance}
\vspace{-0.1in}
\end{figure}

\begin{figure}[t!]
    \centering
    \includegraphics[width = 1\columnwidth]{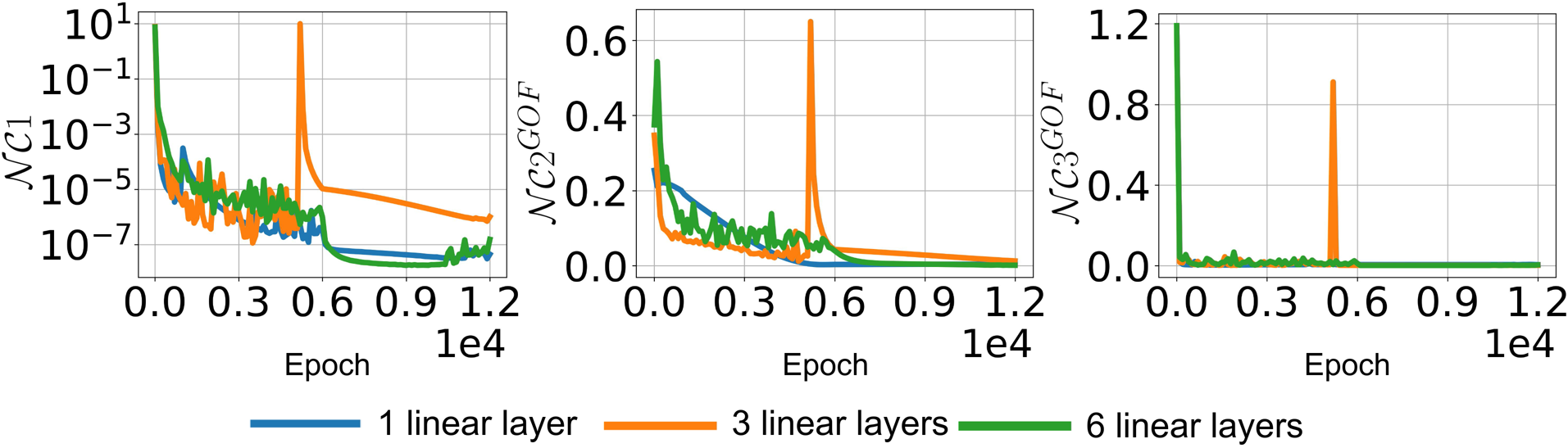}
    \vspace{-2em}
    \caption{\small \color{black}{Illustration of $\mathcal{NC}$ with 6-layer MLP backbone on an imbalanced subset of CIFAR10 for MSE loss and bias-free setting.}}
    \label{fig:mlp_imbalance}
\vspace{-0.1in}
\end{figure}

\subsection{Balanced Data}
\label{subsec:balanced_data_experiment}
Under the balanced data setting, we alternatively substitute between multilayer perceptron (MLP), ResNet18~\citep{DBLP:conf/cvpr/HeZRS16} and VGG16~\citep{Simonyan14} in place of the backbone feature extractor. For all experiments with MLP backbone model, we perform the regularization on the ``unconstrained'' features $\mathbf{H}_{1}$ and on 
subsequent weight layers to replicate the UFM setting in \eqref{eq:UFM_general_linear}. For deep learning experiments with ResNet18 and VGG16 backbone, we enforce the weight decay on all parameters of the network, which aligns to the typical training protocol.
\subsubsection{Image classification experiment on CIFAR10 dataset}
\label{subsec:balanced_data_experiment:cifar10}
{\bf Multilayer perceptron experiment:}
We use a 6-layer MLP model with ReLU activation as the backbone feature extractor in this experiment. For deep linear layers, we cover all depth-width combinations with depth $\in\{1, 3, 6, 9\}$ and width $\in\{512, 1024, 2048\}$.  We run both bias-free and last-layer bias cases to demonstrate the convergence to OF and ETF geometry, with the models trained by Adam optimizer \cite{Kingma2014} for 200 epochs. For a concrete illustration, the results of width-1024 MLP backbone and linear layers for MSE loss are shown in Fig. \ref{fig:MLP_nobias_balance} and Fig. \ref{fig:MLP_bias_balance}. 
We consistently observe the convergence of $\mathcal{NC}$ metrics to small values as training progresses for various depths of the linear networks. Additional results with MLP backbone for other widths and for CE loss can be found in Appendix~\ref{subsec:balance_case_appendix}.



{\bf Deep learning experiment:}
We use ResNet18 and VGG16 as the deep learning backbone for extracting $\mathbf{H}_{1}$ in this experiment. The depths of the deep linear network are selected from the set $\{1, 3, 6, 9\}$ and the widths are chosen to equal the last-layer dimension of the backbone model (i.e., $512$). 
The models are trained with the MSE loss without data augmentation for $200$ epochs using stochastic gradient descent (SGD). As shown in Fig.~\ref{fig:ResNet_balance} above and Fig.\ref{fig:VGG_balance} in the Appendix \ref{subsubsec:addtional_balanced}, $\mathcal{NC}$ properties are obtained for widely used architectures in deep learning contexts. Furthermore, the results empirically confirm the occurrences of $\mathcal{NC}$ across deep linear classifiers described in Theorem~\ref{thm:bias-free}.

\subsubsection{Image classification experiment on EMNIST letter dataset}
Similar to the deep learning experiment described in section \ref{subsec:balanced_data_experiment:cifar10}, we use ResNet18 and VGG16 as deep learning backbones. We consider deep linear networks with depth selected from the set $\{1, 3, 6\}$ and the width is chosen to be $512$. All models are trained with MSE loss for 200 epochs using SGD. As shown in Fig. \ref{fig:ResNet_balance_EMNIST} and Fig. \ref{fig:VGG_balance_EMNIST} in Appendix~\ref{subsec:balance_case_appendix}, the occurrences of $\mathcal{NC}$ across deep linear classifiers described in Theorem~\ref{thm:bias-free} can also be observed when training on the EMNIST letter dataset.


\subsubsection{Direct optimization experiment}
To exactly replicate the problem \eqref{eq:UFM_general_linear}, $\mathbf{W}_{M}, \ldots, \mathbf{W}_{1}$ and $\mathbf{H}_{1}$ are initialized with standard normal distribution scaled by $0.1$ and optimized with gradient descent with step-size 0.1 for MSE loss. In this experiment, we set $K=4, n=100, d_M = d_{M-1} = \ldots = d_1 = 64$ and all $\lambda$'s are set to be $\num{5e-4}$. We cover multiple depth settings with $M$ chosen from the set $\{1, 3, 6, 9\}$. Fig.~\ref{fig:synthetic_balance_nobias} and Fig. \ref{fig:synthetic_balance_bias} in Appendix \ref{subsubsec:addtional_balanced} shows the convergence to $0$ of $\mathcal{NC}$ metrics for bias-free and last-layer bias settings, respectively. The convergence errors are less than $0.001$ at the final iteration, which corroborates Theorem \ref{thm:bias-free}.


\subsection{Imbalanced Data}
\label{subsec:imbalanced_data_experiment}
For imbalanced data setting, we perform three experiments: 
CIFAR10 and EMNIST letter image classification with MLP backbone  and direct optimization with a similar setup as in Section \ref{subsec:balanced_data_experiment}.

{\bf Multilayer perceptron experiment on CIFAR10 dataset:}
In this experiment, we use a 6-layer MLP network with ReLU activation as the backbone model with removed batch normalization. We choose a random subset of CIFAR10 dataset with number of training samples of each class chosen from the list $\{500, 500, 400, 400, 300, 300, 200, 200, 100, 100\}$. The network is trained with batch gradient descent 
for $12000$ epochs. Both the feature extraction model and deep linear model share the hidden width $d=2048$. This experiment is performed with multiple linear model depths $M = 1, 3, 6$ and the results are shown in Fig. \ref{fig:mlp_imbalance}.
The converge of $\mathcal{NC}$ metrics to $0$ (errors are at most $0.05$ at the final epoch) strongly validates Theorem \ref{thm:UFM_imbalance} and \ref{thm:deep_imbalance} with the convergence to GOF structure of learned classifiers and features.

{\bf Multilayer perceptron experiment on EMNIST letter dataset:}
In this experiment, we use the same architecture as descibed in previous CIFAR10 experiment. Our training set is randomly sampled from the EMNIST letter training set. The number of training samples is as followed: 1 major class with $1500$ samples, 5 medium class with $600$ samples per class, and 20 minor classes with $50$ sample per class. We train the model with batch gradient descent 
for $12000$ epochs with the hidden width of both the feature extraction model and deep linear model is chosen to be $d=2048$. We perform the experiment with multiple linear model depths $M = \{1, 3, 6\}$. The results are shown in Fig. \ref{fig:mlp_imbalance_EMNIST} in Appendix \ref{subsubsec:addtional_imbalanced}. The convergence of $\mathcal{NC}$ metrics to  small values also validates the convergence to GOF structure as described in Theorems \ref{thm:UFM_imbalance} and \ref{thm:deep_imbalance}.

{\bf Direct optimization experiment:}
\label{subsubsec:direct_im}
In this experiment, except for the imbalanced data of $K = 4$ and $n_1 = 200, n_2 = 100, n_3=n_4=50$, the settings are identical to the direct optimization experiment in balanced case for MSE loss.
Fig. \ref{fig:synthetic_imbalance_GOF} in Appendix \ref{subsubsec:addtional_imbalanced} corroborates Theorems~\ref{thm:UFM_imbalance} and \ref{thm:deep_imbalance} for various depths $M = 1,3,6$ and $9$.

\section{Concluding Remarks}
\label{sec:conclusion}
In this work, we extend the global optimal analysis of the deep linear networks trained with the mean squared error (MSE) and cross entropy losses under the unconstrained features model. We prove that $\mathcal{NC}$ phenomenon is exhibited by the global solutions across layers. Moreover, we extend our theoretical analysis to the UFM imbalanced data settings for the MSE loss, which are much less studied in the current literature, and thoroughly analyze NC properties under this scenario. The convergence to GOF structure of the last-layer classifier and the last-layer features in a UFM with 1-layer learnable linear classifier (see Theorem \ref{thm:UFM_imbalance}) is relevant to the practical training of deep nonlinear networks.

In our work, we do not include biases in the training problem under imbalanced setting. We leave the study of the collapsed structure with the presence of biases as future work. As the next natural development of our results, characterizing $\mathcal{NC}$ for deep networks with non-linear activations under unconstrained features model is a highly interesting direction for future research. For example,  \cite{He22} recently discovers the decreasing pattern of $\mathcal{NC}1$ across layers of the model through extensive experiments on multiple architectures and datasets.


\section*{Acknowledgements}
This material is based on research sponsored by the AFOSR MURI FA9550-18-1-0502, the ONR grant N00014-20-1-2093, the MURI N00014-20-1-2787, and the NSF under Grant\# 2030859 to the Computing Research Association for the CIFellows Project (CIF2020-UCLA-38). NH
acknowledges support from the NSF IFML 2019844 and the NSF AI Institute for Foundations of Machine Learning.

\bibliography{ICML_ref}

\begin{thebibliography}{53}
\providecommand{\natexlab}[1]{#1}
\providecommand{\url}[1]{\texttt{#1}}
\expandafter\ifx\csname urlstyle\endcsname\relax
  \providecommand{\doi}[1]{doi: #1}\else
  \providecommand{\doi}{doi: \begingroup \urlstyle{rm}\Url}\fi

\bibitem[Arora et~al.(2018)Arora, Cohen, and Hazan]{arora18}
Arora, S., Cohen, N., and Hazan, E.
\newblock On the optimization of deep networks: Implicit acceleration by
  overparameterization.
\newblock In \emph{International Conference on Machine Learning}, pp.\
  244--253. PMLR, 2018.

\bibitem[Baldi \& Hornik(1989)Baldi and Hornik]{Baldi89}
Baldi, P. and Hornik, K.
\newblock Neural networks and principal component analysis: Learning from
  examples without local minima.
\newblock \emph{Neural Networks}, 2\penalty0 (1):\penalty0 53--58, 1989.
\newblock ISSN 0893-6080.
\newblock \doi{https://doi.org/10.1016/0893-6080(89)90014-2}.
\newblock URL
  \url{https://www.sciencedirect.com/science/article/pii/0893608089900142}.

\bibitem[Belkin et~al.(2019{\natexlab{a}})Belkin, Hsu, Ma, and
  Mandal]{Belkin19}
Belkin, M., Hsu, D., Ma, S., and Mandal, S.
\newblock Reconciling modern machine-learning practice and the classical
  bias{\textendash}variance trade-off.
\newblock \emph{Proceedings of the National Academy of Sciences}, 116\penalty0
  (32):\penalty0 15849--15854, jul 2019{\natexlab{a}}.
\newblock \doi{10.1073/pnas.1903070116}.
\newblock URL \url{https://doi.org/10.1073%2Fpnas.1903070116}.

\bibitem[Belkin et~al.(2019{\natexlab{b}})Belkin, Rakhlin, and
  Tsybakov]{Belkin18}
Belkin, M., Rakhlin, A., and Tsybakov, A.~B.
\newblock Does data interpolation contradict statistical optimality?
\newblock In \emph{The 22nd International Conference on Artificial Intelligence
  and Statistics}, pp.\  1611--1619. PMLR, 2019{\natexlab{b}}.

\bibitem[Brown et~al.(2020)Brown, Mann, Ryder, Subbiah, Kaplan, Dhariwal,
  Neelakantan, Shyam, Sastry, Askell, et~al.]{Brown20}
Brown, T., Mann, B., Ryder, N., Subbiah, M., Kaplan, J.~D., Dhariwal, P.,
  Neelakantan, A., Shyam, P., Sastry, G., Askell, A., et~al.
\newblock Language models are few-shot learners.
\newblock \emph{Advances in neural information processing systems},
  33:\penalty0 1877--1901, 2020.

\bibitem[Cao et~al.(2019)Cao, Wei, Gaidon, Arechiga, and Ma]{Cao19}
Cao, K., Wei, C., Gaidon, A., Arechiga, N., and Ma, T.
\newblock Learning imbalanced datasets with label-distribution-aware margin
  loss.
\newblock \emph{Advances in neural information processing systems}, 32, 2019.

\bibitem[Cohen et~al.(2017)Cohen, Afshar, Tapson, and
  Van~Schaik]{cohen2017emnist}
Cohen, G., Afshar, S., Tapson, J., and Van~Schaik, A.
\newblock Emnist: Extending mnist to handwritten letters.
\newblock In \emph{2017 international joint conference on neural networks
  (IJCNN)}, pp.\  2921--2926. IEEE, 2017.

\bibitem[Demirkaya et~al.(2020)Demirkaya, Chen, and Oymak]{Demirkaya20}
Demirkaya, A., Chen, J., and Oymak, S.
\newblock Exploring the role of loss functions in multiclass classification.
\newblock In \emph{2020 54th Annual Conference on Information Sciences and
  Systems (CISS)}, pp.\  1--5, 2020.
\newblock \doi{10.1109/CISS48834.2020.1570627167}.

\bibitem[Ergen \& Pilanci(2021)Ergen and Pilanci]{Ergen20}
Ergen, T. and Pilanci, M.
\newblock Revealing the structure of deep neural networks via convex duality.
\newblock In \emph{International Conference on Machine Learning}, pp.\
  3004--3014. PMLR, 2021.

\bibitem[Fang et~al.(2021)Fang, He, Long, and Su]{Fang21}
Fang, C., He, H., Long, Q., and Su, W.~J.
\newblock Exploring deep neural networks via layer-peeled model: Minority
  collapse in imbalanced training.
\newblock \emph{Proceedings of the National Academy of Sciences}, 118\penalty0
  (43), oct 2021.
\newblock \doi{10.1073/pnas.2103091118}.
\newblock URL \url{https://doi.org/10.1073%2Fpnas.2103091118}.

\bibitem[Goodfellow et~al.(2016)Goodfellow, Bengio, and Courville]{Good16}
Goodfellow, I.~J., Bengio, Y., and Courville, A.
\newblock \emph{Deep Learning}.
\newblock MIT Press, Cambridge, MA, USA, 2016.
\newblock \url{http://www.deeplearningbook.org}.

\bibitem[Graf et~al.(2021)Graf, Hofer, Niethammer, and Kwitt]{graf23}
Graf, F., Hofer, C., Niethammer, M., and Kwitt, R.
\newblock Dissecting supervised contrastive learning.
\newblock In \emph{International Conference on Machine Learning}, pp.\
  3821--3830. PMLR, 2021.

\bibitem[Guo et~al.(2020)Guo, Alvarez, and Salzmann]{guo21}
Guo, S., Alvarez, J.~M., and Salzmann, M.
\newblock Expandnets: Linear over-parameterization to train compact
  convolutional networks.
\newblock \emph{Advances in Neural Information Processing Systems},
  33:\penalty0 1298--1310, 2020.

\bibitem[Han et~al.(2022)Han, Papyan, and Donoho]{Han21}
Han, X., Papyan, V., and Donoho, D.~L.
\newblock Neural collapse under {MSE} loss: Proximity to and dynamics on the
  central path.
\newblock In \emph{International Conference on Learning Representations}, 2022.
\newblock URL \url{https://openreview.net/forum?id=w1UbdvWH_R3}.

\bibitem[Hardt \& Ma(2017)Hardt and Ma]{Hardt18}
Hardt, M. and Ma, T.
\newblock Identity matters in deep learning.
\newblock In \emph{International Conference on Learning Representations}, 2017.
\newblock URL \url{https://openreview.net/forum?id=ryxB0Rtxx}.

\bibitem[Hastie et~al.(2022)Hastie, Montanari, Rosset, and
  Tibshirani]{Hastie20}
Hastie, T., Montanari, A., Rosset, S., and Tibshirani, R.~J.
\newblock Surprises in high-dimensional ridgeless least squares interpolation.
\newblock \emph{The Annals of Statistics}, 50\penalty0 (2):\penalty0 949--986,
  2022.

\bibitem[He \& Su(2022)He and Su]{He22}
He, H. and Su, W.~J.
\newblock A law of data separation in deep learning.
\newblock \emph{arXiv preprint arXiv:2210.17020}, 2022.

\bibitem[He et~al.(2016{\natexlab{a}})He, Zhang, Ren, and
  Sun]{DBLP:conf/cvpr/HeZRS16}
He, K., Zhang, X., Ren, S., and Sun, J.
\newblock Deep residual learning for image recognition.
\newblock In \emph{2016 {IEEE} Conference on Computer Vision and Pattern
  Recognition, {CVPR} 2016, Las Vegas, NV, USA, June 27-30, 2016}, pp.\
  770--778. {IEEE} Computer Society, 2016{\natexlab{a}}.
\newblock \doi{10.1109/CVPR.2016.90}.
\newblock URL \url{https://doi.org/10.1109/CVPR.2016.90}.

\bibitem[He et~al.(2016{\natexlab{b}})He, Zhang, Ren, and Sun]{He15}
He, K., Zhang, X., Ren, S., and Sun, J.
\newblock Deep residual learning for image recognition.
\newblock In \emph{Proceedings of the IEEE conference on computer vision and
  pattern recognition}, pp.\  770--778, 2016{\natexlab{b}}.

\bibitem[Hornik(1991)]{Hornik91}
Hornik, K.
\newblock Approximation capabilities of multilayer feedforward networks.
\newblock \emph{Neural Networks}, 4\penalty0 (2):\penalty0 251--257, 1991.
\newblock ISSN 0893-6080.
\newblock \doi{https://doi.org/10.1016/0893-6080(91)90009-T}.
\newblock URL
  \url{https://www.sciencedirect.com/science/article/pii/089360809190009T}.

\bibitem[Hornik et~al.(1989)Hornik, Stinchcombe, and White]{Hornik89}
Hornik, K., Stinchcombe, M., and White, H.
\newblock Multilayer feedforward networks are universal approximators.
\newblock \emph{Neural Networks}, 2\penalty0 (5):\penalty0 359--366, 1989.
\newblock ISSN 0893-6080.
\newblock \doi{https://doi.org/10.1016/0893-6080(89)90020-8}.
\newblock URL
  \url{https://www.sciencedirect.com/science/article/pii/0893608089900208}.

\bibitem[Huang et~al.(2017)Huang, Liu, Van Der~Maaten, and Weinberger]{Huang17}
Huang, G., Liu, Z., Van Der~Maaten, L., and Weinberger, K.~Q.
\newblock Densely connected convolutional networks.
\newblock In \emph{2017 IEEE Conference on Computer Vision and Pattern
  Recognition (CVPR)}, pp.\  2261--2269, 2017.
\newblock \doi{10.1109/CVPR.2017.243}.

\bibitem[Huh et~al.(2021)Huh, Mobahi, Zhang, Cheung, Agrawal, and Isola]{huh23}
Huh, M., Mobahi, H., Zhang, R., Cheung, B., Agrawal, P., and Isola, P.
\newblock The low-rank simplicity bias in deep networks.
\newblock \emph{CoRR}, abs/2103.10427, 2021.
\newblock URL \url{https://arxiv.org/abs/2103.10427}.

\bibitem[Hui \& Belkin(2021)Hui and Belkin]{Hui20}
Hui, L. and Belkin, M.
\newblock Evaluation of neural architectures trained with square loss vs
  cross-entropy in classification tasks.
\newblock In \emph{International Conference on Learning Representations}, 2021.
\newblock URL \url{https://openreview.net/forum?id=hsFN92eQEla}.

\bibitem[Kang et~al.(2020)Kang, Xie, Rohrbach, Yan, Gordo, Feng, and
  Kalantidis]{Kang19}
Kang, B., Xie, S., Rohrbach, M., Yan, Z., Gordo, A., Feng, J., and Kalantidis,
  Y.
\newblock Decoupling representation and classifier for long-tailed recognition.
\newblock In \emph{International Conference on Learning Representations}, 2020.
\newblock URL \url{https://openreview.net/forum?id=r1gRTCVFvB}.

\bibitem[Kawaguchi(2016)]{Kawaguchi16}
Kawaguchi, K.
\newblock Deep learning without poor local minima.
\newblock In Lee, D., Sugiyama, M., Luxburg, U., Guyon, I., and Garnett, R.
  (eds.), \emph{Advances in Neural Information Processing Systems}, volume~29.
  Curran Associates, Inc., 2016.
\newblock URL
  \url{https://proceedings.neurips.cc/paper_files/paper/2016/file/f2fc990265c712c49d51a18a32b39f0c-Paper.pdf}.

\bibitem[Kingma \& Ba(2014)Kingma and Ba]{Kingma2014}
Kingma, D.~P. and Ba, J.
\newblock Adam: A method for stochastic optimization, 2014.
\newblock URL \url{https://arxiv.org/abs/1412.6980}.

\bibitem[Krizhevsky et~al.(2009)Krizhevsky, Hinton,
  et~al.]{Krizhevsky09learningmultiple}
Krizhevsky, A., Hinton, G., et~al.
\newblock Learning multiple layers of features from tiny images, 2009.

\bibitem[Krizhevsky et~al.(2012)Krizhevsky, Sutskever, and
  Hinton]{Krizhevsky12}
Krizhevsky, A., Sutskever, I., and Hinton, G.~E.
\newblock Imagenet classification with deep convolutional neural networks.
\newblock In \emph{Proceedings of the 25th International Conference on Neural
  Information Processing Systems - Volume 1}, NIPS'12, pp.\  1097–1105, Red
  Hook, NY, USA, 2012. Curran Associates Inc.

\bibitem[Laurent \& Brecht(2018)Laurent and Brecht]{Laurent17}
Laurent, T. and Brecht, J.
\newblock Deep linear networks with arbitrary loss: All local minima are
  global.
\newblock In \emph{International conference on machine learning}, pp.\
  2902--2907. PMLR, 2018.

\bibitem[Lu \& Steinerberger(2020)Lu and Steinerberger]{Lu20}
Lu, J. and Steinerberger, S.
\newblock Neural collapse with cross-entropy loss, 2020.
\newblock URL \url{https://arxiv.org/abs/2012.08465}.

\bibitem[Ma et~al.(2018)Ma, Bassily, and Belkin]{Ma17}
Ma, S., Bassily, R., and Belkin, M.
\newblock The power of interpolation: Understanding the effectiveness of sgd in
  modern over-parametrized learning.
\newblock In \emph{International Conference on Machine Learning}, pp.\
  3325--3334. PMLR, 2018.

\bibitem[Mixon et~al.(2022)Mixon, Parshall, and Pi]{Mixon20}
Mixon, D., Parshall, H., and Pi, J.
\newblock Neural collapse with unconstrained features.
\newblock \emph{Sampling Theory, Signal Processing, and Data Analysis}, 20, 07
  2022.
\newblock \doi{10.1007/s43670-022-00027-5}.

\bibitem[Nakkiran et~al.(2021)Nakkiran, Kaplun, Bansal, Yang, Barak, and
  Sutskever]{Nakkiran19}
Nakkiran, P., Kaplun, G., Bansal, Y., Yang, T., Barak, B., and Sutskever, I.
\newblock Deep double descent: Where bigger models and more data hurt.
\newblock \emph{Journal of Statistical Mechanics: Theory and Experiment},
  2021\penalty0 (12):\penalty0 124003, 2021.

\bibitem[Papyan et~al.(2020)Papyan, Han, and Donoho]{Papyan20}
Papyan, V., Han, X., and Donoho, D.~L.
\newblock Prevalence of neural collapse during the terminal phase of deep
  learning training.
\newblock \emph{Proceedings of the National Academy of Sciences}, 117\penalty0
  (40):\penalty0 24652--24663, 2020.

\bibitem[Rangamani \& Banburski-Fahey(2022)Rangamani and
  Banburski-Fahey]{Rangamani22}
Rangamani, A. and Banburski-Fahey, A.
\newblock Neural collapse in deep homogeneous classifiers and the role of
  weight decay.
\newblock In \emph{ICASSP 2022 - 2022 IEEE International Conference on
  Acoustics, Speech and Signal Processing (ICASSP)}, pp.\  4243--4247, 2022.
\newblock \doi{10.1109/ICASSP43922.2022.9746778}.

\bibitem[Ruder(2016)]{Ruder16}
Ruder, S.
\newblock An overview of gradient descent optimization algorithms, 2016.
\newblock URL \url{https://arxiv.org/abs/1609.04747}.

\bibitem[Safran \& Shamir(2018)Safran and Shamir]{Safran17}
Safran, I. and Shamir, O.
\newblock Spurious local minima are common in two-layer relu neural networks.
\newblock In \emph{International conference on machine learning}, pp.\
  4433--4441. PMLR, 2018.

\bibitem[Saxe et~al.(2013)Saxe, McClelland, and Ganguli]{Saxe14}
Saxe, A.~M., McClelland, J.~L., and Ganguli, S.
\newblock Exact solutions to the nonlinear dynamics of learning in deep linear
  neural networks.
\newblock \emph{arXiv preprint arXiv:1312.6120}, 2013.

\bibitem[Simonyan \& Zisserman(2015)Simonyan and Zisserman]{Simonyan14}
Simonyan, K. and Zisserman, A.
\newblock Very deep convolutional networks for large-scale image recognition.
\newblock In Bengio, Y. and LeCun, Y. (eds.), \emph{3rd International
  Conference on Learning Representations, {ICLR} 2015, San Diego, CA, USA, May
  7-9, 2015, Conference Track Proceedings}, 2015.
\newblock URL \url{http://arxiv.org/abs/1409.1556}.

\bibitem[Thrampoulidis et~al.(2022)Thrampoulidis, Kini, Vakilian, and
  Behnia]{Christos22}
Thrampoulidis, C., Kini, G.~R., Vakilian, V., and Behnia, T.
\newblock Imbalance trouble: Revisiting neural-collapse geometry.
\newblock \emph{Advances in Neural Information Processing Systems},
  35:\penalty0 27225--27238, 2022.

\bibitem[Tirer \& Bruna(2022)Tirer and Bruna]{Tirer22}
Tirer, T. and Bruna, J.
\newblock Extended unconstrained features model for exploring deep neural
  collapse.
\newblock In \emph{International Conference on Machine Learning}, pp.\
  21478--21505. PMLR, 2022.

\bibitem[Xie et~al.(2023)Xie, Yang, Cai, and He]{Xie22}
Xie, L., Yang, Y., Cai, D., and He, X.
\newblock Neural collapse inspired attraction-repulsion-balanced loss for
  imbalanced learning.
\newblock \emph{Neurocomputing}, 2023.

\bibitem[Yang et~al.(2022)Yang, Chen, Li, Xie, Lin, and Tao]{Yang22}
Yang, Y., Chen, S., Li, X., Xie, L., Lin, Z., and Tao, D.
\newblock Inducing neural collapse in imbalanced learning: Do we really need a
  learnable classifier at the end of deep neural network?
\newblock In \emph{Neural Information Processing Systems}, 2022.

\bibitem[Yaras et~al.(2022)Yaras, Wang, Zhu, Balzano, and Qu]{yaras23}
Yaras, C., Wang, P., Zhu, Z., Balzano, L., and Qu, Q.
\newblock Neural collapse with normalized features: A geometric analysis over
  the riemannian manifold.
\newblock In Oh, A.~H., Agarwal, A., Belgrave, D., and Cho, K. (eds.),
  \emph{Advances in Neural Information Processing Systems}, 2022.
\newblock URL \url{https://openreview.net/forum?id=Zvh6lF5b26N}.

\bibitem[Yarotsky(2022)]{Yarotsky18}
Yarotsky, D.
\newblock Universal approximations of invariant maps by neural networks.
\newblock \emph{Constructive Approximation}, 55\penalty0 (1):\penalty0
  407--474, 2022.

\bibitem[Yun et~al.(2018)Yun, Sra, and Jadbabaie]{Yun17}
Yun, C., Sra, S., and Jadbabaie, A.
\newblock Global optimality conditions for deep neural networks.
\newblock In \emph{International Conference on Learning Representations}, 2018.
\newblock URL \url{https://openreview.net/forum?id=BJk7Gf-CZ}.

\bibitem[Yun et~al.(2019)Yun, Sra, and Jadbabaie]{Yun18}
Yun, C., Sra, S., and Jadbabaie, A.
\newblock Small nonlinearities in activation functions create bad local minima
  in neural networks.
\newblock In \emph{International Conference on Learning Representations}, 2019.
\newblock URL \url{https://openreview.net/forum?id=rke_YiRct7}.

\bibitem[Zhou(2020)]{Zhou18}
Zhou, D.-X.
\newblock Universality of deep convolutional neural networks.
\newblock \emph{Applied and computational harmonic analysis}, 48\penalty0
  (2):\penalty0 787--794, 2020.

\bibitem[Zhou et~al.(2022{\natexlab{a}})Zhou, Li, Ding, You, Qu, and
  Zhu]{Zhou22a}
Zhou, J., Li, X., Ding, T., You, C., Qu, Q., and Zhu, Z.
\newblock On the optimization landscape of neural collapse under mse loss:
  Global optimality with unconstrained features.
\newblock In \emph{International Conference on Machine Learning}, pp.\
  27179--27202. PMLR, 2022{\natexlab{a}}.

\bibitem[Zhou et~al.(2022{\natexlab{b}})Zhou, You, Li, Liu, Liu, Qu, and
  Zhu]{Zhou22b}
Zhou, J., You, C., Li, X., Liu, K., Liu, S., Qu, Q., and Zhu, Z.
\newblock Are all losses created equal: A neural collapse perspective.
\newblock \emph{arXiv preprint arXiv:2210.02192}, 2022{\natexlab{b}}.

\bibitem[Zhu et~al.(2020)Zhu, Soudry, Eldar, and Wakin]{Zhu18}
Zhu, Z., Soudry, D., Eldar, Y.~C., and Wakin, M.~B.
\newblock The global optimization geometry of shallow linear neural networks.
\newblock \emph{Journal of Mathematical Imaging and Vision}, 62:\penalty0
  279--292, 2020.

\bibitem[Zhu et~al.(2021)Zhu, Ding, Zhou, Li, You, Sulam, and Qu]{Zhu21}
Zhu, Z., Ding, T., Zhou, J., Li, X., You, C., Sulam, J., and Qu, Q.
\newblock A geometric analysis of neural collapse with unconstrained features.
\newblock \emph{Advances in Neural Information Processing Systems},
  34:\penalty0 29820--29834, 2021.

\end{thebibliography}
\bibliographystyle{icml2023}

\newpage
\appendix
\onecolumn
\begin{center}
{\bf \Large Appendix for ``Neural Collapse in Deep Linear Networks: From Balanced to Imbalanced Data''}
\end{center}

Firstly, we study $\mathcal{NC}$ characteristics for cross-entropy loss function in deep linear networks in Appendix \ref{sec:CEloss}.
The delayed related works discussion are provided in Appendix \ref{sec:relatedworks}.
Next, we present additional numerical results and experiments, details of training hyperparameters and describe $\mathcal{NC}$ metrics used for experiments in Appendix \ref{sec:experiment_details_appendix}. Finally, detailed proofs for Theorems \ref{thm:bias-free}, \ref{thm:UFM_imbalance}, \ref{thm:deep_imbalance} and \ref{thm:CE} are provided in Appendix \ref{sec:proofs_balanced}, \ref{sec:proofs_im}, \ref{sec:proofs_im_deep} and \ref{sec:CE_proof}, respectively.

\DoToC

\section{Neural Collapse in Deep Linear Networks under UFM Setting for CE with Balanced Data}
\label{sec:CEloss}

In this section, we turn to cross-entropy loss and generalize $\mathcal{NC}$ for deep linear networks with last-layer bias under balanced setting, and a mild assumption that all the hidden layers dimension are at least $K-1$ is required. We consider the training problem \eqref{eq:UFM_general_linear} with CE loss as following:
\begin{align}
    &\min_{\mathbf{W}_{M},\ldots, \mathbf{W}_{1}, \mathbf{H}_{1}, \mathbf{b}} 
    \frac{1}{N} \sum_{k=1}^{K} \sum_{i=1}^{n} 
    \mathcal{L}_{CE} (\mathbf{W}_{M} \ldots \mathbf{W}_{1} \mathbf{h}_{k,i} + \mathbf{b}, \mathbf{y}_{k}) 
    +
    \frac{\lambda_{W_{M}}}{2} \| \mathbf{W}_{M} \|^2_F +  \ldots + \frac{\lambda_{H_{1}}}{2} \| \mathbf{H}_{1} \|^2_F
    + \frac{\lambda_{b}}{2} \|  \mathbf{b} \|_2^2,
    \label{eq:CE_loss}
\end{align}
where:
\begin{align}
    \mathcal{L}_{CE} (\mathbf{z}, \mathbf{y}_{k}) 
    := -\log \left(
    \frac{ e^{z_{k}} }{ \sum_{i=1}^{K} e^{z_{i}}}
    \right).
    \nonumber
\end{align}

\begin{theorem}
    \label{thm:CE}
    Assume $d_{k} \geq K - 1 \: \forall \: k \in [M]$, then any global minimizer $(\mathbf{W}_{M}^{*}, \ldots,  \mathbf{W}_{1}^{*}, \mathbf{H}_{1}^{*}, \mathbf{b}^{*})$  of problem \eqref{eq:CE_loss} satisfies:
    \begin{itemize}
        \item $(\mathcal{NC}1) + (\mathcal{NC}3)$:
        \begin{align}
            \mathbf{h}_{k,i}^{*}
            &= \frac{\lambda_{H_{1}}^{M}}{ \lambda_{W_{M}} \lambda_{W_{M-1}} \ldots \lambda_{W_{1}}}
            \frac{\sum_{k=1}^{K-1} s_{k}^{2}}{\sum_{k=1}^{K-1} s_{k}^{2M}} (\mathbf{W}_{M}^{*}  \mathbf{W}_{M-1}^{*}
           \ldots\mathbf{W}_{1}^{*})_{k} \quad \forall k \in [K], i \in [n] 
           \nonumber \\
            \Rightarrow
            \mathbf{h}_{k,i}^{*} &= \mathbf{h}_{k}^{*} \quad \forall \: i \in [n], k \in [K] ,
            \nonumber
        \end{align}
        where $\{ s_{k} \}_{k=1}^{K-1}$ are the singular values of $\mathbf{H}^{*}_{1}$.
        
        \item $(\mathcal{NC}2):$
        $\mathbf{H}_{1}^{*}$ and $\mathbf{W}_{M}^{*} \mathbf{W}_{M-1}^{*} \cdots \mathbf{W}_{1}^{*}$ will converge to a simplex ETF when training progresses:
        \begin{align}
            (\mathbf{W}_{M}^{*} \mathbf{W}_{M-1}^{*} \cdots \mathbf{W}_{1}^{*}) (\mathbf{W}_{M}^{*} \mathbf{W}_{M-1}^{*} \cdots \mathbf{W}_{1}^{*})^{\top} 
           =  
            \frac{ \lambda_{H_{1}}^{M} \sum_{k=1}^{K-1} s_{k}^{2M}}{(K-1) \lambda_{W_{M}} \lambda_{W_{M-1}} \ldots \lambda_{W_{1}}} \left( 
            \mathbf{I}_{K} - \frac{1}{K} \mathbf{1}_{K} \mathbf{1}_{K}^{\top}
            \right). \nonumber
        \end{align}
        
        \item We have $\mathbf{b}^{*} = b^{*} \mathbf{1}$ where either $b^{*} = 0$ or $\lambda_{b} = 0$.
    \end{itemize}
\end{theorem}

\noindent The proof is delayed until Section \ref{sec:CE_proof} and some of the key techniques are extended from the proof for the plain UFM in \cite{Zhu21}. Comparing with the plain UFM with one layer of weight only, we have for deep linear case similar results as the plain UFM case, with the $(\mathcal{NC}2)$ and $(\mathcal{NC}3)$ property now hold for the product $\mathbf{W}_{M} \mathbf{W}_{M-1} \ldots \mathbf{W}_{1}$ instead of $\mathbf{W}$.

\section{Related Works}
\label{sec:relatedworks}

    \textbf{Neural Collapse for balanced data:} In recent years, there has been a rapid increase in interest in $\mathcal{NC}$, resulting in a decent amount of works in a short period of time. Under UFM, these works studied different training problems and proving ETF and $\mathcal{NC}$ properties for the last-layer classifier and last-layer features by treating the last-layer features as unconstrained variables. In particular, a line of works use UFM with CE training to analyze theoretical abstractions of $\mathcal{NC}$ \cite{Zhu21, Fang21, Lu20, yaras23}. Other works study UFM with MSE loss \cite{Tirer22, Zhou22a, Ergen20, Rangamani22}. $\mathcal{NC}$ phenomenon has also been observed and analyzed for supervised contrastive loss \cite{graf23}. For MSE loss, recent extensions to account for additional layers with non-linearity are studied in \cite{Tirer22, Rangamani22}, or with batch normalization \cite{Ergen20}. \cite{Tirer22} extends UFM to account for one additional layer, from one-layer linear classifier to two-layer linear classifier after the "unconstrained" features. \cite{Tirer22} also extends UFM to two-layer case with ReLU activation but requires a strong assumption about nuclear norm equality (see Table \ref{table:1}). The work in \cite{Rangamani22} studies deep homogeneous networks with MSE loss and trained with stochastic gradient descent. Specifically, the critical points of gradient flow satisfying the so-called symmetric quasi-interpolation assumption are proved to exhibit $\mathcal{NC}$ properties, but the other solutions are not investigated. \cite{Ergen20} derives $\mathcal{NC}$ for networks with parallel architectures without requiring UFM. However, their results require a large number of parallel branches in the architecture and require the
    number of nodes in the second-to-last layer in each branch to be at least the total number of training samples in the dataset. On the other hand, \cite{Zhu21, Zhou22a, Zhou22b} show the benign optimization landscape for several loss functions under the plain UFM setting, demonstrating that critical points can only be global minima or strict saddle points. Another line of work exploits the ETF structure to improve the network design by initially fixing the last-layer linear classifier as a simplex ETF and not performing any subsequent learning \cite{Zhu21, Yang22}.

    \textbf{Neural Collapse for imbalanced data:} Most recent papers study $\mathcal{NC}$ under a balanced setting, i.e., the number of training samples in every class is identical. This setting is vital for the existence of the simplex ETF structure. To the best of our knowledge, $\mathcal{NC}$ with imbalanced data is studied in \cite{Fang21, Christos22, Yang22, Xie22}. In particular, \cite{Fang21} is the first to observe that for imbalanced setting, the collapse of features within the same class $\mathcal{NC}1$ is preserved, but the geometry skew away from ETF. They also present a phenomenon called "Minority Collapse": for large levels of imbalance, the minorities' classifiers collapse to the same vector. \cite{Christos22} theoretically studies the SVM problem, whose global minima follows a more general geometry than the ETF, called "SELI". However, this work also makes clear that the unregularized version of CE loss only converges to KKT points of the SVM problem, which are not necessarily global minima. \cite{Yang22} studies the imbalanced setting but with fixed last-layer linear classifiers initialized as a simplex ETF right at the beginning and proves that the optimal features will also converge to ETF structure in this setting. \cite{Xie22} proposed a novel loss function for balancing different components of the gradients for imbalanced learning. A comparison of our results with some existing works regarding the study of  global optimality conditions is shown in Table \ref{table:1}.

    \textbf{Deep linear networks:}
    Analyzing a deep linear network is an important step in studying deep nonlinear networks. The theoretical analysis of deep nonlinear networks is very challenging and, in fact, there has been no rigorous theory for deep nonlinear networks yet to the best of our knowledge. Thus, deep linear networks have been studied to provide insights into the behavior of deep nonlinear networks. For example, using only linear regression, \cite{Hastie20} can recover several phenomena observed in large-scale deep nonlinear networks, including the double descent phenomenon \cite{Nakkiran19}. \cite{Saxe14, Kawaguchi16, Laurent17, Hardt18} empirically show that the optimization of deep linear models exhibits similar properties to those of the optimization of deep nonlinear models. As pointed out in \cite{Saxe14}, despite the linearity of their input-output map, deep linear networks have nonlinear gradient descent dynamics on weights that change with the addition of each new hidden layer. This nonlinear learning phenomenon is proven to be similar to those seen in deep nonlinear networks.
    
    In practice, deep linear networks can help improve the training and performance of deep nonlinear networks \cite{huh23, guo21, arora18}. Specifically, \cite{huh23} empirically proves that linear overparameterization in nonlinear networks improves generalization on classification tasks (see Section 4 in \cite{huh23}). In particular, \cite{huh23} expands each linear layer into a succession of multiple linear layers and does not include any non-linearities in between, which results in a considerable increase in performance. \cite{guo21} applies a similar strategy for compact networks, and their experiments show that training such expanded networks yields better results than training the original compact networks. \cite{arora18} shows that linear overparameterization, i.e., the use of a deep linear network in place of a classic linear model, induces on gradient descent a particular preconditioning scheme that can accelerate optimization. The preconditioning scheme that deep linear layers introduce can be interpreted as using momentum and adaptive learning rate.
    
    \textbf{Relation with previous works on neural networks optimization landscape: } This work also relates to recent advances in studying the optimization landscape in deep neural network training. As pointed out in \cite{Zhu21}, the UFM takes a top-down approach to the analysis of deep neural networks, where last-layer features are treated as free optimization variables, in contrast to the conventional bottom-up
    approach that studies the problem starting from the input \cite{Baldi89, Zhu18, Kawaguchi16, Yun17, Laurent17, Safran17, Yun18}. These works studies the optimization landscape of two-layer linear network \cite{Baldi89, Zhu18}, deep linear network \cite{Kawaguchi16, Yun17, Laurent17} and non-linear network \cite{Safran17, Yun18}. \cite{Zhu21} provides an interesting perspective about the differences between this top-down and bottom-up approach, with how results stemmed from UFM can provide more insights to the network design and the generalization of deep learning.

\begin{table}[t]
\resizebox{\textwidth}{!}{
\begin{tabular}{c|l|l|l|l|l|l}
\multicolumn{1}{l|}{} &
  \textbf{Loss} &
  \textbf{Train model} &
  \textbf{Setting} &
  \textbf{\begin{tabular}[c]{@{}l@{}}Consider\\  $d < K - 1$?\end{tabular}} &
  \textbf{\begin{tabular}[c]{@{}l@{}}Extra \\ assumption\end{tabular}} &
  \textbf{\begin{tabular}[c]{@{}l@{}}$\mathcal{NC}2$ \\ geometry\end{tabular}} \\ \hline
\textbf{\cite{Zhu21}} &
  CE &
  Plain UFM &
  Balanced &
  No &
  N/a &
  Simplex ETF \\ \hline
\textbf{\cite{Fang21}} &
  CE &
  Layer-peeled &
  Balanced &
  No &
  N/a &
  Simplex ETF \\ \hline
\textbf{\cite{Zhou22a}} &
  MSE &
  Plain UFM &
  Balanced &
  Yes &
  N/a &
  Simplex ETF \\ \hline
\multirow{4}{*}{\textbf{\begin{tabular}[c]{@{}c@{}} \cite{Tirer22}\end{tabular}}} &
  MSE &
  Plain UFM, no bias &
  Balanced &
  No &
  N/a &
  OF \\ \cline{2-7} 
 &
  MSE &
  Plain UFM, un-reg. bias &
  Balanced &
  No &
  N/a &
  Simplex ETF \\ \cline{2-7} 
 &
  MSE &
  Extended UFM 2 linear layers, no bias &
  Balanced &
  No &
  N/a &
  OF \\ \cline{2-7} 
 &
  MSE &
  Extended UFM 2 layers with ReLU, no bias &
  Balanced &
  No &
  \begin{tabular}[c]{@{}l@{}}Nuclear norm \\ equality 
  \tablefootnote{\cite{Tirer22} assumes the nuclear norm of $\mathbf{W}^{*}_{1} \mathbf{H}^{*}_{1}$ and $\text{ReLU}(\mathbf{W}^{*}_{1} \mathbf{H}^{*}_{1})$ are equal for any global solution $(\mathbf{W}^{*}_{2}, \mathbf{W}^{*}_{1}, \mathbf{H}^{*}_{1})$.}
  \end{tabular} &
  OF \\ \hline
\textbf{\begin{tabular}[c]{@{}c@{}}  \cite{Rangamani22} \end{tabular}} &
  MSE &
  Deep ReLU network, no bias &
  Balanced &
  No &
  \begin{tabular}[c]{@{}l@{}}Symmetric Quasi-\\ interpolation \tablefootnote{\cite{Rangamani22} assumes having a classifer $f: \mathbb{R}^{D} \to \mathbb{R}^{K}$ where $[f(\mathbf{x}_{k,i})]_{k} = 1 - \epsilon$ and $[f(\mathbf{x}_{k,i})]_{k^{\prime}} = \epsilon / (K-1) \: \forall \: k^{\prime} \neq k$ for all training samples} \end{tabular} &
  Simplex ETF \\ \hline
\textbf{ \cite{Christos22}} &
  CE &
  UFM Support Vector Machine &
  Imbalanced &
  No &
  N/a &
  SELI \\ \hline
\multirow{5}{*}{\textbf{This work}} &
  MSE &
  Extended UFM M linear layers, no bias (Theorem \ref{thm:bias-free}) &
  Balanced &
  Yes &
  N/a &
  OF \\ \cline{2-7} 
 &
  MSE &
  Extended UFM M linear layers, un-reg. last bias (Theorem \ref{thm:bias-free}) &
  Balanced &
  Yes &
  N/a &
  Simplex ETF \\ \cline{2-7} 
 &
  MSE &
  Plain UFM, no bias (Theorem \ref{thm:UFM_imbalance}) &
  Imbalanced &
  Yes &
  N/a &
  GOF \\ \cline{2-7} 
 &
  MSE &
  Extended UFM M linear layers, no bias (Theorem \ref{thm:deep_imbalance}) &
  Imbalanced  &
  Yes &
  N/a &
  GOF \\ \cline{2-7} 
 &
  CE &
  Extended UFM M linear layers (Theorem \ref{thm:CE}) &
  Balanced &
  No &
  N/a &
  Simplex ETF
\end{tabular}}
\caption{Selected comparision of theoretical results on global optimality conditions with $\mathcal{NC}$ occurrence.}
\label{table:1}
\end{table}

\section{Additional Experiments, Network Training and Metrics}
\label{sec:experiment_details_appendix}
\begin{figure}[t!]
	\centering
	\includegraphics[width = 0.7\columnwidth]{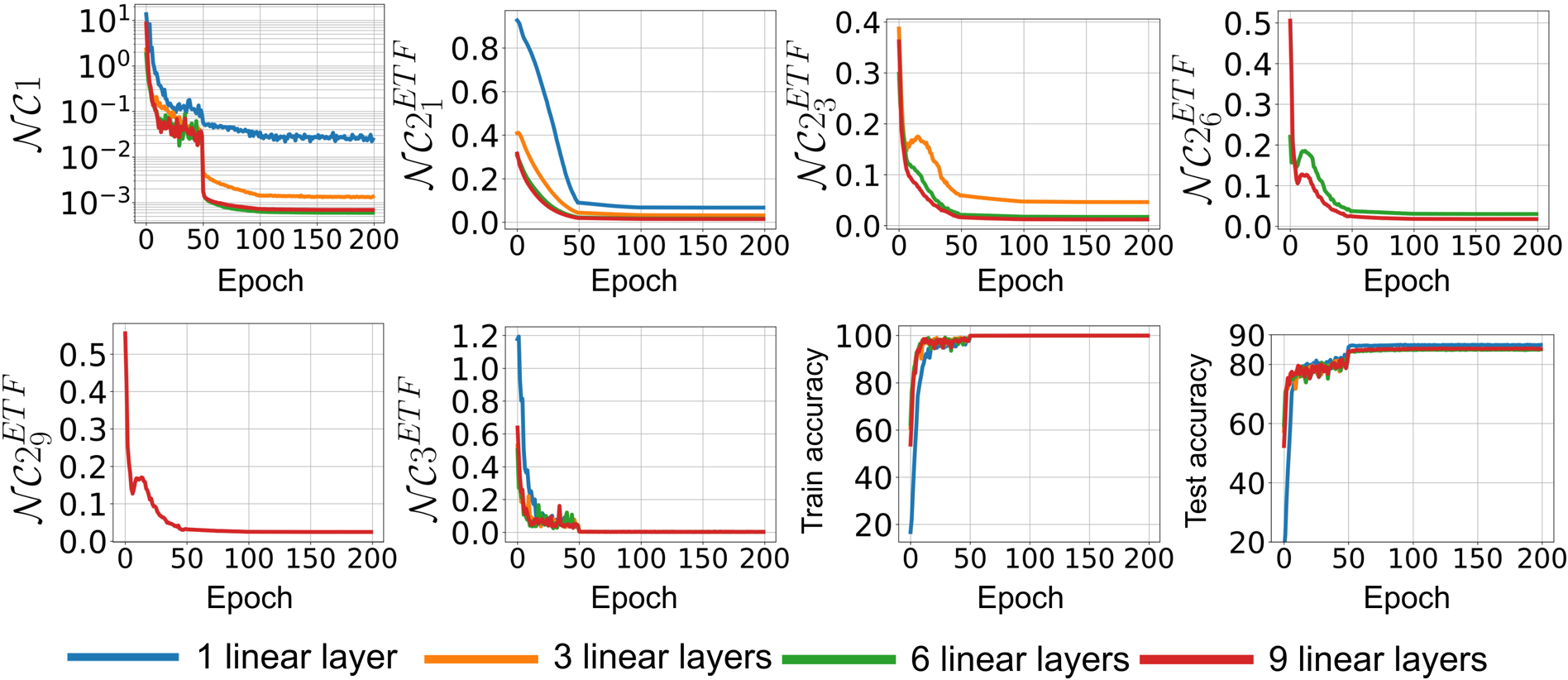}
    \caption{\small \color{black}{Training results with VGG16 backbone on CIFAR10 with MSE loss, balanced data and last-layer bias setting.}}
	\label{fig:VGG_balance}
\end{figure}

\begin{figure}[t]
    \centering
    \includegraphics[width = 0.7\columnwidth]{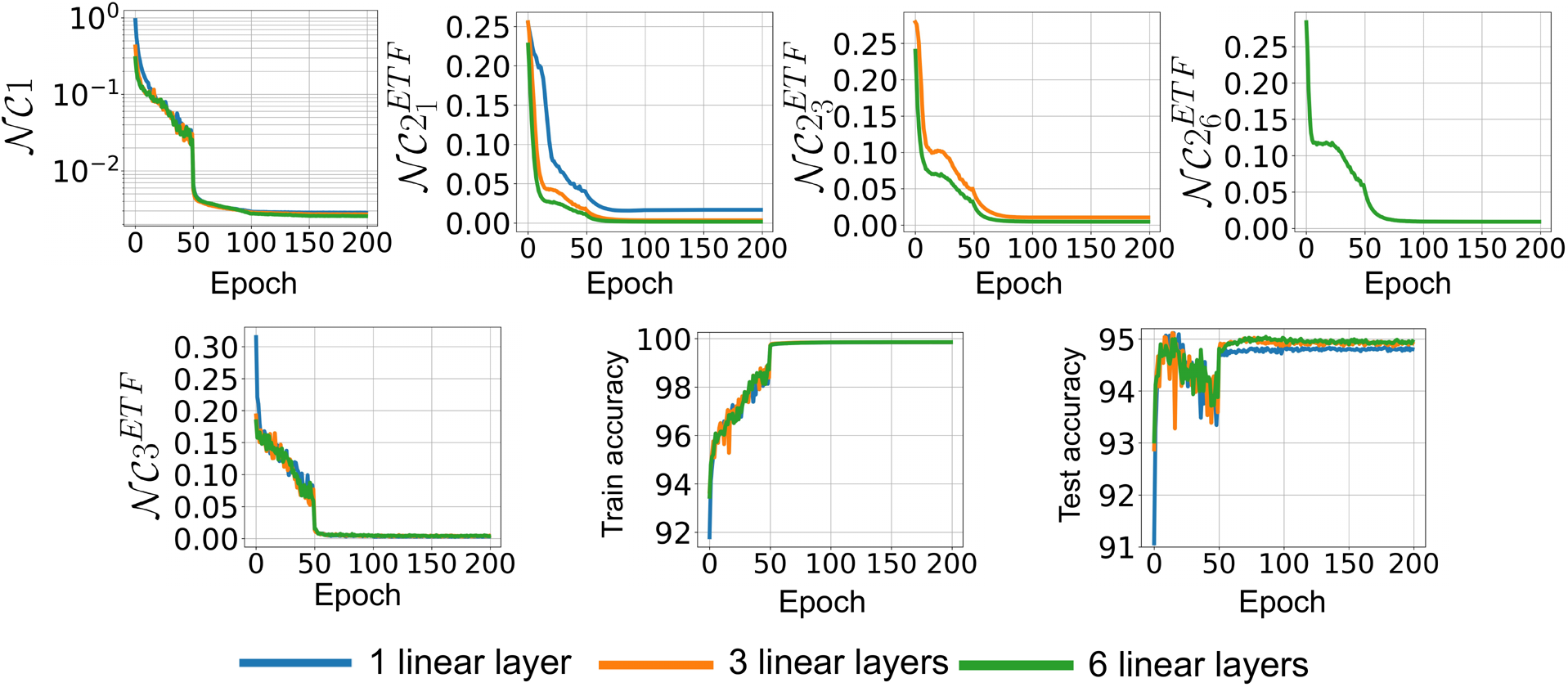}
    \caption{\small \color{black}{Training results with ResNet18 backbone on EMNIST letter dataset with MSE loss, balanced data, and last-layer bias setting.}}
    \label{fig:ResNet_balance_EMNIST}
    \centering
    \includegraphics[width = 
    0.7\columnwidth]{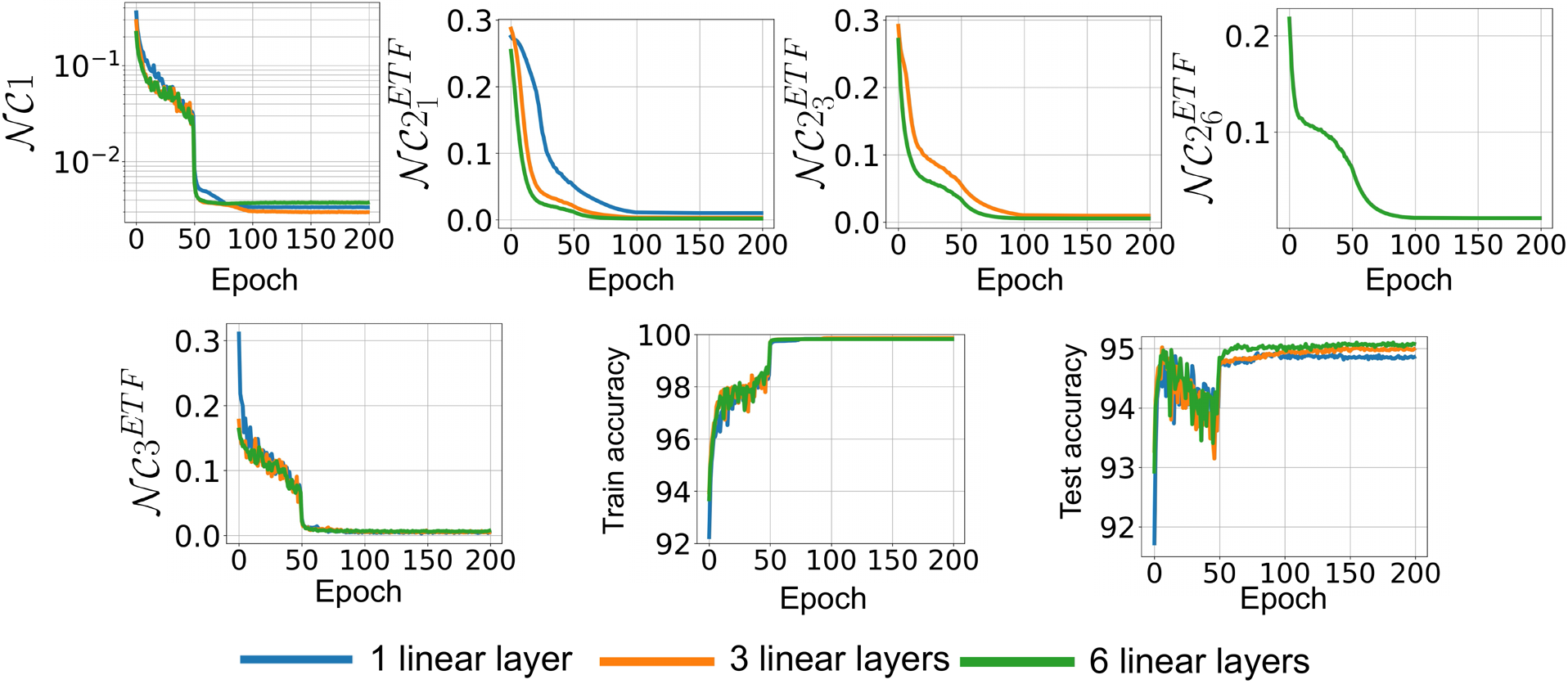}
    \caption{\small \color{black}{Training results with VGG16 backbone on EMNIST letter dataset with MSE loss, balanced data, and last-layer bias setting.}}
	\label{fig:VGG_balance_EMNIST}
\end{figure}
\subsection{Balanced Data}
\label{subsec:balance_case_appendix}
\subsubsection{Metric for measuring $\mathcal{NC}$ in balanced settings}
\label{subsubsec:metric_balance_appendix}
For balanced data, we use similar metrics to those presented in \citep{Zhu21} and \citep{Tirer22}, but also extend them to the multilayer network setting:
\begin{itemize}
    \item \textbf{Features collapse.} 
    Since the collapse of the features of the backbone extractors implies the collapse of the features in subsequent linear layers, we only consider $\mathcal{NC}1$ metric for the output features of the backbone model.
    We recall the definition of the class-means and global-mean of the features $\{ \mathbf{h}_{k, i} \}$ as:\\
    \[ 
    \mathbf{h}_{k} := \frac{1}{n} \sum_{i=1}^n \mathbf{h}_{k,i}, \quad
    \mathbf{h}_{G} := \frac{1}{K n} \sum_{k=1}^K \sum_{i=1}^n \mathbf{h}_{k,i}. \quad 
    \]
    We also define the within-class, between-class covariance matrices, and $\mathcal{NC}1$ metric as following:
    \[
    \mathbf{\Sigma}_W := \frac{1}{N} \sum_{k=1}^{K} \sum_{i=1}^{n} (\mathbf{h}_{k,i} - \mathbf{h}_{k}) (\mathbf{h}_{k,i} - \mathbf{h}_{k})^{\top}, 
    \quad
    \mathbf{\Sigma}_B := \frac{1}{K} \sum_{k=1}^K (\mathbf{h}_{k} - \mathbf{h}_{G}) (\mathbf{h}_{k} - \mathbf{h}_{G})^{\top},
    \]
    \[
    \mathcal{NC}1 := \frac{1}{K} \text{trace}(\mathbf{\Sigma}_W \mathbf{\Sigma}_B^\dagger).
    \]
    where $\mathbf{\Sigma}_B^\dagger$ denotes the pseudo inverse of $\mathbf{\Sigma}_B$.

    \item \textbf{Convergence to OF/Simplex ETF.} To capture the $\mathcal{NC}$ behaviors across layers, we denote $\mathbf{W}^{m} := \mathbf{W}_{M} \mathbf{W}_{M-1} \ldots \mathbf{W}_{M-m+1}$ as the product of last $m$  weight matrices of the deep linear network. We define $\mathcal{NC}2^{OF}_{m}$ and $\mathcal{NC}2^{ETF}_{m}$ to measure the similarity of the learned classifiers  $\mathbf{W}^{m}$ to OF (bias-free case) and  ETF (last-layer bias case) as:
    \begin{align*}
     &\mathcal{NC}2^{OF}_m := \left\Vert \frac{\mathbf{W}^{m} \mathbf{W}^{m \top}}{\left\Vert \mathbf{W}^{m} \mathbf{W}^{m \top} \right\Vert_F} - \frac{1}{\sqrt{K}} \mathbf{I}_K \right\Vert_F,\\
    &\mathcal{NC}2^{ETF}_m := \left\Vert \frac{\mathbf{W}^{m} \mathbf{W}^{m \top}}{\left\Vert \mathbf{W}^{m} \mathbf{W}^{m \top} \right\Vert_F} - \frac{1}{\sqrt{K-1}} \left(\mathbf{I}_K - \frac{1}{K}\mathbf{1}_K \mathbf{1}_K^{\top} \right) \right\Vert_F. \nonumber
    \end{align*}

    \item \textbf{Convergence to self-duality.} 
    We measure the alignment between the learned classifier $\mathbf{W}_M \mathbf{W}_{M-1} \ldots \mathbf{W}_{1}$ and the learned class-means $\overline{\mathbf{H}}$ via:
    \begin{align*}
    &\mathcal{NC}3^{OF} := \left\Vert \frac{\mathbf{W}_{M}\mathbf{W}_{M-1} \ldots \mathbf{W}_{1} \overline{\mathbf{H}}}{\left\Vert \mathbf{W}_{M} \mathbf{W}_{M-1} \ldots \mathbf{W}_{1} \overline{\mathbf{H}} \right\Vert_F} - \frac{1}{\sqrt{K}} \mathbf{I}_K \right\Vert_F, \nonumber \\
    &\mathcal{NC}3^{ETF} := \left\Vert \frac{\mathbf{W}_{M} \mathbf{W}_{M-1} \ldots \mathbf{W}_{1} \overline{\mathbf{H}}}{\left\Vert \mathbf{W}_{M} \mathbf{W}_{M-1}\ldots \mathbf{W}_{1} \overline{\mathbf{H}} \right\Vert_F} -\frac{1}{\sqrt{K-1}}\left( \mathbf{I}_K - \frac{1}{K} \mathbf{1}_K \mathbf{1}_K^{\top} \right) \right\Vert_F,
    \end{align*}
    where $\overline{\mathbf{H}} = [\mathbf{h}_{1}, \ldots, \mathbf{h}_{K}]$ is the class-means matrix.
\end{itemize}
\subsubsection{Additional numerical results for balanced data}
\label{subsubsec:addtional_balanced}

\begin{figure}[t!]
	\centering
	\includegraphics[width = 0.7\columnwidth]{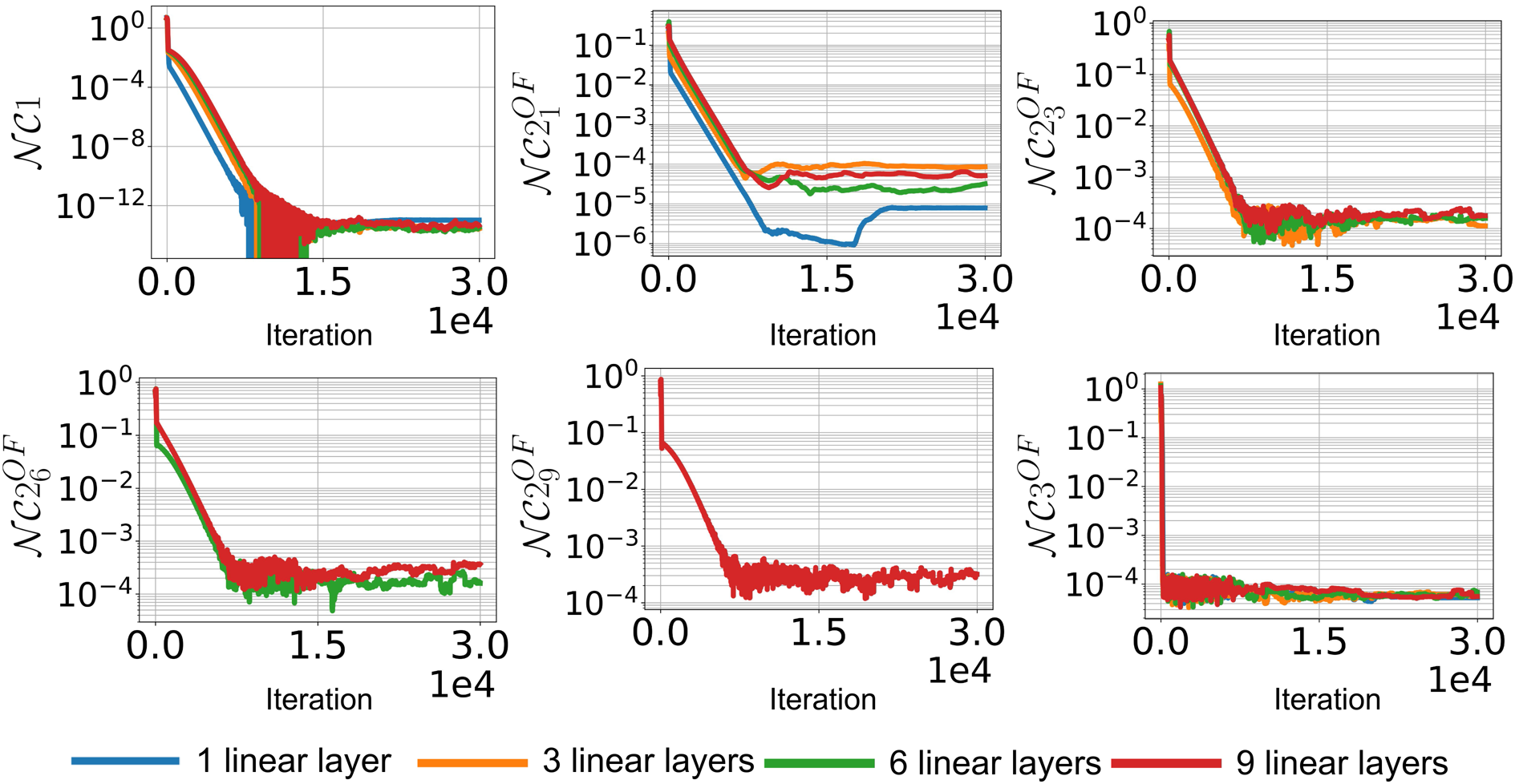}
    \caption{\small \color{black}{Illustration of $\mathcal{NC}$ for direct optimization experiment with MSE loss, balanced data and bias-free setting.}}
	\label{fig:synthetic_balance_nobias}
\end{figure}

\begin{figure}[t!]
	\centering
	\includegraphics[width = 0.7\columnwidth]{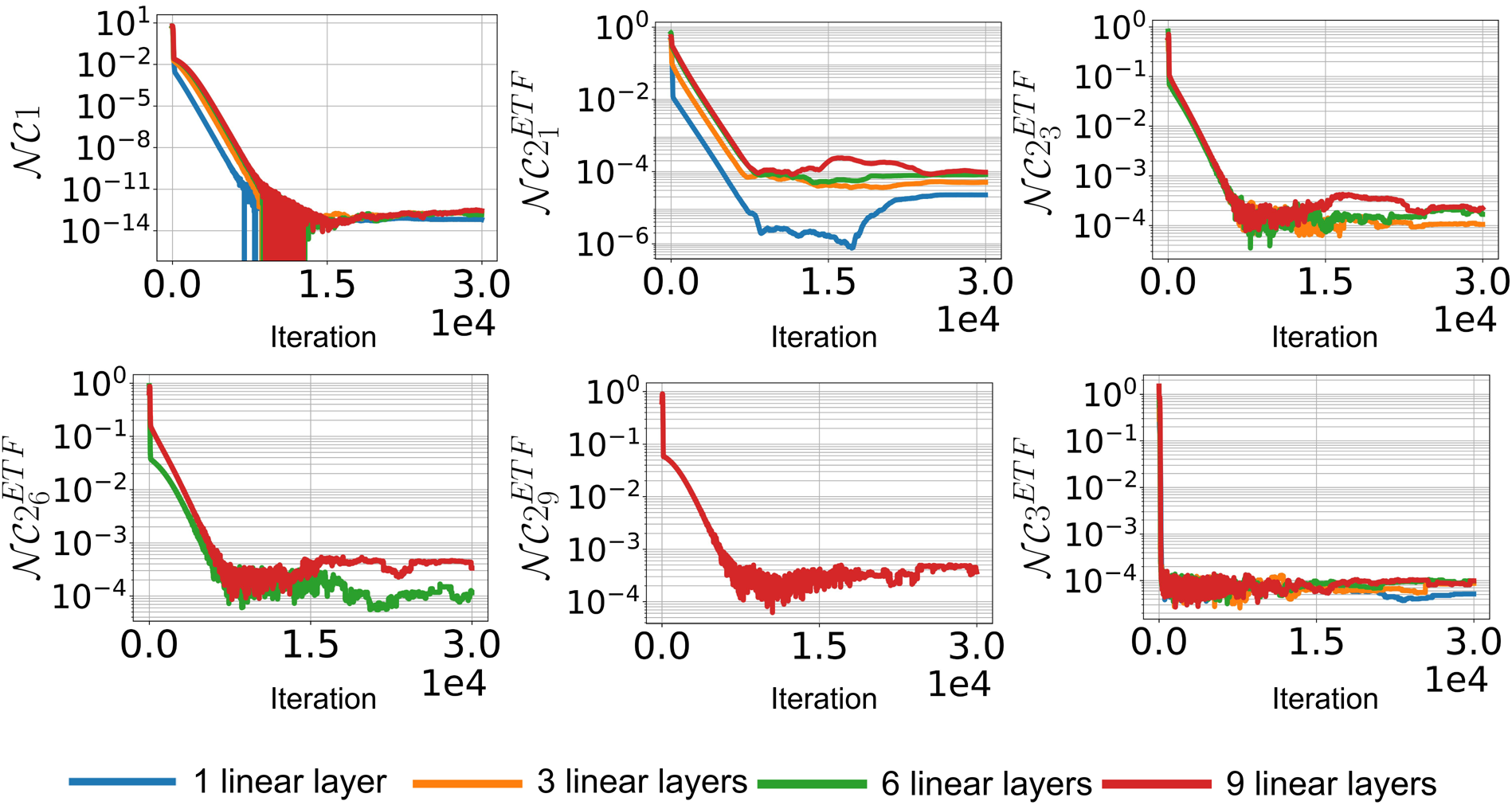}
    \caption{\small \color{black}{Illustration of $\mathcal{NC}$ for direct optimization experiment with MSE loss, balanced data and last-layer bias setting.}}
	\label{fig:synthetic_balance_bias}
\end{figure}

This subsection expands upon the experiment results for balanced data in subsection~\ref{subsec:balanced_data_experiment} by the following points: i) For MLP experiment, we provide $\mathcal{NC}$ metrics measured at the last epoch for the remaining depth-widths combinations mentioned in subsection \ref{subsec:balanced_data_experiment}, ii) Empirically verify Theorem~\ref{thm:CE} of the $\mathcal{NC}$ existence for cross-entropy loss in deep linear network setting, iii) Conduct experiments to verify the consistent of $\mathcal{NC}$ for text classification dataset, and iv) Empirically demonstrate the occurrence of $\mathcal{NC}$ for ReLU network with depth $\in \{2, 3\}$.
\medbreak
\textbf{Last-epoch $\bm{\mathcal{NC}}$ metrics for multilayer perceptron and deep learning experiments:}
 We include the full set of last-epoch $\mathcal{NC}$ metrics for  mentioned MLP depth-width combinations in Table~\ref{table:full_metric_MLP_nobias_appendix} and \ref{table:full_metric_MLP_bias_appendix}. In which, Table~\ref{table:full_metric_MLP_nobias_appendix} corresponds to the bias-free setting and Table~\ref{table:full_metric_MLP_bias_appendix} corresponds to the last-layer bias setting. Similarly, the full set of last-epoch $\mathcal{NC}$ metrics for deep learning experiments with ResNet18 and VGG19 models are also presented in Table~\ref{table:full_metric_deeplearning_appendix}. 

\begin{table}[t]
\resizebox{\textwidth}{!}{\begin{tabular}{cccccccccccccc} 
\toprule
 \scalebox{1}{No. layer} &  \scalebox{1}{Hidden dim} & \scalebox{1}{$\mathcal{NC}1$} & \scalebox{1}{$\mathcal{NC}2_1^{OF}$} & \scalebox{1}{$\mathcal{NC}2_2^{OF}$} & \scalebox{1}{$\mathcal{NC}2_3^{OF}$} & \scalebox{1}{$\mathcal{NC}2_4^{OF}$} & \scalebox{1}{$\mathcal{NC}2_5^{OF}$} & \scalebox{1}{$\mathcal{NC}2_6^{OF}$} &
 \scalebox{1}{$\mathcal{NC}2_7^{OF}$} &
 \scalebox{1}{$\mathcal{NC}2_8^{OF}$} &
 \scalebox{1}{$\mathcal{NC}2_9^{OF}$} &
  \scalebox{1}{$\mathcal{NC}3^{OF}$}\\ 
 \midrule
 \multirow{3}{*}{1} & 512 & $\num{1.819e-3}$ &  $\num{5.856e-2}$ &  &  &  &  &  &  &  &  & $\num{1.769e-2}$ \\
                    & 1024 & $\num{2.437e-4}$ & $\num{3.024e-2}$ & & & & & & & & & $\num{1.528e-2}$ \\
                    & 2048 & $\num{1.259e-4}$ & $\num{1.467e-2}$ & & & & & & & & & $\num{1.712e-2}$ \\
 \cmidrule{1-13}
 \multirow{3}{*}{3} & 512 & $\num{8.992e-3}$ & $\num{5.09e-2}$ & $\num{1.057e-1}$ & $\num{1.486e-1}$ & & & & & & & $\num{2.958e-2}$ \\
                    & 1024 & $\num{2.843e-3}$ & $\num{5.697e-2}$ & $\num{1.009e-1}$ & $\num{1.731e-1}$ & & & & & & & $\num{2.368e-2}$\\
                    & 2048 & $\num{5.165e-4}$ & $\num{3.857e-2}$ & $\num{5.799e-2}$ & $\num{8.648e-2}$ & & & & & & & $\num{2.797e-2}$\\
 \cmidrule{1-13}
 \multirow{3}{*}{6} & 512 & $\num{8.701e-3}$ & $\num{7.833e-2}$ & $\num{1.009e-1}$ & $\num{1.186e-1}$ & $\num{1.340e-1}$ & $\num{1.511e-1}$ & $\num{1.824e-1}$ & & & & $\num{3.478e-2}$\\
                    & 1024 & $\num{2.578e-3}$ & $\num{8.356e-2}$ & $\num{1.066e-1}$ & $\num{1.283e-1}$ & $\num{1.489e-1}$ & $\num{1.725e-1}$ & $\num{2.429e-1}$ & & & & $\num{1.928e-2}$\\
                    & 2048 & $\num{8.231e-4}$ & $\num{7.187e-2}$ & $\num{9.224e-2}$ & $\num{1.078e-1}$ & $\num{1.160e-1}$ & $\num{1.214e-1}$ & $\num{1.386e-1}$ & & & & $\num{3.430e-2}$\\
 \cmidrule{1-13}
 \multirow{3}{*}{9} & 512 & $\num{9.359e-3}$ & $\num{1.149e-1}$ & $\num{1.480e-1}$ & $\num{1.703e-1}$ & $\num{1.824e-1}$ & $\num{1.868e-1}$ & $\num{1.855e-1}$ & $\num{1.821e-1}$ & $\num{1.823e-1}$ & $\num{2.033e-1}$ & $\num{3.074e-2}$\\
                    & 1024 & $\num{2.615e-3}$ & $\num{1.165e-1}$ & $\num{1.488e-1}$ & $\num{1.745e-1}$ & $\num{1.893e-1}$ & $\num{1.961e-1}$ & $\num{1.975e-1}$ & $\num{1.972e-1}$ & $\num{2.013e-1}$ & $\num{2.492e-1}$ & $\num{2.089e-2}$\\
                    & 2048 & $\num{7.694e-4}$ & $\num{1.070e-1}$ & $\num{1.402e-1}$ & $\num{1.701e-1}$ & $\num{1.864e-1}$ & $\num{1.929e-1}$ & $\num{1.892e-1}$ & $\num{1.763e-1}$ & $\num{1.592e-1}$ & $\num{1.371e-1}$ & $\num{2.141e-2}$\\
 \bottomrule
\end{tabular}
}
\caption{Full set of metrics $\mathcal{NC}1$, $\mathcal{NC}2$, and $\mathcal{NC}3$ described in multilayer perceptron experiment in section \ref{subsec:balanced_data_experiment} with bias-free setting.}
\label{table:full_metric_MLP_nobias_appendix}
\end{table}

\begin{table}[t]
\resizebox{\textwidth}{!}{\begin{tabular}{cccccccccccccc} 
\toprule
 \scalebox{1}{No. layer} &  \scalebox{1}{Hidden dim} & \scalebox{1}{$\mathcal{NC}1$} & \scalebox{1}{$\mathcal{NC}2_1^{ETF}$} & \scalebox{1}{$\mathcal{NC}2_2^{ETF}$} & \scalebox{1}{$\mathcal{NC}2_3^{ETF}$} & \scalebox{1}{$\mathcal{NC}2_4^{ETF}$} & \scalebox{1}{$\mathcal{NC}2_5^{ETF}$} & \scalebox{1}{$\mathcal{NC}2_6^{ETF}$} &
 \scalebox{1}{$\mathcal{NC}2_7^{ETF}$} &
 \scalebox{1}{$\mathcal{NC}2_8^{ETF}$} &
 \scalebox{1}{$\mathcal{NC}2_9^{ETF}$} &
  \scalebox{1}{$\mathcal{NC}3^{ETF}$}\\ 
 \midrule
 \multirow{3}{*}{1} & 512 & $\num{2.058e-3}$ &  $\num{4.936e-2}$ & & & & & & & & & $\num{5.406e-3}$ \\
                    & 1024 & $\num{2.791e-4}$ & $\num{2.540e-2}$ & & & & & & & & & $\num{3.862e-3}$ \\
                    & 2048 & $\num{1.434e-4}$ & $\num{9.418e-3}$ & & & & & & & & & $\num{1.750e-3}$ \\
 \cmidrule{1-13}
 \multirow{3}{*}{3} & 512 & $\num{7.601e-3}$ & $\num{5.147e-2}$ & $\num{1.124e-1}$ & $\num{1.586e-1}$ & & & & & & & $\num{1.972e-2}$ \\
                    & 1024 & $\num{2.194e-3}$ & $\num{5.967e-2}$ & $\num{1.071e-1}$ & $\num{1.949e-1}$ & & & & & & & $\num{1.155e-2}$\\
                    & 2048 & $\num{6.397e-4}$ & $\num{3.447e-2}$ & $\num{5.795e-2}$ & $\num{9.811e-2}$ & & & & & & & $\num{5.311e-3}$\\
 \cmidrule{1-13}
 \multirow{3}{*}{6} & 512 & $\num{8.308e-3}$ & $\num{2.006e-2}$ & $\num{5.110e-2}$ & $\num{8.624e-2}$ & $\num{1.221e-1}$ & $\num{1.587e-1}$ & $\num{1.997e-1}$ & & & & $\num{1.757e-2}$\\
                    & 1024 & $\num{2.258e-3}$ & $\num{2.818e-2}$ & $\num{6.244e-1}$ & $\num{9.861e-2}$ & $\num{1.350e-1}$ & $\num{1.710e-1}$ & $\num{2.350e-1}$ & & & & $\num{1.320e-2}$\\
                    & 2048 & $\num{5.653e-4}$ & $\num{1.848e-2}$ & $\num{3.409e-2}$ & $\num{5.134e-2}$ & $\num{6.849e-2}$ & $\num{8.570e-2}$ & $\num{1.279e-1}$ & & & & $\num{4.522e-3}$\\
 \cmidrule{1-13}
 \multirow{3}{*}{9} & 512 & $\num{9.745e-3}$ & $\num{1.608e-2}$ & $\num{2.040e-2}$ & $\num{3.916e-2}$ & $\num{6.095e-2}$ & $\num{8.494e-2}$ & $\num{1.107e-1}$ & $\num{1.383e-1}$ & $\num{1.679e-1}$ & $\num{2.102e-1}$ & $\num{1.772e-2}$\\
                    & 1024 & $\num{2.587e-3}$ & $\num{1.522e-2}$ & $\num{2.462e-2}$ & $\num{4.350e-2}$ & $\num{6.525e-2}$ & $\num{8.910e-2}$ & $\num{1.147e-1}$ & $\num{1.422e-1}$ & $\num{1.711e-1}$ & $\num{2.370e-1}$ & $\num{1.245e-2}$\\
                    & 2048 & $\num{6.943e-4}$ & $\num{1.217e-2}$ & $\num{2.043e-2}$ & $\num{3.218e-2}$ & $\num{4.517e-2}$ & $\num{5.899e-1}$ & $\num{7.350e-2}$ & $\num{8.881e-2}$ & $\num{1.042e-1}$ & $\num{1.414e-1}$ & $\num{7.937e-3}$\\
 \bottomrule
\end{tabular}
}
\caption{Full set of metrics $\mathcal{NC}1$, $\mathcal{NC}2$, and $\mathcal{NC}3$ in multilayer perceptron experiment in section \ref{subsec:balanced_data_experiment} with last-layer bias setting.}
\label{table:full_metric_MLP_bias_appendix}
\end{table}

\begin{table}[t]
\resizebox{\textwidth}{!}{\begin{tabular}{cccccccccccccc} 
\toprule
 \scalebox{1}{Model name} &  \scalebox{1}{No.layer} & \scalebox{1}{$\mathcal{NC}1$} & \scalebox{1}{$\mathcal{NC}2_1^{ETF}$} & \scalebox{1}{$\mathcal{NC}2_2^{ETF}$} & \scalebox{1}{$\mathcal{NC}2_3^{ETF}$} & \scalebox{1}{$\mathcal{NC}2_4^{ETF}$} & \scalebox{1}{$\mathcal{NC}2_5^{ETF}$} & \scalebox{1}{$\mathcal{NC}2_6^{ETF}$} &
 \scalebox{1}{$\mathcal{NC}2_7^{ETF}$} &
 \scalebox{1}{$\mathcal{NC}2_8^{ETF}$} &
 \scalebox{1}{$\mathcal{NC}2_9^{ETF}$} &
  \scalebox{1}{$\mathcal{NC}3^{ETF}$}\\ 
 \midrule
  \multirow{4}{*}{ResNet18} & 1 & $\num{1.556e-3}$ & $\num{4.376e-2}$ & & & & & & & & & $\num{3.598e-3}$ \\
                    & 3 & $\num{4.713e-4}$ & $\num{2.191e-2}$ & $\num{4.714e-2}$ & $\num{7.813e-2}$ &  & & & & & & $\num{2.131e-3}$\\
                    & 6 & $\num{1.824e-4}$ & $\num{4.295e-3}$ & $\num{4.868e-3}$ & $\num{7.651e-3}$ & $\num{1.156e-2}$ & $\num{1.681e-2}$ & $\num{2.459e-2}$ & & & & $\num{1.817e-3}$\\
                    & 9 & $\num{2.156e-4}$ & $\num{3.609e-3}$ & $\num{6.459e-3}$ & $\num{7.835e-3}$ & $\num{8.056e-3}$ & $\num{8.096e-3}$ & $\num{8.362e-3}$ & $\num{9.400e-3}$ & $\num{1.212e-2}$ & $\num{1.683e-2}$ & $\num{2.210e-3}$\\
                    
 \cmidrule{1-13}
 \multirow{4}{*}{VGG16} & 1 & $\num{2.447e-2}$ & $\num{6.689e-2}$ &  & & & & & & & & $\num{1.977e-3}$ \\
                    & 3 & $\num{1.347e-3}$ & $\num{3.120e-2}$ & $\num{3.035e-2}$& $\num{4.606e-2}$ & & & & & & & $\num{2.767e-3}$\\
                    & 6 & $\num{5.959e-4}$ & $\num{1.645e-2}$ & $\num{1.266e-2}$ & $\num{1.703e-2}$ & $\num{2.183e-2}$ & $\num{2.473e-2}$ & $\num{3.015e-2}$ & & & & $\num{2.483e-3}$\\
                    & 9 & $\num{6.893e-4}$ & $\num{1.438e-2}$ & $\num{9.511e-3}$ & $\num{1.198e-2}$ & $\num{1.314e-2}$ & $\num{1.619e-2}$ & $\num{1.774e-2}$ & $\num{2.030e-2}$ & $\num{2.218e-2}$ & $\num{2.445e-2}$ & $\num{2.434e-3}$\\
\bottomrule
\end{tabular}
}
\caption{Full set of metrics $\mathcal{NC}1$, $\mathcal{NC}2$, and $\mathcal{NC}3$ described in deep learning experiment in section \ref{subsec:balanced_data_experiment} for ResNet18 and VGG16 backbones with last-layer bias setting.}
\label{table:full_metric_deeplearning_appendix}
\end{table}

\textbf{Verification of Theorem~\ref{thm:CE} for CE loss:} We run two experiments to verify neural collapse for CE loss described in Theorem~\ref{thm:CE} in two settings: MLP backbone model and direct optimization. Our network training procedure is similar to multilayer perceptron experiment and direct optimization experiment for last-layer bias setting described in subsection~\ref{subsec:balanced_data_experiment}. For MLP experiment, we only change the learning rate to $\num{2e-4}$ and substitute cross entropy loss in place of MSE loss. We run the experiment with all depth-width combinations with linear layer depth $\in\{1, 3\}$ and width $\in\{512, 1024, 2048\}$. For direct optimization experiment, we change learning rate to $0.02$, width to $256$, substitute cross entropy loss in place of MSE loss, and keep other settings to be the same. 

Theorem~\ref{thm:CE} indicates that all the features of the same class converge to a single vector, and the alignment between the learned classifier $\mathbf{W}_M \mathbf{W}_{M-1} \ldots \mathbf{W}_{1}$ and the learned class-means $\overline{\mathbf{H}}$ has ETF form. Therefore, we use the same $\mathcal{NC}1$ and $\mathcal{NC}3$ as in the balanced data, last-layer bias case. Theorem~\ref{thm:CE} also indicates that $\mathbf{W}_M \mathbf{W}_{M-1} \ldots \mathbf{W}_{1}$ converges to ETF form. Hence, the metric used for CE loss to measure the convergence of $\mathbf{W}_M \mathbf{W}_{M-1} \ldots \mathbf{W}_{1}$ is defined as  $\mathcal{NC}2_{CE}^{ETF} := \mathcal{NC}2_{M}^{ETF}$, where $\mathcal{NC}2_{M}^{ETF}$ is defined in \ref{subsubsec:metric_balance_appendix}. Fig.~\ref{fig:mlp_balance_CE} and Fig.~\ref{fig:synthetic_balance_CE} demonstrate the convergence of $\mathcal{NC}$ for MLP and direct optimization experiments, respectively. The convergence to 0 of the $\mathcal{NC}$ metrics verifies theorem~\ref{thm:CE}.

\begin{figure}[t!]
	\centering
	\includegraphics[width = 0.7\columnwidth]{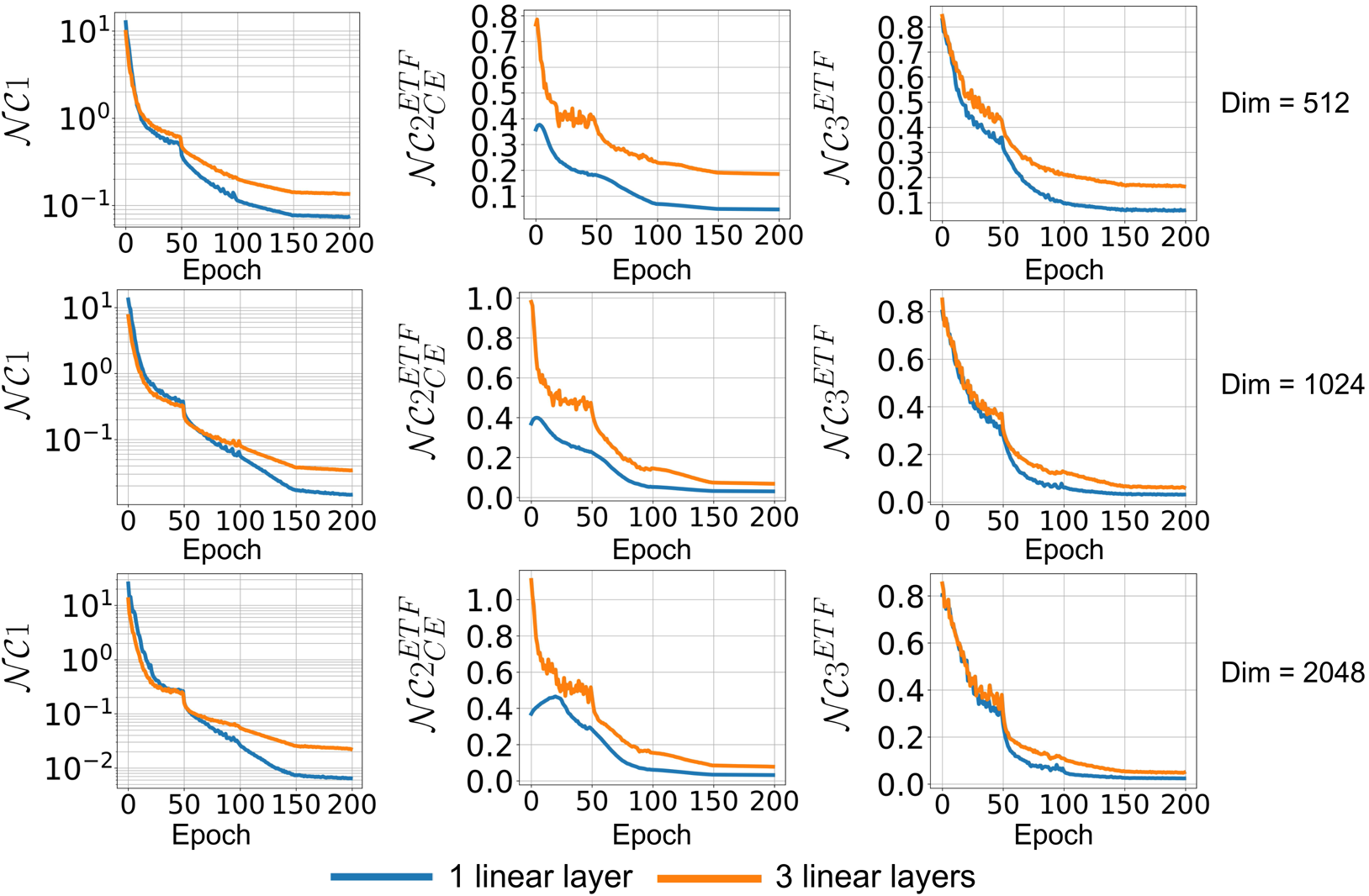}
    \caption{\small \color{black}{Illustration of $\mathcal{NC}$ with 6-layer MLP backbone on CIFAR10 for cross entropy loss, balanced data and last-layer bias setting.}}
    \label{fig:mlp_balance_CE}
\end{figure}

\begin{figure}[t!]
	\centering
	\includegraphics[width = 0.7\columnwidth]{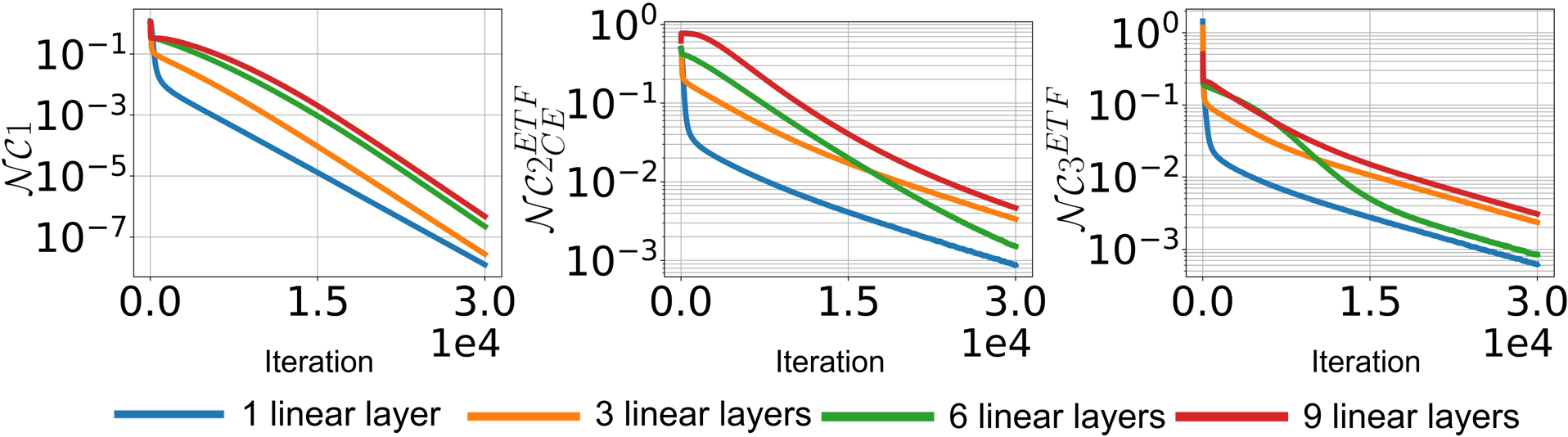}
    \caption{\small \color{black}{Illustration of $\mathcal{NC}$ for direct optmization experiment with cross-entropy loss, balanced data and last-layer bias setting.}}
    \label{fig:synthetic_balance_CE}
\end{figure}

\textbf{Text classification experiment:} To further validate the consistent of $\mathcal{NC}$ through different datasets, we conduct experiments on 4 subsets of text classification datasets including: AG News, IMDB, Sogou News, and Yelp Review Polarity datasets. For each dataset, we randomly choose 3000 samples per class for the training set. We use average word embedding as the backbone model, followed by a linear network with depth=$\{1, 3\}$. The model for AG News dataset has width=$\{2048\}$. Both IMDB and Yelp Review Polarity datasets share width=$\{128\}$, while width=$\{256\}$ is used for Sogou News dataset. All models are trained with MSE loss for until convergence using SGD. Fig.~\ref{fig:AG_News_balance}, Fig.~\ref{fig:IMDB_balance}, Fig.~\ref{fig:Sogou_News_balance}, and Fig.~\ref{fig:Yelp_Review_Polarity_balance} show the convergence to $0$ of $\mathcal{NC}$ metrics. The results demonstrate that the $\mathcal{NC}$ phenomenon described in Theorem~\ref{thm:bias-free} can also be observed in when training with text classification datasets. 

\begin{figure}[t]
    \centering
    \includegraphics[width = 0.7\columnwidth]{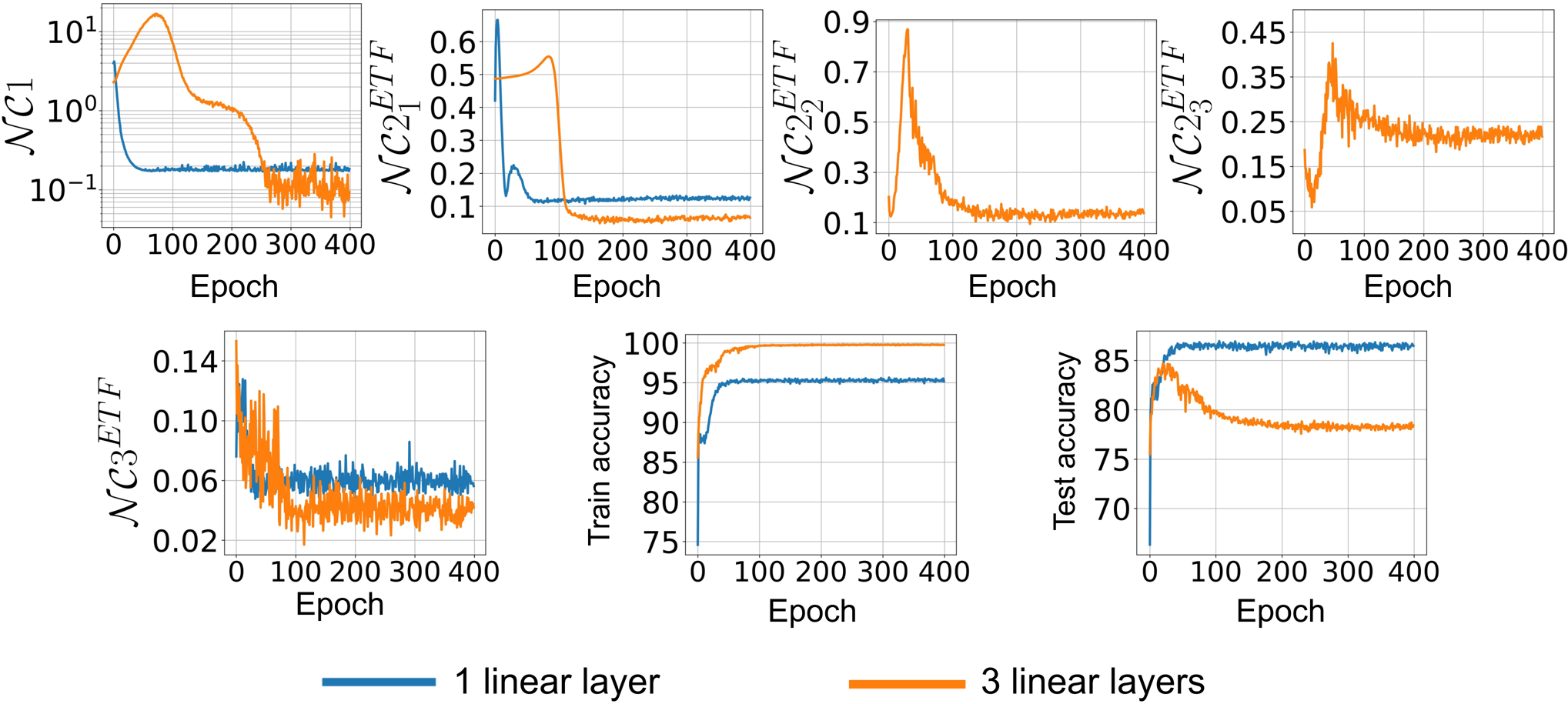}
    \caption{\small \color{black}{Training results with average word embedding backbone on AG News dataset with MSE loss, balanced data and last-layer bias setting.}}
    \label{fig:AG_News_balance}
\end{figure}

\begin{figure}[t]
    \centering
    \vspace{-5pt}
    \includegraphics[width = 
    0.7\columnwidth]{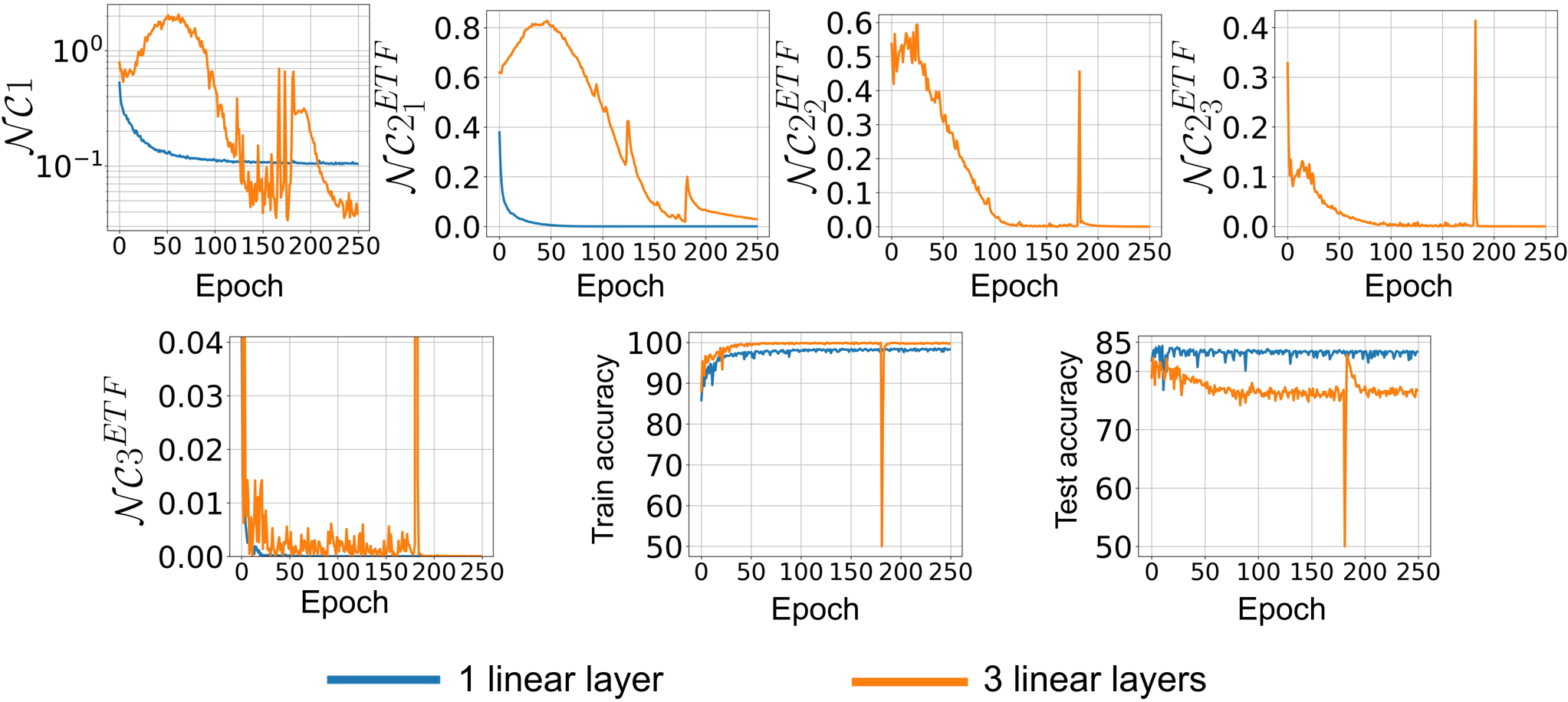}
    \caption{\small \color{black}{Training results with average word embedding backbone on IMDB dataset with MSE loss, balanced data and last-layer bias setting.}}
	\label{fig:IMDB_balance}
\end{figure}

\begin{figure}[t]
    \centering
    \includegraphics[width = 
    0.7\columnwidth]{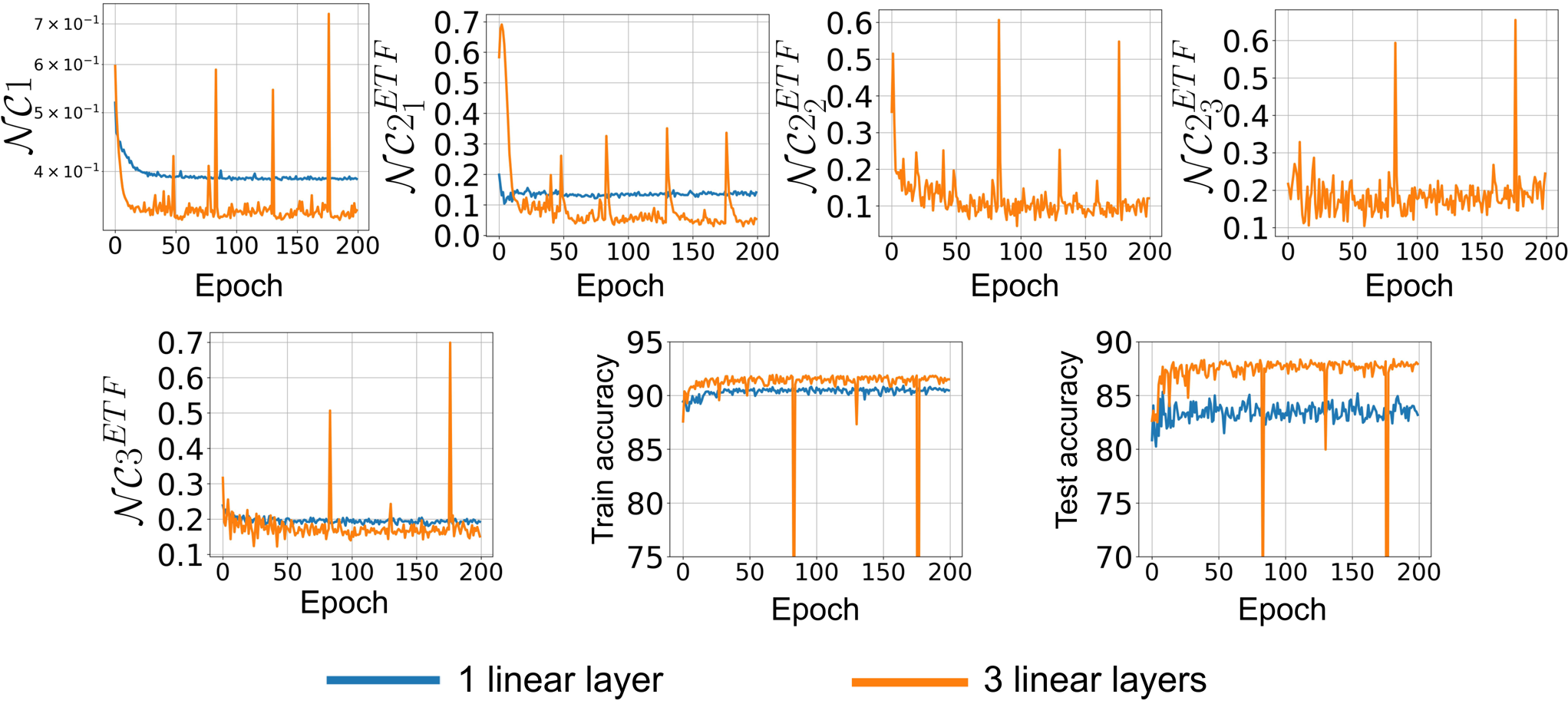}
    \caption{\small \color{black}{Training results with average word embedding backbone on Sogou News dataset with MSE loss, balanced data and last-layer bias setting.}}
	\label{fig:Sogou_News_balance}
\end{figure}

\begin{figure}[t]
    \centering
    \includegraphics[width = 
    0.7\columnwidth]{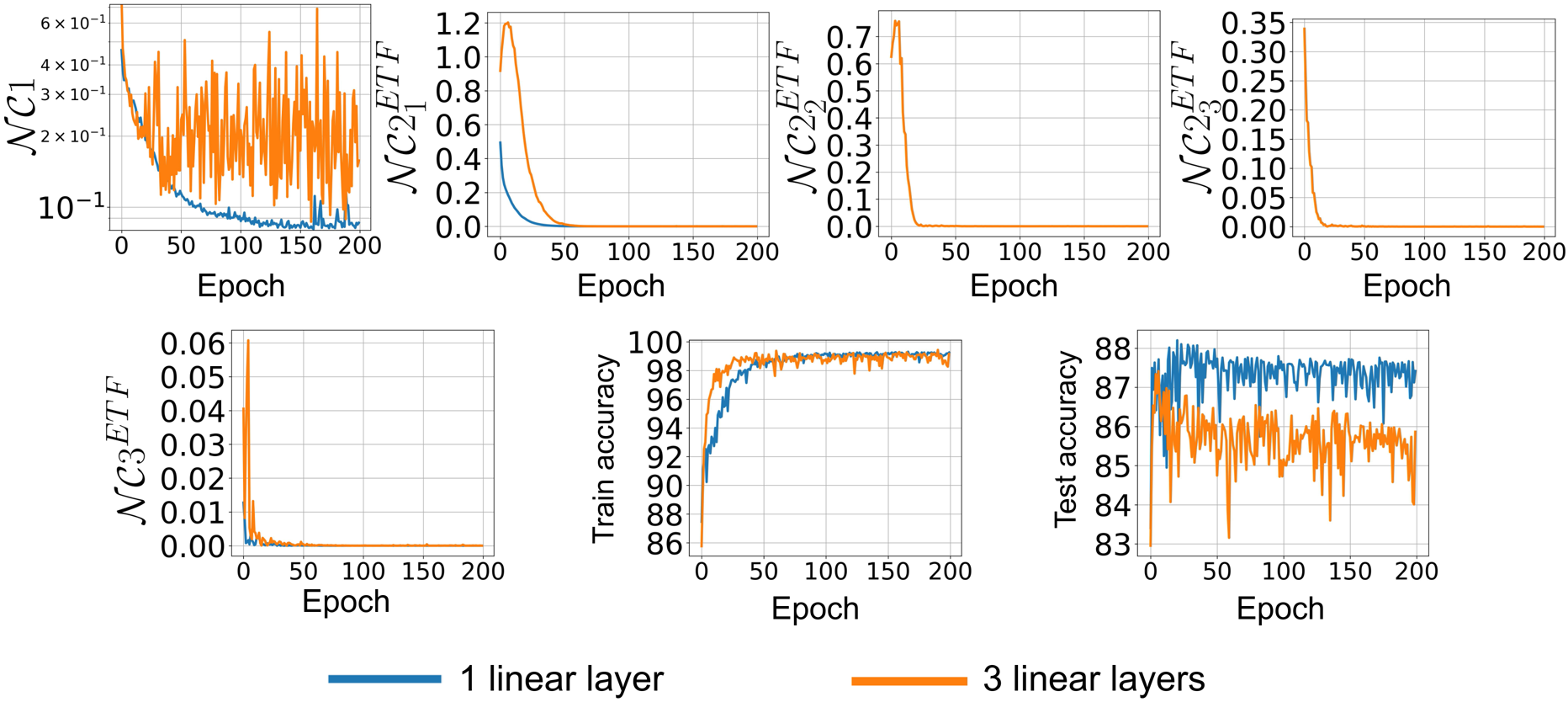}
    \caption{\small \color{black}{Training results with average word embedding backbone on Yelp Review Polarity dataset with MSE loss, balanced data and last-layer bias setting.}}
	\label{fig:Yelp_Review_Polarity_balance}
\end{figure}

\textbf{ReLU experiment:} We conjecture that the occurrence of ETF structure across layers also holds true with nonlinear ReLU activation included. To empirically verify the conjecture, we replace the deep linear network by a deep ReLU network and use batch normalization after each ReLU activation layer. We conduct the experiment on CIFAR10 dataset with ResNet18 backbone under the same setup as the deep learning experiment described in section \ref{subsec:balanced_data_experiment}. Fig.~\ref{fig:ReLU_conjecture} demonstrates that the $\mathcal{NC}$ phenomenon described in Theorem~\ref{thm:bias-free} can still be observed for ReLU network with depth $\in \{2, 3\}$.

\begin{figure}[t]
    \centering
    \includegraphics[width = 
    0.7\columnwidth]{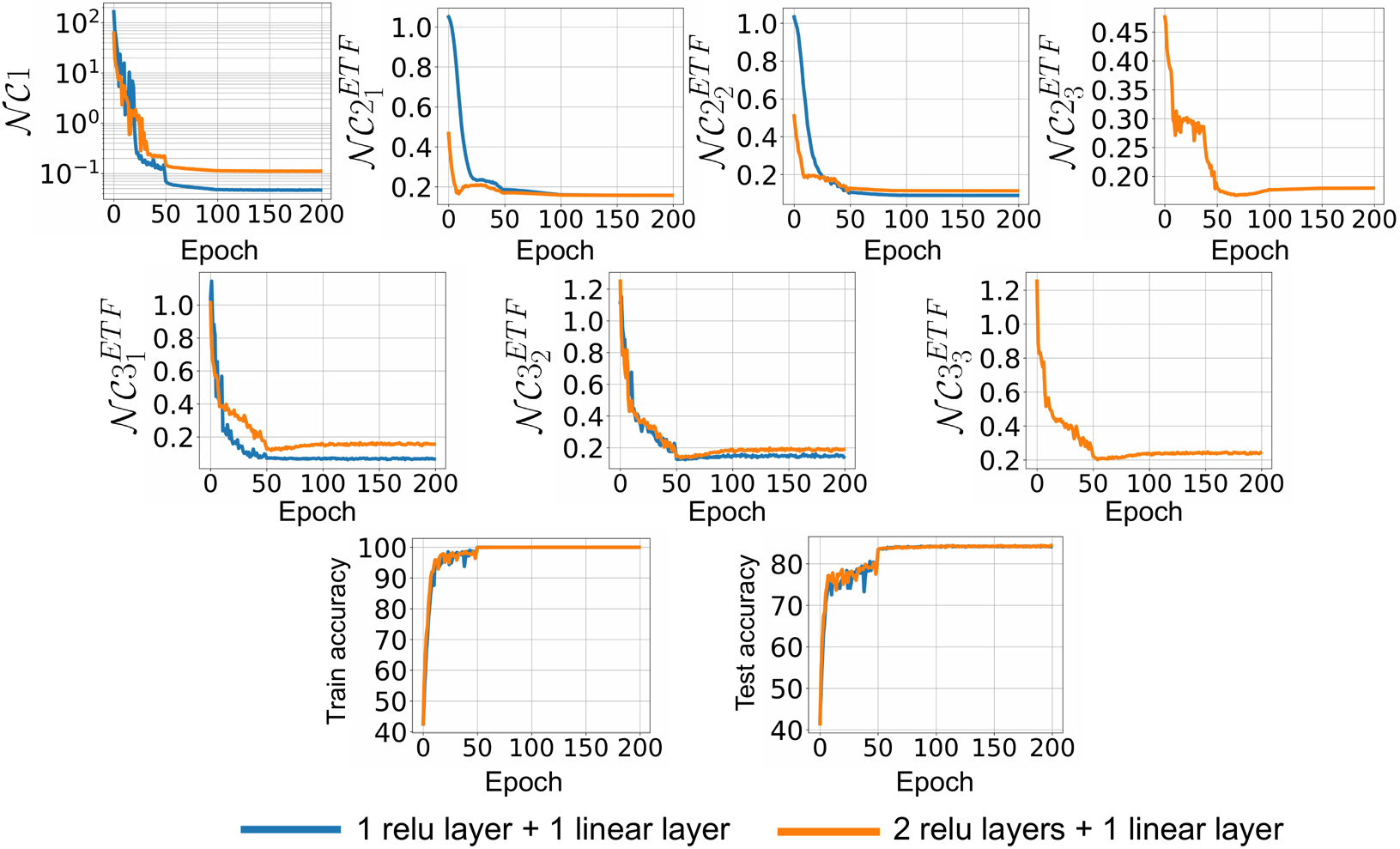}
    \caption{\small \color{black}{Training results with ResNet18 backbone on CIFAR10 dataset with deep ReLU network in place of deep linear network trained with MSE loss, balanced data and last-layer bias setting.}}
	\label{fig:ReLU_conjecture}
\end{figure}

\subsubsection{Details of network training and hyperparameters for balanced data experiments}
\label{subsubsec:detail_training_balanced}
\textbf{Multilayer perceptron experiment with CIFAR10 dataset:} In this experiment, we use a 6-layer MLP model with ReLU activation as the backbone feature extractor. Hidden width of the backbone model and the deep linear network are set to be equal. We cover all depth-width combinations with depth $\in\{1, 3, 6, 9\}$ and width $\in\{512, 1024, 2048\}$ for two settings, bias-free and last-layer bias. All models are trained with Adam optimizer with MSE loss for 200 epochs with batch size $128$ and learning rate $\num{1e-4}$ (divided by 10 every 50 epochs). Weight decay and feature decay are set to $\num{1e-4}$.\

\textbf{Deep learning experiment with CIFAR10 dataset:} In deep learning experiment, we use ResNet18 and VGG16 as backbones feature extractors. We train both models with SGD optimizer with batch size $128$ for MSE loss. Data augmentation is not used in this experiment. The learning rate decays $0.1$ every $50$ epochs for $200$ epochs. Depth of the deep linear layers are selected from the set $\{1, 3, 6, 9\}$. Width of the deep linear layers are set to $512$ to be equal to the last-layer dimension of the backbone model. Weight decay in both models is enforced on all network parameters to align with the typical training protocol. For ResNet18 backbone models, we use the learning rate of $0.05$ and weight decay of $\num{2e-4}$. For VGG16 backbone, the learning rate is $0.02$. Except for VGG16-backbone with $1$ linear layer using weight decay of $\num{5e-4}$, all other VGG16-backbone models shares the weight decay of $\num{3e-4}$.

\textbf{Deep learning experiment with EMNIST letter dataset:} In this experiment, our models and optimization schemes are identical to the deep learning experiment with CIFAR10 dataset. For ResNet18 bacbone models, we use the learning rate of $0.05$ and weight decay of $\num{2e-4}$ for all depths. For all VGG16 backbone models, the learning rate is $0.02$ and weight decay is $\num{3e-4}$.

\textbf{Direct optimization experiment:} In this experiment, we replicate the optimization problem \eqref{eq:UFM_general_linear}. $\mathbf{W}_{M}, \ldots, \mathbf{W}_{1}$ and $\mathbf{H}_{1}$ are initialized with standard normal distribution scaled by $0.1$. We set $K=4, n=100, d_M = \ldots = d_1 = 64$ and all $\lambda$'s are set to be $\num{5e-4}$. Depth of the linear layers are selected from the set $\{1, 3, 6, 9\}$. $\mathbf{W}_{M}, \ldots, \mathbf{W}_{1}$ and $\mathbf{H}_{1}$ are optimized by gradient descent for $30000$ iterations with learning rate $0.1$.

\textbf{Text classification experiment:} In this experiment, we use average word embedding as the backbone feature extractor and train the models on subsets of 4 text classification datasets including AG News, IMDB, Sogou News, and Yelp Review Polarity. Followed the backbone feature extractor is a linear network with depth$=\{1, 3\}$. For each dataset, 3000 samples of each class in the full training set in randomly sampled to create the training subset. Both IMDB and Yelp Review Polarity models share width $=128$, AG News model has width $=2048$, and model for Sogou News dataset has width $=256$. Each model is trained with SGD optimizer, batch size $128$ and MSE loss until convergence. We perform hyperparameter search with learning rate $\in \{ \num{1e-4}, \num{5e-4}, 0.001, 0.005, 0.01\}$. Weight decay for all models is enforced on all network parameters and set to $\num{1e-4}$.

\textbf{ReLU experiment:} In this experiment, we run the experiment on CIFAR10 dataset with ResNet18 backbone and replace the deep linear network by a deep ReLU network with depth $\in \{2, 3\}$. The depth of all models are set to $512$, learning rate is $0.05$ (divided by 10 every 50 epochs) and weight decay is $\num{2e-4}$. We train all models with SGD optimizer with batch size $128$ for MSE loss.

\subsection{Imbalanced Data}
\label{subsec:imbalance_case_appendix}
\subsubsection{Metric for measuring $\mathcal{NC}$ in imbalanced data}
For imbalanced setting,
$\mathcal{NC}1$ metric is identical to the balanced setting's. While for $\mathcal{NC}2$ and $\mathcal{NC}3$, we measure the closeness of learned classifiers and features to GOF structure as follows:
\begin{align}
\begin{gathered}
     \mathcal{NC}2^{GOF} 
    := \left\| 
    \frac{(\mathbf{W}_{M} \mathbf{W}_{M-1} \ldots \mathbf{W}_{1}) (\mathbf{W}_{M} \mathbf{W}_{M-1} \ldots \mathbf{W}_{1})^{\top} }{\| (\mathbf{W}_{M} \mathbf{W}_{M-1} \ldots \mathbf{W}_{1}) (\mathbf{W}_{M} \mathbf{W}_{M-1} \ldots \mathbf{W}_{1})^{\top} \|_{F} } - \frac{\operatorname{diag}
    \{cs_{k}^{2M}\}_{k=1}^{K}}{\| \operatorname{diag}
    \{cs_{k}^{2M}\}_{k=1}^{K} \|_{F}} \right\|_{F}, \nonumber \\
     \mathcal{NC}3^{GOF} 
    := 
    \left\| 
    \frac{\mathbf{W}_{M} \mathbf{W}_{M-1} \ldots \mathbf{W}_{1} \overline{\mathbf{H}}}{ \left\| \mathbf{W}_{M} \mathbf{W}_{M-1} \ldots \mathbf{W}_{1} \overline{\mathbf{H}} \right\|_{F}} 
    - 
    \frac{\operatorname{diag}
    \left\{\frac{c s_{k}^{2M}}{c s_{k}^{2M} + N \lambda_{H_{1}}} \right\}_{k=1}^{K}}
    {\left\| \operatorname{diag}
    \left\{\frac{c s_{k}^{2M}}{c s_{k}^{2M} + N \lambda_{H_{1}}} \right\}_{k=1}^{K} \right\|_{F}}
    \right\|_{F}, \nonumber   
\end{gathered}
\end{align}
where $\overline{\mathbf{H}} = [\mathbf{h}_{1}, \ldots, \mathbf{h}_{K}]$ is the class-means matrix, $c$ and $\{s_{k} \}_{k=1}^{K}$ are as defined in Theorem \ref{thm:deep_imbalance}.

\subsubsection{Additional numerical results for imbalanced data}
\label{subsubsec:addtional_imbalanced}

\begin{figure}[t!]
    \centering
    \includegraphics[width = 0.7\columnwidth]{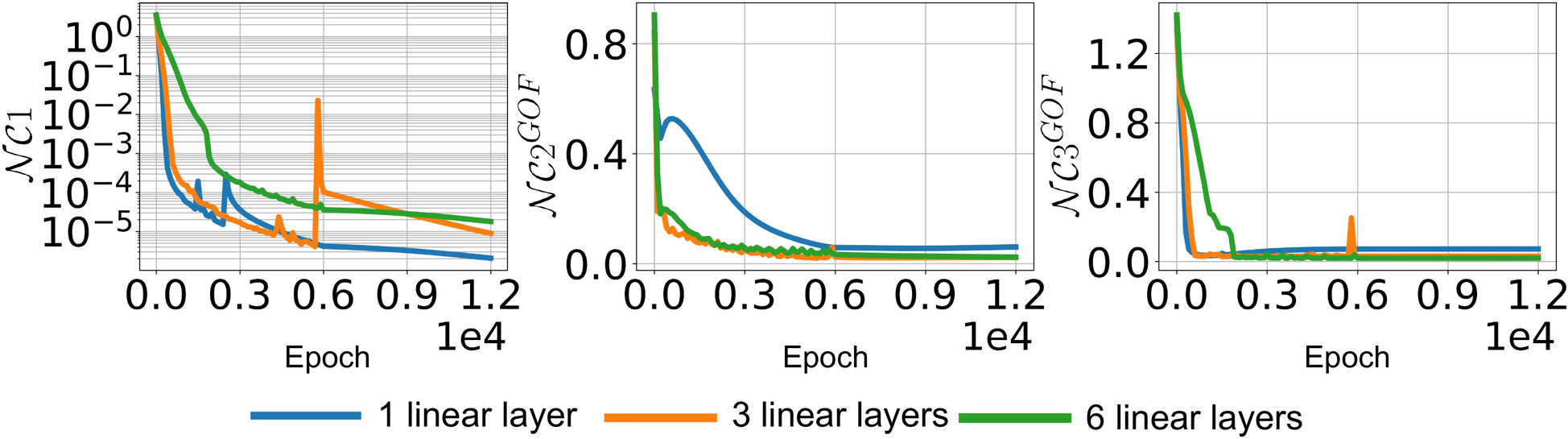}
    \caption{\small \color{black}{Illustration of $\mathcal{NC}$ with 6-layer MLP backbone on an imbalanced subset of EMNIST letter dataset with MSE loss and bias-free setting.}}
    \label{fig:mlp_imbalance_EMNIST}
\end{figure}

\begin{figure}[t!]
	\centering
	\vspace{-5pt}
	\includegraphics[width = 0.7\columnwidth]{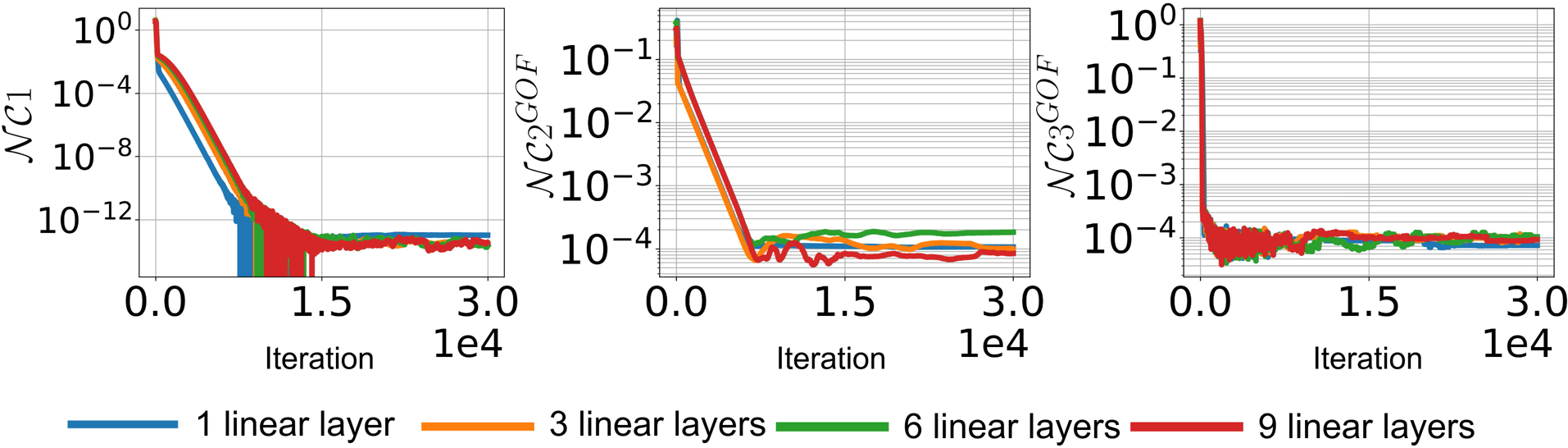}
    \caption{\small \color{black}{Illustration of $\mathcal{NC}$ for direct optimization experiment with MSE loss, imbalanced data and bias-free setting.}}
	\label{fig:synthetic_imbalance_GOF}
\end{figure}

Continue from subsection \ref{subsec:imbalanced_data_experiment}, to empirically validate the Minority Collapse of the problems \eqref{eq:UFM_imbalance} and \eqref{eq:deep_imbalance}, we run two direct optimization schemes similar as Section \ref{subsubsec:direct_im} with heavy imbalanced data of $K = 4$ and $n_{1} = 2000, n_{2} = n_{3} = 495$ and $n_{4} = 10$ for $M = 1$ ($d = 16$) and $M = 3$ ($d = 40$). Both models are trained by gradient descent for 30000 iterations. The final weight matrices of these models are as following (results are rounded to 2 decimal places):
\begin{align}
    \mathbf{W}_{1} = 
    \begin{smallbmatrix}
      -1.55
      & 1.50 & 2.19 & -1.36 & -0.65 & 3.08 & -0.81 &
      -1.76 & -0.96 & -0.48 & -1.21 & -1.06 & 1.01 & 1.72
      & 0.30 & -1.73 \\
      -1.26 & -0.56 & -0.94
      & -1.24 & 0.11 & -1.46 & -0.51 & -1.75 & -0.69 &
      0.11 & 1.09 & -0.89 & -0.56 & 0.57 & 0.48 & 0.27 \\
      0.76 & -0.31 & 0.32
      & -1.30 & -0.42 & 0.09 & 2.22 & -1.07 & 1.15 &
      -0.58 & -0.28 & -0.88 & -0.03 & -0.40 & -1.29 & 0.43
      \\
      0.00 & 0.00 & 0.00
      & 0.00 & 0.00 & 0.00 & 0.00 & 0.00 & 0.00 & 0.00
      & 0.00 & 0.00 & 0.00 & 0.00 & 0.00 & 0.00 \\
    \end{smallbmatrix},
    \nonumber
\end{align}
for case $M = 1$. For case $M = 3$, we have:
\begin{align}
    \mathbf{W}_{3}
    = 
    \begin{smallbmatrix}
      0.65
      & -0.96 & 0.49 & -0.15 & 0.50 & -0.11 & -0.14 &
      0.40 & \ldots & 0.02 & 0.05 & 0.27 & 0.13 & 0.71
      & -0.29 & 0.14 & -0.30 \\
      -0.25 & 0.13 & -0.40
      & -0.33 & 0.14 & 0.11 & -0.32 & 0.15 & \ldots &
      0.40 & -0.10 & -0.86 & 0.34 & 0.20 & 0.54 & 0.66
      & 0.18 \\
      0.36 & -0.15 & -0.04
      & -0.23 & -0.66 & -0.04 & -0.51 & -0.33 & \ldots
      & -0.07 & -0.52 & 0.15 & -0.03 & 0.04 & -0.36 &
      0.35 & 0.02 \\
      0.00 & 0.00 & 0.00
      & 0.00 & 0.00 & 0.00 & 0.00 & 0.00 & \ldots &
      0.00 & 0.00 & 0.00 & 0.00 & 0.00 & 0.00 & 0.00 &
      0.00 \\
    \end{smallbmatrix}.
\end{align}
As can be seen from both cases, the classifier of the fourth class converges to zero vector (with the convergence error are less than 1e-8), due to the heavy imbalance level of the dataset, which align to Theorem \ref{thm:UFM_imbalance} and Theorem \ref{thm:deep_imbalance}.

We further perform an image classification task on a heavy imbalanced subset of the CIFAR-10 dataset using a 6-layer MLP model with ReLU activation, followed by 1-layer linear classifier with the other settings the same as in EMNIST letter experiment described in Section \ref{subsubsec:hyperparam_imbalanced}. The subset includes 10 classes, with 7 major classes with 1000 samples per class and the other 3 minor classes with only 1 sample per class. Thus, the maximum imbalance ratio is $R = 1000$. To measure the Minority Collapse phenomenon, we follow Theorem 5 in~\cite{Fang21} and calculate the L2-norm of $\mathbf{w}_{i} - \mathbf{w}_{j}$ to show that for minority classes, their classifiers $\mathbf{w}_{i}$ are hardly distinguishable. Specifically, we denote $\mathbf{w}_{1}, \ldots, \mathbf{w}_{7}$ as the classifiers of 7 major classes, $\mathbf{w}_{8}, \mathbf{w}_{9}, \mathbf{w}_{10}$ as the classifiers of 3 minor classes. The matrix $\mathbf{W}_{\text{diff}}$ with $i$-th row, $j$-column entries are squared L2-norm of $\mathbf{w}_{i} - \mathbf{w}_{j}$ is as following (results are rounded to 2 decimal places):

\begin{align}
    \mathbf{W}_{\text{diff}} = 
    \begin{bmatrix} 
    0.00 & 61.80 & 61.70 & 61.70 & 61.78 & 61.73 & 61.78 & 31.17 & 31.17 & 31.17 \\ 61.80 & 0.00 & 61.75 & 61.77 & 61.82 & 61.78 & 61.84 & 31.22 & 31.22 & 31.22 \\ 61.70 & 61.75 & 0.00 & 61.66 & 61.68 & 61.69 & 61.74 & 31.13 & 31.13 & 31.13 \\ 61.70 & 61.77 & 61.66 & 0.00 & 61.75 & 61.67 & 61.76 & 31.14 & 31.14 & 31.14 \\ 61.78 & 61.82 & 61.68 & 61.75 & 0.00 & 61.74 & 61.81 & 31.19 & 31.19 & 31.19 \\ 61.73 & 61.78 & 61.69 & 61.67 & 61.74 & 0.00 & 61.77 & 31.16 & 31.16 & 31.16 \\ 61.78 & 61.84 & 61.74 & 61.76 & 61.81 & 61.77 & 0.00 & 31.21 & 31.21 & 31.21 \\ 31.17 & 31.22 & 31.13 & 31.14 & 31.19 & 31.16 & 31.21 & 0.00 & 0.60 & 0.60 \\
    31.17 & 31.22 & 31.13 & 31.14 & 31.19 & 31.16 & 31.21 & 0.60 & 0.00 & 0.60 \\
    31.17 & 31.22 & 31.13 & 31.14 & 31.19 & 31.16 & 31.21 & 0.60 & 0.60 & 0.00
    \end{bmatrix}.
    \nonumber
\end{align}

We observe from matrix $W_{\text{diff}}$ that the distances between minority classes’ classifiers is significantly small (0.60), and thus they are very close to each other. This observation is aligned with “Minority Collapse” phenomenon and our result in Theorem~\ref{thm:UFM_imbalance}.

\subsubsection{Details of network training and hyperparameters for imbalanced data experiments}
\label{subsubsec:hyperparam_imbalanced}
\label{subsubsec:detail_training_imbalanced}
\textbf{Multilayer perceptron experiment with CIFAR10 dataset:} In this experiment, we use a subset of CIFAR10 dataset with training samples of each class in the list $\{500, 500, 400, 400, 300, 300, 200, 200, 100, 100\}$. We use a 6-layer MLP model with ReLU activation with removed activation as the backbone feature extractor. Hidden width of both the backbone model and the deep linear networks are set to be $2048$. Depth of the linear layers are selected from the set $\{1, 3, 6\}$. All models are trained with Adam optimizer and MSE loss for $12000$ epochs, no data augmentation, full batch gradient descent, learning rate $\num{1e-4}$ (divided by $10$ every $6000$ epochs), feature decay and weight decay are set to be $\num{1e-5}$.

\textbf{Multilayer perceptron experiment with EMNIST letter dataset:} In this experiment, we use the same settings as described in MLP experiment on CIFAR10 dataset. The imbalanced training set is randomly sampled from EMNIST letter training set. We sample 1 major class with $5000$ samples, $5$ medium classes with $600$ samples per class, and $20$ minor class with $50$ samples per class. The optimization scheme is identical to the aforementioned MLP experiment on CIFAR10 imbalanced dataset.

\textbf{Direct optimization experiment:} In this experiment, we replicate the optimization problem \eqref{eq:UFM_general_linear} in imbalance data setting. We set $K = 4$ and $n_1 = 200, n_2 = 100, n_3 = n_4 = 50 , d_M = \ldots = d_1 = 64$. Similar to the direct optimization experiment in balance case, all $\lambda$'s are set to be $\num{5e-4}$. $\mathbf{W}_{M}, \ldots, \mathbf{W}_{1}$ and $\mathbf{H}_{1}$ are optimized by stochastic gradient descent for $30000$ iterations, with learning rate $0.1$.

\section{Proof of Theorem \ref{thm:bias-free}}
\label{sec:proofs_balanced}
First we state the proof for UFM bias-free with three layers of weights with same width across layers, as a warm-up for our approach in the next proofs.

\subsection{Warm-up Case: UFM with Three Layers of Weights}
Consider the following bias-free optimization problem:
\begin{align}
    \min _{\mathbf{W}_{3}, \mathbf{W}_{2}, \mathbf{W}_{1}, \mathbf{H}_{1}} \frac{1}{2N} \| \mathbf{W}_{3} \mathbf{W}_{2} \mathbf{W}_{1}
    \mathbf{H}_{1} - \mathbf{Y} \|_F^2 + \frac{\lambda_{W_{3}}}{2} \| \mathbf{W}_{3} \|^2_F +
    \frac{\lambda_{W_{2}}}{2} \| \mathbf{W}_{2} \|^2_F + \frac{\lambda_{W_{1}}}{2} \| \mathbf{W}_{1} \|^2_F +
    \frac{\lambda_{H_{1}}}{2} \| \mathbf{H}_{1} \|^2_F
    \label{eq:3layers}
\end{align}
where $\lambda_{W_{3}}, \lambda_{W_{2}}, \lambda_{W_{1}}, \lambda_{H_{1}}$ are regularization hyperparameters, and $\mathbf{W}_{3} \in \mathbb{R}^{K \times d}$, $\mathbf{W}_{2} \in \mathbb{R}^{d \times d}$,
$\mathbf{W}_{1} \in \mathbb{R}^{d \times d}$, $\mathbf{H}_{1} \in \mathbb{R}^{d \times N}$ and $\mathbf{Y} \in \mathbb{R}^{K \times N}$. We assume $d \geq K$ for this problem.

\begin{proof}[Proof of Theorem \ref{thm:bias-free} with 3 layers of weight and $d \geq K$]
    \noindent By definition, any critical point $(\mathbf{W}_{3}, \mathbf{W}_{2}, \mathbf{W}_{1}, \mathbf{H}_{1})$ of the loss function \eqref{eq:3layers} satisfies the following :
\begin{align}
  \frac{\partial f}{\partial \mathbf{W}_{3}} =
  \frac{1}{N} (\mathbf{W}_{3} \mathbf{W}_{2} \mathbf{W}_{1} \mathbf{H}_{1} - \mathbf{Y})\mathbf{H}_{1}^{\top} \mathbf{W}_{1}^{\top} 
  \mathbf{W}_{2}^{\top} + \lambda_{W_{3}} \mathbf{W}_{3} = \mathbf{0} , \\
  \frac{\partial f}{\partial \mathbf{W}_{2}} =
  \frac{1}{N} \mathbf{W}_{3}^{\top} (\mathbf{W}_{3} \mathbf{W}_{2} \mathbf{W}_{1} \mathbf{H}_{1} - \mathbf{Y}) \mathbf{H}_{1}^{\top} \mathbf{W}_{1}^{\top} + \lambda_{W_{2}} \mathbf{W}_{2} = \mathbf{0} , \\
  \frac{\partial f}{\partial \mathbf{W}_{1}} =
  \frac{1}{N} \mathbf{W}_{2}^{\top} \mathbf{W}_{3}^{\top} (\mathbf{W}_{3} \mathbf{W}_{2} \mathbf{W}_{1} \mathbf{H}_{1} - \mathbf{Y}) \mathbf{H}_{1}^{\top}  + \lambda_{W_{1}} \mathbf{W}_{1} = \mathbf{0} , \\
  \frac{\partial f}{\partial \mathbf{H}_{1}} =
  \frac{1}{N} \mathbf{W}_{1}^{\top} \mathbf{W}_{2}^{\top} \mathbf{W}_{3}^{\top} (\mathbf{W}_{3} \mathbf{W}_{2} \mathbf{W}_{1} \mathbf{H}_{1} - \mathbf{Y})  +  \lambda_{H_{1}} \mathbf{H}_{1} = \mathbf{0} . \label{eq:derivative_H}
\end{align}

\noindent Next, from $\mathbf{W}_{3}^{\top}  \frac{\partial f}{\partial \mathbf{W}_{3}} - \frac{\partial f}{\partial \mathbf{W}_{2}} \mathbf{W}_{2}^{\top} = \mathbf{0}$, we have:
\begin{align}
    \lambda_{W_{3}} \mathbf{W}_{3}^{\top} \mathbf{W}_{3} = \lambda_{W_{2}} \mathbf{W}_{2} \mathbf{W}_{2}^{\top} . \label{eq:14}
\end{align}

Similarly, we also have:
\begin{align}
     \lambda_{W_{2}} \mathbf{W}_{2}^{\top} \mathbf{W}_{2} = \lambda_{W_{1}} \mathbf{W}_{1} \mathbf{W}_{1}^{\top} , \label{eq:15} 
     \\
    \lambda_{W_{1}} \mathbf{W}_{1}^{\top} \mathbf{W}_{1} =  \lambda_{H_{1}}\mathbf{H}_{1} \mathbf{H}_{1}^{\top} .
    \label{eq:16}
\end{align}

\noindent Also, from equation \eqref{eq:derivative_H}, by solving for $\mathbf{H}_{1}$, we have:
\begin{align}
     \mathbf{H}_{1} &= 
    ( \mathbf{W}_{1}^{\top}
    \mathbf{W}_{2}^{\top} 
    \mathbf{W}_{3}^{\top}
    \mathbf{W}_{3}
    \mathbf{W}_{2}
    \mathbf{W}_{1} + N \lambda_{H_{1}} \mathbf{I}
    )^{-1}
    \mathbf{W}_{1}^{\top}
    \mathbf{W}_{2}^{\top}  
    \mathbf{W}_{3}^{\top} \mathbf{Y} \nonumber
    \\
    &= 
     \left( \frac{\lambda_{W_{2}}}{\lambda_{W_{3}}} \mathbf{W}_{1}^{\top}
    (\mathbf{W}_{2}^{\top} 
    \mathbf{W}_{2})^{2}
    \mathbf{W}_{1} + N \lambda_{H_{1}} \mathbf{I}
    \right)^{-1}
    \mathbf{W}_{1}^{\top}
    \mathbf{W}_{2}^{\top}  
    \mathbf{W}_{3}^{\top} \mathbf{Y} \nonumber
    \\
    &= 
    \left( \frac{\lambda_{W_{1}}^{2}}{\lambda_{W_{3}}\lambda_{W_{2}}} (\mathbf{W}_{1}^{\top}
    \mathbf{W}_{1})^{3} + N \lambda_{H_{1}} \mathbf{I}
    \right)^{-1}
    \mathbf{W}_{1}^{\top}
    \mathbf{W}_{2}^{\top}  
    \mathbf{W}_{3}^{\top} \mathbf{Y} ,
    \label{eq:17}
\end{align}
where we use equations \eqref{eq:14} and \eqref{eq:15} for the derivation.
\\

Now, let $\mathbf{W}_{1} = \mathbf{U}_{W_{1}} \mathbf{S}_{W_{1}} \mathbf{V}_{W_{1}}^{\top}$ be the SVD decomposition of $\mathbf{W}_{1}$ with $\mathbf{U}_{W_{1}}, \mathbf{V}_{W_{1}} \in \mathbb{R}^{d \times d}$ are orthonormal matrix and $\mathbf{S}_{W_{1}} \in \mathbb{R}^{d \times d}$ is a diagonal matrix with \textbf{decreasing} non-negative singular values. We note that from equations \eqref{eq:14}-\eqref{eq:16}, we have $\operatorname{rank}(\mathbf{W}_{3}^{\top} \mathbf{W}_{3}) = \operatorname{rank} (\mathbf{W}_{3}) = \operatorname{rank} (\mathbf{W}_{2}) = \operatorname{rank} (\mathbf{W}_{1}) = \operatorname{rank} (\mathbf{H}_{1})$ and is at most $K$. We denote the $K$ singular values (some of them can be $0$'s) of $\mathbf{W}_{1}$ as $\left\{s_{k}\right\}_{k=1}^{K}$.
\\

\noindent From equation \eqref{eq:15}, we have:
\begin{align}
    \mathbf{W}_{2}^{\top} \mathbf{W}_{2}
    = \frac{\lambda_{W_{1}}}{\lambda_{W_{2}}} \mathbf{W}_{1} \mathbf{W}_{1}^{\top} = \frac{\lambda_{W_{1}}}{\lambda_{W_{2}}} 
    \mathbf{U}_{W_{1}} \mathbf{S}_{W_{1}}^{2} \mathbf{U}_{W_{1}}^{\top} = \mathbf{U}_{W_{1}} \mathbf{S}_{W_{2}}^{2} \mathbf{U}_{W_{1}}^{\top}, \nonumber
\end{align}
where $\mathbf{S}_{W_{2}} = \sqrt{\frac{\lambda_{W_{1}}}{\lambda_{W_{2}}}} \mathbf{S}_{W_{1}} \in \mathbb{R}^{d \times d}$. This means that $\mathbf{S}_{W_{2}}^{2}$ contains the eigenvalues and the columns of $\mathbf{U}_{W_{1}}$ are the eigenvectors of $\mathbf{W}_{2}^{\top} \mathbf{W}_{2}$. Hence, we can write the SVD decomposition of $\mathbf{W}_{2}$ as $\mathbf{W}_{2} = \mathbf{U}_{W_{2}} \mathbf{S}_{W_{2}} \mathbf{U}_{W_{1}}^{\top}$ with orthonormal matrix $\mathbf{U}_{W_{2}} \in \mathbb{R}^{d \times d}$. 
\\  

\noindent By making similar arguments for $\mathbf{W}_{3}$, from equation \eqref{eq:14}:
\begin{align}
    \mathbf{W}_{3}^{\top} \mathbf{W}_{3}
    = \frac{\lambda_{W_{2}}}{\lambda_{W_{3}}} \mathbf{W}_{2} \mathbf{W}_{2}^{\top} = \frac{\lambda_{W_{2}}}{\lambda_{W_{3}}} \mathbf{U}_{W_{2}} \mathbf{S}_{W_{2}}^{2} \mathbf{U}_{W_{2}}^{\top} =
    \frac{\lambda_{W_{1}}}{\lambda_{W_{3}}}
    \mathbf{U}_{W_{2}} \mathbf{S}_{W_{1}}^{2} \mathbf{U}_{W_{2}}^{\top}
    = \mathbf{U}_{W_{2}} \mathbf{S}_{W_{3}}^{\top} \mathbf{S}_{W_{3}} \mathbf{U}_{W_{2}}^{\top}, \nonumber
\end{align}
with $\mathbf{S}_{W_{3}} = 
\sqrt{\frac{\lambda_{W_{1}}}{\lambda_{W_{3}}}} \begin{bmatrix}
\operatorname{diag} (s_{1}, s_{2},\ldots, s_{K}) & \mathbf{0}_{K \times (d-K)}
\end{bmatrix}
\in \mathbb{R}^{K \times d}$, we can write SVD decomposition of $\mathbf{W}_{3}$ as $\mathbf{W}_{3} = \mathbf{U}_{W_{3}} \mathbf{S}_{W_{3}} \mathbf{U}_{W_{2}}^{\top}$ with orthonormal matrix $\mathbf{U}_{W_{3}} \in \mathbb{R}^{d \times d}$. 
\\

\noindent Using these SVD in the RHS of equation \eqref{eq:17} yields:
\begin{align}
    \mathbf{H}_{1}
    &= 
    \left( \frac{\lambda_{W_{1}}^{2}}{\lambda_{W_{3}}\lambda_{W_{2}}} (\mathbf{W}_{1}^{\top}
    \mathbf{W}_{1})^{3} + N \lambda_{H_{1}} \mathbf{I}
    \right)^{-1}
    \mathbf{W}_{1}^{\top}
    \mathbf{W}_{2}^{\top}  
    \mathbf{W}_{3}^{\top} \mathbf{Y} \nonumber \\
    &= 
    \left( \frac{\lambda_{W_{1}}^{2}}{\lambda_{W_{3}}\lambda_{W_{2}}} \mathbf{V}_{W_{1}} \mathbf{S}_{W_{1}}^{6} \mathbf{V}_{W_{1}}^{\top} + N \lambda_{H_{1}} \mathbf{I} \right)^{-1}
     \mathbf{W}_{1}^{\top}
    \mathbf{W}_{2}^{\top}  
    \mathbf{W}_{3}^{\top} \mathbf{Y} \nonumber
    \\
    &= \left( \frac{\lambda_{W_{1}}^{2}}{\lambda_{W_{3}}\lambda_{W_{2}}} \mathbf{V}_{W_{1}} \mathbf{S}_{W_{1}}^{6} \mathbf{V}_{W_{1}}^{\top} + N \lambda_{H_{1}} \mathbf{I} \right)^{-1}
    \mathbf{V}_{W_{1}} \mathbf{S}_{W_{1}} 
    \mathbf{S}_{W_{2}} \mathbf{S}_{W_{3}}^{\top} \mathbf{U}_{W_{3}}^{\top} \mathbf{Y} \nonumber \\
    &= \mathbf{V}_{W_{1}} 
    \left( \frac{\lambda_{W_{1}}^{2}}{\lambda_{W_{3}}\lambda_{W_{2}}} \mathbf{S}_{W_{1}}^{6}  + N \lambda_{H_{1}} \mathbf{I} \right)^{-1}
    \mathbf{S}_{W_{1}} 
    \mathbf{S}_{W_{2}} \mathbf{S}_{W_{3}}^{\top}
    \mathbf{U}_{W_{3}}^{\top} \mathbf{Y} \nonumber \\
    &= \mathbf{V}_{W_{1}}
     \left( \frac{\lambda_{W_{1}}^{2}}{\lambda_{W_{3}}\lambda_{W_{2}}} \mathbf{S}_{W_{1}}^{6} + N \lambda_{H_{1}} \mathbf{I} \right)^{-1}
     \sqrt{\frac{\lambda_{W_{1}}^{2}}{\lambda_{W_{3}}\lambda_{W_{2}}}} \begin{bmatrix}
    \operatorname{diag} (s_{1}^{3}, s_{2}^{3},\ldots, s_{K}^{3}) \\ \mathbf{0}_{(d-K) \times K}
    \end{bmatrix} 
    \mathbf{U}_{W_{3}}^{\top} \mathbf{Y} \nonumber \\
    &= \mathbf{V}_{W_{1}}
    \underbrace{
    \begin{bmatrix}
    \operatorname{diag} \left(
        \frac{\sqrt{c} s_{1}^{3}}
        {c s_{1}^{6} + N \lambda_{H_{1}}},\ldots,
        \frac{\sqrt{c} s_{K}^{3}}
        {c s_{K}^{6} + N \lambda_{H_{1}}}
        \right) \\ \mathbf{0}
    \end{bmatrix}}_{\mathbf{C} \in \mathbb{R}^{d \times K}}
    \mathbf{U}_{W_{3}}^{\top} \mathbf{Y} \nonumber \\
    &= \mathbf{V}_{W_{1}} \mathbf{C} \mathbf{U}_{W_{3}}^{\top} \mathbf{Y} , \label{eq:H_form}
\end{align}
with $c :=  \frac{\lambda_{W_{1}}^{2}}{\lambda_{W_{3}}\lambda_{W_{2}}}$. We further have:
\begin{align}
    \mathbf{W}_{3} \mathbf{W}_{2} \mathbf{W}_{1} \mathbf{H} 
    &= \mathbf{U}_{W_{3}} \mathbf{S}_{W_{3}}
    \mathbf{S}_{W_{2}} \mathbf{S}_{W_{1}}
    \mathbf{V}_{W_{1}}^{\top}
    \mathbf{V}_{W_{1}} \mathbf{C} \mathbf{U}_{W_{3}}^{\top} \mathbf{Y} \nonumber \\
    &= \mathbf{U}_{W_{3}} \operatorname{diag}
    \left(\frac{c s_{1}^{6}}
        {c s_{1}^{6} + N \lambda_{H_{1}}},\ldots, 
        \frac{c s_{K}^{6}}
        {c s_{K}^{6} + N \lambda_{H_{1}}}
        \right) \mathbf{U}_{W_{3}}^{\top} \mathbf{Y} \\
    \Rightarrow \mathbf{W}_{3} \mathbf{W}_{2} \mathbf{W}_{1} \mathbf{H} - \mathbf{Y}
    &= \mathbf{U}_{W_{3}} \left( \operatorname{diag}
    \left(\frac{c s_{1}^{6}}
        {c s_{1}^{6} + N \lambda_{H_{1}}},\ldots, 
        \frac{c s_{K}^{6}}
        {c s_{K}^{6} + N \lambda_{H_{1}}}
        \right) - \mathbf{I}_{K}
        \right)
        \mathbf{U}_{W_{3}}^{\top} \mathbf{Y} \nonumber \\
    &= \mathbf{U}_{W_{3}} 
    \underbrace{\operatorname{diag}
    \left(\frac{- N \lambda_{H_{1}} }
        {c s_{1}^{6} + N \lambda_{H_{1}}},\ldots, 
        \frac{- N \lambda_{H_{1}}}
        {c s_{K}^{6} + N \lambda_{H_{1}}}
        \right)}_{\mathbf{D} \in \mathbb{R}^{K \times K}}
    \mathbf{U}_{W_{3}}^{\top} \mathbf{Y} \nonumber \\
    &= \mathbf{U}_{W_{3}} \mathbf{D} \mathbf{U}_{W_{3}}^{\top} \mathbf{Y} .
\end{align} 

\noindent Next, we will calculate the Frobenius norm of $ \mathbf{W}_{3} \mathbf{W}_{2} \mathbf{W}_{1} \mathbf{H} - \mathbf{Y}$:
\begin{align}
    \| \mathbf{W}_{3} \mathbf{W}_{2} \mathbf{W}_{1} \mathbf{H}_{1} - \mathbf{Y}  \|_F^2
     &= \| \mathbf{U}_{W_{3}}
     \mathbf{D}
     \mathbf{U}_{W_{3}}^{\top}
     \mathbf{Y}  \|_F^2
     = \operatorname{trace}
     (\mathbf{U}_{W_{3}}
     \mathbf{D}
     \mathbf{U}_{W_{3}}^{\top}
     \mathbf{Y} (\mathbf{U}_{W_{3}}
     \mathbf{D}
     \mathbf{U}_{W_{3}}^{\top}
     \mathbf{Y})^{\top}  ) \nonumber \\
     &= \operatorname{trace}
     (\mathbf{U}_{W_{3}}
     \mathbf{D}
     \mathbf{U}_{W_{3}}^{\top}
     \mathbf{Y} \mathbf{Y}^{\top} \mathbf{U}_{W_{3}} \mathbf{D}
     \mathbf{U}_{W_{3}}^{\top}
     ) 
     =  \operatorname{trace}
     ( \mathbf{D}^{2} \mathbf{U}_{W_{3}}^{\top}  \mathbf{Y} \mathbf{Y}^{\top}   \mathbf{U}_{W_{3}} )  \nonumber \\
     &= n \operatorname{trace}
     ( \mathbf{D}^{2}) = n \sum_{k=1}^{K} 
     \left( \frac{- N \lambda_{H_{1}} }
    {c s_{k}^{6} + N \lambda_{H_{1}}} \right)^{2} .
    \label{eq:3_W_norm}
\end{align}
where we use the fact $\mathbf{Y} \mathbf{Y}^{\top} = n \mathbf{I}_{K}$ and $\mathbf{U}_{W_{3}}$ is orthonormal matrix.
\\

\noindent Similarly, from the RHS of equation \eqref{eq:H_form}, we have:
\begin{align}
     \| \mathbf{H}_{1} \|_F^2
    &= \operatorname{trace}
    ( \mathbf{V}_{W_{1}} 
    \mathbf{C} \mathbf{U}_{W_{3}}^{\top}
    \mathbf{Y} \mathbf{Y}^{\top}
    \mathbf{U}_{W_{3}} \mathbf{C}^{\top}
    \mathbf{V}_{W_{1}}^{\top}
    ) = 
    \operatorname{trace} 
    (  \mathbf{C}^{\top}  \mathbf{C} \mathbf{U}_{W_{3}}^{\top}
    \mathbf{Y} \mathbf{Y}^{\top}
    \mathbf{U}_{W_{3}}  ) \nonumber \\
    &= n \operatorname{trace} 
    (  \mathbf{C}^{\top} \mathbf{C}) 
    = n \sum_{k=1}^{K} \left( \frac{\sqrt{c} s_{k}^{3}}
        {c s_{k}^{6} + N \lambda_{H_{1}}} \right)^{2} .
    \label{eq:3_H_norm}
\end{align}

\noindent Now, we will plug equations \eqref{eq:3_W_norm}, \eqref{eq:3_H_norm}, and the SVD decomposition of $\mathbf{W}_{2}, \mathbf{W}_{1}, \mathbf{H}$ into the function \eqref{eq:3layers} and note that orthonormal matrix does not change the Frobenius form:
\begin{align}
    &f(\mathbf{W}_{3}, \mathbf{W}_{2}, \mathbf{W}_{1}, \mathbf{H}_{1}) = 
     \frac{1}{2 N}\left\|\mathbf{W}_{3} \mathbf{W}_{2} \mathbf{W}_{1} \mathbf{H}-\mathbf{I}_{K}\right\|_{F}^{2} + \frac{\lambda_{W_{3}}}{2}\left\|\mathbf{W}_{3}\right\|_{F}^{2} + \frac{\lambda_{W_{2}}}{2}\left\|\mathbf{W}_{2}\right\|_{F}^{2}+\frac{\lambda_{W_{1}}}{2}\left\|\mathbf{W}_{1}\right\|_{F}^{2} 
    +\frac{ \lambda_{H_{1}}}{2}\left\|\mathbf{H_{1}}\right\|_{F}^{2} \nonumber
    \\
    &= \frac{1}{2K} \sum_{k=1}^{K} 
     \left( \frac{- N \lambda_{H_{1}} }
    {c s_{k}^{6} + N \lambda_{H_{1}}} \right)^{2} 
    + \frac{\lambda_{W_{3}}}{2}
    \sum_{k=1}^{K} \frac{\lambda_{W_{1}}}{\lambda_{W_{3}}} s_{k}^{2}
    + \frac{\lambda_{W_{2}}}{2}
    \sum_{k=1}^{K} \frac{\lambda_{W_{1}}}{\lambda_{W_{2}}} s_{k}^{2}
    + \frac{\lambda_{W_{1}}}{2}
    \sum_{k=1}^{K} s_{k}^{2} 
    + \frac{n \lambda_{H_{1}}}{2} 
    \sum_{k=1}^{K}  \frac{c s_{k}^{6}}
        {(c s_{k}^{6} + N \lambda_{H_{1}})^{2}} \nonumber \\
    &= \frac{n \lambda_{H_{1}}}{2} \sum_{k=1}^{K} 
     \frac{1}{c s_{k}^{6} + N \lambda_{H_{1}}} + \frac{3 \lambda_{W_{1}}}{2}  \sum_{k=1}^{K} s_{k}^{2} \nonumber \\
    &= \frac{1}{2K}  \sum_{k=1}^{K} \left(
    \frac{1}{ \frac{c s_{k}^{6}}{N \lambda_{H_{1}}} + 1} + 3 K \lambda_{W_{1}}
    \frac{\sqrt[3]{N \lambda_{H_{1}}}}{\sqrt[3]{c}}
    \frac{\sqrt[3]{c} s_{k}^{2}}{\sqrt[3]{N \lambda_{H_{1}}}}
    \right) \nonumber \\
    &= \frac{1}{2K} \sum_{k=1}^{K} \left(
    \frac{1}{x_{k}^{3}+1} + b x_{k}
    \right),
    \label{eq:3_f_form}
\end{align}
with $x_{k} := \frac{\sqrt[3]{c} s_{k}^{2}}{\sqrt[3]{N \lambda_{H_{1}}}}$ and $b:= 3 K \lambda_{W_{1}}
    \frac{\sqrt[3]{N \lambda_{H_{1}}}}{\sqrt[3]{c}} = 3K \sqrt[3]{N \lambda_{W_{3}} \lambda_{W_{2}} \lambda_{W_{1}} \lambda_{H_{1}}}$.
\\

\noindent Next, we consider the function:
\begin{align}
    g(x) = \frac{1}{x^{3} + 1} + bx \text{ with } x \geq 0, b > 0.
\end{align}
Clearly, $g(0) = 1$. As in equation \eqref{eq:3_f_form}, $f(\mathbf{W}_{3}, \mathbf{W}_{2}, \mathbf{W}_{1}, \mathbf{H})$ is the sum of $g(x_{k})$ (with separable $x_{k}$). Hence, if we can minimize $g(x)$, we will finish lower bounding $f(\mathbf{W}_{3}, \mathbf{W}_{2}, \mathbf{W}_{1}, \mathbf{H})$. We consider the following cases for $g(x)$:
\begin{itemize}
    \item If $b > \frac{\sqrt[3]{4}}{3}$: For $x > 0$, we always have $g(x) > \frac{1}{x^{3} + 1} +  \frac{\sqrt[3]{4}}{3}x \geq 1 = g(0)$. Indeed, the second inequality is equivalent to:
    \begin{align}
        &\frac{1}{x^{3} + 1} +  \frac{\sqrt[3]{4}}{3}x \geq 1 \nonumber \\
        \Leftrightarrow \quad &\frac{\sqrt[3]{4}}{3} x^{4} - x^{3} + \frac{\sqrt[3]{4}}{3} x \geq 0 \nonumber \\
        \Leftrightarrow \quad &x (x + \frac{1}{\sqrt[3]{4}}) (x - \sqrt[3]{2})^{2} \geq 0 . \nonumber
    \end{align}
    Therefore, in this case, $g(x)$ is minimized at $x = 0$ with minimal value of $1$.
    
    \item If $b = \frac{\sqrt[3]{4}}{3}$: Similar as above, we have:
    \begin{align}
        &g(x) \geq 1 \nonumber \\
        \Leftrightarrow \quad &x (x + \frac{1}{\sqrt[3]{4}}) (x - \sqrt[3]{2})^{2} \geq 0 . \nonumber
    \end{align}
    In this case, $g(x)$ is minimized at $x = 0$ or $x = \sqrt[3]{2}$.
    
    \item If $b < \frac{\sqrt[3]{4}}{3}$: We take the first and second derivatives of $g(x)$:
    \begin{align}
        g^{\prime}(x) &= b - \frac{3x^{2}}{(x^{3}+1)^{2}}, \nonumber \\
        g^{\prime \prime}(x) &= \frac{12x^4 - 6x}{(x^{3}+1)^{3}} . \nonumber
    \end{align}
    We have: $g^{\prime \prime}(x) = 0 \Leftrightarrow x = 0$ or $x = \sqrt[3]{\frac{1}{2}}$. Therefore, with $x \geq 0$, $g^{\prime}(x) = 0$ has at most two solutions. We also have $g^{\prime} \left(\sqrt[3]{\frac{1}{2}} \right) = b - \frac{2 \sqrt[3]{2}}{3} < 0$ (since $b < \frac{\sqrt[3]{4}}{3}$). Thus, together with the fact that $g^{\prime}(0) = b > 0$ and $g(+ \infty) > 0$, $g^{\prime}(x) = 0$ has exactly two solutions, we call it $x_1$ and $x_2$ ($x_{1} < \sqrt[3]{\frac{1}{2}} < x_{2}$). Next, we note that $g^{\prime} (x_2) = 0$ and $g^{\prime} (x) > 0 \quad \forall x > x_2$ (since $g^{\prime \prime} (x) > 0 \quad  \forall x > x_2$). In the meanwhile, $g^{\prime}(\sqrt[3]{2}) = b - \frac{ \sqrt[3]{4}}{3} < 0$. Hence, we must have $x_{2} > \sqrt[3]{2}$.
    \\
    
    From the variation table, we can see that $g(x_{2}) < g(\sqrt[3]{2}) = \frac{1}{3} + b \sqrt[3]{2} < \frac{1}{3} + \frac{2}{3} = 1 = g(0)$. Hence, the minimizer in this case is the largest solution $x > \sqrt[3]{2}$ of the equation $g^{\prime} (x) = 0$.
\end{itemize}
\begin{table}[h]
\centering
\begin{tabular}{c|cccccc}
$x$                 & 0 & $x_1$    & $\sqrt[3]{\frac{1}{2}}$    & $\sqrt[3]{2}$                & $x_2$    & $\infty$ \\ \hline
$g^{\prime \prime}$ & 0 & -        & 0                       & +                 & +        & +        \\ \hline
$g^{\prime}$        & + & 0        & -                       & -                 & 0        & +        \\ \hline
$g$                 & 1 & $g(x_1)$ & $g\left(\sqrt[3]{\frac{1}{2}}\right)$ & $\frac{1}{3} + b \sqrt[3]{2}$ & $g(x_2)$ & $\infty$ \\ \hline
\end{tabular}
\end{table}    

\noindent From the above result, we can summarize the original problem as follows:
\begin{itemize}
    \item If $b = 3 K  \sqrt[3]{K n \lambda_{W_{3}} \lambda_{W_{2}} \lambda_{W_{1}} \lambda_{H_{1}}} > \frac{\sqrt[3]{4}}{3}$: all the singular values of $\mathbf{W}_{1}^{*}$ are $0$'s. Therefore, the singular values of $\mathbf{W}_{3}^{*}, \mathbf{W}_{1}^{*}, \mathbf{H}^{*}$ are also all $0$'s. In this case, $f(\mathbf{W}_{3}, \mathbf{W}_{2}, \mathbf{W}_{1}, \mathbf{H}_{1})$ is minimized at $(\mathbf{W}_{3}^{\ast}, \mathbf{W}_{2}^{\ast}, \mathbf{W}_{1}^{\ast}, \mathbf{H}_{1}^{\ast}) = (\mathbf{0}, \mathbf{0}, \mathbf{0}, \mathbf{0})$.
    
    \item If $b = 3 K  \sqrt[3]{K n \lambda_{W_{3}} \lambda_{W_{2}} \lambda_{W_{1}} \lambda_{H_{1}}} < \frac{\sqrt[3]{4}}{3}$: In this case, $\mathbf{W}_{1}^{\ast}$ has $K$ singular values, all of which are multiplier of the largest positive solution of the equation $b - \frac{3x^{2}}{(x^{3}+1)^{2}} = 0$, denoted as $s$. Hence, we have the compact SVD form (with a bit of notation abuse) of $\mathbf{W}_{1}^{\ast}$ as $\mathbf{W}_{1}^{\ast} = s \mathbf{U}_{W_{1}} \mathbf{V}_{W_{1}}^{\top}$ with semi-orthonormal matrices $\mathbf{U}_{W_{1}}, \mathbf{V}_{W_{1}} \in \mathbb{R}^{d \times K}$. We also have $\mathbf{U}_{W_{1}}^{\top} \mathbf{U}_{W_{1}} = \mathbf{I}_{K}$ and $\mathbf{V}_{W_{1}}^{\top} \mathbf{V}_{W_{1}} = \mathbf{I}_{K}$.
    \\
    
    Similarly, since the singular matrices of $\mathbf{W}_{3}, \mathbf{W}_{1}$ are aligned to $\mathbf{W}_{1}$'s, we also have:
    \begin{align}
         \mathbf{W}_{3}^{\ast}  &= \sqrt{\frac{\lambda_{W_{1}}}{\lambda_{W_{3}}}} s \mathbf{U}_{W_{3}} \mathbf{U}_{W_{2}}^{T},  \nonumber \\
        \mathbf{W}_{2}^{\ast} &= \sqrt{\frac{\lambda_{W_{1}}}{\lambda_{W_{2}}}} s
        \mathbf{U}_{W_{2}} \mathbf{U}_{W_{1}}^{\top}, \nonumber \\
        \mathbf{W}_{1}^{\ast} &= s \mathbf{U}_{W_{1}} \mathbf{V}_{W_{1}}^{\top}, \nonumber \\
        \mathbf{H}^{\ast}_{1} &= \frac{\sqrt{c} s^{3}}{c s^{6} + N \lambda_{H_{1}}} \mathbf{V}_{W_{1}} \mathbf{U}_{W_{3}}^{\top} \mathbf{Y}, \nonumber
    \end{align}
    with orthonormal matrices $\mathbf{U}_{W_{3}} \in \mathbb{R}^{K \times K}$, semi-orthonormal matrix $\mathbf{U}_{W_{2}}, \mathbf{U}_{W_{1}}, \mathbf{V}_{W_{1}} \in \mathbb{R}^{d \times K}$. 
    Let $\overline{\mathbf{H}}^{*} = \frac{\sqrt{c} s^{3}}{c s^{6} + N \lambda_{H_{1}}} \mathbf{V}_{W_{1}} \mathbf{U}_{W_{3}}^{\top} \in \mathbb{R}^{K \times K}$, we have: $\mathbf{H}^{*}_{1} = \overline{\mathbf{H}}^{*} \mathbf{Y} = \overline{\mathbf{H}}^{*} \otimes \mathbf{1}_{n}^{\top}$.
    \\
    
    We have the geometry of the global solutions as follows:
    \begin{align}
    \begin{gathered}
        \mathbf{W}_{3}^{\ast} \mathbf{W}_{3}^{\top \ast} \propto \mathbf{U}_{W_{3}} \mathbf{U}_{W_{2}}^{\top} \mathbf{U}_{W_{2}} \mathbf{U}_{W_{3}}^{\top} \propto \mathbf{I}_{K}, \\
        \overline{\mathbf{H}}^{\ast \top} \overline{\mathbf{H}}^{\ast} \propto \mathbf{U}_{W_{3}} \mathbf{V}_{W_{1}}^{\top} \mathbf{V}_{W_{1}} \mathbf{U}_{W_{3}}^{\top} \propto  \mathbf{I}_{K}, \\
        (\mathbf{W}_{3}^{\ast} \mathbf{W}_{2}^{\ast}) (\mathbf{W}_{3}^{\ast} \mathbf{W}_{2}^{\ast})^{\top} \propto (\mathbf{U}_{W_{3}} \mathbf{U}_{W_{2}}^{T} \mathbf{U}_{W_{2}} \mathbf{U}_{W_{1}}^{\top})(\mathbf{U}_{W_{3}} \mathbf{U}_{W_{2}}^{T} \mathbf{U}_{W_{2}} \mathbf{U}_{W_{1}}^{\top})^{\top} \propto \mathbf{I}_{K}, \\
        (\mathbf{W}_{1}^{\ast} \overline{\mathbf{H}}^{*})^{\top} (\mathbf{W}_{1}^{\ast} \overline{\mathbf{H}}^{*}) \propto (\mathbf{U}_{W_{1}} \mathbf{V}_{W_{1}}^{\top} \mathbf{V}_{W_{1}} \mathbf{U}_{W_{3}}^{\top})^{\top} (\mathbf{U}_{W_{1}} \mathbf{V}_{W_{1}}^{\top} \mathbf{V}_{W_{1}} \mathbf{U}_{W_{3}}^{\top}) \propto \mathbf{I}_{K}, \\
        (\mathbf{W}_{3}^{\ast} \mathbf{W}_{2}^{\ast} \mathbf{W}_{1}^{\ast}) (\mathbf{W}_{3}^{\ast} \mathbf{W}_{2}^{\ast} \mathbf{W}_{1}^{\ast})^{\top} \propto (\mathbf{U}_{W_{3}} \mathbf{V}_{W_{1}}^{\top})
        (\mathbf{U}_{W_{3}} \mathbf{V}_{W_{1}}^{\top})^{\top} \propto \mathbf{I}_{K}, \\
        (\mathbf{W}_{2}^{\ast} \mathbf{W}_{1}^{\ast} \overline{\mathbf{H}}^{*})^{\top} (\mathbf{W}_{2}^{\ast} \mathbf{W}_{1}^{\ast} \overline{\mathbf{H}}^{*}) \propto (\mathbf{U}_{W_{2}} \mathbf{U}_{W_{3}}^{\top})^{\top} (\mathbf{U}_{W_{2}} \mathbf{U}_{W_{3}}^{\top}) \propto \mathbf{I}_{K}, \\
    \end{gathered}
    \end{align}
    and,
    \begin{align}
        \mathbf{W}_{3}^{\ast} \mathbf{W}_{2}^{\ast} \mathbf{W}_{1}^{\ast} \overline{\mathbf{H}}^{*} \propto \mathbf{U}_{W_{3}} \mathbf{U}_{W_{2}}^{\top} \mathbf{U}_{W_{2}} \mathbf{V}_{W_{2}}^{\top} \mathbf{V}_{W_{2}} \mathbf{V}_{W_{1}}^{\top} \mathbf{V}_{W_{1}} \mathbf{U}_{W_{3}}^{\top} \propto \mathbf{I}_{K}.
    \end{align}
    
    Next, we can derive the alignments between weights and features as following:
    \begin{align}
    \begin{aligned}
         \mathbf{W}_{3}^{\ast} \mathbf{W}_{2}^{\ast} \mathbf{W}_{1}^{\ast} \propto \mathbf{U}_{W_{3}} \mathbf{V}_{W_{1}}^{\top} \propto \overline{\mathbf{H}}^{* \top}, \\
          \mathbf{W}_{2}^{\ast} \mathbf{W}_{1}^{\ast} \overline{\mathbf{H}}^{*} \propto \mathbf{U}_{W_{2}}   \mathbf{U}_{W_{3}}^{\top} \propto  \mathbf{W}_{3}^{\ast \top}, \\
        \mathbf{W}_{3}^{*} \mathbf{W}_{2}^{*} \propto 
        \mathbf{U}_{W_{3}} \mathbf{V}_{W_{2}}^{\top} \propto  (\mathbf{W}_{1}^{*} \overline{\mathbf{H}}^{*})^{\top}.
    \end{aligned}
    \end{align}
    
    \item If $b = 3 K  \sqrt[3]{K n \lambda_{W_{3}} \lambda_{W_{2}} \lambda_{W_{1}} \lambda_{H_{1}}} = \frac{\sqrt[3]{4}}{3}$: For this case, $x_{k}^{*}$ can either be $0$ or $\sqrt[3]{2}$, as long as $\{ x_{k}^{*} \}_{k=1}^{K}$ is a decreasing sequence. If all the singular values are $0$'s, we have the trivial global minima $(\mathbf{W}_{3}^{\ast}, \mathbf{W}_{2}^{\ast}, \mathbf{W}_{1}^{\ast}, \mathbf{H}_{1}^{\ast}) = (\mathbf{0}, \mathbf{0}, \mathbf{0}, \mathbf{0})$. If there are exactly $r \leq K$ positive singular values  $s_{1} = s_{2} = \ldots = s_{r} := s > 0$ and $s_{r+1} = \ldots = s_{K} = 0$, then we can write the compact SVD form of weight matrices and $\mathbf{H}_{1}^{*}$ as following:
     \begin{align}
         \mathbf{W}_{3}^{\ast}  &= \sqrt{\frac{\lambda_{W_{1}}}{\lambda_{W_{3}}}} s \mathbf{U}_{W_{3}} \mathbf{U}_{W_{2}}^{T},  \nonumber \\
        \mathbf{W}_{2}^{\ast} &= \sqrt{\frac{\lambda_{W_{1}}}{\lambda_{W_{2}}}} s
        \mathbf{U}_{W_{2}} \mathbf{U}_{W_{1}}^{\top}, \nonumber \\
        \mathbf{W}_{1}^{\ast} &= s \mathbf{U}_{W_{1}} \mathbf{V}_{W_{1}}^{\top}, \nonumber \\
        \mathbf{H}_{1}^{\ast} &= \frac{\sqrt{c} s^{3}}{c s^{6} + N \lambda_{H_{1}}} \mathbf{V}_{W_{1}} \mathbf{U}_{W_{3}}^{\top} \mathbf{Y} = \overline{\mathbf{H}}^{*} \mathbf{Y}, \nonumber
    \end{align}
    where $\mathbf{U}_{W_{3}}, \mathbf{U}_{W_{2}}, \mathbf{U}_{W_{1}}, \mathbf{V}_{W_{1}} 
   $ are semi-orthonormal matrices consist $r$ orthogonal columns. Additionally, we note that $\mathbf{U}_{W_{3}} \in \mathbb{R}^{K \times r}$ are created from orthonormal matrices size $K \times K$ with the removal of columns corresponding with singular values equal $0$. Thus, $\mathbf{U}_{W_{3}} \mathbf{U}_{W_{3}}^{\top}$ is the best rank-$r$ approximation of $\mathbf{I}_{K}$. From here, we can deduce the geometry of the following:
    \begin{align}
    \begin{gathered}
        \mathbf{W}_{3}^{*}
        \mathbf{W}_{3}^{* \top}
        \propto \overline{\mathbf{H}}^{* \top}
        \overline{\mathbf{H}}^{*} \propto  \mathbf{W}_{3}^{\ast} \mathbf{W}_{2}^{\ast} \mathbf{W}_{1}^{\ast} \overline{\mathbf{H}}^{*} \\
        \propto
        (\mathbf{W}_{3}^{\ast} \mathbf{W}_{2}^{\ast}) (\mathbf{W}_{3}^{\ast} \mathbf{W}_{2}^{\ast})^{\top}
        \propto 
        (\mathbf{W}_{1}^{\ast} \overline{\mathbf{H}})^{\top} (\mathbf{W}_{1}^{\ast} \overline{\mathbf{H}}) \\
        \propto 
         (\mathbf{W}_{3}^{\ast} \mathbf{W}_{2}^{\ast} \mathbf{W}_{1}^{\ast}) (\mathbf{W}_{3}^{\ast} \mathbf{W}_{2}^{\ast} \mathbf{W}_{1}^{\ast})^{\top}
         \propto
           (\mathbf{W}_{2}^{\ast} \mathbf{W}_{1}^{\ast} \overline{\mathbf{H}})^{\top} (\mathbf{W}_{2}^{\ast} \mathbf{W}_{1}^{\ast} \overline{\mathbf{H}})
         \propto  \mathcal{P}_{r}(\mathbf{I}_{K}),
    \end{gathered}
    \nonumber
    \end{align}
    where $\mathcal{P}_{r}(\mathbf{I}_{K})$ denotes the best rank-$r$ approximation of $\mathbf{I}_{K}$.
    The collapse of features $(\mathcal{NC}1)$ and the alignments between weights and features $(\mathcal{NC}3)$ are identical as the case $b < \frac{\sqrt[3]{4}}{3}$.
\end{itemize}
\end{proof}

\subsection{Supporting Lemmas for UFM Deep Linear Networks with M Layers of Weights}
\label{sec:lemmas}
Before deriving the proof for M layers linear network, from the proof of three layers of weights, we generalize some useful results that support the main proof. 

\noindent Consider MSE loss function with M layers linear network and arbitrary target matrix $\mathbf{Y} \in \mathbb{R}^{K \times N}$:
\begin{align}
    f(\mathbf{W}_{M}, \mathbf{W}_{M-1},\ldots, \mathbf{W}_{2}, \mathbf{W}_{1}, \mathbf{H}_{1})
    = \frac{1}{2N} \| \mathbf{W}_{M} \mathbf{W}_{M-1} \ldots \mathbf{W}_{2} \mathbf{W}_{1}
    \mathbf{H}_{1} - \mathbf{Y} \|_F^2 + \frac{\lambda_{W_{M}}}{2} \| \mathbf{W}_{M} \|^2_F \nonumber \\ +
    \frac{\lambda_{W_{M-1}}}{2} \| \mathbf{W}_{M-1} \|^2_F +
    \ldots   +
    \frac{\lambda_{W_{2}}}{2} \| \mathbf{W}_{2} \|^2_F + \frac{\lambda_{W_{1}}}{2} \| \mathbf{W}_{1} \|^2_F +
    \frac{\lambda_{H_{1}}}{2} \| \mathbf{H}_{1} \|^2_F,
\end{align}
with $\mathbf{W}_{M} \in \mathbb{R}^{K \times d_{M}}$, $\mathbf{W}_{M-1} \in \mathbb{R}^{d_{M} \times d_{M-1}}, \mathbf{W}_{M-2} \in \mathbb{R}^{d_{M-1} \times d_{M-2}}, \ldots, \mathbf{W}_{2} \in \mathbb{R}^{d_{3} \times d_{2}}, \mathbf{W}_{1}  \in \mathbb{R}^{d_{2} \times d_{1}}, \mathbf{H}_{1} \in \mathbb{R}^{d_{1} \times K}$ with $d_{M}, d_{M-1}, \ldots, d_{2}, d_{1}$ are arbitrary positive integers.

\begin{lemma}
\label{lm:1}
    The partial derivative of $\| \mathbf{W}_{M} \mathbf{W}_{M-1} \ldots \mathbf{W}_{2} \mathbf{W}_{1}
    \mathbf{H}_{1} - \mathbf{Y} \|_F^2 $ w.r.t $\mathbf{W}_{i}$  $(i= 1,2,\ldots, M)$:
    \begin{align}
    \begin{gathered}
    \frac{1}{2}
    \frac{\partial \| \mathbf{W}_{M} \mathbf{W}_{M-1} \ldots \mathbf{W}_{i} \ldots \mathbf{W}_{2} \mathbf{W}_{1}
    \mathbf{H}_{1} - \mathbf{Y} \|_F^2 }{\partial \mathbf{W}_{i}}
    = \nonumber \\ \mathbf{W}_{i+1}^{\top} \mathbf{W}_{i+2}^{\top} \ldots \mathbf{W}_{M}^{\top}
    (\mathbf{W}_{M} \mathbf{W}_{M-1} \ldots \mathbf{W}_{i} \ldots \mathbf{W}_{2} \mathbf{W}_{1}
    \mathbf{H}_{1} - \mathbf{Y} )\mathbf{H}_{1}^{\top} \mathbf{W}_{1}^{\top} \ldots \mathbf{W}_{i-1}^{\top}. \nonumber
    \end{gathered}
    \end{align}
\end{lemma}
\noindent This result is common and the proof can be found in \cite{Yun17}, for example.

\begin{lemma}
\label{lm:2}
    For any critical point $(\mathbf{W}_{M}, \mathbf{W}_{M-1},\ldots, \mathbf{W}_{2}, \mathbf{W}_{1}, \mathbf{H}_{1})$ of $f$, we have the following:
    \begin{align}
    \begin{gathered}
        \lambda_{W_{M}} \mathbf{W}^{\top}_{M} \mathbf{W}_{M} = \lambda_{W_{M-1}} \mathbf{W}_{M-1} \mathbf{W}^{\top}_{M-1}, \\
        \lambda_{W_{M-1}} \mathbf{W}^{\top}_{M-1} \mathbf{W}_{M-1} = \lambda_{W_{M-2}} \mathbf{W}_{M-2} \mathbf{W}^{\top}_{M-2}, \\
        \ldots, \nonumber \\
        \lambda_{W_{2}} \mathbf{W}_{2}^{\top} \mathbf{W}_{2} = \lambda_{W_{1}} \mathbf{W}_{1} \mathbf{W}_{1}^{\top}, \\
        \lambda_{W_{1}} \mathbf{W}_{1}^{\top} \mathbf{W}_{1} =  \lambda_{H_{1}}\mathbf{H}_{1} \mathbf{H}_{1}^{\top},   \nonumber        
    \end{gathered}
    \end{align}
    and:
    \begin{align}
        \mathbf{H}_{1} 
        = (c (\mathbf{W}_{1}^{\top} \mathbf{W}_{1})^{M} + N \lambda_{H_{1}} \mathbf{I} )^{-1} \mathbf{W}_{1}^{\top}
        \mathbf{W}_{2}^{\top}  
        \ldots
        \mathbf{W}_{M}^{\top} \mathbf{Y},
    \end{align}
    with $c := \frac{\lambda_{W_{1}}^{M-1}}
    {\lambda_{W_{M}} \lambda_{W_{M-1}} \ldots \lambda_{W_{2}} }$.
\end{lemma}

\begin{proof}[Proof of Lemma \ref{lm:2}]
    By definition and using Lemma \ref{lm:1}, any critical point $(\mathbf{W}_{M}, \mathbf{W}_{M-1},\ldots, \mathbf{W}_{2}, \mathbf{W}_{1},\mathbf{H}_{1})$ satisfies the following :
    \begin{align}
    \begin{aligned}
      &\frac{\partial f}{\partial \mathbf{W}_{M}} =
      \frac{1}{N} (\mathbf{W}_{M} \mathbf{W}_{M-1} \ldots \mathbf{W}_{2} \mathbf{W}_{1}   \mathbf{H}_{1} - \mathbf{Y})\mathbf{H}_{1}^{\top} \mathbf{W}_{1}^{\top} \ldots
      \mathbf{W}_{M-1}^{\top} + \lambda_{W_{M}} \mathbf{W}_{M} = \mathbf{0}, \\
      &\frac{\partial f}{\partial \mathbf{W}_{M-1}} =
      \frac{1}{N} \mathbf{W}_{M}^{\top} (\mathbf{W}_{M} \mathbf{W}_{M-1} \ldots \mathbf{W}_{2} \mathbf{W}_{1}   \mathbf{H}_{1} - \mathbf{Y}) \mathbf{H}_{1}^{\top} \mathbf{W}_{1}^{\top} \ldots \mathbf{W}_{M-2}^{\top}  + \lambda_{W_{M-1}} \mathbf{W}_{M-1} = \mathbf{0}, \\
      &\ldots, \nonumber \\
      &\frac{\partial f}{\partial \mathbf{W}_{1}} =
      \frac{1}{N} \mathbf{W}_{2}^{\top} \mathbf{W}_{3}^{\top} \ldots \mathbf{W}_{M}^{\top} (\mathbf{W}_{M} \mathbf{W}_{M-1} \ldots \mathbf{W}_{2} \mathbf{W}_{1}   \mathbf{H}_{1} - \mathbf{Y}) \mathbf{H}_{1}^{\top}  + \lambda_{W_{1}} \mathbf{W}_{1} = \mathbf{0}, \\
      &\frac{\partial f}{\partial \mathbf{H}_{1}} =
      \frac{1}{N} \mathbf{W}_{1}^{\top} \mathbf{W}_{2}^{\top} \ldots \mathbf{W}_{M}^{\top} (\mathbf{W}_{M} \mathbf{W}_{M-1} \ldots \mathbf{W}_{2} \mathbf{W}_{1}   \mathbf{H}_{1} - \mathbf{Y})  +  \lambda_{H_{1}} \mathbf{H}_{1} = \mathbf{0} .
    \end{aligned} \nonumber
    \end{align}
    
    \noindent Next, we have:
    \begin{align}
        &\mathbf{0} = \mathbf{W}^{\top}_{M} \frac{\partial f}{\partial \mathbf{W}_{M}} - \frac{\partial f}{\partial \mathbf{W}_{M-1}} \mathbf{W}_{M-1}^{\top} = \lambda_{W_{M}} \mathbf{W}^{\top}_{M} \mathbf{W}_{M} - \lambda_{W_{M-1}} \mathbf{W}_{M-1} \mathbf{W}^{\top}_{M-1} \nonumber \\ 
        &\Rightarrow \lambda_{W_{M}} \mathbf{W}^{\top}_{M} \mathbf{W}_{M} = \lambda_{W_{M-1}} \mathbf{W}_{M-1} \mathbf{W}^{\top}_{M-1}. \nonumber \\
        &\mathbf{0} = \mathbf{W}^{\top}_{M-1} \frac{\partial f}{\partial \mathbf{W}_{M-1}} - \frac{\partial f}{\partial \mathbf{W}_{M-2}} \mathbf{W}_{M-2}^{\top} = \lambda_{W_{M-1}} \mathbf{W}^{\top}_{M-1} \mathbf{W}_{M-1} - \lambda_{W_{M-2}} \mathbf{W}_{M-2} \mathbf{W}^{\top}_{M-2} \nonumber \\
        &\Rightarrow \lambda_{W_{M-1}} \mathbf{W}^{\top}_{M-1} \mathbf{W}_{M-1} = \lambda_{W_{M-2}} \mathbf{W}_{M-2} \mathbf{W}^{\top}_{M-2}. \nonumber 
    \end{align}
    Making similar argument for the other derivatives, we have:
    \begin{align}
    \begin{gathered}
        \lambda_{W_{M}} \mathbf{W}^{\top}_{M} \mathbf{W}_{M} = \lambda_{W_{M-1}} \mathbf{W}_{M-1} \mathbf{W}^{\top}_{M-1}, \\
        \lambda_{W_{M-1}} \mathbf{W}^{\top}_{M-1} \mathbf{W}_{M-1} = \lambda_{W_{M-2}} \mathbf{W}_{M-2} \mathbf{W}^{\top}_{M-2}, \\
        \ldots, \nonumber \\
        \lambda_{W_{2}} \mathbf{W}_{2}^{\top} \mathbf{W}_{2} = \lambda_{W_{1}} \mathbf{W}_{1} \mathbf{W}_{1}^{\top}, \\
        \lambda_{W_{1}} \mathbf{W}_{1}^{\top} \mathbf{W}_{1} = \lambda_{H_{1}}\mathbf{H}_{1} \mathbf{H}_{1}^{\top}.    
    \end{gathered}
    \nonumber
    \end{align}
    
   \noindent Also, from $\frac{\partial f}{\partial \mathbf{H}_{1}} = \mathbf{0}$, solving for $\mathbf{H}_{1}$ yields:
    \begin{align}
    \mathbf{H}_{1} &= 
    ( \mathbf{W}_{1}^{\top}
    \mathbf{W}_{2}^{\top} 
    \ldots
    \mathbf{W}_{M-1}^{\top}
    \mathbf{W}_{M}^{\top}
    \mathbf{W}_{M}
    \mathbf{W}_{M-1}
    \ldots
     \mathbf{W}_{2}
    \mathbf{W}_{1} + N \lambda_{H_{1}} \mathbf{I}
    )^{-1}
    \mathbf{W}_{1}^{\top}
    \mathbf{W}_{2}^{\top}  
    \ldots
    \mathbf{W}_{M}^{\top} \mathbf{Y} \nonumber \\
    &= \left( \frac{\lambda_{W_{M-1}}}{\lambda_{W_{M}}}
    \mathbf{W}_{1}^{\top}
    \mathbf{W}_{2}^{\top} 
    \ldots
    (\mathbf{W}_{M-1}^{\top} \mathbf{W}_{M-1})^{2}
    \ldots
    \mathbf{W}_{2}
    \mathbf{W}_{1} + N \lambda_{H_{1}} \mathbf{I}
     \right)^{-1}
    \mathbf{W}_{1}^{\top}
    \mathbf{W}_{2}^{\top}  
    \ldots
    \mathbf{W}_{M}^{\top} \mathbf{Y}
    \nonumber \\
    &= \ldots \nonumber \\
    &= \left(
    \underbrace{\frac{\lambda_{W_{1}}^{M-1}}
    {\lambda_{W_{M}} \lambda_{W_{M-1}} \ldots \lambda_{W_{2}} }}_{c}  
    (\mathbf{W}_{1}^{\top} \mathbf{W}_{1})^{M} + N \lambda_{H_{1}}
    \right)^{-1}
    \mathbf{W}_{1}^{\top}
    \mathbf{W}_{2}^{\top}  
    \ldots
    \mathbf{W}_{M}^{\top} \mathbf{Y} \nonumber \\
    &=
    (c (\mathbf{W}_{1}^{\top} \mathbf{W}_{1})^{M} + N \lambda_{H_{1}} \mathbf{I} )^{-1} \mathbf{W}_{1}^{\top}
    \mathbf{W}_{2}^{\top}  
    \ldots
    \mathbf{W}_{M}^{\top} \mathbf{Y}. \nonumber
    \end{align} 
\end{proof}
\begin{lemma}
\label{lm:3}
    For any critical point $(\mathbf{W}_{M}, \mathbf{W}_{M-1},\ldots, \mathbf{W}_{2}, \mathbf{W}_{1}, \mathbf{H}_{1})$, we have
    $r:= \operatorname{rank}(\mathbf{W}_{M}) = \operatorname{rank}(\mathbf{W}_{M-1}) = \operatorname{rank}(\mathbf{W}_{M-2}) = \ldots = \operatorname{rank}(\mathbf{W}_{1} ) = \operatorname{rank}(\mathbf{H}_{1}) \leq \min(K, d_{M}, d_{M-1},\ldots, d_{1}) := R$.
\end{lemma}

\begin{proof}[Proof of Lemma \ref{lm:3}]
    The result is deduced from Lemma \ref{lm:2} and the matrix rank property $\operatorname{rank}(\mathbf{A}) = \operatorname{rank}(\mathbf{A}^{\top} \mathbf{A}) =  \operatorname{rank}(\mathbf{A} \mathbf{A}^{\top})$.
\end{proof}

\begin{lemma}
\label{lm:4}
    For any critical point $(\mathbf{W}_{M}, \mathbf{W}_{M-1},\ldots, \mathbf{W}_{2}, \mathbf{W}_{1}, \mathbf{H}_{1})$ of $f$, let $\mathbf{W}_{1} = \mathbf{U}_{W_{1}} \mathbf{S}_{W_{1}} \mathbf{V}_{W_{1}}^{\top}$ be the SVD decomposition of $\mathbf{W}_{1}$ with $\mathbf{U}_{W_{1}} \in \mathbb{R}^{d_{2} \times d_{2}}, \mathbf{V}_{W_{1}} \in \mathbb{R}^{d_{1} \times d_{1}}$ are orthonormal matrices and $\mathbf{S}_{W_{1}} \in \mathbb{R}^{d_{2} \times d_{1}}$ is a diagonal matrix with \textbf{decreasing} non-negative singular values. We denote the $r := \operatorname{rank}(\mathbf{W}_{1})$ singular values of $\mathbf{W}_{1}$  as $\left\{s_{k}\right\}_{k=1}^{r}$ ($r \leq R := \min(K, d_{M}, \ldots, d_{1})$, from Lemma \ref{lm:3}).
    \\
    
    Then, we can write the SVD of weight matrices as:
    \begin{align}
    \begin{gathered}
        \mathbf{W}_{M} = \mathbf{U}_{W_{M}} \mathbf{S}_{W_{M}}
        \mathbf{U}_{W_{M-1}}^{\top},
        \\
        \mathbf{W}_{M-1} = \mathbf{U}_{W_{M-1}} \mathbf{S}_{W_{M-1}} \mathbf{U}_{W_{M-2}}^{\top},     \\
        \mathbf{W}_{M-2} = \mathbf{U}_{W_{M-2}} \mathbf{S}_{W_{M-2}} \mathbf{U}_{W_{M-3}}^{\top}, \\
        \mathbf{W}_{M-3} = \mathbf{U}_{W_{M-3}} \mathbf{S}_{W_{M-3}} \mathbf{U}_{W_{M-4}}^{\top}, \\
        \ldots, \\
        \mathbf{W}_{2} = \mathbf{U}_{W_{2}} \mathbf{S}_{W_{2}} \mathbf{U}_{W_{1}}^{\top},
        \\
        \mathbf{W}_{1} = \mathbf{U}_{W_{1}} \mathbf{S}_{W_{1}} \mathbf{V}_{W_{1}}^{\top}, 
        \nonumber
    \end{gathered}
    \end{align}
    with:
    \begin{align}
       \mathbf{S}_{W_{j}} = \sqrt{\frac{\lambda_{W_{1}}}{\lambda_{W_{j}}}}
        \begin{bmatrix}
        \operatorname{diag}(s_{1},\ldots, s_{r}) & \mathbf{0}_{r \times (d_{j} - r)}  \\
        \mathbf{0}_{(d_{j+1} - r) \times r} & \mathbf{0}_{(d_{j+1} - r) \times (d_{j} - r)}  \\
        \end{bmatrix} \in \mathbb{R}^{d_{j+1} \times d_{j}}, \quad \forall \: j \in [M],
        \nonumber
    \end{align} 
    and $\mathbf{U}_{W_{M}}, \mathbf{U}_{W_{M-1}}, \mathbf{U}_{W_{M-2}}, \mathbf{U}_{W_{M-3}},\ldots, \mathbf{U}_{W_{1}}, \mathbf{V}_{W_{1}} $ are all orthonormal matrices.
\end{lemma}

\begin{proof}[Proof of Lemma \ref{lm:4}]
    From Lemma \ref{lm:2}, we have:
    \begin{align}
        \mathbf{W}_{2}^{\top} \mathbf{W}_{2}
        = \frac{\lambda_{W_{1}}}{\lambda_{W_{2}}} \mathbf{W}_{1} \mathbf{W}_{1}^{\top} = \frac{\lambda_{W_{1}}}{\lambda_{W_{2}}} 
        \mathbf{U}_{W_{1}} \mathbf{S}_{W_{1}}
        \mathbf{S}_{W_{1}}^{\top}
        \mathbf{U}_{W_{1}}^{\top} = \mathbf{U}_{W_{1}} \mathbf{S}_{W_{2}}^{\top}
        \mathbf{S}_{W_{2}} \mathbf{U}_{W_{1}}^{\top}, \nonumber
    \end{align}
    where:
    \begin{align}
        \mathbf{S}_{W_{2}} := \sqrt{\frac{\lambda_{W_{1}}}{\lambda_{W_{2}}}} 
        \begin{bmatrix}
        \operatorname{diag}(s_{1},\ldots, s_{r}) & \mathbf{0}_{r \times (d_{2} - r)}  \\
        \mathbf{0}_{(d_{3} - r) \times r} & \mathbf{0}_{(d_{3} - r) \times (d_{2} - r)}  \\
        \end{bmatrix} \in \mathbb{R}^{d_{3} \times d_{2}}. \nonumber
    \end{align}
    This means the diagonal matrix $\mathbf{S}_{W_{2}}^{\top}
    \mathbf{S}_{W_{2}}$ contains the eigenvalues and the columns of $\mathbf{U}_{W_{1}}$ are the eigenvectors of $\mathbf{W}_{2}^{\top} \mathbf{W}_{2}$. Hence, we can write the SVD decomposition of $\mathbf{W}_{2}$ as $\mathbf{W}_{2} = \mathbf{U}_{W_{2}} \mathbf{S}_{W_{2}} \mathbf{U}_{W_{1}}^{\top}$ with orthonormal matrix $\mathbf{U}_{W_{2}} \in \mathbb{R}^{d_{3} \times d_{3}}$. 
    \\  
    
    \noindent By making similar arguments as above for $\mathbf{W}_{3}$, from:
    \begin{align}
        \mathbf{W}_{3}^{\top} \mathbf{W}_{3}
        = \frac{\lambda_{W_{2}}}{\lambda_{W_{3}}} \mathbf{W}_{2} \mathbf{W}_{2}^{\top} = \frac{\lambda_{W_{2}}}{\lambda_{W_{3}}} \mathbf{U}_{W_{2}} \mathbf{S}_{W_{2}}
        \mathbf{S}_{W_{2}}^{\top}
        \mathbf{U}_{W_{2}}^{\top} = \mathbf{U}_{W_{2}} \mathbf{S}_{W_{3}}^{\top}
        \mathbf{S}_{W_{3}} \mathbf{U}_{W_{2}}^{\top},  \nonumber
    \end{align}
    where:
    \begin{align}
        \mathbf{S}_{W_{3}} := \sqrt{\frac{\lambda_{W_{1}}}{\lambda_{W_{3}}}} 
        \begin{bmatrix}
        \operatorname{diag}(s_{1},\ldots, s_{r}) & \mathbf{0}_{r \times (d_{3} - r)}  \\
        \mathbf{0}_{(d_{4} - r) \times r} & \mathbf{0}_{(d_{4} - r) \times (d_{3} - r)}  \\
        \end{bmatrix} \in \mathbb{R}^{d_{4} \times d_{3}}, \nonumber
    \end{align}
    and thus, we can write SVD decomposition of $\mathbf{W}_{3}$ as $\mathbf{W}_{3} = \mathbf{U}_{W_{3}} \mathbf{S}_{W_{3}} \mathbf{U}_{W_{2}}^{\top}$ with orthonormal matrix $\mathbf{U}_{W_{3}} \in \mathbb{R}^{d_{4} \times d_{4}}$. 
    Repeating the process for other weight matrices, we got the desired result.
\end{proof}
\begin{lemma}
\label{lm:5}
    Continue from the setting and result of Lemma \ref{lm:4}, we have:
    \begin{align}
    \mathbf{H}_{1} 
        &= \mathbf{V}_{W_{1}}
        \underbrace{
         \begin{bmatrix}
        \operatorname{diag}\left(
            \frac{\sqrt{c} s_{1}^{M} }{c s_{1}^{2M} + N \lambda_{H_{1}}}, \ldots, \frac{\sqrt{c} s_{r}^{M}}{c s_{r}^{2M} + N \lambda_{H_{1}}}
            \right) &  \mathbf{0}_{r \times (K-r)} \\
        \mathbf{0}_{(d_{1}-r) \times r} & \mathbf{0}_{(d_{1}-r) \times (K - r)}\\
        \end{bmatrix}}_{\mathbf{C} \in \mathbb{R}^{d_{1} \times K}}
        \mathbf{U}_{W_{M}}^{\top}
        \mathbf{Y}, \nonumber 
        \\
     \mathbf{W}_{M} \mathbf{W}_{M-1} \ldots \mathbf{W}_{2} \mathbf{W}_{1}   \mathbf{H} - \mathbf{Y} 
    &= \mathbf{U}_{W_{M}} 
        \underbrace{
        \begin{bmatrix}
        \operatorname{diag}\left(
            \frac{- N \lambda_{H_{1}} }{c s_{1}^{2M} + N \lambda_{H_{1}}}, \ldots, \frac{- N \lambda_{H_{1}}}{c s_{r}^{2M} + N \lambda_{H_{1}}}
            \right ) &  \mathbf{0}_{r \times (K-r)}\\
        \mathbf{0}_{(K-r) \times r} &  -\mathbf{I}_{K-r}\\
        \end{bmatrix}}_{\mathbf{D} \in \mathbb{R}^{K \times K}}
        \mathbf{U}_{W_{M}}^{\top}
        \mathbf{Y}, \nonumber
    \end{align}
    with $c := \frac{\lambda_{W_{1}}^{M-1}}
    {\lambda_{W_{M}} \lambda_{W_{M-1}} \ldots \lambda_{W_{2}} }$.
\end{lemma}

\begin{proof}[Proof of Lemma \ref{lm:5}]
    From Lemma \ref{lm:2}, together with the SVD of weight matrices and the form of singular matrix $\mathbf{S}_{W_{j}}$ derived in Lemma \ref{lm:4}, we have:
 \begin{align}
    \begin{aligned}
    \mathbf{H}_{1} &= 
    (c (\mathbf{W}_{1}^{\top} \mathbf{W}_{1})^{M} + N \lambda_{H_{1}} \mathbf{I} )^{-1} \mathbf{W}_{1}^{\top}
    \mathbf{W}_{2}^{\top}  
    \ldots
    \mathbf{W}_{M}^{\top} \mathbf{Y} \\
    &= 
    (c \mathbf{V}_{W_{1}} (\mathbf{S}_{W_{1}}^{\top}
    \mathbf{S}_{W_{1}})^{M}
    \mathbf{V}_{W_{1}}^{\top}
    + N \lambda_{H_{1}} \mathbf{I} )^{-1}
    \mathbf{V}_{W_{1}}
    \mathbf{S}_{W_{1}}^{\top}
    \mathbf{S}_{W_{2}}^{\top} 
    \ldots
    \mathbf{S}_{W_{M}}^{\top}
    \mathbf{U}_{W_{M}}^{\top}
    \mathbf{Y} \\
    &= 
    \mathbf{V}_{W_{1}} 
    (c  (\mathbf{S}_{W_{1}}^{\top}
    \mathbf{S}_{W_{1}})^{M}
    + N \lambda_{H_{1}} \mathbf{I} )^{-1}
    \mathbf{S}_{W_{1}}^{\top}
    \mathbf{S}_{W_{2}}^{\top} 
    \ldots
    \mathbf{S}_{W_{M}}^{\top}
    \mathbf{U}_{W_{M}}^{\top}
    \mathbf{Y} \\
    &= 
    \mathbf{V}_{W_{1}} 
     (c  (\mathbf{S}_{W_{1}}^{\top}
    \mathbf{S}_{W_{1}})^{M}
    + N \lambda_{H_{1}} \mathbf{I} )^{-1}
    \sqrt{c} 
    \begin{bmatrix}
    \operatorname{diag}(s_{1}^{M},\ldots, s_{r}^{M}) & \mathbf{0}_{r \times (K-r)}  \\
    \mathbf{0}_{(d_{1}-r) \times r} & \mathbf{0}_{(d_{1}-r) \times (K - r)} \\
    \end{bmatrix}
    \mathbf{U}_{W_{M}}^{\top}
    \mathbf{Y} \\
    &= 
    \mathbf{V}_{W_{1}}
    \underbrace{
    \begin{bmatrix}
    \operatorname{diag}\left(
        \frac{\sqrt{c} s_{1}^{M} }{c s_{1}^{2M} + N \lambda_{H_{1}}}, \ldots, \frac{\sqrt{c} s_{r}^{M}}{c s_{r}^{2M} + N \lambda_{H_{1}}}
        \right) &  \mathbf{0}_{r \times (K-r)} \\
    \mathbf{0}_{(d_{1}-r) \times r} & \mathbf{0}_{(d_{1}-r) \times (K - r)}\\
    \end{bmatrix}}_{\mathbf{C} \in \mathbb{R}^{d_{1} \times K}}
    \mathbf{U}_{W_{M}}^{\top}
    \mathbf{Y} \\
    &= \mathbf{V}_{W_{1}}
    \mathbf{C}
    \mathbf{U}_{W_{M}}^{\top}
    \mathbf{Y} 
    \end{aligned} \nonumber
    \end{align}
    
    \begin{align}
        &\Rightarrow \mathbf{W}_{M} \mathbf{W}_{M-1} \ldots \mathbf{W}_{2} \mathbf{W}_{1} \mathbf{H}_{1} 
        = \mathbf{U}_{W_{M}}
        \mathbf{S}_{W_{M}}
        \mathbf{S}_{W_{M-1}}
        \ldots
        \mathbf{S}_{W_{1}}
        \mathbf{C}
        \mathbf{U}_{W_{M}}^{\top}
        \mathbf{Y}  \nonumber \\
        &=
        \sqrt{\frac{\lambda_{W_{1}}}{\lambda_{W_{M}}}}
        \mathbf{U}_{W_{M}}
        \begin{bmatrix}
        \operatorname{diag}(s_{1},\ldots, s_{r}) & \mathbf{0}  \\
        \mathbf{0} & \mathbf{0}  \\
        \end{bmatrix}
        \mathbf{S}_{W_{M-1}}
        \ldots
        \mathbf{S}_{W_{1}}
        \mathbf{C}
        \mathbf{U}_{W_{M}}^{\top}
        \mathbf{Y}
        \nonumber \\
        &= \ldots \nonumber \\
        &= \mathbf{U}_{W_{M}} \sqrt{c} 
        \begin{bmatrix}
        \operatorname{diag}(s_{1}^{M},\ldots, s_{r}^{M}) & \mathbf{0}  \\
        \mathbf{0} & \mathbf{0}  \\
        \end{bmatrix} 
        \mathbf{C} \mathbf{U}_{W_{M}}^{\top}
        \mathbf{Y}
        \nonumber \\
        &=  \mathbf{U}_{W_{M}}
        \begin{bmatrix}
        \operatorname{diag}\left(
            \frac{c s_{1}^{2M} }{c s_{1}^{2M} + N \lambda_{H_{1}}}, \ldots, \frac{c s_{r}^{2M}}{c s_{r}^{2M} + N \lambda_{H_{1}}}
            \right ) &  \mathbf{0}\\
        \mathbf{0} &  \mathbf{0}\\
        \end{bmatrix} 
        \mathbf{U}_{W_{M}}^{\top}
        \mathbf{Y}
        \nonumber \\ %
        &\Rightarrow 
        \mathbf{W}_{M} \ldots \mathbf{W}_{1}   \mathbf{H}_{1} - \mathbf{Y} 
        = 
        \mathbf{U}_{W_{M}} 
        \left(  
        \begin{bmatrix}
        \operatorname{diag}\left(
            \frac{c s_{1}^{2M} }{c s_{1}^{2M} + N \lambda_{H_{1}}}, \ldots, \frac{c s_{r}^{2M}}{c s_{r}^{2M} + N \lambda_{H_{1}}}
            \right ) &  \mathbf{0}_{r \times (K-r)}\\
        \mathbf{0}_{(K-r) \times r} &  \mathbf{0}_{(K-r) \times (K-r)} \\
        \end{bmatrix}  - \mathbf{I}_{K} \right)
        \mathbf{U}_{W_{M}}^{\top}
        \mathbf{Y}  \nonumber \\
        &= \mathbf{U}_{W_{M}} 
        \underbrace{
        \begin{bmatrix}
        \operatorname{diag}\left(
            \frac{- N \lambda_{H_{1}} }{c s_{1}^{2M} + N \lambda_{H_{1}}}, \ldots, \frac{- N \lambda_{H_{1}}}{c s_{r}^{2M} + N \lambda_{H_{1}}}
            \right ) &  \mathbf{0}_{r \times (K-r)}\\
        \mathbf{0}_{(K-r) \times r} &  -\mathbf{I}_{K-r} \\
        \end{bmatrix}}_{\mathbf{D} \in \mathbb{R}^{K \times K}}
        \mathbf{U}_{W_{M}}^{\top}
        \mathbf{Y} \nonumber  \\
        &= \mathbf{U}_{W_{M}}
        \mathbf{D}
        \mathbf{U}_{W_{M}}^{\top}
        \mathbf{Y}. \nonumber
    \end{align}
\end{proof}

\subsubsection{Minimizer of the function $g(x) = \frac{1}{x^{M} + 1} + bx$}
\label{sec:study_g}
\noindent Next,  we study the minimization problem of the following function, this result will be used frequently in proofs of theorems in the main paper:
    \begin{align}
    g(x) = \frac{1}{x^{M} + 1} + bx \text{ with } x \geq 0, b > 0, M 
    \geq 2. \nonumber
    \end{align}

\noindent Clearly, $g(0) = 1$. We consider the following cases for parameter $b$:

\begin{itemize}
    \item If $b > \frac{(M-1)^{\frac{M-1}{M}}}{M}$:
    We have with $x > 0$: $g(x) > \frac{1}{x^{M}+1} + \frac{(M-1)^{\frac{M-1}{M}}}{M} x$. We will prove:
    \begin{align}
    \begin{aligned}
        &\frac{1}{x^{M}+1} + \frac{(M-1)^{\frac{M-1}{M}}}{M} x \geq 1 \\
        &\Leftrightarrow\frac{(M-1)^{\frac{M-1}{M}}}{M} x^{M+1} - x^{M} +  \frac{(M-1)^{\frac{M-1}{M}}}{M} x \geq 0 \\
        &\Leftrightarrow x (x^{M} - \frac{M}{(M-1)^{\frac{M-1}{M}}} x^{M-1} + 1) \geq 0 \\
        &\Leftrightarrow x^{M} - \frac{M}{(M-1)^{\frac{M-1}{M}}} x^{M-1} + 1 \geq 0.
    \end{aligned}
    \end{align}
    Let $h(x) =  x^{M} - \frac{M}{(M-1)^{\frac{M-1}{M}}} x^{M-1} + 1$ with $x \geq 0$, we have:
    \begin{align}
        h^{\prime}(x) = M x^{M-1} - M (M-1)^{1/M} x^{M-2}, \nonumber \\
        h^{\prime}(x) = 0 \Leftrightarrow x = 0 \text{ or } x = (M-1)^{1/M}.
    \end{align}
    We also have: $h(0) = 1$ and $h((M-1)^{1/M}) = M - 1 - M + 1 = 0$. From the variation table, we clearly have $h(x) \geq 0 \: \forall \: x \geq 0$. 
    
    \begin{table}[ht]
    \centering
    \begin{tabular}{c|ccc}
    $x$             & 0 & $(M-1)^{1/M}$ & $\infty$ \\ \hline
    $h^{\prime}(x)$ & - & 0             & +        \\ \hline
    $h(x)$          & 1 & 0             & $\infty$
    \end{tabular}
    \end{table}
    
    Hence, in this case, $g(x) > 1 \: \forall \: x > 0$, therefore, $g(x)$ is minimized at $x = 0$.
    
    \item If $b = \frac{(M-1)^{\frac{M-1}{M}}}{M}$: We have $g(x) = \frac{1}{x^{M}+1} + \frac{(M-1)^{\frac{M-1}{M}}}{M} x \geq 1$. Thus, $g(x)$ is minimized at $x = 0$ or $x = (M-1)^{1/M}$.
    
    \item If $b < \frac{(M-1)^{\frac{M-1}{M}}}{M}$: We take the first and second derivatives of $g(x)$:
    \begin{align}
        g^{\prime}(x) &= b - \frac{M x^{M-1}}{(x^{M} + 1)^{2}}, \nonumber \\
        g^{\prime \prime}(x) &= -M \left(
        \frac{(M-1) x^{M-2}}{(x^{M}+1)^{2}} - \frac{2 M x^{2M-2}}{(x^{M}+1)^{3}}
        \right). \nonumber
        \\
        &= 
        \frac{(M^{2}+M)x^{2M-2} - (M^{2}-M) x^{M-2}}
        {(x^{M}+1)^{3}} \nonumber
    \end{align}
    We have: $g^{\prime \prime}(x) = 0 \Leftrightarrow x = 0$ or $x = \sqrt[M]{\frac{M-1}{M+1}}$. Therefore, with $x \geq 0$, $g^{\prime}(x) = 0$ has at most 2 solutions. We further have $g^{\prime} (\sqrt[M]{\frac{M-1}{M+1}}) = b -
    M (\frac{M-1}{M+1})^{\frac{M-1}{M}}  / 
    (\frac{M-1}{M+1} + 1)^{2} < (M-1)^{\frac{M-1}{M}} / M -  M (\frac{M-1}{M+1})^{\frac{M-1}{M}}  / 
    (\frac{M-1}{M+1} + 1)^{2} $. Actually, we have:
    \begin{align}
        &\frac{(M-1)^{\frac{M-1}{M}}}{M} < 
        \frac{M (\frac{M-1}{M+1})^{\frac{M-1}{M}}}
        {(\frac{M-1}{M+1} + 1)^{2}} \nonumber \\
        \Leftrightarrow \: &\left(\frac{M-1}{M+1} + 1 \right)^{2} < \frac{M^{2}}{ (M+1)^{\frac{M-1}{M}} } \nonumber \\
        \Leftrightarrow \: &\frac{4M^2}{(M+1)^{2} } < \frac{M^{2}}{ (M+1)^{\frac{M-1}{M}} } \nonumber \\
        \Leftrightarrow \: &4 < (M+1)^{2 - \frac{M-1}{M}} \nonumber \\
        \Leftrightarrow \: &4 < (M+1)^{1 +  \frac{1}{M}} \quad (\text{true } \forall M \geq 2). \nonumber
    \end{align}
    Therefore, $g^{\prime} (\sqrt[M]{\frac{M-1}{M+1}}) < 0$.
    Together with the fact that $g^{\prime}(0) = b > 0$ and $g^{\prime}(+ \infty) > 0$ , $g^{\prime}(x) = 0$ has exactly two solutions, we call it $x_1$ and $x_2$ ($x_{1} < \sqrt[M]{\frac{M-1}{M+1}} < x_{2}$). Next, we note that $g^{\prime} (x_2) = 0$ and $g^{\prime} (x) > 0 \quad \forall x > x_2$ (since $g^{\prime \prime} (x) > 0 \quad  \forall x > x_2$). In the meanwhile, $g^{\prime}(\sqrt[M]{M-1}) = b - \frac{ M (M-1)^{\frac{M-1}{M}}}{M^{2}} = b - \frac{(M-1)^{\frac{M-1}{M}}}{M} < 0$. Hence, we must have $x_{2} > \sqrt[M]{M-1}$.
    
    \begin{table}[h]
    \centering
    \begin{tabular}{c|cccccc}
    $x$    & 0 & $x_{1}$    & $\sqrt[M]{\frac{M-1}{M+1}}$    & $\sqrt[M]{M-1}$                 & $x_{2}$    & $+ \infty$ \\ \hline
    $g^{\prime \prime}(x)$ & 0 & - & 0 & + & + & + \\ \hline
    $g^{\prime}(x)$        & + & 0 & - & - & 0 & + \\ \hline
    $g(x)$ & 1 & $g(x_{1})$ & $g(\sqrt[M]{\frac{M-1}{M+1}})$ & $\frac{1}{M} + b \sqrt[M]{M-1}$ & $g(x_{2})$ & $+ \infty$
    \end{tabular}
    \end{table}
    
    From the variation table, we can see that $g(x_{2}) < g(\sqrt[M]{M-1}) = \frac{1}{M} + b \sqrt[M]{M-1} < \frac{1}{M} + \frac{(M-1)^{\frac{M-1}{M}}}{M} \sqrt[M]{M-1} = \frac{1}{M} + \frac{M-1}{M} = 1 = g(0)$. 
    
    In conclusion, in this case, $g(x)$ is minimized at $x_{2} > \sqrt[M]{M-1}$, i.e. the largest solution of the equation $b - \frac{M x^{M-1}}{(x^{M} + 1)^{2}} = 0$.
\end{itemize}

\subsection{Full Proof of Theorem \ref{thm:bias-free} with Bias-Free}
Now, we state the proof of Theorem \ref{thm:bias-free} for general setting with $M$ layers of weight with no bias (i.e., excluding $\mathbf{b}$) with arbitrary widths $d_{M}, d_{M-1},\ldots, d_{1}$.

\begin{proof}[Proof of Theorem \ref{thm:bias-free} (bias-free)]
    First, by using Lemma \ref{lm:2}, we have for any critical point  $(\mathbf{W}_{M}, \mathbf{W}_{M-1},\ldots, \mathbf{W}_{2}, \mathbf{W}_{1}, \mathbf{H}_{1})$ of $f$, we have the following:
    \begin{align}
    \begin{gathered}
        \lambda_{W_{M}} \mathbf{W}^{\top}_{M} \mathbf{W}_{M} = \lambda_{W_{M-1}} \mathbf{W}_{M-1} \mathbf{W}^{\top}_{M-1}, \\
        \lambda_{W_{M-1}} \mathbf{W}^{\top}_{M-1} \mathbf{W}_{M-1} = \lambda_{W_{M-2}} \mathbf{W}_{M-2} \mathbf{W}^{\top}_{M-2}, \\
        \ldots, \nonumber \\
        \lambda_{W_{2}} \mathbf{W}_{2}^{\top} \mathbf{W}_{2} = \lambda_{W_{1}} \mathbf{W}_{1} \mathbf{W}_{1}^{\top}, \\
        \lambda_{W_{1}} \mathbf{W}_{1}^{\top} \mathbf{W}_{1} =  \lambda_{H_{1}}\mathbf{H}_{1} \mathbf{H}_{1}^{\top}. 
    \end{gathered}
    \end{align}
    
    \noindent Let $\mathbf{W}_{1} = \mathbf{U}_{W_{1}} \mathbf{S}_{W_{1}} \mathbf{V}_{W_{1}}^{\top}$ be the SVD decomposition of $\mathbf{W}_{1}$ with $\mathbf{U}_{W_{1}} \in \mathbb{R}^{d_{2} \times d_{2}}, \mathbf{V}_{W_{1}} \in \mathbb{R}^{d_{1} \times d_{1}}$ are orthonormal matrices and $\mathbf{S}_{W_{1}} \in \mathbb{R}^{d_{2} \times d_{1}}$ is a diagonal matrix with \textbf{decreasing} non-negative singular values. We denote the $r$ singular values of $\mathbf{W}_{1}$  as $\left\{s_{k}\right\}_{k=1}^{r}$ ($r \leq R:= \min(K, d_{M}, \ldots, d_{1})$, from Lemma \ref{lm:3}). From Lemma \ref{lm:4}, we have the SVD of other weight matrices as:
    \begin{align}
    \begin{gathered}
        \mathbf{W}_{M} = \mathbf{U}_{W_{M}} \mathbf{S}_{W_{M}}
        \mathbf{U}_{W_{M-1}}^{\top},
        \\
        \mathbf{W}_{M-1} = \mathbf{U}_{W_{M-1}} \mathbf{S}_{W_{M-1}} \mathbf{U}_{W_{M-2}}^{\top},     \\
        \mathbf{W}_{M-2} = \mathbf{U}_{W_{M-2}} \mathbf{S}_{W_{M-2}} \mathbf{U}_{W_{M-3}}^{\top}, \\
        \mathbf{W}_{M-3} = \mathbf{U}_{W_{M-3}} \mathbf{S}_{W_{M-3}} \mathbf{U}_{W_{M-4}}^{\top}, \\
        \ldots \\
        \mathbf{W}_{2} = \mathbf{U}_{W_{2}} \mathbf{S}_{W_{2}} \mathbf{U}_{W_{1}}^{\top},
        \\
        \mathbf{W}_{1} = \mathbf{U}_{W_{1}} \mathbf{S}_{W_{1}} \mathbf{V}_{W_{1}}^{\top},
        \nonumber
    \end{gathered}
    \end{align}
    where:
    \begin{align}
       \mathbf{S}_{W_{j}} = \sqrt{\frac{\lambda_{W_{1}}}{\lambda_{W_{j}}}}
        \begin{bmatrix}
        \operatorname{diag}(s_{1},\ldots, s_{r}) & \mathbf{0}_{r \times (d_{j} - r)}  \\
        \mathbf{0}_{(d_{j+1} - r) \times r} & \mathbf{0}_{(d_{j+1} - r) \times (d_{j} - r)}  \\
        \end{bmatrix} \in \mathbb{R}^{d_{j+1} \times d_{j}}, \quad \forall \: j \in [M],
        \nonumber
    \end{align} 
    and $\mathbf{U}_{W_{M}}, \mathbf{U}_{W_{M-1}}, \mathbf{U}_{W_{M-2}}, \mathbf{U}_{W_{M-3}},\ldots, \mathbf{U}_{W_{1}}, \mathbf{V}_{W_{1}} $ are all orthonormal matrices.
    \\
    
    \noindent From Lemma \ref{lm:5}, denote $c := \frac{\lambda_{W_{1}}^{M-1}}
    {\lambda_{W_{M}} \lambda_{W_{M-1}} \ldots \lambda_{W_{2}} }$, we have:
    \begin{align}
    \begin{aligned}
    \mathbf{H}_{1} &= 
    \mathbf{V}_{W_{1}}
    \underbrace{
    \begin{bmatrix}
    \operatorname{diag}\left(
        \frac{\sqrt{c} s_{1}^{M} }{c s_{1}^{2M} + N \lambda_{H_{1}}}, \ldots, \frac{\sqrt{c} s_{r}^{M}}{c s_{r}^{2M} + N \lambda_{H_{1}}}
        \right) &  \mathbf{0}\\
    \mathbf{0} &  \mathbf{0}\\
    \end{bmatrix}}_{\mathbf{C} \in \mathbb{R}^{d_{1} \times K}}
    \mathbf{U}_{W_{M}}^{\top}
    \mathbf{Y} \\
    &= \mathbf{V}_{W_{1}}
    \mathbf{C}
    \mathbf{U}_{W_{M}}^{\top}
    \mathbf{Y},
    \end{aligned}
    \end{align} 
    \begin{align}
    \begin{aligned}
        \mathbf{W}_{M} \mathbf{W}_{M-1} \ldots \mathbf{W}_{2} \mathbf{W}_{1}   \mathbf{H} - \mathbf{Y} 
        &= \mathbf{U}_{W_{M}} 
        \underbrace{
        \begin{bmatrix}
        \operatorname{diag}\left(
            \frac{- N \lambda_{H_{1}} }{c s_{1}^{2M} + N \lambda_{H_{1}}}, \ldots, \frac{- N \lambda_{H_{1}}}{c s_{r}^{2M} + N \lambda_{H_{1}}}
            \right ) &  \mathbf{0}\\
        \mathbf{0} &  -\mathbf{I}_{K-r}\\
        \end{bmatrix}}_{\mathbf{D} \in \mathbb{R}^{K \times K}}
        \mathbf{U}_{W_{M}}^{\top}
        \mathbf{Y} \\
        &= \mathbf{U}_{W_{M}}
        \mathbf{D}
        \mathbf{U}_{W_{M}}^{\top}
        \mathbf{Y}.
    \end{aligned}
    \end{align}
    
    \noindent Next, we will calculate the Frobenius norm of $\mathbf{W}_{M} \mathbf{W}_{M-1} \ldots \mathbf{W}_{2} \mathbf{W}_{1} \mathbf{H} - \mathbf{Y}$:
    \begin{align}
         \| \mathbf{W}_{M} \mathbf{W}_{M-1} \ldots \mathbf{W}_{2} \mathbf{W}_{1} \mathbf{H}_{1} - \mathbf{Y}  \|_F^2
         &= \| \mathbf{U}_{W_{M}}
         \mathbf{D}
         \mathbf{U}_{W_{M}}^{\top}
         \mathbf{Y}  \|_F^2 \nonumber \\
         &= \operatorname{trace}
         (\mathbf{U}_{W_{M}}
         \mathbf{D}
         \mathbf{U}_{W_{M}}^{\top}
         \mathbf{Y} (\mathbf{U}_{W_{M}}
         \mathbf{D}
         \mathbf{U}_{W_{M}}^{\top}
         \mathbf{Y})^{\top}  ) \nonumber \\
         &= \operatorname{trace}
         (\mathbf{U}_{W_{M}}
         \mathbf{D}
         \mathbf{U}_{W_{M}}^{\top}
         \mathbf{Y} \mathbf{Y}^{\top} \mathbf{U}_{W_{M}} \mathbf{D}
         \mathbf{U}_{W_{M}}^{\top}
         ) \nonumber \\
         &=  \operatorname{trace}
         ( \mathbf{D}^{2} \mathbf{U}_{W_{M}}^{\top}  \mathbf{Y} \mathbf{Y}^{\top}   \mathbf{U}_{W_{M}} )
         \nonumber \\
         &= n \operatorname{trace}
         (\mathbf{D}^{2})
         = n \left[ \sum_{k=1}^{r} 
          \left( \frac{- N \lambda_{H_{1}} }{c s_{1}^{2M} + N \lambda_{H_{1}}} \right)^{2}
          + K - r
         \right].
        \label{eq:WH_norm_no_bias}
     \end{align}
     where we use the fact $\mathbf{Y} \mathbf{Y}^{\top} = (\mathbf{I}_{K} \otimes \mathbf{1}_{n}^{\top}) (\mathbf{I}_{K} \otimes \mathbf{1}_{n}^{\top})^{\top} = n \mathbf{I}_{K}$ and $\mathbf{U}_{W_{M}}$ is an orthonormal matrix.
     \\
     
     \noindent Similarly, for $\mathbf{H}_{1}$, we have:
     \begin{align}
         \| \mathbf{H}_{1} \|_F^2
        &= \operatorname{trace}
        ( \mathbf{V}_{W_{1}} 
        \mathbf{C} \mathbf{U}_{W_{M}}^{\top}
        \mathbf{Y} \mathbf{Y}^{\top}
        \mathbf{U}_{W_{M}} \mathbf{C}^{\top}
        \mathbf{V}_{W_{1}}^{\top}
        ) = 
        \operatorname{trace} 
        (  \mathbf{C}^{\top}  \mathbf{C} \mathbf{U}_{W_{M}}^{\top}
        \mathbf{Y} \mathbf{Y}^{\top}
        \mathbf{U}_{W_{M}}  ) \nonumber \\
        &=
        n \sum_{k=1}^{r} \frac{c s_{k}^{2M}}{c s_{k}^{2M} + N \lambda_{H_{1}}}.
        \label{eq:H_norm_no_bias}
     \end{align}
    
    Now, we plug equations \eqref{eq:WH_norm_no_bias}, \eqref{eq:H_norm_no_bias} and the SVD of weight matrices into the function $f$ and note that orthonormal matrix does not change Frobenius norm, we got:
    \begin{align}
        f\left(\mathbf{W}_{M},\ldots, \mathbf{W}_{1}, \mathbf{H}_{1} \right) 
        &=
        \frac{1}{2N} 
        \| \mathbf{W}_{M} \mathbf{W}_{M-1} \ldots \mathbf{W}_{2} \mathbf{W}_{1}   \mathbf{H} - \mathbf{Y} \|_F^2   
        + \frac{\lambda_{W_{M}}}{2} \| \mathbf{W}_{M} \|^2_F + \ldots
        + \frac{\lambda_{W_{1}}}{2} \| \mathbf{W}_{1} \|^2_F +\frac{ \lambda_{H_{1}}}{2}\left\|\mathbf{H}_{1}
        \right\|_{F}^{2} \nonumber \\
        &= 
        \frac{1}{2K} \sum_{k=1}^{r}
         \frac{(-N \lambda_{H_{1}})^{2}  }{ ( c s_{k}^{2M} + N \lambda_{H_{1}})^{2}} 
         + 
         \frac{K - r}{2K}
        + 
        \frac{\lambda_{W_{M}}}{2} \sum_{k = 1}^{r}
        \frac{\lambda_{W_{1}}}{\lambda_{W_{M}}} s_{k}^{2}
        + \frac{\lambda_{W_{M-1}}}{2} \sum_{k=1}^{r}  \frac{\lambda_{W_{1}}}{\lambda_{W_{M-1}}} s_{k}^{2}
        \nonumber \\
        &+ \ldots + \frac{\lambda_{W_{1}}}{2} \sum_{k=1}^{r} s_{k}^{2} + \frac{n \lambda_{H_{1}}}{2} 
        \sum_{k=1}^{r} 
        \frac{c s_{k}^{2M}  }{ ( cs_{k}^{2M} + N \lambda_{H_{1}})^{2}}
        \nonumber \\
        &= \frac{n \lambda_{H_{1}}}{2} \sum_{k=1}^{r}
        \frac{1}
        {c s_{k}^{2M} + N \lambda_{H_{1}}} 
        + \frac{K-r}{2K} 
        + \frac{M \lambda_{W_{1}}}{2} \sum_{k=1}^{r} s_{k}^{2} \nonumber \\
        &= \frac{1}{2K} 
        \sum_{k=1}^{r} 
        \left( 
        \frac{1}{\frac{c s_{k}^{2M}}{N \lambda_{H_{1}}} + 1} + M N \lambda_{W_{1}} \sqrt[M]{\frac{N \lambda_{H_{1}}}{c}} \left(\sqrt[M]{\frac{c s_{k}^{2M}}{N \lambda_{H_{1}}}} \right) 
        \right)
        + \frac{K-r}{2K} \nonumber \\
        &=  \frac{1}{2K} 
        \sum_{k=1}^{r} 
        \left( \frac{1}{x^{M}_{k} +1} + bx_{k} 
        \right)
        + \frac{K-r}{2K},
        \label{eq:f_no_bias}
    \end{align}
    with $x_{k} := \sqrt[M]{\frac{c s_{k}^{2M}}{N \lambda_{H_{1}}}} $ and $b:= M K \lambda_{W_{1}} \sqrt[M]{\frac{N \lambda_{H_{1}}}{c}} = M K \lambda_{W_{1}}
    \sqrt[M]{\frac{N \lambda_{W_{M}} \lambda_{W_{M-2}} \ldots \lambda_{W_{1}} \lambda_{H_{1}}}{ \lambda_{W_{1}}^{M-1}}}  =
    M K \sqrt[M]{K n  \lambda_{W_{M}} \lambda_{W_{M-1}} \ldots \lambda_{W_{1}} \lambda_{H_{1}}}$.
    \\
    
    \noindent Recall that we have studied the minimizer of function $g(x) = \frac{1}{x^{M}+1} + bx$  in Section \ref{sec:study_g}. From equation \eqref{eq:f_no_bias}, $f$ can be written as $\frac{1}{2K} \sum_{k=1}^{r} g(x_{k}) + \frac{K-r}{2N}$. By applying the result from Section \ref{sec:study_g} for each $g(x_{k})$, we finish bounding $f$ and the equality conditions are as following:
    
    \begin{itemize}
        \item If $b = MK \sqrt[M]{K n  \lambda_{W_{M}} \lambda_{W_{M-1}} \ldots \lambda_{W_{1}} \lambda_{H_{1}}} > \frac{(M-1)^{\frac{M-1}{M}}}{M}$: all the singular values of $\mathbf{W}_{1}$ are zeros. Therefore, the singular values of $\mathbf{W}_{M}, \mathbf{W}_{M-1}, \ldots, \mathbf{H}_{1}$ are also all zeros. In this case, $f(\mathbf{W}_{M}, \mathbf{W}_{M-1},\ldots, \mathbf{W}_{2}, \mathbf{W}_{1}, \mathbf{H}_{1})$ is minimized at $(\mathbf{W}_{M}^{\ast}, \mathbf{W}_{M-1}^{\ast}, \ldots, \mathbf{W}_{1}^{\ast}, \mathbf{H}_{1}^{\ast}) = (\mathbf{0}, \mathbf{0},\ldots \mathbf{0}, \mathbf{0})$.
        
        \item If $b = MK \sqrt[M]{K n  \lambda_{W_{M}} \lambda_{W_{M-1}} \ldots \lambda_{W_{1}} \lambda_{H_{1}}} < \frac{(M-1)^{\frac{M-1}{M}}}{M}$: In this case, $\mathbf{W}_{1}^{\ast}$ have $r$ singular values, all of which are equal a multiplier of the largest positive solution of the equation $b - \frac{M x^{M-1}}{(x^{M} + 1)^{2}} = 0$, we denote that singular value as $s$. Hence, we can write the compact SVD form (with a bit of notation abuse) of $\mathbf{W}_{M-1}^{\ast}$ as $\mathbf{W}_{1}^{\ast} = s \mathbf{U}_{W_{1}} \mathbf{V}_{W_{1}}^{\top}$ with semi-orthonormal matrices $\mathbf{U}_{W_{1}} \in \mathbb{R}^{d_{2} \times r}, \mathbf{V}_{W_{1}} \in \mathbb{R}^{d_{1} \times r}$. (note that $\mathbf{U}_{W_{1}}^{\top} \mathbf{U}_{W_{1}} = \mathbf{I}$ and $\mathbf{V}_{W_{1}}^{\top} \mathbf{V}_{W_{1}} = \mathbf{I}$). Since $\frac{1}{x^{* M} + 1} + bx^{*} < 1$, we have $r = R = \min(K, d_{M}, \ldots, d_{1})$ in this case.
        \\
            
        Similarly, we also have the compact SVD form of other weight matrices and feature matrix as:
            \begin{align}
                \mathbf{W}_{M}^{\ast}  &= \sqrt{\frac{\lambda_{W_{1}}}{\lambda_{W_{M}}}} s \mathbf{U}_{W_{M}} \mathbf{U}_{W_{M-1}}^{T}, \nonumber \\
                \mathbf{W}_{M-1}^{\ast} &= \sqrt{\frac{\lambda_{W_{1}}}{\lambda_{W_{M-1}}}} s 
                \mathbf{U}_{W_{M-1}} \mathbf{U}_{W_{M-2}}^{\top}, \nonumber \\
                \ldots \nonumber \\
                \mathbf{W}_{1}^{\ast} &= s \mathbf{U}_{W_{1}} \mathbf{V}_{W_{1}}^{\top}, \nonumber \\
                \mathbf{H}_{1}^{\ast} &= \frac{\sqrt{c} s^{M}}{c s^{2M} + N \lambda_{H_{1}}} \mathbf{V}_{W_{1}} \mathbf{U}_{W_{M}}^{\top}
                \mathbf{Y} \quad (\text{from equation } \eqref{eq:H_norm_no_bias}),
                \nonumber
            \end{align}
            with semi-orthonormal matrices $\mathbf{U}_{W_{M}}, \mathbf{U}_{W_{M-1}}, \mathbf{U}_{W_{M-2}}, \ldots, \mathbf{U}_{W_{1}}, \mathbf{V}_{W_{1}}$ that each has $R$ orthogonal columns, i.e. $\mathbf{U}_{W_{M}}^{\top} \mathbf{U}_{W_{M}} = \mathbf{U}_{W_{M-1}}^{\top} \mathbf{U}_{W_{M-1}} =  \ldots = \mathbf{U}_{W_{1}}^{\top} \mathbf{U}_{W_{1}} = \mathbf{V}_{W_{1}}^{\top} \mathbf{V}_{W_{1}} = \mathbf{I}_{R}$. Furthermore, $\mathbf{U}_{W_{M}}, \mathbf{U}_{W_{M-1}}, \ldots, \mathbf{U}_{W_{1}}, \mathbf{V}_{W_{1}}$ are truncated matrices from orthonormal matrices (remove columns that do not correspond with non-zero singular values), hence $\mathbf{U}_{W_{M}} \mathbf{U}_{W_{M}}^{\top}, \mathbf{U}_{W_{M-1}} \mathbf{U}_{W_{M-1}}^{\top}, \ldots,  \mathbf{U}_{W_{1}} \mathbf{U}_{W_{1}}^{\top}, \mathbf{V}_{W_{1}} \mathbf{V}_{W_{1}}^{\top} $ are the best rank-$R$ approximations of the identity matrix of the same size.
            \\
            
            Let $\overline{\mathbf{H}}^{*} = \frac{\sqrt{c} s^{M}}{c s^{2M} + N \lambda_{H_{1}}} \mathbf{V}_{W_{1}} \mathbf{U}_{W_{M}}^{\top} \in \mathbb{R}^{d_{1} \times K}$, then we have $(\mathcal{NC}1)$ $\mathbf{H}_{1}^{*} = \overline{\mathbf{H}}^{*} \mathbf{Y} = \overline{\mathbf{H}}^{*} \otimes \mathbf{1}_{n}^{\top}$, thus we conclude the features within the same class collapse to their class-mean and $\overline{\mathbf{H}}^{*}$ is the class-means matrix.
            \\
            
            From above arguments, we can deduce the geometry of the following $(\mathcal{NC}2)$:
            \begin{align}
            \begin{gathered}
                 \mathbf{W}_{M}^{\ast} \mathbf{W}_{M}^{\top \ast} \propto 
                 \mathbf{U}_{W_{M}} \mathbf{U}_{W_{M}}^{\top} \propto \mathcal{P}_{R}(\mathbf{I}_{K}), \\
                \overline{\mathbf{H}}^{\ast \top} \overline{\mathbf{H}}^{\ast} \propto 
                 \mathbf{U}_{W_{M}} \mathbf{U}_{W_{M}}^{\top} \propto \mathcal{P}_{R}(\mathbf{I}_{K}),
                \\
                \mathbf{W}_{M}^{\ast} \mathbf{W}_{M-1}^{\ast}  \mathbf{W}_{M-2}^{\ast} 
                \ldots \mathbf{W}_{2}^{\ast} 
                \mathbf{W}_{1}^{\ast} \overline{\mathbf{H}}^{*} \propto 
                 \mathbf{U}_{W_{M}} \mathbf{U}_{W_{M}}^{\top} \propto \mathcal{P}_{R}(\mathbf{I}_{K}),  \\
                (\mathbf{W}_{M}^{\ast} \mathbf{W}_{M-1}^{\ast} \ldots \mathbf{W}_{j}^{\ast})(\mathbf{W}_{M}^{\ast} \mathbf{W}_{M-1}^{\ast} \ldots \mathbf{W}_{j}^{\ast})^{\top} \propto 
                 \mathbf{U}_{W_{M}} \mathbf{U}_{W_{M}}^{\top} \propto \mathcal{P}_{R}(\mathbf{I}_{K}), \quad \forall \: j \in [M].
            \end{gathered}
            \end{align}
            Note that if $R = K$, we have $\mathcal{P}_{R}(\mathbf{I}_{K}) = \mathbf{I}_{K}$.
            \\
            
            Also, the product of each weight matrix or features with its transpose will be the multiplier of one of the best rank-$r$ approximations of the identity matrix of the same size. For example, $\mathbf{W}_{M-1}^{\ast \top} \mathbf{W}_{M-1}^{\ast} \propto \mathbf{U}_{W_{M-2}} \mathbf{U}_{W_{M-2}}^{\top}$ and $\mathbf{W}_{M-1}^{\ast} \mathbf{W}_{M-1}^{\ast \top} \propto \mathbf{U}_{W_{M-1}} \mathbf{U}_{W_{M-1}}^{\top}$ are two best rank-$R$ approximations of $\mathbf{I}_{d_{M-1}}$ and $\mathbf{I}_{d_{M}}$, respectively.
            \\

            Next, we can derive the alignments between weights and features as following $(\mathcal{NC}3)$:
            \begin{align}
            \begin{gathered}
                \mathbf{W}_{M}^{\ast} \mathbf{W}_{M-1}^{\ast} \ldots \mathbf{W}_{1}^{\ast} \propto \mathbf{U}_{W_{M}} \mathbf{V}_{W_{1}}^{\top} \propto \overline{\mathbf{H}}^{* \top}, \\
                \mathbf{W}_{M-1}^{\ast} \mathbf{W}_{M-2}^{\ast} \ldots \mathbf{W}_{1}^{\ast} \overline{\mathbf{H}}^{*} \propto \mathbf{U}_{W_{M-1}}   \mathbf{U}_{W_{M}}^{\top} \propto  \mathbf{W}_{M}^{\ast \top}, \\
                \mathbf{W}_{M}^{*} \mathbf{W}_{M-1}^{*} \ldots \mathbf{W}_{j}^{*} \propto  
                \mathbf{U}_{W_{M}} \mathbf{U}_{W_{j-1}}^{\top} \propto  (\mathbf{W}_{j-1}^{*} \ldots \mathbf{W}_{1}^{*} \overline{\mathbf{H}}^{*})^{\top}.
            \end{gathered}
            \end{align}
            
            \item If $b = MK \sqrt[M]{K n  \lambda_{W_{M}} \lambda_{W_{M-1}} \ldots \lambda_{W_{1}} \lambda_{H_{1}}} = \frac{(M-1)^{\frac{M-1}{M}}}{M}$:
            In this case, $x^{*}_{k}$ can either be $0$ or the largest positive solution of the equation $b - \frac{M x^{M-1}}{(x^{M} + 1)^{2}} = 0$. If all the singular values are $0$'s, we have the trivial global minima $(\mathbf{W}_{M}^{\ast}, \ldots, \mathbf{W}_{1}^{\ast}, \mathbf{H}_{1}^{\ast}) = (\mathbf{0},\ldots, \mathbf{0}, \mathbf{0})$.
            \\
            
            If there are exactly $0 < r \leq R$ positive singular values  $s_{1} = s_{2} = \ldots = s_{r} := s > 0$ and $s_{r+1} = \ldots = s_{R} = 0$, then similar as the case $b < \frac{(M-1)^{\frac{M-1}{M}}}{M}$, we also have similar compact SVD form (with exactly $r$ singular vectors, instead of $R$ as the above case). Thus, the nontrivial solutions exhibit  $(\mathcal{NC}1)$ and $(\mathcal{NC}3)$ property similarly as the case $b < \frac{(M-1)^{\frac{M-1}{M}}}{M}$ above.
            \\
            
            For $(\mathcal{NC}2)$ property, for $j = 1, \ldots, M$, we have:
            \begin{align}
                \begin{gathered}
                   \mathbf{W}_{M}^{\ast} \mathbf{W}_{M}^{\ast \top} \propto 
                   \overline{\mathbf{H}}^{\ast \top} \overline{\mathbf{H}}^{\ast}
                    \propto
                   \mathbf{W}_{M}^{\ast} \mathbf{W}_{M-1}^{\ast}  \mathbf{W}_{M-2}^{\ast} 
                    \ldots \mathbf{W}_{2}^{\ast} 
                   \mathbf{W}_{1}^{\ast} \overline{\mathbf{H}}^{*}
                    \\ \propto 
                   (\mathbf{W}_{M}^{\ast} \mathbf{W}_{M-1}^{\ast} \ldots \mathbf{W}_{j}^{\ast})(\mathbf{W}_{M}^{\ast} \mathbf{W}_{M-1}^{\ast} \ldots \mathbf{W}_{j}^{\ast})^{\top}
                    \propto \mathcal{P}_{r}(\mathbf{I}_{K}).
                   \nonumber
                \end{gathered}
                \end{align}
            \end{itemize}
    We finish the proof of Theorem \ref{thm:bias-free} for bias-free case.
    \end{proof}

\subsection{Full Proof of Theorem \ref{thm:bias-free} with Last-layer Unregularized Bias}
\label{sec:proofs_balanced_bias}
Now, we state the proof of Theorem \ref{thm:bias-free} for general setting with $M$ layers of weight with last-layer bias (i.e., including $\mathbf{b}$) with arbitrary widths $d_{M}, d_{M-1},\ldots, d_{1}$.

\begin{proof}[Proof of Theorem \ref{thm:bias-free} (last-layer bias)]
    First, we have that the objective function $f$ is convex w.r.t $\mathbf{b}$. Hence, we can derive the optimal $\mathbf{b}^{*}$ through its derivative w.r.t $\mathbf{b}$ (note that $N = K n$):
        \begin{align}
            &\frac{1}{N} (\mathbf{W}_{M} \mathbf{W}_{M-1} \ldots \mathbf{W}_{2} \mathbf{W}_{1} \mathbf{H}_{1} + \mathbf{b}^{*} \mathbf{1}_{N}^{\top} - \mathbf{Y}) \mathbf{1}_{N} = \mathbf{0} \nonumber \\
            \Rightarrow \: &\mathbf{b}^{*} = \frac{1}{N} ( \mathbf{Y} - \mathbf{W}_{M} \mathbf{W}_{M-1} \ldots \mathbf{W}_{2} \mathbf{W}_{1}   \mathbf{H}_{1}) \mathbf{1}_{N}
            = \frac{1}{N} \sum_{k=1}^{K} \sum_{i=1}^{n} (\mathbf{y}_{k} - \mathbf{W}_{M} \mathbf{W}_{M-1} \ldots \mathbf{W}_{2} \mathbf{W}_{1}   \mathbf{h}_{k,i}).
        \end{align} 
        
        \noindent Since $\{ \mathbf{y}_{k} \}$ are one-hot vectors, we have:
        \begin{align}
            \mathbf{b}^{*}_{k^{\prime}} =
            \frac{n}{N} - \frac{1}{N}
            \sum_{k=1}^{K} \sum_{i=1}^{n} 
            (\mathbf{W}_{M} \mathbf{W}_{M-1} \ldots \mathbf{W}_{2} \mathbf{W}_{1}  )_{k^{\prime}}^{\top} \mathbf{h}_{k,i} = \frac{1}{K} - (\mathbf{W}_{M} \mathbf{W}_{M-1} \ldots \mathbf{W}_{2} \mathbf{W}_{1}  )_{k^{\prime}}^{\top} \mathbf{h_{G}},
            \label{eq:bias_form}
        \end{align}
        where $\mathbf{h}_{G}:= \frac{1}{N} \sum_{k=1}^{K} \sum_{i=1}^{n}  \mathbf{h}_{k,i}$ is the features' global-mean and $(\mathbf{W}_{M} \mathbf{W}_{M-1} \ldots \mathbf{W}_{2} \mathbf{W}_{1}  )_{k^{\prime}}$ is $k^{\prime}$-th row of $\mathbf{W}_{M} \mathbf{W}_{M-1} \ldots \mathbf{W}_{2} \mathbf{W}_{1}  $.

    Next, we plug $\mathbf{b}^{*}$ into $f$:
    \allowdisplaybreaks
    \begin{align}
    \begin{aligned}
    f &= \frac{1}{2Kn} \| \mathbf{W}_{M} \mathbf{W}_{M-1} \ldots \mathbf{W}_{2} \mathbf{W}_{1}  
    \mathbf{H}_{1} + \mathbf{b}^{*} \mathbf{1}_{N}^{\top} - \mathbf{Y} \|_F^2 + \frac{\lambda_{W_{M}}}{2} \| \mathbf{W}_{M} \|^2_F +  \ldots +   \frac{\lambda_{W_{2}}}{2} \| \mathbf{W}_{2} \|^2_F +
    \frac{\lambda_{W_{1}}}{2} \| \mathbf{W}_{1} \|^2_F \\
    &+ \frac{\lambda_{H_{1}}}{2} \| \mathbf{H}_{1} \|^2_F \\
    &= \frac{1}{2 K n} \sum_{k=1}^{K} \sum_{i=1}^{n}
    \| \mathbf{W}_{M} \mathbf{W}_{M-1} \ldots \mathbf{W}_{2} \mathbf{W}_{1} \mathbf{h}_{k,i} + \mathbf{b}^{*} - \mathbf{y}_{k} \|_{2}^{2}  + \frac{\lambda_{W_{M}}}{2} \| \mathbf{W}_{M} \|^2_F +   \ldots +   \frac{\lambda_{W_{2}}}{2} \| \mathbf{W}_{2} \|^2_F 
    + \frac{\lambda_{W_{1}}}{2} \| \mathbf{W}_{1} \|^2_F \\
    &+
    \sum_{k=1}^{K} \sum_{i=1}^{n} \| \mathbf{h}_{k,i} \|_{2}^{2} \\
    &= 
    \frac{1}{2 K n} \sum_{k=1}^{K} \sum_{i=1}^{n} \sum_{k^{\prime}=1}^{K}
    \left( (\mathbf{W}_{M} \mathbf{W}_{M-1} \ldots \mathbf{W}_{2} \mathbf{W}_{1})_{k^{\prime}}^{\top} (\mathbf{h}_{k,i} - \mathbf{h}_{G}) + \frac{1}{K} - \mathbf{1}_{k = k^{\prime}} 
    \right)^{2}
    + \frac{\lambda_{W_{M}}}{2} \| \mathbf{W}_{M} \|^2_F +  \ldots 
    \\
    &+ \frac{\lambda_{W_{1}}}{2} \| \mathbf{W}_{1} \|^2_F 
    +
    \sum_{k=1}^{K} \sum_{i=1}^{n} \| \mathbf{h}_{k,i} \|_{2}^{2}
    \end{aligned}
    \nonumber
    \end{align}
    
    \begin{align}
    \begin{aligned}
    &\geq 
    \frac{1}{2 K n} \sum_{k=1}^{K} \sum_{i=1}^{n} \sum_{k^{\prime}=1}^{K}
    \left( (\mathbf{W}_{M} \mathbf{W}_{M-1} \ldots \mathbf{W}_{2} \mathbf{W}_{1})_{k^{\prime}}^{\top} (\mathbf{h}_{k,i} - \mathbf{h}_{G}) + \frac{1}{K} - \mathbf{1}_{k = k^{\prime}} 
    \right)^{2}
    + \frac{\lambda_{W_{M}}}{2} \| \mathbf{W}_{M} \|^2_F +  \ldots 
    \\
    &+ \frac{\lambda_{W_{1}}}{2} \| \mathbf{W}_{1} \|^2_F 
    +
    \sum_{k=1}^{K} \sum_{i=1}^{n} \| \mathbf{h}_{k,i} - \mathbf{h}_{G} \|_{2}^{2} \\
    &= 
    \frac{1}{2Kn} \| \mathbf{W}_{M} \mathbf{W}_{M-1} \ldots \mathbf{W}_{2} \mathbf{W}_{1}  
    \mathbf{H}_{1}^{'}  - (\mathbf{Y} - \frac{1}{K} \mathbf{1}_{K} \mathbf{1}_{N}^{\top}) \|_F^2 + \frac{\lambda_{W_{M}}}{2} \| \mathbf{W}_{M} \|^2_F +  \ldots +   \frac{\lambda_{W_{2}}}{2} \| \mathbf{W}_{2} \|^2_F 
    \\
    &+ \frac{\lambda_{W_{1}}}{2} \| \mathbf{W}_{1} \|^2_F +  \frac{\lambda_{H_{1}}}{2} \| \mathbf{H}_{1}^{'} \|^2_F := f^{'} (\mathbf{W}_{M}, \mathbf{W}_{M-1},\ldots, \mathbf{W}_{2}, \mathbf{W}_{1}, \mathbf{H}_{1}^{'}),
    \end{aligned}
    \nonumber
    \end{align}
    where $\mathbf{H}_{1}^{'} = [\mathbf{h}_{1,1} - \mathbf{h}_{G}, \ldots, \mathbf{h}_{K,n} - \mathbf{h}_{G}] \in \mathbb{R}^{d \times N}$ and the inequality is from:
    \begin{align}
        \sum_{k=1}^{K} \sum_{i=1}^{n} \| \mathbf{h}_{k,i} \|_{2}^{2} &= \sum_{k=1}^{K} \sum_{i=1}^{n} 
        \left( \| \mathbf{h}_{k,i} - \mathbf{h}_{G} \|_{2}^{2} 
        + 2 (\mathbf{h}_{k,i} - \mathbf{h}_{G})^{\top} \mathbf{h}_{G} + \|  \mathbf{h}_{G} \|_{2}^{2} \right) \nonumber \\
        &= \sum_{k=1}^{K} \sum_{i=1}^{n} \| \mathbf{h}_{k,i} - \mathbf{h}_{G} \|_{2}^{2} + N \| \mathbf{h}_{G} \|_2^2 \nonumber \\
        &\geq \sum_{k=1}^{K} \sum_{i=1}^{n} \| \mathbf{h}_{k,i} - \mathbf{h}_{G} \|_{2}^{2}
        \label{eq:hg=0},
    \end{align}
    where the equality happens when $\mathbf{h}_{G} = 0$.
    \\
    
    \noindent Noting that $f^{'}$ has similar form as function $f$ for bias-free case (except the difference of the target matrix $\mathbf{Y}$), we can use the lemmas derived at Section \ref{sec:lemmas} for $f^{'}$. First, by using Lemma \ref{lm:2}, we have for any critical point $(\mathbf{W}_{M}, \mathbf{W}_{M-1},\ldots, \mathbf{W}_{2}, \mathbf{W}_{1}, \mathbf{H}_{1}^{'})$ of $f^{'}$, we have the following:
    \begin{align}
    \begin{gathered}
        \lambda_{W_{M}} \mathbf{W}^{\top}_{M} \mathbf{W}_{M} = \lambda_{W_{M-1}} \mathbf{W}_{M-1} \mathbf{W}^{\top}_{M-1}, \\
        \lambda_{W_{M-1}} \mathbf{W}^{\top}_{M-1} \mathbf{W}_{M-1} = \lambda_{W_{M-2}} \mathbf{W}_{M-2} \mathbf{W}^{\top}_{M-2}, \\
        \ldots, \nonumber \\
        \lambda_{W_{2}} \mathbf{W}_{2}^{\top} \mathbf{W}_{2} = \lambda_{W_{1}} \mathbf{W}_{1} \mathbf{W}_{1}^{\top}, \\
        \lambda_{W_{1}} \mathbf{W}_{1}^{\top} \mathbf{W}_{1} =  \lambda_{H_{1}}\mathbf{H}_{1}^{'} \mathbf{H}^{' \top}_{1}.
    \end{gathered}
    \end{align}
    
    \noindent Let $\mathbf{W}_{1} = \mathbf{U}_{W_{1}} \mathbf{S}_{W_{1}} \mathbf{V}_{W_{1}}^{\top}$ be the SVD decomposition of $\mathbf{W}_{1}$ with $\mathbf{U}_{W_{1}} \in \mathbb{R}^{d_{2} \times d_{2}}, \mathbf{V}_{W_{1}} \in \mathbb{R}^{d_{1} \times d_{1}}$ are orthonormal matrices and $\mathbf{S}_{W_{1}} \in \mathbb{R}^{d_{2} \times d_{1}}$ is a diagonal matrix with \textbf{decreasing} non-negative singular values. We denote the $r$ singular values of $\mathbf{W}_{1}$ as $\left\{s_{k}\right\}_{k=1}^{r}$
    ($r \leq R:= \min(K, d_{M}, \ldots, d_{1})$, from Lemma \ref{lm:3})
    . From Lemma \ref{lm:4}, we have the SVD of other weight matrices as:
\begin{align}
    \begin{gathered}
        \mathbf{W}_{M} = \mathbf{U}_{W_{M}} \mathbf{S}_{W_{M}}
        \mathbf{U}_{W_{M-1}}^{\top},
        \\
        \mathbf{W}_{M-1} = \mathbf{U}_{W_{M-1}} \mathbf{S}_{W_{M-1}} \mathbf{U}_{W_{M-2}}^{\top},     \\
        \mathbf{W}_{M-2} = \mathbf{U}_{W_{M-2}} \mathbf{S}_{W_{M-2}} \mathbf{U}_{W_{M-3}}^{\top}, \\
        \mathbf{W}_{M-3} = \mathbf{U}_{W_{M-3}} \mathbf{S}_{W_{M-3}} \mathbf{U}_{W_{M-4}}^{\top}, \\
        \ldots, \\
        \mathbf{W}_{2} = \mathbf{U}_{W_{2}} \mathbf{S}_{W_{2}} \mathbf{U}_{W_{1}}^{\top},
        \\
        \mathbf{W}_{1} = \mathbf{U}_{W_{1}} \mathbf{S}_{W_{1}} \mathbf{V}_{W_{1}}^{\top} ,
        \nonumber
    \end{gathered}
    \end{align}
    where:
    \begin{align}
       \mathbf{S}_{W_{j}} = \sqrt{\frac{\lambda_{W_{1}}}{\lambda_{W_{j}}}}
        \begin{bmatrix}
        \operatorname{diag}(s_{1},\ldots, s_{r}) & \mathbf{0}_{r \times (d_{j} - r)}  \\
        \mathbf{0}_{(d_{j+1} - r) \times r} & \mathbf{0}_{(d_{j+1} - r) \times (d_{j} - r)}  \\
        \end{bmatrix} \in \mathbb{R}^{d_{j+1} \times d_{j}}, \quad \forall \: j \in [M],
        \nonumber
    \end{align} 
    and $\mathbf{U}_{W_{M}}, \mathbf{U}_{W_{M-1}}, \mathbf{U}_{W_{M-2}}, \mathbf{U}_{W_{M-3}},\ldots, \mathbf{U}_{W_{1}}, \mathbf{V}_{W_{1}} $ are all orthonormal matrices.
    \\
    
    \noindent From Lemma \ref{lm:5}, denote $c := \frac{\lambda_{W_{1}}^{M-1}}
    {\lambda_{W_{M}} \lambda_{W_{M-1}} \ldots \lambda_{W_{2}} }$, we have:
    \begin{align}
    \begin{aligned}
    \mathbf{H}_{1}^{'} &= 
    \mathbf{V}_{W_{1}}
    \underbrace{
    \begin{bmatrix}
    \operatorname{diag}\left(
        \frac{\sqrt{c} s_{1}^{M} }{c s_{1}^{2M} + N \lambda_{H_{1}}}, \ldots, \frac{\sqrt{c} s_{r}^{M}}{c s_{r}^{2M} + N \lambda_{H_{1}}}
        \right) &  \mathbf{0}\\
    \mathbf{0} &  \mathbf{0}\\
    \end{bmatrix}}_{\mathbf{C} \in \mathbb{R}^{d_{1} \times K}}
    \mathbf{U}_{W_{M}}^{\top}
    \left( \mathbf{Y} - \frac{1}{K} \mathbf{1}_{K} \mathbf{1}_{N}^{\top} \right) \\
    &= \mathbf{V}_{W_{1}}
    \mathbf{C}
    \mathbf{U}_{W_{M}}^{\top}
    \left( \mathbf{Y} - \frac{1}{K} \mathbf{1}_{K} \mathbf{1}_{N}^{\top} \right).
    \label{eq:no_bias_H_form}
    \end{aligned}
    \end{align} 
    
    \begin{align}
    \begin{aligned}
        &\mathbf{W}_{M} \mathbf{W}_{M-1} \ldots \mathbf{W}_{2} \mathbf{W}_{1} \mathbf{H}^{'}_{1} - \mathbf{Y} \nonumber \\
        &= \mathbf{U}_{W_{M}} 
        \underbrace{
        \begin{bmatrix}
        \operatorname{diag}\left(
            \frac{- N \lambda_{H_{1}} }{c s_{1}^{2M} + N \lambda_{H_{1}}}, \ldots, \frac{- N \lambda_{H_{1}}}{c s_{r}^{2M} + N \lambda_{H_{1}}}
            \right ) &  \mathbf{0}\\
        \mathbf{0} &  -\mathbf{I}_{K-r}\\
        \end{bmatrix}}_{\mathbf{D} \in \mathbb{R}^{K \times K}}
        \mathbf{U}_{W_{M}}^{\top}
        \left( \mathbf{Y} - \frac{1}{K} \mathbf{1}_{K} \mathbf{1}_{N}^{\top} \right) \\
        &= \mathbf{U}_{W_{M}}
        \mathbf{D}
        \mathbf{U}_{W_{M}}^{\top}
        \left( \mathbf{Y} - \frac{1}{K} \mathbf{1}_{K} \mathbf{1}_{N}^{\top} \right).
    \end{aligned}
    \end{align}
    
    \noindent Next, we will calculate the Frobenius norm of $\mathbf{W}_{M} \mathbf{W}_{M-1} \ldots \mathbf{W}_{2} \mathbf{W}_{1} \mathbf{H}^{'}_{1} - \mathbf{Y}$:
    \begin{align}
         &\| \mathbf{W}_{M} \mathbf{W}_{M-1} \ldots \mathbf{W}_{2} \mathbf{W}_{1} \mathbf{H}_{1}^{'} - \mathbf{Y}  \|_F^2
         = \left\| \mathbf{U}_{W_{M}}
         \mathbf{D}
         \mathbf{U}_{W_{M}}^{\top}
         \left( \mathbf{Y} - \frac{1}{K} \mathbf{1}_{K} \mathbf{1}_{N}^{\top} \right)  \right \|_F^2 \nonumber \\
         &= \operatorname{trace}
         \left(\mathbf{U}_{W_{M}}
         \mathbf{D}
         \mathbf{U}_{W_{M}}^{\top}
         \left( \mathbf{Y} - \frac{1}{K} \mathbf{1}_{K} \mathbf{1}_{N}^{\top} \right) \left(\mathbf{U}_{W_{M}}
         \mathbf{D}
         \mathbf{U}_{W_{M}}^{\top}
         \left( \mathbf{Y} - \frac{1}{K} \mathbf{1}_{K} \mathbf{1}_{N}^{\top} \right) \right)^{\top} \right) \nonumber \\
         &= \operatorname{trace}
         \left(\mathbf{U}_{W_{M}}
         \mathbf{D}
         \mathbf{U}_{W_{M}}^{\top}
         \left( \mathbf{Y} - \frac{1}{K} \mathbf{1}_{K} \mathbf{1}_{N}^{\top} \right) \left( \mathbf{Y} - \frac{1}{K} \mathbf{1}_{K} \mathbf{1}_{N}^{\top} \right)^{\top} \mathbf{U}_{W_{M}} \mathbf{D}
         \mathbf{U}_{W_{M}}^{\top}
         \right)          \nonumber \\
         &= \operatorname{trace}
         \left( \mathbf{D}^{2} \mathbf{U}_{W_{M}}^{\top}  \left( \mathbf{Y} - \frac{1}{K} \mathbf{1}_{K} \mathbf{1}_{N}^{\top} \right) \left( \mathbf{Y} - \frac{1}{K} \mathbf{1}_{K} \mathbf{1}_{N}^{\top} \right)^{\top}   \mathbf{U}_{W_{M}} \right).
         \label{eq:bias_WH_form}
    \end{align}
    
    \noindent Note that:
    \begin{align}
        \mathbf{Y} - \frac{1}{K} \mathbf{1}_{K} \mathbf{1}_{N}^{\top} &= \left( \mathbf{I}_{K} - \frac{1}{K} \mathbf{1}_{K} \mathbf{1}_{K}^{\top} \right) \otimes \mathbf{1}_{n}^{\top}, \nonumber \\
        \left( \mathbf{Y} - \frac{1}{K} \mathbf{1}_{K} \mathbf{1}_{N}^{\top} \right) \left( \mathbf{Y} - \frac{1}{K} \mathbf{1}_{K} \mathbf{1}_{N}^{\top} \right)^{\top}  
        &= \left( \left( \mathbf{I}_{K} - \frac{1}{K} \mathbf{1}_{K} \mathbf{1}_{K}^{\top} \right) \otimes \mathbf{1}_{n}^{\top} \right) \left( \left( \mathbf{I}_{K} - \frac{1}{K} \mathbf{1}_{K} \mathbf{1}_{K}^{\top} \right) \otimes \mathbf{1}_{n}^{\top} \right)^{\top}
        \nonumber \\
        &= \left( \left( \mathbf{I}_{K} - \frac{1}{K} \mathbf{1}_{K} \mathbf{1}_{K}^{\top} \right) \otimes \mathbf{1}_{n}^{\top} \right) \left( \left( \mathbf{I}_{K} - \frac{1}{K} \mathbf{1}_{K} \mathbf{1}_{K}^{\top} \right) \otimes \mathbf{1}_{n} \right)
        \nonumber \\
        &= \left( \left( \mathbf{I}_{K} - \frac{1}{K} \mathbf{1}_{K} \mathbf{1}_{K}^{\top} \right) \left( \mathbf{I}_{K} - \frac{1}{K} \mathbf{1}_{K} \mathbf{1}_{K}^{\top} \right)  \right) \otimes \left( \mathbf{1}_{n}^{\top} \mathbf{1}_{n}  \right) \nonumber \\
        &= n \left( \mathbf{I}_{K} - \frac{1}{K} \mathbf{1}_{K} \mathbf{1}_{K}^{\top} \right), \nonumber
    \end{align}
    since $\mathbf{I}_{K} - \frac{1}{K} \mathbf{1}_{K} \mathbf{1}_{K}^{\top}$ is an idempotent matrix.
    \\
    
    \noindent Next, we have:
    \begin{align}
        \mathbf{U}_{W_{M}}^{\top}  \left( \mathbf{Y} - \frac{1}{K} \mathbf{1}_{K} \mathbf{1}_{N}^{\top} \right) \left( \mathbf{Y} - \frac{1}{K} \mathbf{1}_{K} \mathbf{1}_{N}^{\top} \right)^{\top}   \mathbf{U}_{W_{M}} 
        &= 
        n \mathbf{U}_{W_{M}}^{\top}
        \left( \mathbf{I}_{K} - \frac{1}{K} \mathbf{1}_{K} \mathbf{1}_{K}^{\top} \right)
        \mathbf{U}_{W_{M}} \nonumber \\
        &= n \left( \mathbf{I}_{K} - \frac{1}{K}  \mathbf{U}_{W_{M}}^{\top}
        \mathbf{1}_{K} \mathbf{1}_{K}^{\top}
        \mathbf{U}_{W_{M}}
        \right). \nonumber
    \end{align}
    
    \noindent We denote $\mathbf{q} = \mathbf{U}_{W_{M}}^{\top}
    \mathbf{1}_{K} 
    = [q_{1},\ldots, q_{K}]^{\top} \in \mathbb{R}^{K}$, then $q_{k}$ will equal the sum of entries of the $k$-th column of $\mathbf{U}_{W_{M}}$. Hence, $\mathbf{U}_{W_{M}}^{\top}
    \mathbf{1}_{K} \mathbf{1}_{K}^{\top}
    \mathbf{U}_{W_{M}} = \mathbf{q} \mathbf{q}^{\top} = (q_{i} q_{j})_{i,j}$. Note that from the orthonormality of $\mathbf{U}_{W_{M}}$, we can deduce $\sum_{k=1}^{K} q_{k}^{2} = K$. Thus, continue from equation \eqref{eq:bias_WH_form}:
    \begin{align}
         \| \mathbf{W}_{M} \mathbf{W}_{M-1} \ldots \mathbf{W}_{2} \mathbf{W}_{1} \mathbf{H}_{1}^{'} - \mathbf{Y}  \|_F^2
         &= n
         \operatorname{trace}
         \left(
         \mathbf{D}^{2}
         \left(\mathbf{I}_{K} - \frac{1}{K} \mathbf{q} \mathbf{q}^{\top}\right)
         \right) \nonumber \\
         &= n \left( \sum_{k=1}^{r}
         \left (1 - \frac{1}{K} q_{k}^{2} \right)
         \frac{(-N \lambda_{H_{1}})^{2}}{(c s_{k}^{2M} + N \lambda_{H_{1}})^{2}}
        + \sum_{h= r +1 }^{K}
        \left (1 - \frac{1}{K} q_{h}^{2} \right)
         \right).
         \label{eq:bias_WH_norm}
    \end{align}
    
    \noindent Similarly, we calculate the Frobenius norm for $\mathbf{H}_{1}^{'}$, continue from the RHS of equation \eqref{eq:no_bias_H_form}:
    \begin{align}
        \| \mathbf{H}_{1}^{'} \|_F^2
        &= \operatorname{trace}
        \left( \mathbf{V}_{W_{1}} 
        \mathbf{C} \mathbf{U}_{W_{M}}^{\top}
        \left( \mathbf{Y} - \frac{1}{K} \mathbf{1}_{K} \mathbf{1}_{N}^{\top} \right)
        \left( \mathbf{Y} - \frac{1}{K} \mathbf{1}_{K} \mathbf{1}_{N}^{\top} \right)^{\top}
        \mathbf{U}_{W_{M}} \mathbf{C}^{\top}
        \mathbf{V}_{W_{1}}^{\top}
        \right) \nonumber \\
        &= n \operatorname{trace}
        \left(
        \mathbf{C}^{\top} \mathbf{C}
        \left( \mathbf{I}_{K} - \frac{1}{K} \mathbf{q} \mathbf{q}^{\top}
        \right)
        \right) \nonumber \\
        &= n \sum_{k=1}^{r}
        \left( 1 - \frac{1}{K} q_{k}^{2} 
        \right)
        \frac{c s_{k}^{2M}}{(c s_{k}^{2M} + N \lambda_{H_{1}})^{2}}.
        \label{eq:bias_H_norm}
    \end{align}
    
    \noindent Plug the equations \eqref{eq:bias_WH_norm}, \eqref{eq:bias_H_norm} and the SVD of weight matrices into $f^{'}$ yields:
    \begin{align}
        &\frac{1}{2 Kn}\left\|\mathbf{W}_{M} \mathbf{W}_{M-1} \ldots \mathbf{W}_{1}   \mathbf{H}_{1}^{'} 
        - (\mathbf{Y} - \frac{1}{K} \mathbf{1}_{K} \mathbf{1}_{N}^{T}) \right\|_{F}^{2} + \frac{\lambda_{W_{M}}}{2} \| \mathbf{W}_{M} \|^2_F + \ldots \frac{\lambda_{W_{1}}}{2} \| \mathbf{W}_{1} \|^2_F +  \frac{\lambda_{H_{1}}}{2} \| \mathbf{H}_{1}^{'} \|^2_F \nonumber \\
        &= \frac{1}{2K} \sum_{k=1}^{r} \left(1 - \frac{1}{K} q_{k}^{2} \right) 
        \left( \frac{-N \lambda_{H_{1}} } {c
        s_{k}^{2M}  + N \lambda_{H_{1}}} \right)^{2} 
        + 
        \frac{1}{2K}
        \sum_{h=r+1}^{K} \left(1 - \frac{1}{K} q_{h}^{2} \right)
        + \frac{\lambda_{W_{M}}}{2}
        \sum_{k=1}^{r} \frac{\lambda_{W_{1}}}{\lambda_{W_{M}}} s_{k}^{2} \nonumber \\ &
        +
        \frac{\lambda_{W_{M-1}}}{2} \sum_{k=1}^{r} 
        \frac{\lambda_{W_{1}}}{\lambda_{W_{M-1}}} s_{k}^{2}
        + \ldots + \frac{\lambda_{W_{1}}}{2} \sum_{k=1}^{r}  s_{k}^{2} 
        +
        \frac{n \lambda_{H_{1}}}{2} \sum_{k=1}^{r} 
        \left( 1 - \frac{1}{K} q_{k}^{2} 
        \right)
        \frac{c s_{k}^{2M}}{(c s_{k}^{2M} + N \lambda_{H_{1}})^{2}}
        \nonumber \\
        &= \frac{1}{2K} \sum_{k=1}^{r}
        \left(1 - \frac{1}{K} q_{k}^{2} \right)
        \frac{(N \lambda_{H_{1}})^{2}}{( c s_{k}^{2M} + N \lambda_{H_{1}} )^{2}} 
        + \frac{n \lambda_{H_{1}}}{2} \sum_{k=1}^{r} 
        \left( 1 - \frac{1}{K} q_{k}^{2} 
        \right)
        \frac{c s_{k}^{2M}}{(c s_{k}^{2M} + N \lambda_{H_{1}})^{2}}
        + \frac{M \lambda_{W_{1}}}{2} \sum_{k=1}^{r} s_{k}^{2} \nonumber \\
        &+ \frac{1}{2K} \sum_{h=r+1}^{K} \left(1 - \frac{1}{K} q_{h}^{2} \right)
         \nonumber \\
        &=  \frac{n \lambda_{H_{1}}}{2}
        \sum_{k=1}^{r}
        \frac{1 - \frac{1}{K} q_{k}^{2}}{c s_{k}^{2M} + N \lambda_{H_{1}}}  + \frac{M \lambda_{W_{1}}}{2} \sum_{k=1}^{r} s_{k}^{2} + 
        \frac{1}{2K} \sum_{h=r+1}^{K} \left(1 - \frac{1}{K} q_{h}^{2} \right)
        \nonumber \\ 
        &= \frac{1}{2K} 
        \sum_{k=1}^{r} 
        \left( 
        \frac{1 - \frac{1}{K} q_{k}^{2}}{\frac{c s_{k}^{2M}}{N \lambda_{H_{1}}} + 1} + M K \lambda_{W_{1}} \sqrt[M]{\frac{N \lambda_{H_{1}}}{c}} \left(\sqrt[M]{\frac{c s^{2M}_{k}}{N \lambda_{H_{1}}}} \right) 
        \right) + 
        \frac{1}{2K} \sum_{h=r+1}^{K} \left(1 - \frac{1}{K} q_{h}^{2} \right) \nonumber \\
        &=  \frac{1}{2K} 
        \sum_{k=1}^{r} 
        \left( \frac{1 - \frac{1}{K} q_{k}^{2}}{x^{M}_{k} +1} + bx_{k}
        \right) + 
         \frac{1}{2K} \sum_{h=r+1}^{K} \left(1 - \frac{1}{K} q_{h}^{2} \right),
         \label{eq:bias_f_form}
    \end{align}
    with $x_{k} := \sqrt[M]{\frac{c s_{k}^{2M}}{N \lambda_{H_{1}}}} $ and $b:= M K \lambda_{W_{1}} \sqrt[M]{\frac{N \lambda_{H_{1}}}{c}} = MK \lambda_{W_{1}}  \sqrt[M]{\frac{K n \lambda_{W_{M}} \lambda_{W_{M-2}} \ldots \lambda_{W_{1}} \lambda_{H_{1}}}{ \lambda_{W_{1}}^{M-1}}}  =
    MK \sqrt[M]{K n  \lambda_{W_{M}} \lambda_{W_{M-1}} \ldots \lambda_{W_{1}} \lambda_{H_{1}}}$.
    \\
    
    \noindent Before continue optimizing the RHS of equation \eqref{eq:bias_f_form}, we first simplify it by proving if $s_{k} > 0$ then $q_{k} = 0$, i.e. sum of entries of $k$-th column of $\mathbf{U}_{W_{M}}$ equals 0. To prove this, we will utilize a property of $\mathbf{H}_{1}^{'}  = [\mathbf{h}_{1,1} - \mathbf{h}_{G}, \ldots, \mathbf{h}_{K,n} - \mathbf{h}_{G}]$, which is the sum of entries on every row equals $0$. First, we connect $\mathbf{W}_{M}$ and $\mathbf{H}_{1}^{'}$ through:
    \begin{align}
       &\frac{\partial f^{'}}{\partial \mathbf{W}_{M}} =
      \frac{1}{N} \left(\mathbf{W}_{M} \mathbf{W}_{M-1} \ldots  \mathbf{W}_{1}   \mathbf{H}_{1}^{'} - \left( \mathbf{Y} - \frac{1}{K} \mathbf{1}_{K} \mathbf{1}_{N}^{\top} \right) \right)\mathbf{H}_{1}^{' \top} \mathbf{W}_{1}^{\top} \ldots
      \mathbf{W}_{M-1}^{\top} + \lambda_{W_{M}} \mathbf{W}_{M} = \mathbf{0} \nonumber \\
      \Rightarrow & \mathbf{W}_{M} = 
      \left( \mathbf{Y} - \frac{1}{K} \mathbf{1}_{K} \mathbf{1}_{N}^{\top} \right) 
      \mathbf{H}_{1}^{' \top}
      \underbrace{
      \mathbf{W}_{1}^{\top} \ldots
      \mathbf{W}_{M-1}^{\top}
      \left( \mathbf{W}_{M-1} \ldots  \mathbf{W}_{1}  \mathbf{H}_{1}^{'}  \mathbf{H}_{1}^{' \top} \mathbf{W}_{1}^{\top}  \ldots \mathbf{W}_{M-1}^{\top} + N \lambda_{W_{M}} \mathbf{I}_{K} \right)^{-1}}_{\mathbf{G}}.
    \end{align}
    
    \noindent From the definition of $\mathbf{H}_{1}^{'}$, we know that the sum of entries of every column of $\mathbf{H}_{1}^{' \top}$ is $0$. Recall the class-mean definition $\mathbf{h}_{k} = \frac{1}{n} \sum_{i=1}^{n} \mathbf{h}_{k,i}$, we have:
    \begin{align}
        &\left( \mathbf{Y} - \frac{1}{K} \mathbf{1}_{K} \mathbf{1}_{N}^{\top} \right) 
        \mathbf{H}_{1}^{' \top} =
        \mathbf{Y}  \mathbf{H}_{1}^{' \top} = 
        n \begin{bmatrix}
        (\mathbf{h}_{1} - \mathbf{h}_{G})^{\top}  \\ (\mathbf{h}_{2} - \mathbf{h}_{G})^{\top}
        \\ \ldots 
        \\  (\mathbf{h}_{K} - \mathbf{h}_{G})^{\top}
        \end{bmatrix} \nonumber \\
        \Rightarrow &\mathbf{W}_{M} = 
        n \begin{bmatrix}
        (\mathbf{h}_{1} - \mathbf{h}_{G})^{\top}  \\ (\mathbf{h}_{2} - \mathbf{h}_{G})^{\top}
        \\ \ldots 
        \\  (\mathbf{h}_{K} - \mathbf{h}_{G})^{\top}
        \end{bmatrix} \mathbf{G}, \nonumber
    \end{align}
    and thus, the sum of entries of every column of $\mathbf{W}_{M}$ equals $0$. From the SVD $\mathbf{W}_{M} = \mathbf{U}_{W_{M}} \mathbf{S}_{W_{M}} \mathbf{V}_{W_{M}}^{\top}$, denote $\mathbf{u}_{j}$ and $\mathbf{v}_{j}$ the $j$-th column of $\mathbf{U}_{W_{M}}$ and $\mathbf{V}_{W_{M}}$, respectively. We have from the definition of left and right singular vectors:
    \begin{align}
        \mathbf{W}_{M} \mathbf{v}_{j} = s_{j} \mathbf{u}_{j} ,
    \end{align}
    and since the sum of entries of every column of $\mathbf{W}_{M}$ equals $0$, we have the sum of entries of vector $\mathbf{W}_{M} \mathbf{v}_{j}$ equals $0$. Thus, if $s_{j} > 0$, we have $q_{j} = 0$.
    \\
    
    \noindent Return to the expression of $f^{'}$ as the RHS of equation \eqref{eq:bias_f_form}, notice that it is separable w.r.t each singular value $s_{j}$, we will analyze how each singular value contribute to the value of the expression \eqref{eq:bias_f_form}. For every singular value $s_{j}$ with $j = 1,\ldots, r$, if $s_{j} > 0$, then $q_{j} = 0$, and its contribution to the expression \eqref{eq:bias_f_form} will be $\frac{1}{2K} (  \frac{1}{x_{j}^{M} + 1} + bx_{j} ) = \frac{1}{2K} g(x_{j})$ (with the minimizer of $g(x)$ has been studied in Section \ref{sec:study_g}).
    Otherwise, if $s_{j} = 0$ (hence $x_{j} = 0$), its contribution to the value of the expression \eqref{eq:bias_f_form} will be $\frac{1 - \frac{1}{K} q_{j}^{2}}{2K}$, and it eventually be $\frac{1}{2K}$ because $\sum_{k=1}^{K} \frac{1}{K} q_{j}^{2}$ always equal $1$, thus $\frac{1}{K} q_{j}^{2}$ has no additional contribution to the expression \eqref{eq:bias_f_form}. Therefore, it is a comparision between $\frac{1}{2K}$ and $\frac{1}{2K} \min_{x_{j} > 0} g(x_{j})$ to decide whether $s_{j}^{*} = 0$ or $s_{j}^{*} = \sqrt[2M]{\frac{N \lambda_{H_{1}}}{c}} \sqrt{x_{j}^{*}}$ with $x_{j}^{*} =  \arg \min_{x > 0} g(x)$. Therefore, we consider three cases:
    \begin{itemize}
        \item If $b > \frac{(M-1)^{\frac{M-1}{M}}}{M}$: In this case, $g(x)$ is minimized at $x = 0$ and $g(0) = 1$. Hence, $\frac{1}{2K} < \frac{1}{2K} \min_{x_{j} > 0} g(x_{j})$ and thus, $s_{j}^{*} = 0 \: \forall j = 1, \ldots, r$.
        
        \item If $b < \frac{(M-1)^{\frac{M-1}{M}}}{M}$: In this case, $g(x)$ is minimized at some $x_{0} > \sqrt[M]{M-1}$ and $g(x_{0}) < 1$. Hence, $\frac{1}{2K} \min_{x_{j} > 0} g(x_{j}) < \frac{1}{2K}$ and thus, $s_{j}^{*} = \sqrt[2M]{\frac{N \lambda_{H_{1}}}{c}} \sqrt{x_{0}} \: \forall \: j = 1, \ldots, r$. 
        
        We also note that in this case, we have $q_{j} = 0 \: \forall j = 1,\ldots, r$ (meaning the sum of entries of every column in the first $r$ columns of $\mathbf{U}_{W_{M}}$ is equal $0$). 
        
        \item If  $b = \frac{(M-1)^{\frac{M-1}{M}}}{M}$: In this case, $g(x)$ is minimized at $x = 0$ or some $x = x_{0} > \sqrt[M]{M-1}$ with $g(0) = g(x_{0}) = 1$. Therefore, $s_{j}^{*}$ can either be $0$ or $x_{0}$ as long as $\{ s_{k} \}_{k=1}^{r}$ is a decreasing sequence. 
    \end{itemize}
    
    \noindent To help for the conclusion of the geometry properties of weight matrices and features, we state a lemma as following:
    \begin{lemma}
    \label{lm:6}
        Let $\mathbf{W} \in \mathbb{R}^{K \times d_{M}}$ be a matrix with $r \leq K - 1$ singular values equal a positive constant $s > 0$. If there exists a compact SVD form of $\mathbf{W}$ as $\mathbf{W} = s \mathbf{U} \mathbf{V}^{\top}$ with semi-orthonormal matrices $\mathbf{U} \in \mathbb{R}^{K \times r}, \mathbf{V} \in \mathbb{R}^{d_{M} \times r}$ such that the sum of entries of every column of $\mathbf{U}$ equals 0. Then, 
        $\mathbf{W} \mathbf{W}^{\top} \propto \mathbf{U} \mathbf{U}^{\top}$ and
        $\mathbf{U} \mathbf{U}^{\top}$ is a best rank-$r$ approximation of the simplex ETF $(\mathbf{I}_{K} - \frac{1}{K} \mathbf{1}_{K} \mathbf{1}_{K}^{\top})$.
    \end{lemma}
    \begin{proof}
    Let's denote $\mathbf{U} = [\mathbf{u}_{1}, \ldots, \mathbf{u}_{r}]$ with $\mathbf{u}_{1}, \ldots, \mathbf{u}_{r}$ are $r$ orthonormal vectors. Since the sum of entries in each $\mathbf{u}_{i}$ equals $0$, $\frac{1}{\sqrt{K}} \mathbf{1}_{K}$ can be added to the set $\{\mathbf{u}_{1}, \ldots, \mathbf{u}_{r} \}$ to form $r+1$ orthonormal vectors. Let $\hat{\mathbf{U}} = [\mathbf{u}_{1}, \ldots, \mathbf{u}_{r}, \frac{1}{\sqrt{K}} \mathbf{1}_{K}]$, we have $\operatorname{dim (Col} \hat{\mathbf{U}} ) = r + 1$. Hence, $\operatorname{dim (Null} \hat{\mathbf{U}}^{\top} ) = K - r -1$ and thus, we can choose an orthonormal basis of $\operatorname{Null} \hat{\mathbf{U}}^{\top}$ including $K - r - 1$ orthonormal vectors $\{\mathbf{u}_{r+1}, \mathbf{u}_{r+2},\ldots, \mathbf{u}_{K-1} \}$. And because these $K - r - 1$ orthonormal vectors are in $\operatorname{Null} \hat{\mathbf{U}}^{\top}$, we can add these vectors to the set $\{ \mathbf{u}_{1}, \ldots, \mathbf{u}_{r}, \frac{1}{\sqrt{K}} \mathbf{1}_{K}\}$ to form a basis of $\mathbb{R}^{K}$ including $K$ orthonormal vectors $\{ \mathbf{u}_{1}, \ldots, \mathbf{u}_{r}, \mathbf{u}_{r+1}, \mathbf{u}_{r+2},\ldots, \mathbf{u}_{K-1}, \frac{1}{\sqrt{K}} \mathbf{1}_{K} \}$. We denote $\overline{\mathbf{U}} = [\mathbf{u}_{1}, \ldots, \mathbf{u}_{r}, \mathbf{u}_{r+1}, \mathbf{u}_{r+2},\ldots, \mathbf{u}_{K-1}, \frac{1}{\sqrt{K}} \mathbf{1}_{K}] \in \mathbb{R}^{K \times K}$. We have $\overline{\mathbf{U}}^{\top} \overline{\mathbf{U}} = \mathbf{I}_{K}$. From the Inverse Matrix Theorem, we deduce that $\overline{\mathbf{U}}^{-1} = \overline{\mathbf{U}}^{\top}$ and thus, $\overline{\mathbf{U}}$ is an orthonormal matrix. We have $\overline{\mathbf{U}}$ is an orthonormal matrix with the last column $\frac{1}{\sqrt{K}} \mathbf{1}_{K}$, hence by simple matrix multiplication, we have:
    \begin{align}
        &[\mathbf{u}_{1}, \ldots, \mathbf{u}_{r}, \mathbf{u}_{r+1}, \mathbf{u}_{r+2},\ldots, \mathbf{u}_{K-1}]
        [\mathbf{u}_{1}, \ldots, \mathbf{u}_{r}, \mathbf{u}_{r+1}, \mathbf{u}_{r+2},\ldots, \mathbf{u}_{K-1}]^{\top} = \mathbf{I}_{K} - \frac{1}{K} \mathbf{1}_{K} \mathbf{1}_{K}^{\top}
        \nonumber \\
        &\Rightarrow \overline{\mathbf{U}} \begin{bmatrix}
        \mathbf{I}_{K-1} & \mathbf{0}  \\
        \mathbf{0} & 0 \\
        \end{bmatrix} \overline{\mathbf{U}}^{\top} = \mathbf{I}_{K} - \frac{1}{K} \mathbf{1}_{K} \mathbf{1}_{K}^{\top}.
    \end{align}
    Therefore, $\mathbf{U} \mathbf{U}^{\top}$ is the best rank-$r$ approximation of $\mathbf{I}_{K} - \frac{1}{K} \mathbf{1}_{K} \mathbf{1}_{K}^{\top}$, and the proof for the lemma is finished.
    \end{proof}
    
    \noindent Thus, we finish bounding $f$ and the equality conditions are as following:
    
    \begin{itemize}
        \item If $b = MK \sqrt[M]{K n  \lambda_{W_{M}} \lambda_{W_{M-1}} \ldots \lambda_{W_{1}} \lambda_{H_{1}}} > \frac{(M-1)^{\frac{M-1}{M}}}{M}$: all the singular values of $\mathbf{W}_{1}$ are zeros. Therefore, the singular values of $\mathbf{W}_{M}, \mathbf{W}_{M-1}, \ldots, \mathbf{H}_{1}^{'}$ are also all zeros. In this case, $f(\mathbf{W}_{M}, \mathbf{W}_{M-1},\ldots, \mathbf{W}_{2}, \mathbf{W}_{1}, \mathbf{H}_{1}, \mathbf{b})$ is minimized at $(\mathbf{W}_{M}^{\ast}, \mathbf{W}_{M-1}^{\ast}, \ldots, \mathbf{W}_{1}^{\ast}, \mathbf{H}_{1}^{\ast}, \mathbf{b}^{*}) = (\mathbf{0}, \mathbf{0},\ldots \mathbf{0}, \mathbf{0}, \frac{1}{K} \mathbf{1}_{K})$.
        
        \item If $b = MK \sqrt[M]{K n  \lambda_{W_{M}} \lambda_{W_{M-1}} \ldots \lambda_{W_{1}} \lambda_{H_{1}}} < \frac{(M-1)^{\frac{M-1}{M}}}{M}$: In this case, $\mathbf{W}_{1}^{\ast}$ will have the its $r$ ($r$ will be specified later) singular values all equal a multiplier of the largest positive solution of the equation $b - \frac{M x^{M-1}}{(x^{M} + 1)^{2}} = 0$, denoted as $s$. Hence, we can write the compact SVD form (with a bit of notation abuse) of $\mathbf{W}_{M-1}^{\ast}$ as $\mathbf{W}_{1}^{\ast} = s \mathbf{U}_{W_{1}} \mathbf{V}_{W_{1}}^{\top}$ with semi-orthonormal matrices $\mathbf{U}_{W_{1}} \in \mathbb{R}^{d_{2} \times r}, \mathbf{V}_{W_{1}} \in \mathbb{R}^{d_{1} \times r}$ (note that $\mathbf{U}_{W_{1}}^{\top} \mathbf{U}_{W_{1}} = \mathbf{I}$ and $\mathbf{V}_{W_{1}}^{\top} \mathbf{V}_{W_{1}} = \mathbf{I}$).
            \\
            
        Similarly, we also have the compact SVD form of other weight matrices and feature matrix as:
            \begin{align}
                \mathbf{W}_{M}^{\ast}  &= \sqrt{\frac{\lambda_{W_{1}}}{\lambda_{W_{M}}}} s \mathbf{U}_{W_{M}} \mathbf{U}_{W_{M-1}}^{\top},  \nonumber \\
                \mathbf{W}_{M-1}^{\ast} &= \sqrt{\frac{\lambda_{W_{1}}}{\lambda_{W_{M-1}}}} s 
                \mathbf{U}_{W_{M-1}} \mathbf{U}_{W_{M-2}}^{\top}, \nonumber \\
                \ldots \nonumber \\
                \mathbf{W}_{1}^{\ast} &= s \mathbf{U}_{W_{1}} \mathbf{V}_{W_{1}}^{\top}, \nonumber \\
                \mathbf{H}_{1}^{' \ast} &= \frac{\sqrt{c} s^{M}}{c s^{2M} + N \lambda_{H_{1}}} \mathbf{V}_{W_{1}} \mathbf{U}_{W_{M}}^{\top}
                \left( \mathbf{Y} - \frac{1}{K} \mathbf{1}_{K} \mathbf{1}_{N}^{\top} \right),
                \nonumber
            \end{align}
            with semi-orthonormal matrices $\mathbf{U}_{W_{M}}, \mathbf{U}_{W_{M-1}}, \ldots, \mathbf{U}_{W_{1}}, \mathbf{V}_{W_{1}}$ that each has $r$ orthogonal columns, i.e., $\mathbf{U}_{W_{M}}^{\top} \mathbf{U}_{W_{M}} = \mathbf{U}_{W_{M-1}}^{\top} \mathbf{U}_{W_{M-1}} =  \ldots = \mathbf{U}_{W_{1}}^{\top} \mathbf{U}_{W_{1}} = \mathbf{V}_{W_{1}}^{T} \mathbf{V}_{W_{1}} = \mathbf{I}_{r}$. Furthermore, $\mathbf{U}_{W_{M}}, \mathbf{U}_{W_{M-1}}, \ldots, \mathbf{U}_{W_{1}}, \mathbf{V}_{W_{1}}$ are truncated matrices from orthonormal matrices (remove columns that does not correspond with non-zero singular values), hence  $\mathbf{U}_{W_{M}} \mathbf{U}_{W_{M}}^{\top}, \mathbf{U}_{W_{M-1}} \mathbf{U}_{W_{M-1}}^{\top}, \ldots,  \mathbf{U}_{W_{1}} \mathbf{U}_{W_{1}}^{\top}, \mathbf{V}_{W_{1}} \mathbf{V}_{W_{1}}^{\top} $ are the best rank-$r$ approximations of the identity matrix of the same size.
            \\
            
            Since $\left( \mathbf{Y} - \frac{1}{K} \mathbf{1}_{K} \mathbf{1}_{N}^{\top} \right) = \left( \mathbf{I}_{K} - \frac{1}{K} \mathbf{1}_{K} \mathbf{1}_{K}^{\top} \right) \mathbf{Y} = \left( \mathbf{I}_{K} - \frac{1}{K} \mathbf{1}_{K} \mathbf{1}_{K}^{\top} \right) \otimes \mathbf{1}_{n}^{\top}$ , let  $\overline{\mathbf{H}}^{*} = \frac{\sqrt{c} s^{M}}{c s^{2M} + N \lambda_{H_{1}}} \mathbf{V}_{W_{1}} \mathbf{U}_{W_{M}}^{\top} \left( \mathbf{I}_{K} - \frac{1}{K} \mathbf{1}_{K} \mathbf{1}_{K}^{\top} \right) \in \mathbb{R}^{d_{1} \times K}$, then we have $(\mathcal{NC}1)$ $\mathbf{H}_{1}^{' *} = \overline{\mathbf{H}}^{*} \mathbf{Y} = \overline{\mathbf{H}}^{*} \otimes \mathbf{1}_{n}^{\top}$, thus we conclude the features within the same class collapse to their class-mean and $\overline{\mathbf{H}}^{*}$ is the class-means matrix. We also have $\mathbf{h}_{G} = \mathbf{0}$ (the equality condition of inequality \eqref{eq:hg=0}), hence $\mathbf{H}^{*}_{1} = \mathbf{H}^{' *}_{1}$. Furthermore, clearly we have $\operatorname{rank}(\mathbf{H}_{1}^{' *})
            = \operatorname{rank}(\overline{\mathbf{H}}^{*})$ and since $\mathbf{h}_{G} = 0$, we have $r = \operatorname{rank}(\mathbf{H}_{1}^{' *})
            = \operatorname{rank}(\overline{\mathbf{H}}^{*}) \leq K - 1$. Hence, $r = \min(R, K-1)$.
            \\
            
            By using Lemma \ref{lm:6} for $\mathbf{W}_{M}$ with the note $q_{j} = 0 \: \forall \: j \leq r$, we have $\mathbf{U}_{W} \mathbf{U}_{W}^{\top}$ is a best rank-$r$ approximation of the simplex ETF $\mathbf{I}_{K} - \frac{1}{K} \mathbf{1}_{K} \mathbf{1}_{K}^{\top}$. Thus, we can deduce the geometry of the following $(\mathcal{NC}2)$:
            \begin{align}
            \begin{gathered}
                 \mathbf{W}_{M}^{\ast} \mathbf{W}_{M}^{\top \ast} \propto 
                 \mathbf{U}_{W_{M}} \mathbf{U}_{W_{M}}^{\top} \propto \mathcal{P}_{r}(\mathbf{I}_{K} - \frac{1}{K} \mathbf{1}_{K} \mathbf{1}_{K}^{\top}), \\
                \overline{\mathbf{H}}^{\ast \top} \overline{\mathbf{H}}^{\ast} \propto 
                 (\mathbf{I}_{K} - \frac{1}{K} \mathbf{1}_{K} \mathbf{1}_{K}^{\top}) \mathbf{U}_{W_{M}} \mathbf{U}_{W_{M}}^{\top} (\mathbf{I}_{K} - \frac{1}{K} \mathbf{1}_{K} \mathbf{1}_{K}^{\top})
                 \propto
                 \mathbf{U}_{W_{M}} \mathbf{U}_{W_{M}}^{\top}
                 \propto \mathcal{P}_{r}(\mathbf{I}_{K} - \frac{1}{K} \mathbf{1}_{K} \mathbf{1}_{K}^{\top}),
                \\
                \mathbf{W}_{M}^{\ast} \mathbf{W}_{M-1}^{\ast}
                \ldots \mathbf{W}_{2}^{\ast} 
                \mathbf{W}_{1}^{\ast} \overline{\mathbf{H}}^{*} \propto 
                 \mathbf{U}_{W_{M}} \mathbf{U}_{W_{M}}^{\top} (\mathbf{I}_{K} - \frac{1}{K} \mathbf{1}_{K} \mathbf{1}_{K}^{\top}) 
                 \propto  \mathbf{U}_{W_{M}} \mathbf{U}_{W_{M}}^{\top}
                 \propto \mathcal{P}_{r}(\mathbf{I}_{K} - \frac{1}{K} \mathbf{1}_{K} \mathbf{1}_{K}^{\top}),  \\
                (\mathbf{W}_{M}^{\ast} \mathbf{W}_{M-1}^{\ast} \ldots \mathbf{W}_{j}^{\ast})(\mathbf{W}_{M}^{\ast} \mathbf{W}_{M-1}^{\ast} \ldots \mathbf{W}_{j}^{\ast})^{\top} \propto 
                 \mathbf{U}_{W_{M}} \mathbf{U}_{W_{M}}^{\top} \propto \mathcal{P}_{r}(\mathbf{I}_{K} - \frac{1}{K} \mathbf{1}_{K} \mathbf{1}_{K}^{\top}) \quad \forall \: j \in [M].
            \end{gathered}
            \end{align}
            Note that if $r = K - 1$, we have $\mathcal{P}_{r}(\mathbf{I}_{K} - \frac{1}{K} \mathbf{1}_{K} \mathbf{1}_{K}^{\top}) = \mathbf{I}_{K} - \frac{1}{K} \mathbf{1}_{K} \mathbf{1}_{K}^{\top}$.
            \\
            
            Also, the product of each weight matrix or features with its transpose will be the multiplier of one of the best rank-$r$ approximations of the identity matrix of the same size. For example, $\mathbf{W}_{M-1}^{\ast \top} \mathbf{W}_{M-1}^{\ast} \propto \mathbf{U}_{W_{M-2}} \mathbf{U}_{W_{M-2}}^{\top}$ and $\mathbf{W}_{M-1}^{\ast} \mathbf{W}_{M-1}^{\ast \top} \propto \mathbf{U}_{W_{M-1}} \mathbf{U}_{W_{M-1}}^{\top}$ are two best rank-$r$ approximations of $\mathbf{I}_{d_{M-1}}$ and $\mathbf{I}_{d_{M}}$, respectively.
            \\

            Next, we can derive the alignments between weights and features as following $(\mathcal{NC}3)$:
            \begin{align}
            \begin{gathered}
                \mathbf{W}_{M}^{\ast} \mathbf{W}_{M-1}^{\ast} \ldots \mathbf{W}_{1}^{\ast} \propto \mathbf{U}_{W_{M}} \mathbf{V}_{W_{1}}^{\top} \propto \overline{\mathbf{H}}^{* \top}, \\
                \mathbf{W}_{M-1}^{\ast} \mathbf{W}_{M-2}^{\ast} \ldots \mathbf{W}_{1}^{\ast} \overline{\mathbf{H}}^{*} \propto \mathbf{U}_{W_{M-1}}   \mathbf{U}_{W_{M}}^{\top} \propto  \mathbf{W}_{M}^{\ast \top}, \\
                \mathbf{W}_{M}^{*} \mathbf{W}_{M-1}^{*} \ldots \mathbf{W}_{j}^{*} \propto  
                \mathbf{U}_{W_{M}} \mathbf{U}_{W_{j-1}}^{\top} \propto  (\mathbf{W}_{j-1}^{*} \ldots \mathbf{W}_{1}^{*} \overline{\mathbf{H}}^{*})^{\top}.
            \end{gathered}
            \end{align}
            
            \item If $b = MK \sqrt[M]{K n  \lambda_{W_{M}} \lambda_{W_{M-1}} \ldots \lambda_{W_{1}} \lambda_{H_{1}}} = \frac{(M-1)^{\frac{M-1}{M}}}{M}$:
            In this case, $x^{*}_{k}$ can either be $0$ or the largest positive solution of the equation $b - \frac{M x^{M-1}}{(x^{M} + 1)^{2}} = 0$. If all the singular values are $0$'s, we have the trivial global minima $(\mathbf{W}_{M}^{\ast}, \ldots, \mathbf{W}_{1}^{\ast}, \mathbf{H}_{1}^{\ast}, \mathbf{b}^{*}) = (\mathbf{0},\ldots, \mathbf{0}, \mathbf{0}, \frac{1}{K} \mathbf{1}_{K})$.
            \\
            
            If there are exactly $0 < t \leq r = \min(R, K-1)$ positive singular values  $s_{1} = s_{2} = \ldots = s_{t} := s > 0$ and $s_{t+1} = \ldots = s_{r} = 0$, we also have compact SVD form similar as the case $b < \frac{(M-1)^{\frac{M-1}{M}}}{M}$, (with exactly $t$ singular vectors, instead of $r$ as the above case). Thus, the nontrivial solutions exhibit  $(\mathcal{NC}1)$ and $(\mathcal{NC}3)$ property similarly as the case $b < \frac{(M-1)^{\frac{M-1}{M}}}{M}$ above.
            \\
            
            For $(\mathcal{NC}2)$ property, for $j = 1, \ldots, M$, we have:
            \begin{align}
                \begin{gathered}
                   \mathbf{W}_{M}^{\ast} \mathbf{W}_{M}^{\ast \top} \propto 
                   \overline{\mathbf{H}}^{\ast \top} \overline{\mathbf{H}}^{\ast}
                    \propto
                   \mathbf{W}_{M}^{\ast} \mathbf{W}_{M-1}^{\ast}  \mathbf{W}_{M-2}^{\ast} 
                    \ldots \mathbf{W}_{2}^{\ast} 
                   \mathbf{W}_{1}^{\ast} \overline{\mathbf{H}}^{*}
                    \\ \propto 
                   (\mathbf{W}_{M}^{\ast} \mathbf{W}_{M-1}^{\ast} \ldots \mathbf{W}_{j}^{\ast})(\mathbf{W}_{M}^{\ast} \mathbf{W}_{M-1}^{\ast} \ldots \mathbf{W}_{j}^{\ast})^{\top}
                    \propto \mathcal{P}_{t}(\mathbf{I}_{K} - \frac{1}{K} \mathbf{1}_{K} \mathbf{1}_{K}^{\top}).
                   \nonumber
                \end{gathered}
                \end{align}
    \end{itemize}
We finish the proof.
\end{proof}

\section{Proof of Theorem \ref{thm:UFM_imbalance}}
\label{sec:proofs_im}

First, we state the results for the case that the hidden dimension $d$ is at least the number of classes $K$.

\begin{theorem}
\label{thm:im_UFM_normal}
    Let $d \geq K$ and $(\mathbf{W}^{*}, \mathbf{H}^{*})$ be any global minimizer of problem \eqref{eq:UFM_imbalance}. Then, we have:
    \\
    
    $(\mathcal{NC}1) \quad
        \mathbf{H}^{*} = \overline{\mathbf{H}}^{*} \mathbf{Y} 
        \Leftrightarrow \mathbf{h}_{k,i}^{*} = \mathbf{h}_{k}^{*} \: \forall \: k \in [K], i \in [n_{k}], \nonumber$
    where $\overline{\mathbf{H}}^{*} = [\mathbf{h}_{1}^{*},\ldots,\mathbf{h}_{K}^{*} ] \in \mathbb{R}^{d \times K}$.
    \\
    
    $(\mathcal{NC}3) \quad 
        \mathbf{w}_{k}^{*} = \sqrt{\frac{n_{k} \lambda_{H}}{\lambda_{W}}} \mathbf{h}_{k}^{*} \quad \forall \: k \in [K]. \nonumber$
    \\
    $(\mathcal{NC}2)$ Let $a:= N^{2} \lambda_{W} \lambda_{H}$, we have:
    \begin{align}
    \begin{gathered}
             \mathbf{W}^{*} \mathbf{W}^{* \top}
             = \operatorname{diag}
             \left\{s_{k}^{2} \right\}_{k=1}^{K},
             \\
            \overline{\mathbf{H}}^{* \top}
            \overline{\mathbf{H}}^{*} 
            =
            \operatorname{diag}
            \left\{
           \frac{s_{k}^{2}}{(s_{k}^{2} + N \lambda_{H})^{2}}
            \right\}_{k=1}^{K}, \nonumber
     \end{gathered}    
    \end{align}
    \begin{align}
    \begin{gathered}
            \mathbf{W}^{*} \mathbf{H}^{*}
            = \operatorname{diag}
            \left\{
            \frac{s_{k}^{2}}{s_{k}^{2} + N \lambda_{H}}
            \right\}_{k=1}^{K}
            \mathbf{Y}
            \nonumber \\
            = \begin{bmatrix}
            \frac{s_{1}^{2}}{s_{1}^{2} + N \lambda_{H}} \mathbf{1}_{n_{1}}^{\top}  & \ldots & \mathbf{0} \\
            \vdots & \ddots & \vdots \\
            \mathbf{0}  & \ldots & \frac{s_{K}^{2}}{s_{K}^{2} + N \lambda_{H}} \mathbf{1}_{n_{K}}^{\top} \\
            \end{bmatrix}. \nonumber
    \end{gathered}    
    \end{align}
    where:
    \begin{itemize}
        \item If $\frac{a}{n_{1}} \leq \frac{a}{n_{2}} \leq \ldots \leq \frac{a}{n_{K}} \leq 1$:
        \begin{align}
        \begin{aligned}
            s_{k} = 
            \sqrt{ \sqrt{\frac{n_{k} \lambda_{H}}{\lambda_{W}} }
          - N \lambda_{H}} \quad &\forall \: k  \nonumber  
        \end{aligned}
        \end{align}
        
        \item If there exists a $j \in [K-1]$ s.t. $\frac{a}{n_{1}} \leq \frac{a}{n_{2}} \leq \ldots \leq \frac{a}{n_{j}} \leq 1 < \frac{a}{n_{j+1}} \leq \ldots \leq \frac{a}{n_{K}}$:
        \begin{align}
        \begin{aligned}
        s_{k} = \left\{\begin{matrix}
        \sqrt{ \sqrt{\frac{n_{k} \lambda_{H}}{\lambda_{W}} }
          - N \lambda_{H}} \quad &\forall \: k \leq j \\ 0 \quad &\forall \: k > j
        \end{matrix}\right. .
        \nonumber
        \end{aligned}
        \end{align}
        
        \item If $1 < \frac{a}{n_{1}} \leq \frac{a}{n_{2}} \leq \ldots \leq \frac{a}{n_{K}} $:
        \begin{align}
           (s_{1}, s_{2}, \ldots, s_{K} ) &= (0,0,\ldots,0), \nonumber
        \end{align}
        and $(\mathbf{W}^{*}, \mathbf{H}^{*}) = (\mathbf{0}, \mathbf{0})$ in this case.
    \end{itemize}
     And, for any $k$ such that $s_{k} = 0$, we have:
        \begin{align}
            \mathbf{w}_{k}^{*} = \mathbf{h}_{k}^{*} = \mathbf{0}. \nonumber
        \end{align}
\end{theorem}

Next, for the bottleneck architecture that $d < K$, as mentioned in main paper, the $\mathcal{NC}2$ geometry may not be diagonal. The details are as follows. 

\begin{theorem}
\label{thm:im_UFM_bottleneck}
    Let $d < K$, thus $R = \min(d, K) = d$ and $(\mathbf{W}^{*}, \mathbf{H}^{*})$ be any global minimizer of problem \eqref{eq:UFM_imbalance}. Then, we have:
    \\
    
    $(\mathcal{NC}1) \quad
        \mathbf{H}^{*} = \overline{\mathbf{H}}^{*} \mathbf{Y} 
        \Leftrightarrow \mathbf{h}_{k,i}^{*} = \mathbf{h}_{k}^{*} \: \forall \: k \in [K], i \in [n_{k}], \nonumber$
    where $\overline{\mathbf{H}}^{*} = [\mathbf{h}_{1}^{*},\ldots,\mathbf{h}_{K}^{*} ] \in \mathbb{R}^{d \times K}$.
    \\
    
    $(\mathcal{NC}3) \quad 
        \mathbf{w}_{k}^{*} = \sqrt{\frac{n_{k} \lambda_{H}}{\lambda_{W}}} \mathbf{h}_{k}^{*} \quad \forall \: k \in [K]. \nonumber$
    \\
    
    $(\mathcal{NC}2)$ Let $a:= N^{2} \lambda_{W} \lambda_{H}$, we define $\left\{ s_{k} \right\}_{k=1}^{K}$ as follows:
    \begin{itemize}
        \item If $\frac{a}{n_{1}} \leq \frac{a}{n_{2}} \leq \ldots \leq \frac{a}{n_{R}} \leq 1$:
        \begin{align}
        \begin{aligned}
            s_{k} = 
            \left\{\begin{matrix}
        \sqrt{ \sqrt{\frac{n_{k} \lambda_{H}}{\lambda_{W}} }
          - N \lambda_{H}} \quad &\forall \: k \leq R \\ 0 \quad &\forall \: k > R
        \end{matrix}\right. .
        \end{aligned}
        \end{align}

        Then, if $b/n_{R} = 1$ or $n_{R} > n_{R+1}$, we have:
        \begin{align}
        \begin{gathered}
             \mathbf{W}^{*} \mathbf{W}^{* \top}
             = \operatorname{diag}
             \left\{s_{k}^{2} \right\}_{k=1}^{K},
             \\
            \overline{\mathbf{H}}^{* \top}
            \overline{\mathbf{H}}^{*} 
            =
            \operatorname{diag}
            \left\{
           \frac{s_{k}^{2}}{(s_{k}^{2} + N \lambda_{H})^{2}}
            \right\}_{k=1}^{K}, \nonumber
            \\
            \mathbf{W}^{*} \overline{\mathbf{H}}^{*}
            = \operatorname{diag}
            \left\{
            \frac{s_{k}^{2}}{s_{k}^{2} + N \lambda_{H}}
            \right\}_{k=1}^{K}, 
            \end{gathered}    
            \end{align}
            and for any $k > R$, we have  $\mathbf{w}_{k}^{*} = \mathbf{h}_{k}^{*} = \mathbf{0}$.
            \\
            
            If $b/n_{R} < 1$ and there exists $k \leq R$, $l > R$ such that $n_{k-1} > n_{k} = n_{k+1} = \ldots = n_{R} = \ldots = n_{l} > n_{l+1}$, then:
                        \begin{align}
            \mathbf{W}^{*} \mathbf{W}^{* \top}
            &= \begin{bmatrix}
            s_{1}^{2} & \ldots & \mathbf{0} & \mathbf{0} & \mathbf{0}  \\
            \vdots & \ddots & \vdots & \vdots & \vdots  \\
            \mathbf{0} & \ldots & s_{k-1}^{2} & \mathbf{0} & \mathbf{0}  \\
            \mathbf{0} & \ldots & \mathbf{0} & s_{k}^{2} \mathcal{P}_{R-k+1}(\mathbf{I}_{l-k+1}) &  \mathbf{0}  \\
            \mathbf{0} & \ldots & \mathbf{0} & \mathbf{0} & \mathbf{0}_{(K-l) \times (K-l)}   \\
            \end{bmatrix},  \\
             \overline{\mathbf{H}}^{* \top}
            \overline{\mathbf{H}}^{*} 
            &= 
            \begin{bmatrix}
            \frac{s_{1}^{2}}{(s_{1}^{2} + N \lambda_{H})^{2}} & \ldots & \mathbf{0} & \mathbf{0} & \mathbf{0}  \\
            \vdots & \ddots & \vdots & \vdots & \vdots  \\
            \mathbf{0} & \ldots & \frac{s_{k-1}^{2}}{(s_{k-1}^{2} + N \lambda_{H})^{2}} & \mathbf{0} & \mathbf{0}  \\
            \mathbf{0} & \ldots & \mathbf{0} & \frac{s_{k}^{2}}{(s_{k}^{2} + N \lambda_{H})^{2}} \mathcal{P}_{R-k+1}(\mathbf{I}_{l-k+1}) &  \mathbf{0}  \\
            \mathbf{0} & \ldots & \mathbf{0} & \mathbf{0} & \mathbf{0}_{(K-l) \times (K-l)}   \\
            \end{bmatrix}, \\
            \mathbf{W}^{*} \overline{\mathbf{H}}^{*}
            &= 
            \begin{bmatrix}
            \frac{s_{1}^{2}}{s_{1}^{2} + N \lambda_{H}} & \ldots & \mathbf{0} & \mathbf{0} & \mathbf{0}  \\
            \vdots & \ddots & \vdots & \vdots & \vdots  \\
            \mathbf{0} & \ldots & \frac{s_{k-1}^{2}}{s_{k-1}^{2} + N \lambda_{H}} & \mathbf{0} & \mathbf{0}  \\
            \mathbf{0} & \ldots & \mathbf{0} & \frac{s_{k}^{2}}{s_{k}^{2} + N \lambda_{H}} \mathcal{P}_{R-k+1}(\mathbf{I}_{l-k+1}) &  \mathbf{0}  \\
            \mathbf{0} & \ldots & \mathbf{0} & \mathbf{0} & \mathbf{0}_{(K-l) \times (K-l)}   \\
            \end{bmatrix},
            \end{align}
            and for any $k > l > R$, we have $\mathbf{w}_{k}^{*} = \mathbf{h}_{k}^{*} = \mathbf{0}$. 
        
        \item If there exists a $j \in [R-1]$ s.t. $\frac{a}{n_{1}} \leq \frac{a}{n_{2}} \leq \ldots \leq \frac{a}{n_{j}} \leq 1 < \frac{a}{n_{j+1}} \leq \ldots \leq \frac{a}{n_{R}}$:
        \begin{align}
        \begin{aligned}
        s_{k} = \left\{\begin{matrix}
        \sqrt{ \sqrt{\frac{n_{k} \lambda_{H}}{\lambda_{W}} }
          - N \lambda_{H}} \quad &\forall \: k \leq j \\ 0 \quad &\forall \: k > j
        \end{matrix}\right. .
        \nonumber
        \end{aligned}
        \end{align}
        Then, we have:
        \begin{align}
        \begin{gathered}
             \mathbf{W}^{*} \mathbf{W}^{* \top}
             = \operatorname{diag}
             \left\{s_{k}^{2} \right\}_{k=1}^{K},
             \\
            \overline{\mathbf{H}}^{* \top}
            \overline{\mathbf{H}}^{*} 
            =
            \operatorname{diag}
            \left\{
           \frac{s_{k}^{2}}{(s_{k}^{2} + N \lambda_{H})^{2}}
            \right\}_{k=1}^{K}, \nonumber
            \\
            \mathbf{W}^{*} \overline{\mathbf{H}}^{*}
            = \operatorname{diag}
            \left\{
            \frac{s_{k}^{2}}{s_{k}^{2} + N \lambda_{H}}
            \right\}_{k=1}^{K}, 
            \end{gathered}    
            \end{align}
        and for any $k > j$, we have  $\mathbf{w}_{k}^{*} = \mathbf{h}_{k}^{*} = \mathbf{0}$
        
        \item If $1 < \frac{a}{n_{1}} \leq \frac{a}{n_{2}} \leq \ldots \leq \frac{a}{n_{R}} $:
        \begin{align}
           (s_{1}, s_{2}, \ldots, s_{K} ) &= (0,0,\ldots,0), \nonumber
        \end{align}
        and $(\mathbf{W}^{*}, \mathbf{H}^{*}) = (\mathbf{0}, \mathbf{0})$ in this case.
    \end{itemize}
\end{theorem}

We derive proofs for both theorem as following.

\begin{proof}[Proof of Theorem \ref{thm:im_UFM_normal} and \ref{thm:im_UFM_bottleneck}]
 By definition, any critical point $(\mathbf{W}, \mathbf{H})$ of $f(\mathbf{W}, \mathbf{H})$ satisfies the following:
 \begin{align}
  \frac{\partial f}{\partial \mathbf{W}} =
  \frac{1}{N} 
  ( \mathbf{W} \mathbf{H} - \mathbf{Y} ) \mathbf{H}^{\top} 
  + \lambda_{W} \mathbf{W} = \mathbf{0}, \\
  \frac{\partial f}{\partial \mathbf{H}} =
  \frac{1}{N} \mathbf{W}^{\top}
  ( \mathbf{W} \mathbf{H} - \mathbf{Y} ) +
  \lambda_{H} \mathbf{H} = \mathbf{0}.
 \end{align}
 From $\mathbf{0} = \mathbf{W}^{\top} \frac{\partial f}{\partial \mathbf{W}} - \frac{\partial f}{\partial \mathbf{H}} \mathbf{H}^{\top}$, we have:
 \begin{align}
     \lambda_{W} \mathbf{W}^{\top} \mathbf{W}
     = 
     \lambda_{H}
     \mathbf{H}
     \mathbf{H}^{\top}.
 \end{align}
 
 \noindent Also, from $\frac{\partial f}{\partial \mathbf{H}} = \mathbf{0}$, solving for $\mathbf{H}$ yields:
 \begin{align}
     \mathbf{H} = ( \mathbf{W}^{\top} \mathbf{W} + N \lambda_{H} \mathbf{I}  )^{-1} \mathbf{W}^{\top}
     \mathbf{Y}.
     \label{eq:im_H_form}
 \end{align}
 
 \noindent Let $\mathbf{W} = \mathbf{U}_{W} \mathbf{S}_{W} \mathbf{V}_{W}^{\top} $ be the SVD decomposition of $\mathbf{W}$ with orthonormal matrices $\mathbf{U}_{W} \in \mathbb{R}^{K \times K},  \mathbf{V}_{W} \in \mathbb{R}^{d \times d}$ and diagonal matrix $\mathbf{S}_{W} \in \mathbb{R}^{K \times d}$ with non-decreasing singular values. We denote $r$ singular values of $\mathbf{W}$ as $\left\{ s_{k} \right\}_{k=1}^{r}$ (we have $r \leq R:= \min(K,d)$).
 
 \noindent From equation \eqref{eq:im_H_form} and the SVD of $\mathbf{W}$:
 \begin{align}
 \begin{aligned}
      \mathbf{H} &= ( \mathbf{W}^{\top} \mathbf{W} + N \lambda_{H} \mathbf{I}  )^{-1} \mathbf{W}^{\top}
     \mathbf{Y} \\
     &=
     ( \mathbf{V}_{W} \mathbf{S}_{W}^{\top}
     \mathbf{S}_{W}
     \mathbf{V}_{W}^{\top} 
     + N \lambda_{H} \mathbf{I}
     )^{-1}
     \mathbf{V}_{W} 
     \mathbf{S}_{W}^{\top}
     \mathbf{U}_{W}^{\top} 
     \mathbf{Y}. \\
     &= 
     \mathbf{V}_{W}
     ( \mathbf{S}_{W}^{\top}
     \mathbf{S}_{W} + N \lambda_{H} \mathbf{I} )^{-1}
     \mathbf{S}_{W}^{\top}
     \mathbf{U}_{W}^{\top} 
     \mathbf{Y} \\
     &= \mathbf{V}_{W}
    \underbrace{
    \begin{bmatrix}
    \operatorname{diag}\left(
        \frac{ s_{1} }{ s_{1}^{2} + N \lambda_{H_{1}}}, \ldots, \frac{s_{r}}{ s_{r}^{2} + N \lambda_{H_{1}}}
        \right) &  \mathbf{0}\\
    \mathbf{0} &  \mathbf{0}\\
    \end{bmatrix}}_{\mathbf{C} \in \mathbb{R}^{d \times K}}
    \mathbf{U}_{W}^{\top} \mathbf{Y} \\
    &= \mathbf{V}_{W} \mathbf{C} \mathbf{U}_{W}^{\top} \mathbf{Y},
    \label{eq:im_H_form2}
 \end{aligned}
 \end{align}
 \begin{align}
 \begin{aligned}
 \mathbf{W} \mathbf{H}
 &= \mathbf{U}_{W}
 \mathbf{S}_{W} 
 \begin{bmatrix}
    \operatorname{diag}\left(
        \frac{ s_{1} }{ s_{1}^{2} + N \lambda_{H_{1}}}, \ldots, \frac{s_{r}}{ s_{r}^{2} + N \lambda_{H_{1}}}
        \right) &  \mathbf{0}\\
    \mathbf{0} &  \mathbf{0}\\
  \end{bmatrix}
  \mathbf{U}_{W}^{\top}
  \mathbf{Y} \\
  &= \mathbf{U}_{W} 
    \operatorname{diag} \left( \frac{s_{1}^{2}}{s_{1}^{2} + N \lambda_{H}} , \ldots, \frac{s_{r}^{2}}{s_{r}^{2} + N \lambda_{H}}, 0,\ldots,0 \right)
  \mathbf{U}_{W}^{\top}
  \mathbf{Y}
 \end{aligned}     
 \end{align}

 \begin{align}
    \begin{aligned}
     \Rightarrow  \mathbf{W} \mathbf{H} - \mathbf{Y}
     &= \mathbf{U}_{W} 
     \left[  
      \operatorname{diag} \left( \frac{s_{1}^{2}}{s_{1}^{2} + N \lambda_{H}} , \ldots, \frac{s_{r}^{2}}{s_{r}^{2} + N \lambda_{H}}, 0,\ldots,0 \right) - \mathbf{I}_{K}
     \right] \mathbf{U}_{W}^{\top}
     \mathbf{Y} \\
     &= \mathbf{U}_{W}
     \underbrace{
     \operatorname{diag} \left( \frac{- N \lambda_{H}}{s_{1}^{2} + N \lambda_{H}} , \ldots, \frac{- N \lambda_{H}}{s_{r}^{2} + N \lambda_{H}}, -1,\ldots, -1 \right) }_{\mathbf{D} \in \mathbb{R}^{K \times K}}
     \mathbf{U}_{W}^{\top}
     \mathbf{Y} \\
     &= \mathbf{U}_{W}
     \mathbf{D}
     \mathbf{U}_{W}^{\top}
     \mathbf{Y}.
     \label{eq:im_WH_form}
    \end{aligned}    
 \end{align}
 
 \noindent Based on this result, we now calculate the Frobenius norm of $\mathbf{W} \mathbf{H} - \mathbf{Y}$:
 \begin{align}
     \| \mathbf{W} \mathbf{H} - \mathbf{Y}  \|_F^2
     &= \| \mathbf{U}_{W}
     \mathbf{D}
     \mathbf{U}_{W}^{\top}
     \mathbf{Y}  \|_F^2
     = \operatorname{trace}
     (\mathbf{U}_{W}
     \mathbf{D}
     \mathbf{U}_{W}^{\top}
     \mathbf{Y} (\mathbf{U}_{W}
     \mathbf{D}
     \mathbf{U}_{W}^{\top}
     \mathbf{Y})^{\top}  ) \nonumber \\
     &= \operatorname{trace}
     (\mathbf{U}_{W}
     \mathbf{D}
     \mathbf{U}_{W}^{\top}
     \mathbf{Y} \mathbf{Y}^{\top} \mathbf{U}_{W} \mathbf{D}
     \mathbf{U}_{W}^{\top}
     ) 
     =  \operatorname{trace}
     ( \mathbf{D}^{2} \mathbf{U}_{W}^{\top}  \mathbf{Y} \mathbf{Y}^{\top}   \mathbf{U}_{W} ).
 \end{align}
    We denote $\mathbf{u}^{k}$ and $\mathbf{u}_{k}$ are the $k$-th row and column of $\mathbf{U}_{W}$, respectively. Let $\mathbf{n} = (n_{1}, \ldots, n_{K})$, we have the following:
    \begin{align}
    \begin{gathered}    
        \mathbf{U}_{W} =
        \begin{bmatrix}
         -\mathbf{u}^{1}-  \\
         \ldots  \\
        -\mathbf{u}^{K}-  \\
        \end{bmatrix}
        = \begin{bmatrix}
        | & | & |  \\
        \mathbf{u}_{1} & \ldots & \mathbf{u}_{K}  \\
        | & | & | \\
        \end{bmatrix},  \\
        \mathbf{Y} \mathbf{Y}^{\top} = \operatorname{diag}(n_{1}, n_{2}, \ldots, n_{K}) \in \mathbb{R}^{K \times K}  \\
        \Rightarrow 
         \mathbf{U}_{W}^{\top}
         \mathbf{Y} \mathbf{Y}^{\top} \mathbf{U}_{W} =
         \begin{bmatrix}
        | & | & |  \\
        (\mathbf{u}^{1})^{\top} & \ldots & (\mathbf{u}^{K})^{\top}  \\
        | & | & | \\
        \end{bmatrix}
        \operatorname{diag}(n_{1}, n_{2}, \ldots, n_{K})
        \begin{bmatrix}
         -\mathbf{u}^{1}-  \\
         \ldots  \\
        -\mathbf{u}^{K}-  \\
        \end{bmatrix}  \\
        =  \begin{bmatrix}
        | & | & |  \\
        (\mathbf{u}^{1})^{\top} & \ldots & (\mathbf{u}^{K})^{\top}  \\
        | & | & | \\
        \end{bmatrix}
        \begin{bmatrix}
         - n_{1}\mathbf{u}^{1}-  \\
         \ldots  \\
        - n_{k} \mathbf{u}^{K}-  \\
        \end{bmatrix} \\
        \Rightarrow 
         (\mathbf{U}_{W}^{\top}
         \mathbf{Y} \mathbf{Y}^{\top} \mathbf{U}_{W})_{kk} 
         = n_{1} u_{1k}^{2} + n_{2} u_{2k}^{2} + \ldots + n_{k} 
         u_{K k}^{2} = (\mathbf{u}_{k} \odot \mathbf{u}_{k} )^{\top} \mathbf{n} \\
         \Rightarrow
          \| \mathbf{W} \mathbf{H} - \mathbf{Y} \|_F^2 
         = 
         \operatorname{trace}
         (\mathbf{D}^{2} \mathbf{U}_{W}^{\top}
         \mathbf{Y} \mathbf{Y}^{\top} \mathbf{U}_{W}) 
         =
         \sum_{k=1}^{r}
         (\mathbf{u}_{k} \odot \mathbf{u}_{k} )^{\top} \mathbf{n} \frac{(-N \lambda_{H})^{2}  }{ ( s_{k}^{2} + N \lambda_{H})^{2}} + \sum_{h = r + 1}^{K}  (\mathbf{u}_{h} \odot \mathbf{u}_{h} )^{\top} \mathbf{n},
         \label{eq:im_WH_norm}
    \end{gathered}
    \end{align}
    where the last equality is from the fact that $\mathbf{D}^{2}$ is a diagonal matrix, so the diagonal of  $\mathbf{D}^{2} \mathbf{U}_{W}^{\top}
    \mathbf{Y} \mathbf{Y}^{\top} \mathbf{U}_{W}$ is the element-wise product between the diagonal of $\mathbf{D}^{2}$ and $\mathbf{U}_{W}^{\top}
    \mathbf{Y} \mathbf{Y}^{\top} \mathbf{U}_{W}$.
    \\
    
    \noindent Similarly, we calculate the Frobenius norm of $\mathbf{H}$, from equation \eqref{eq:im_H_form2}, we have:
    \begin{align}
        \| \mathbf{H} \|_F^2
        &= \operatorname{trace}
        ( \mathbf{V}_{W} 
        \mathbf{C} \mathbf{U}_{W}^{\top}
        \mathbf{Y} \mathbf{Y}^{\top}
        \mathbf{U}_{W} \mathbf{C}^{\top}
        \mathbf{V}_{W}^{\top}
        ) = 
        \operatorname{trace} 
        (  \mathbf{C}^{\top}  \mathbf{C} \mathbf{U}_{W}^{\top}
        \mathbf{Y} \mathbf{Y}^{\top}
        \mathbf{U}_{W}  ) \nonumber \\
        &= \sum_{k=1}^{K} 
        (\mathbf{u}_{k} \odot \mathbf{u}_{k} )^{\top} \mathbf{n} \frac{s_{k}^{2}  }{ ( s_{k}^{2} + N \lambda_{H})^{2}}.
        \label{eq:im_H_norm}
    \end{align}
    
    \noindent Now, we plug the equations \eqref{eq:im_WH_norm} and \eqref{eq:im_H_norm} into the function $f$, we get:
    \begin{align}
    \begin{aligned}
        f(\mathbf{W}, \mathbf{H}) 
        &= \frac{1}{2N} 
        \sum_{k=1}^{r}
        (\mathbf{u}_{k} \odot \mathbf{u}_{k} )^{\top} \mathbf{n} \frac{(-N \lambda_{H})^{2}  }{ ( s_{k}^{2} + N \lambda_{H})^{2}}
        + \frac{1}{2N}
        \sum_{h = r + 1}^{K}
        (\mathbf{u}_{h} \odot \mathbf{u}_{h} )^{\top} \mathbf{n}
        + \frac{\lambda_{W}}{2} \sum_{k=1}^{r} s_{k}^{2} \\
        &+ \frac{\lambda_{H}}{2}
        \sum_{k=1}^{K} 
        (\mathbf{u}_{k} \odot \mathbf{u}_{k} )^{\top} \mathbf{n} \frac{s_{k}^{2}  }{ ( s_{k}^{2} + N \lambda_{H})^{2}} \\
        &= 
        \frac{\lambda_{H}}{2}
        \sum_{k=1}^{r}
        \frac{ (\mathbf{u}_{k} \odot \mathbf{u}_{k} )^{\top} \mathbf{n}}{s_{k}^{2} + N \lambda_{H}}
        + \frac{\lambda_{W}}{2} \sum_{k=1}^{r} s_{k}^{2} + \frac{1}{2N} 
        \sum_{h = r + 1}^{K}
        (\mathbf{u}_{h} \odot \mathbf{u}_{h} )^{\top} \mathbf{n}
        \\
        &= \frac{1}{2N} \sum_{k=1}^{r}
        \left(  
        \frac{(\mathbf{u}_{k} \odot \mathbf{u}_{k} )^{\top} \mathbf{n}}{ \frac{s_{k}^{2}}{N \lambda_{H}} + 1} 
        + N^{2} \lambda_{W} \lambda_{H} \left( \frac{s_{k}^{2}}{N \lambda_{H}} \right)
        \right)
        +  \frac{1}{2N} 
        \sum_{h = r + 1}^{K}
        (\mathbf{u}_{h} \odot \mathbf{u}_{h} )^{\top} \mathbf{n}
        \\
        &= \frac{1}{2N} \sum_{k=1}^{r}
        \left( \frac{(\mathbf{u}_{k} \odot \mathbf{u}_{k} )^{\top} \mathbf{n}}{ x_{k} + 1} + b x_{k} 
        \right) + \frac{1}{2N} 
        \sum_{h = r + 1}^{K}
        (\mathbf{u}_{h} \odot \mathbf{u}_{h} )^{\top} \mathbf{n}
        \\
        &= \frac{1}{2N} \sum_{k=1}^{r}
        \left( \frac{a_{k}}{ x_{k} + 1} + b x_{k} 
        \right) + \frac{1}{2N} 
        \sum_{h = r + 1}^{K}
        a_{h},
        \label{eq:im_f_form}
    \end{aligned}
    \end{align}
    with $x_{k} := \frac{s_{k}^{2}}{ N \lambda_{H}}$, $a_{k} := (\mathbf{u}_{k} \odot \mathbf{u}_{k} )^{\top} \mathbf{n}$ and $b:= N^{2} \lambda_{W} \lambda_{H}$. 
    \\
    
    \noindent From the fact that $\mathbf{U}_{W}$ is an orthonormal matrix, we have:
    \begin{align}
        \sum_{k=1}^{K} a_{k} = 
        \sum_{k=1}^{K} 
        \left(\mathbf{u}_{k} \odot \mathbf{u}_{k} \right)^{\top} \mathbf{n} = 
        \left(\sum_{k=1}^{K} 
        \mathbf{u}_{k} \odot \mathbf{u}_{k} \right)^{\top} \mathbf{n}
        = \mathbf{1}^{\top} \mathbf{n}
        = \sum_{k=1}^{K} n_{k} = N,
        \label{eq:sum_a}
    \end{align}
    and, for any $j \in [K]$, denote $p_{i, j} := u_{i1}^{2} + u_{i2}^{2} + \ldots + u_{ij}^{2} \: \forall \: i \in [K]$, we have:
    \begin{align}
        \sum_{k=1}^{j} a_{k}
        &= 
       \sum_{k=1}^{j} (\mathbf{u}_{k} \odot \mathbf{u}_{k} )^{\top} \mathbf{n} 
        = 
        n_{1}( u_{11}^{2} + u_{12}^{2} + \ldots + u_{1j}^{2})
        + n_{2}( u_{21}^{2} + 
        u_{22}^{2} + \ldots + u_{2j}^{2})
        + \ldots \nonumber \\
        &+ n_{K} ( u_{K1}^{2} + 
        u_{K2}^{2} + \ldots + u_{K j}^{2}) 
        \nonumber \\
        &= \sum_{k=1}^{K} p_{k, j} n_{k} \leq p_{1, j}n_{1} + p_{2, j}n_{2} + \ldots + p_{j,j}n_{j} + (p_{j+1, j} + p_{j+2, j} + \ldots + p_{K, j})n_{j} \nonumber \\
        &= p_{1, j}n_{1} + p_{2, j}n_{2} + \ldots + p_{j-1, j}n_{j-1} + (j - p_{1,j} - \ldots - p_{j-1}, j)n_{j}
        \nonumber \\
        &= \sum_{k=1}^{j} n_{k} + \sum_{h=1}^{j-1} (n_{h}-n_{j})(p_{h, j} - 1) \leq \sum_{k=1}^{j} n_{k} \nonumber \\
        \Rightarrow \sum_{k = j+1}^{K} a_{k} &\geq N - \sum_{k=1}^{j} n_{k} = \sum_{k = j+1}^{K} n_{k} \quad \forall \: j \in [K],
        \label{eq:partial_a_1}
        \end{align}
    where we used the fact that $\sum_{k=1}^{K} p_{k, j} = j$ since it is the sum of squares of all entries of the first $j$ columns of an orthonormal matrix, and $p_{i, j} \leq 1 \: \forall \: i$ because it is the sum of squares of some entries on the $i$-th row of $\mathbf{U}_{W}$.
    \\
    
    We state a lemma regarding minimizing a weighted sum as following.
    \begin{lemma}
    \label{lm:weighted_sum}
    Consider a weighted sum $ \sum_{k=1}^{K} a_{k} z_{k}$ with $\{a_{k} \}_{k=1}^{K}$ satisfies \eqref{eq:sum_a} and \eqref{eq:partial_a_1} and $0 < z_{1} \leq z_{2} \leq \ldots \leq z_{K}$. Then, we have: 
    \begin{align}
        \min_{a_{1},\ldots,a_{K}} 
        \sum_{k=1}^{K} a_{k} z_{k}
        = \sum_{k=1}^{K} n_{k} z_{k} \nonumber.
    \end{align}
    The equality happens when for any $k \geq 1, z_{k+1} = z_{k}$ or $a_{k+1} + a_{k+2} + \ldots + a_{K} = n_{k+1} + n_{k+2} + \ldots + n_{K}$ (equivalently, $a_{1} + a_{2} + \ldots + a_{k} = n_{1} + n_{2} + \ldots + n_{k}$).
    \end{lemma}
    \begin{proof}[Proof of Lemma \ref{lm:weighted_sum}]
        We have:
        \begin{align}
            \sum_{k=1}^{K} a_{k} z_{k}
            &= 
            ( a_{1} + a_{2} + \ldots + a_{K} ) z_{1} 
            +
            (a_{2} + \ldots + a_{K} ) (z_{2} - z_{1}) 
            + \ldots 
            + a_{K} (z_{K} - z_{K-1}) \nonumber \\
            &\geq 
            ( n_{1} + n_{2} + \ldots + n_{K} ) z_{1} 
            +
            (n_{2} + \ldots + n_{K} ) (z_{2} - z_{1}) 
            + \ldots 
            + n_{K} (z_{K} - z_{K-1}) \nonumber \\
            &= \sum_{k=1}^{K} n_{k} z_{k}. \nonumber
        \end{align}
    \end{proof}

    By applying Lemma \ref{lm:weighted_sum} to the RHS of equation \eqref{eq:im_f_form} with $z_{k} = \frac{1}{x_{k} + 1} \: \forall \: k \leq r$ and $z_{k} = 1$ otherwise, we obtain:
    \begin{align}
        f(\mathbf{W}, \mathbf{H}) &\geq 
        \frac{1}{2N} \sum_{k=1}^{r}
        \left( 
        \frac{n_{k}}{x_{k} + 1} + bx_{k}
        \right) 
        +  \frac{1}{2N} 
        \sum_{h=r+1}^{K} n_{h} \label{eq:im_ine} \\
        &= \frac{1}{2N} \sum_{k=1}^{r}
        n_{k} \left( 
        \frac{1}{x_{k} + 1} + \frac{b}{n_{k}}  x_{k}
        \right) 
        + \frac{1}{2N} 
        \sum_{h=r+1}^{K} n_{h}.
        \label{eq:im_f_final}
    \end{align}

    \noindent Consider the function:
    \begin{align}
        g(x) = \frac{1}{x + 1} + ax \: \text{ with } x \geq 0, a > 0.
    \end{align}
    We consider two cases:
    \begin{itemize}
        \item If $a > 1$, $g(0) = 1$ and $g(x) > g(0) \: \forall x > 0$. Hence, $g(x)$ is minimized at $x = 0$ in this case.
        \item If $a \leq 1$, by using AM-GM, we have $g(x) = \frac{1}{x + 1} + a(x + 1) - a \geq 2\sqrt{a} - a$ with the equality holds iff $x = \sqrt{\frac{1}{a}} - 1$.
    \end{itemize}
    
    By applying this result to each term in the lower bound \eqref{eq:im_f_final}, we finish bounding $f(\mathbf{W}, \mathbf{H})$.
    \\

    Now, we study the equality conditions. In the lower bound \eqref{eq:im_f_final}, by letting $x_{k}^{*}$ be the minimizer of $\frac{1}{x_{k} + 1} + \frac{b}{n_{k}}x_{k}$ for all $k \leq r$ and $x_{k}^{*} = 0$ for all $k > r$ , there are only four possibilities as following:
    \begin{itemize}
        \item \textbf{Case A:} If $x^{*}_{1} > 0$ and $n_{1} > n_{2}$: we have $x^{*}_{1} = \sqrt{\frac{n_{1}}{b}} - 1 > \max(0, \sqrt{\frac{n_{2}}{b}} - 1) \geq x^{*}_{2}$ and therefore from the equality condition of Lemma \ref{lm:weighted_sum}, we have $a_{1} = n_{1}$. From the orthonormal property of $\mathbf{u}_{k}$, we have:
            \begin{align}
                a_{1} = (\mathbf{u}_{1} \odot \mathbf{u}_{1} )^{\top} \mathbf{n} = n_{1} u_{11}^{2} + 
                n_{2} u_{21}^{2} + \ldots + n_{k} u_{K1}^{2} \leq n_{1} (u_{11}^{2} + u_{21}^{2} + \ldots +  u_{K1}^{2}) = n_{1}.
                \nonumber
            \end{align}
        The equality holds when and only when $u_{11}^{2} = 1$ and $u_{21} = \ldots = u_{K1} = 0$.
        
        \item \textbf{Case B:} If $x^{*}_{1} > 0$ and there exists $1 < j \leq r$ such that $n_{1} = n_{2} = \ldots = n_{j} > n_{j+1}$, we have:
       \begin{align}
           \frac{1}{{x} + 1} + \frac{b}{n_{1}}x
           = \frac{1}{{x} + 1} + \frac{b}{n_{2}}x = \ldots
           = \frac{1}{{x} + 1} + \frac{b}{n_{j}}x, \nonumber
       \end{align}
       and thus, $x^{*}_{1} = x^{*}_{2} = \ldots = x^{*}_{j} > x^{*}_{j+1}$.
       Hence, from the equality condition of Lemma \ref{lm:weighted_sum}, we have $a_{1} + a_{2} + \ldots + a_{j} = n_{1} + \ldots + n_{j}$. 
       We have:
        \begin{align}
        \begin{aligned}
            \sum_{k=1}^{j} (\mathbf{u}_{k} \odot \mathbf{u}_{k} )^{\top} \mathbf{n} 
            &= 
            n_{1}( u_{11}^{2} + u_{12}^{2} + \ldots + u_{1j}^{2})
            + n_{2}( u_{21}^{2} + 
            u_{22}^{2} + \ldots + u_{2j}^{2})
            \nonumber \\ 
            &+ \ldots
            + n_{K} ( u_{K1}^{2} + 
            u_{K2}^{2} + \ldots + u_{K j}^{2}) 
            \leq \sum_{k=1}^{j} n_{k},
        \end{aligned}
        \end{align}
        where the inequality is from the fact that for any $k \in [K]$, $(u_{k1}^{2} + 
        u_{k2}^{2} + \ldots + u_{k j}^{2}) \leq 1$ and $\sum_{k=1}^{K} (u_{k1}^{2} +
        u_{k2}^{2} + \ldots + u_{k j}^{2}) = j$ and $n_{j} > n_{j+1}$. The equality holds iff $u_{k1}^{2} + 
        u_{k2}^{2} + \ldots + u_{k j}^{2} = 1 \: \forall \: k = 1, 2, \ldots, j$ and $u_{k1} = u_{k2} = \ldots = u_{kj} = 0 \: \forall \: k = j+1, \ldots, K$, i.e. the upper left sub-matrix size $j \times j$ of $\mathbf{U}_{W}$ is an orthonormal matrix and other entries of $\mathbf{U}_{W}$ lie on the same rows or columns with this sub-matrix must all equal $0$'s.

        \item \textbf{Case C:} If $x_{1}^{*} > 0$, $r < K$
        and there exists $r < j \leq K$ such that $n_{1} = n_{2} = \ldots = n_{r} = \ldots = n_{j} > n_{j+1}$, thus we have $x_{1}^{*} = x_{2}^{*} = \ldots = x_{r}^{*} > 0$ and $x_{r+1}^{*} = \ldots = x_{K}^{*} = 0$. Hence, from the equality condition of Lemma \ref{lm:weighted_sum}, we have $a_{1} + a_{2} + \ldots + a_{r} = n_{1} + \ldots + n_{r}$. 
       We have:
        \begin{align}
        \begin{aligned}
            \sum_{k=1}^{r} (\mathbf{u}_{k} \odot \mathbf{u}_{k} )^{\top} \mathbf{n} 
            &= 
            n_{1}( u_{11}^{2} + u_{12}^{2} + \ldots + u_{1r}^{2})
            + n_{2}( u_{21}^{2} + 
            u_{22}^{2} + \ldots + u_{2r}^{2})
            \nonumber \\ 
            &+ \ldots
            + n_{K} ( u_{K1}^{2} + 
            u_{K2}^{2} + \ldots + u_{K r}^{2}) 
            \leq \sum_{k=1}^{r} n_{k},
        \end{aligned}
        \end{align}
        where the inequality is from the fact that for any $k \in [K]$, $(u_{k1}^{2} + 
        u_{k2}^{2} + \ldots + u_{k r}^{2}) \leq 1$ and $\sum_{k=1}^{K} (u_{k1}^{2} +
        u_{k2}^{2} + \ldots + u_{k r}^{2}) = r$. 
        The equality holds iff $u_{k1} = u_{k2} = \ldots = u_{kr} = 0 \: \forall \: k = j+1, \ldots, K$, i.e., the upper left sub-matrix size $j \times r$ of $\mathbf{U}_{W}$ includes $r$ orthonormal vectors in $\mathbb{R}^{j}$ and the bottom left sub-matrix size $(K-j) \times r$ are all zeros. The other $K-r$ columns of $\mathbf{U}_{W}$ does not matter because $\mathbf{W}^{*}$ can be written as:
        \begin{align}
            \mathbf{W}^{*} =
            \sum_{k=1}^{r} s_{k}^{*} \mathbf{u}_{k} \mathbf{v}_{k}^{\top}, \nonumber
        \end{align}
        with $\mathbf{v}_{k}$ is the right singular vector that satisfies $\mathbf{W}^{* \top} \mathbf{u}_{k} = s_{k}^{*} \mathbf{v}_{k}$. Note that since $s_{1}^{*} = s_{2}^{*} = \ldots = s_{r}^{*} := s^{*}$, we have the compact SVD form as follows:
        \begin{align}
             \mathbf{W}^{*} = s^{*} \mathbf{U}_{W}^{'} \mathbf{V}_{W}^{' \top},
        \end{align}
        where $\mathbf{U}_{W}^{'} \in \mathbb{R}^{K \times r}$ and $\mathbf{V}_{W}^{'} \in \mathbb{R}^{d \times r}$. Especially, the last $K-j$ rows of $\mathbf{W}^{*}$ will be zeros since the last $K-j$ rows of $\mathbf{U}_{W}^{'}$ are zeros. Furthermore, tbhe matrix $\mathbf{U}_{W}^{'} \mathbf{U}_{W}^{' \top}$ after removing the last $K-j$ zero rows and the last $K-j$ zero columns  is the best rank-$r$ approximation of $\mathbf{I}_{j}$.
        \\
        
        We note that if \textbf{Case C} happens, then the number of positive singular values are limited by the matrix rank $r$ (e.g., by $r \leq R = \min(d,K) = d$ when $d < K$), and $n_{r} = n_{r+1}$, thus $x_{r}^{*} > 0$ and $x_{r+1}^{*} = 0$ ($x_{r+1}^{*}$ should equal $x_{r}^{*} > 0$ if it is not forced to be zero).
        
        \item \textbf{Case D:} If $x_{1}^{*} = 0$, we must have $x_{2}^{*} = \ldots = x_{K}^{*} = 0$, $\sum_{k=1}^{K} (\mathbf{u}_{k} \odot \mathbf{u}_{k} )^{\top} \mathbf{n}$ always equal $N$ and thus, $\mathbf{U}_{W}$ can be an arbitrary size $K \times K$ orthonormal matrix. 
    \end{itemize}

    We perform similar arguments as above  for all subsequent $x^{*}_{k}$'s, after we finish reasoning for prior ones. Before going to the conclusion, we first study the matrix $\mathbf{U}_{W}$. If $\textbf{Case C}$ does not happen for any $x_{k}^{*}$'s, we have:
            \begin{align}
            \mathbf{U}_{W} = \begin{bmatrix}
            \mathbf{A}_{1} & \mathbf{0} & \mathbf{0} & \mathbf{0} \\
            \mathbf{0} & \mathbf{A}_{2} & \mathbf{0} & \mathbf{0} \\
            \vdots & \vdots &  \ddots & \vdots \\
             \mathbf{0}& \mathbf{0} & \mathbf{0} & \mathbf{A}_{l} \\
            \end{bmatrix},
            \label{eq:U_W_normal}
            \end{align}
            where each $\mathbf{A}_{i}$ is an orthonormal block which corresponds with one or a group of classes that have the same number of training samples and their $x^{*} > 0$ (\textbf{Case A} and \textbf{Case B}) or corresponds with all classes with $x^{*} = 0$ (\textbf{Case D}). If \textbf{Case C} happens, we have:
            \begin{align}
            \mathbf{U}_{W} = \begin{bmatrix}
            \mathbf{A}_{1} & \mathbf{0} & \mathbf{0} & \mathbf{0} \\
            \mathbf{0} & \mathbf{A}_{2} & \mathbf{0} & \mathbf{0} \\
            \vdots & \vdots &  \ddots & \vdots \\
             \mathbf{0}& \mathbf{0} & \mathbf{0} & \mathbf{A}_{l} \\
            \end{bmatrix},
            \label{eq:U_W_abnormal}
            \end{align}
             where each $\mathbf{A}_{i}, i \in [l-1] $ is an orthonormal block which corresponds with one or a group of classes that have the same number of training samples and their $x^{*} > 0$ (\textbf{Case A} and \textbf{Case B}). $\mathbf{A}_{l}$ is the orthonormal block has the same property as $\mathbf{U}_{W}$ in \textbf{Case C}. 
    \\
    
    We consider \textbf{the case $d \geq K$} from now on. By using arguments about the minimizer of $g(x)$ applied to the lower bound \eqref{eq:im_f_final}, we consider three cases as following:
        \begin{itemize}
            \item \textbf{Case 1a:} $\frac{b}{n_{1}} \leq \frac{b}{n_{2}} \leq \ldots \leq \frac{b}{n_{K}} \leq 1$. \\

            Then, the lower bound \eqref{eq:im_f_final} is minimized at $(x_{1}^{*},x_{2}^{*},\ldots, x_{K}^{*}) 
            =   \left(\sqrt{\frac{n_1}{b}} - 1, \sqrt{\frac{n_2}{b}} - 1,\ldots,\sqrt{\frac{n_K}{b}} - 1 \right)$. 
            Therefore:
            \begin{align}
            (s_{1}^{*}, s_{2}^{*}, \ldots, s_{K}^{*} ) 
            =
            \left(
            \sqrt{ \sqrt{\frac{n_{1} \lambda_{H}}{\lambda_{W}} }
            - N \lambda_{H}},
            \sqrt{ \sqrt{\frac{n_{2} \lambda_{H}}{\lambda_{W}} }
            - N \lambda_{H}}, \ldots,
             \sqrt{ \sqrt{\frac{n_{K} \lambda_{H}}{\lambda_{W}} }
            - N \lambda_{H}}   \right).
            \end{align}

            First, we have the property that the features in each class $\mathbf{h}_{k,i}^{*}$ collapsed to their class-mean $\mathbf{h}_{k}^{*}$ $(\mathcal{NC}1)$. Let $\overline{\mathbf{H}}^{*} = \mathbf{V}_{W} \mathbf{C} \mathbf{U}_{W}^{\top}$, we know that $\mathbf{H}^{*} = \overline{\mathbf{H}}^{*} \mathbf{Y}$ from equation \eqref{eq:im_H_form2}. Then, columns from the $(n_{k-1} + 1)$-th until $(n_{k})$-th of $\mathbf{H}$ will all equals the $k$-th column of $\overline{\mathbf{H}}^{*}$, thus the features in class $k$ are collapsed to their class-mean $\mathbf{h}_{k}^{*}$ (which is the $k$-th column of $\overline{\mathbf{H}}^{*}$), i.e., $ \mathbf{h}_{k,1}^{*} = \mathbf{h}_{k,2}^{*} = \ldots =  \mathbf{h}_{k,n_{k}}^{*} \forall k \in [K]$.

            \textbf{Case C} never happens because if we assume we have $r < K$ positive singular values, meaning $s_{r}^{*} > 0$. Then, if $n_{r+1} = n_{r}$, we must have $s_{r+1}^{*} > 0$ (contradiction!).  Hence, $\mathbf{U}_{W}$ must have the form as in equation \eqref{eq:U_W_normal}, thus we can conclude the geometry of the following :
            \begin{align}
                \mathbf{W}^{*} \mathbf{W}^{* \top}
                &= \mathbf{U}_{W} \mathbf{S}_{W} \mathbf{S}_{W}^{\top} \mathbf{U}_{W}^{\top}
                =
                \operatorname{diag} \left\{ \sqrt{\frac{n_{1} \lambda_{H}}{\lambda_{W}} }
                    - N \lambda_{H},\sqrt{\frac{n_{2} \lambda_{H}}{\lambda_{W}} }
                    - N \lambda_{H}, \ldots, \sqrt{\frac{n_{K} \lambda_{H}}{\lambda_{W}} }
                    - N \lambda_{H} \right\}
                \in \mathbb{R}^{K \times K},
                \\
                \mathbf{W}^{*} \mathbf{H}^{*}
                &= \mathbf{U}_{W} 
                \operatorname{diag} \left\{\frac{s_{1}^{2}}{s_{1}^{2} + N \lambda_{H}} , \ldots, \frac{s_{K}^{2}}{s_{K}^{2}+ N \lambda_{H} } \right\}
              \mathbf{U}_{W}^{\top}
              \mathbf{Y} \nonumber \\
              &=
              \begin{bmatrix}
                \frac{s_{1}^{2}}{s_{1}^{2} + N \lambda_{H}} & 0 & \ldots & 0 \\
                0 & \frac{s_{2}^{2}}{s_{2}^{2} + N \lambda_{H}} &  \ldots & 0 \\
                \vdots & \vdots & \ddots & \vdots \\
                0 & 0 & \ldots & \frac{s_{K}^{2}}{s_{K}^{2} + N \lambda_{H}}
                \\
               \end{bmatrix}
               \begin{bmatrix}
                1 & \ldots & 1 & 0 & \ldots & 0 & \ldots & 0 & \ldots & 0 \\
                0 & \ldots & 0 & 1 & \ldots & 1 & \ldots & 0 & \ldots & 0 \\
                \vdots & \ddots & \vdots & \vdots & \ddots & \vdots  & \ldots & \vdots & \ddots & \vdots \\
                0 & \ldots & 0 & 0 & \ldots & 0 & \ldots & 1 & \ldots & 1 \\
                \end{bmatrix} \nonumber \\
                &= 
                \begin{bmatrix}
                 \frac{s_{1}^{2}}{s_{1}^{2} + N \lambda_{H}} \mathbf{1}_{n_{1}}^{\top} & \ldots & \mathbf{0}  \\
                 \vdots & \ddots & \vdots \\
                \mathbf{0} &  \ldots&  \frac{s_{K}^{2}}{s_{K}^{2} + N \lambda_{H}} \mathbf{1}_{n_{K}}^{\top} \\
                \end{bmatrix},
                \nonumber 
                \\
                \mathbf{H}^{* \top} \mathbf{H}^{*}
                &= \mathbf{Y}^{\top} \mathbf{U}_{W} \mathbf{C}^{T} \mathbf{C}
                \mathbf{U}_{W}^{\top}
                \mathbf{Y} 
                \nonumber 
                \\
                &= \mathbf{Y}^{\top}
                \begin{bmatrix}
                \frac{s_{1}^{2}}{(s_{1}^{2} + N \lambda_{H})^{2}} & 0 & \ldots & 0 \\
                0 & \frac{s_{2}^{2}}{(s_{2}^{2} + N \lambda_{H})^{2}} &  \ldots & 0 \\
                \vdots & \vdots & \ddots & \vdots \\
                0 & 0 & \ldots & \frac{s_{K}^{2}}{(s_{K}^{2} + N \lambda_{H})^{2}} \\
               \end{bmatrix}
                \mathbf{Y} \nonumber \\
                &= \begin{bmatrix}
                \frac{s_{1}^{2}}{(s_{1}^{2} + N \lambda_{H})^{2}} \mathbf{1}_{n_{1}}  \mathbf{1}_{n_{1}}^{\top} & \mathbf{0} & \ldots & \mathbf{0} \\
                \mathbf{0} &\frac{s_{2}^{2}}{(s_{2}^{2} + N \lambda_{H})^{2}} \mathbf{1}_{n_{2}}  \mathbf{1}_{n_{2}}^{\top}  & \ldots & \mathbf{0} \\
                \vdots & \vdots & \ddots & \vdots \\
                \mathbf{0} & \mathbf{0} & \ldots & \frac{s_{K}^{2}}{(s_{K}^{2} + N \lambda_{H})^{2}} \mathbf{1}_{n_{K}}  \mathbf{1}_{n_{K}}^{\top}  \\
                \end{bmatrix} \in \mathbb{R}^{N \times N},
            \end{align}
            where $\mathbf{1}_{n_{k}} \mathbf{1}_{n_{k}}^{\top}$ is a $n_{k} \times n_{k}$ matrix will all entries are $1$'s.
            \\
            
            We additionally have the structure of the class-means matrix:
            \begin{align}
                \overline{\mathbf{H}}^{* \top}
                \overline{\mathbf{H}}^{*} &= \mathbf{U}_{W}^{\top} \mathbf{C}^{\top} 
                \mathbf{C}
                \mathbf{U}_{W}
                = \begin{bmatrix}
                \frac{s_{1}^{2}}{(s_{1}^{2} + N \lambda_{H})^{2}} & 0 & \ldots & 0 \\
                0 & \frac{s_{2}^{2}}{(s_{2}^{2} + N \lambda_{H})^{2}} &  \ldots & 0 \\
                \vdots & \vdots & \ddots & \vdots \\
                0 & 0 & \ldots & \frac{s_{K}^{2}}{(s_{K}^{2} + N \lambda_{H})^{2}} \\
               \end{bmatrix} \in \mathbb{R}^{K \times K}, \\
               \mathbf{W}^{*} \overline{\mathbf{H}}^{*} 
               &= \mathbf{U}_{W} \mathbf{S}_{W} \mathbf{C} \mathbf{U_{W}}^{\top}
               = \begin{bmatrix}
                \frac{s_{1}^{2}}{s_{1}^{2} + N \lambda_{H}} & 0 & \ldots & 0 \\
                0 & \frac{s_{2}^{2}}{s_{2}^{2} + N \lambda_{H}} &  \ldots & 0 \\
                \vdots & \vdots & \ddots & \vdots \\
                0 & 0 & \ldots &  \frac{s_{K}^{2}}{s_{K}^{2} + N \lambda_{H}} \\
               \end{bmatrix} \in \mathbb{R}^{K \times K}.
            \end{align}
            
            And the alignment between the linear classifier and features are as following. For any $k \in [K]$, denote $\mathbf{w}_{k}$ the $k$-th row of $\mathbf{W}^{*}$:
            \begin{align}
               \mathbf{W}^{*} &= \mathbf{U}_{W} \mathbf{S}_{W} \mathbf{V}_{W}^{\top}, \nonumber \\
                \overline{\mathbf{H}}^{*} &= \mathbf{V}_{W} 
                \mathbf{C}
                \mathbf{U}_{W}^{\top}
                 \nonumber \\
                \Rightarrow \mathbf{w}_{k}^{*} &= (s_{k}^{2} + N \lambda_{H}) \mathbf{h}_{k}^{*} 
                = \sqrt{\frac{n_{k} \lambda_{H} }{\lambda_{W}}} \mathbf{h}_{k}^{*}.
            \end{align}
            
            \item \textbf{Case 2a:} There exists $j \in [K-1]$ s.t. $ \frac{b}{n_{1}} \leq \frac{b}{n_{2}} \leq \ldots \leq \frac{b}{n_{j}} \leq 1 < \frac{b}{n_{j+1}} \leq \ldots \leq \frac{b}{n_{K}} $ \\
            
            Then, the lower bound \eqref{eq:im_f_final} is minimized at:
            \begin{align}
            (s_{1}^{*}, \ldots, s_{j}^{*}, s_{j+1}^{*} \ldots, s_{K}^{*} ) =
            \left(
            \sqrt{ \sqrt{\frac{n_{1} \lambda_{H}}{\lambda_{W}} }
            - N \lambda_{H}}, \ldots,
            \sqrt{ \sqrt{\frac{n_{j} \lambda_{H}}{\lambda_{W}} }
            - N \lambda_{H}}, 0,  \ldots,
             0
            \right).
            \end{align}
            
             First, we have the property that the features in each class $\mathbf{h}_{k,i}^{*}$ collapsed to their class-mean $\mathbf{h}_{k}^{*}$ $(\mathcal{NC}1)$. Let $\overline{\mathbf{H}}^{*} = \mathbf{V}_{W} \mathbf{C} \mathbf{U}_{W}^{\top}$, we know that $\mathbf{H}^{*} = \overline{\mathbf{H}}^{*}$ from equation \eqref{eq:im_H_form2}. Then, columns from the $(n_{k-1} + 1)$-th until $(n_{k})$-th of $\mathbf{H}^{*}$ will all equals the $k$-th column of $\overline{\mathbf{H}}^{*}$, thus the features in class $k$ are collapsed to their class-mean $\mathbf{h}_{k}^{*}$ (which is the $k$-th column of $\overline{\mathbf{H}}$), i.e $ \mathbf{h}_{k,1}^{*} = \mathbf{h}_{k,2}^{*} = \ldots =  \mathbf{h}_{k,n_{k}}^{*} \: \forall k \in [K]$. 
             \\
            
            Recall $\mathbf{U}_{W}$ with the form \eqref{eq:U_W_normal} (\textbf{Case C} cannot happen with the same reason as in \textbf{Case 1a}). From equations \eqref{eq:im_H_form2} and \eqref{eq:im_WH_form}, we can conclude the geometry of the following:
            \begin{align}
                \mathbf{W^{*}} \mathbf{W}^{* \top} 
                &= \mathbf{U}_{W} \mathbf{S}_{W}
                \mathbf{S}_{W}^{\top} \mathbf{U}_{W}^{\top} \nonumber \\
                &= \operatorname{diag} \left(\sqrt{\frac{n_{1} \lambda_{H}}{\lambda_{W}} }
                    - N \lambda_{H}, \sqrt{\frac{n_{2} \lambda_{H}}{\lambda_{W}} }
                    - N \lambda_{H}, \ldots, \sqrt{\frac{n_{j} \lambda_{H}}{\lambda_{W}} }
                    - N \lambda_{H}, 0,\ldots,0 \right),
                \\
                 \mathbf{W}^{*} \mathbf{H}^{*}
                &= \mathbf{U}_{W} 
                \operatorname{diag} \left( \frac{s_{1}^{2}}{s_{1}^{2} + N \lambda_{H}} , \ldots, \frac{s_{j}^{2}}{s_{j}^{2} + N \lambda_{H}}, 0,\ldots,0 \right)
              \mathbf{U}_{W}^{\top}
              \mathbf{Y} \nonumber \\
                &= 
                \begin{bmatrix}
                \frac{s_{1}^{2}}{s_{1}^{2} + N \lambda_{H}} \mathbf{1}_{n_{1}}^{\top} & \mathbf{0} & \ldots & \mathbf{0} \\
                \mathbf{0} & \frac{s_{2}^{2}}{s_{2}^{2} + N \lambda_{H}} \mathbf{1}_{n_{2}}^{\top} &  \ldots & \mathbf{0} \\
                \vdots & \vdots & \ddots & \vdots \\
                \mathbf{0} & \mathbf{0} & \ldots & \mathbf{0}_{n_{K}}^{\top} \\
                \end{bmatrix} \in \mathbb{R}^{K \times N}, 
                \nonumber \\
                \mathbf{H}^{* \top} \mathbf{H}^{*}
                &=
                \begin{bmatrix}
                \frac{s_{1}^{2}}{(s_{1}^{2} + N \lambda_{H})^{2}} \mathbf{1}_{n_{1}}  \mathbf{1}_{n_{1}}^{\top} & \mathbf{0} & \ldots & \mathbf{0} \\
                \mathbf{0} &\frac{s_{2}^{2}}{(s_{2}^{2} + N \lambda_{H})^{2}} \mathbf{1}_{n_{2}}  \mathbf{1}_{n_{2}}^{\top}  & \ldots & \mathbf{0} \\
                \vdots & \vdots & \ddots & \vdots \\
                \mathbf{0} & \mathbf{0} & \ldots & \mathbf{0}_{n_{K} \times n_{K}} \\
                \end{bmatrix} \in \mathbb{R}^{N \times N},
            \end{align}
            where $\mathbf{1}_{n_{k}} \mathbf{1}_{n_{k}}^{\top}$ is a $n_{k} \times n_{k}$ matrix will all entries are $1$'s.
            \\

            For any $k \in [K]$, denote $\mathbf{w}_{k}^{*}$ the $k$-th row of $\mathbf{W}^{*}$ and $\mathbf{v}_{k}$ the $k$-th column of $\mathbf{V}_{W}$, we have:
            \begin{align}
               \mathbf{W}^{*} &= \mathbf{U}_{W} \mathbf{S}_{W} \mathbf{V}_{W}^{\top}, \nonumber \\
                \overline{\mathbf{H}}^{*} &= \mathbf{V}_{W} 
                \mathbf{C}
                \mathbf{U}_{W}^{\top}
                 \nonumber \\
                \Rightarrow \mathbf{w}_{k}^{*} &= (s_{k}^{2} + N \lambda_{H}) \mathbf{h}_{k}^{*} 
                = \sqrt{\frac{n_{k} \lambda_{H} }{\lambda_{W}}} \mathbf{h}_{k}^{*}.
            \end{align}
            And, for $k > j$, we have $\mathbf{w}_{k}^{*} = \mathbf{h}_{k}^{*} = \mathbf{0}$, which means the optimal classifiers and features of class $k > j$ will be $\mathbf{0}$.
            
            \item \textbf{Case 3a:} $1 < \frac{b}{n_{1}} \leq \frac{b}{n_{2}} \leq \ldots 
            \leq \frac{b}{n_{R}}$
            \\
            
            Then, the lower bound \eqref{eq:im_f_final} is minimized at:
            \begin{align}
                (s_{1}^{*}, s_{2}^{*}, \ldots, s_{K}^{*} ) = (0,0,\ldots,0).
            \end{align}
            Hence, the global minimizer of $f$ in this case is $(\mathbf{W}^{*},\mathbf{H}^{*}) = (\mathbf{0}, \mathbf{0})$.
        \end{itemize}

    Now, we turn to consider \textbf{the bottleneck case $d < K$}, and thus, $r \leq R = d < K$. Again, we consider the following cases:
    \begin{itemize}
        \item \textbf{Case 1b:} $\frac{b}{n_{1}} \leq \frac{b}{n_{2}} \leq \ldots \leq \frac{b}{n_{R}} \leq 1$. \\

            Then, the lower bound \eqref{eq:im_f_final} is minimized at $(x_{1}^{*},x_{2}^{*},\ldots, x_{K}^{*}) 
            =   (\sqrt{\frac{n_1}{b}} - 1, \sqrt{\frac{n_2}{b}} - 1,\ldots,\sqrt{\frac{n_R}{b}} - 1, 0,\ldots,0)
            = (\sqrt{\frac{n_1}{N^{2} \lambda_{W} \lambda_{H}}} - 1,\sqrt{\frac{n_2}{N^{2} \lambda_{W} \lambda_{H}}} - 1,\ldots,\sqrt{\frac{n_R}{N^{2} \lambda_{W} \lambda_{H}}} - 1,0,\ldots,0)$. 
            Therefore:
            \begin{align}
            &(s_{1}^{*}, s_{2}^{*}, \ldots, s_{R}^{*}, s_{R+1}^{*},\ldots s_{K}^{*} ) \nonumber \\
            &=
            \left(
            \sqrt{ \sqrt{\frac{n_{1} \lambda_{H}}{\lambda_{W}} }
            - N \lambda_{H}},
            \sqrt{ \sqrt{\frac{n_{2} \lambda_{H}}{\lambda_{W}} }
            - N \lambda_{H}}, \ldots,
             \sqrt{ \sqrt{\frac{n_{R} \lambda_{H}}{\lambda_{W}} }
            - N \lambda_{H}}
            ,0,\ldots,0    \right).
            \end{align}

            We have $(\mathcal{NC}1)$ and $(\mathcal{NC}3)$ properties are the same as \textbf{Case 1a}.
            
            We have \textbf{Case C} happens iff $b / n_{R} < 1$ (i.e., $x_{R}^{*} > 0$) and $n_{R} = n_{R+1}$.
            Then, if $b/ n_{R} = 1$ or $n_{R} > n_{R+1}$, we have:
            \begin{align}
                \mathbf{W}^{*} \mathbf{W}^{* \top}
                &= \mathbf{U}_{W} \mathbf{S}_{W} \mathbf{S}_{W}^{\top} \mathbf{U}_{W}^{\top}
                =
                \begin{bmatrix}
                 \sqrt{\frac{n_{1} \lambda_{H}}{\lambda_{W}} }
                    - N \lambda_{H} & \ldots & 0 & \ldots & 0 &  \\
                 \vdots & \ddots & \vdots & \ddots & \vdots  \\
                0 &  \ldots&  \sqrt{\frac{n_{R} \lambda_{H}}{\lambda_{W}} }
                    - N \lambda_{H} & \ldots & 0 \\
                 \vdots & \ddots & \vdots & \ddots & \vdots  \\
                0 & \ldots & 0 & \ldots & 0 \\
                \end{bmatrix}
                \in \mathbb{R}^{K \times K}, \\
                  \overline{\mathbf{H}}^{* \top}
                \overline{\mathbf{H}}^{*} &= \mathbf{U}_{W}^{\top} \mathbf{C}^{\top} 
                \mathbf{C}
                \mathbf{U}_{W}
                = \begin{bmatrix}
                \frac{s_{1}^{2}}{(s_{1}^{2} + N \lambda_{H})^{2}} & 0 & \ldots & 0 \\
                0 & \frac{s_{2}^{2}}{(s_{2}^{2} + N \lambda_{H})^{2}} &  \ldots & 0 \\
                \vdots & \vdots & \ddots & \vdots \\
                0 & 0 & \ldots & 0 \\
               \end{bmatrix} \in \mathbb{R}^{K \times K}, \\
               \mathbf{W}^{*} \overline{\mathbf{H}}^{*} 
               &= \mathbf{U}_{W} \mathbf{S}_{W} \mathbf{C} \mathbf{U_{W}}^{\top}
               = \begin{bmatrix}
                \frac{s_{1}^{2}}{s_{1}^{2} + N \lambda_{H}} & 0 & \ldots & 0 \\
                0 & \frac{s_{2}^{2}}{s_{2}^{2} + N \lambda_{H}} &  \ldots & 0 \\
                \vdots & \vdots & \ddots & \vdots \\
                0 & 0 & \ldots & 0 \\
               \end{bmatrix} \in \mathbb{R}^{K \times K}.
            \end{align}
            Furthermore, we have $\mathbf{w}_{k}^{*} =\mathbf{h}_{k}^{*} = \mathbf{0}$ for $k > R$.
            \\
            
            If \textbf{Case C} happens, there exists $k \leq R$, $l > R$ such that $n_{k-1} > n_{k} = n_{k+1} = \ldots = n_{R} = \ldots = n_{l} > n_{l+1}$. Recall the form of $\mathbf{U}_{W}$ as in equation \eqref{eq:U_W_abnormal}, then:
            \begin{align}
            \mathbf{W}^{*} \mathbf{W}^{* \top}
            &= \begin{bmatrix}
            \sqrt{\frac{n_{1} \lambda_{H}}{\lambda_{W}} }
                                - N \lambda_{H} & \ldots & \mathbf{0} & \mathbf{0} & \mathbf{0}  \\
            \vdots & \ddots & \vdots & \vdots & \vdots  \\
            \mathbf{0} & \ldots & \sqrt{\frac{n_{k-1} \lambda_{H}}{\lambda_{W}} }
                                - N \lambda_{H} & \mathbf{0} & \mathbf{0}  \\
            \mathbf{0} & \ldots & \mathbf{0} & \left( \sqrt{\frac{n_{k} \lambda_{H}}{\lambda_{W}} }
            - N \lambda_{H} \right) \mathcal{P}_{R-k+1}(\mathbf{I}_{l-k+1}) &  \mathbf{0}  \\
            \mathbf{0} & \ldots & \mathbf{0} & \mathbf{0} & \mathbf{0}_{(K-l) \times (K-l)}   \\
            \end{bmatrix},  \\
             \overline{\mathbf{H}}^{* \top}
            \overline{\mathbf{H}}^{*} 
            &= 
            \begin{bmatrix}
            \frac{s_{1}^{2}}{(s_{1}^{2} + N \lambda_{H})^{2}} & \ldots & \mathbf{0} & \mathbf{0} & \mathbf{0}  \\
            \vdots & \ddots & \vdots & \vdots & \vdots  \\
            \mathbf{0} & \ldots & \frac{s_{k-1}^{2}}{(s_{k-1}^{2} + N \lambda_{H})^{2}} & \mathbf{0} & \mathbf{0}  \\
            \mathbf{0} & \ldots & \mathbf{0} & \frac{s_{k}^{2}}{(s_{k}^{2} + N \lambda_{H})^{2}} \mathcal{P}_{R-k+1}(\mathbf{I}_{l-k+1}) &  \mathbf{0}  \\
            \mathbf{0} & \ldots & \mathbf{0} & \mathbf{0} & \mathbf{0}_{(K-l) \times (K-l)}   \\
            \end{bmatrix}, \\
            \mathbf{W}^{*} \overline{\mathbf{H}}^{*}
            &= 
            \begin{bmatrix}
            \frac{s_{1}^{2}}{s_{1}^{2} + N \lambda_{H}} & \ldots & \mathbf{0} & \mathbf{0} & \mathbf{0}  \\
            \vdots & \ddots & \vdots & \vdots & \vdots  \\
            \mathbf{0} & \ldots & \frac{s_{k-1}^{2}}{s_{k-1}^{2} + N \lambda_{H}} & \mathbf{0} & \mathbf{0}  \\
            \mathbf{0} & \ldots & \mathbf{0} & \frac{s_{k}^{2}}{s_{k}^{2} + N \lambda_{H}} \mathcal{P}_{R-k+1}(\mathbf{I}_{l-k+1}) &  \mathbf{0}  \\
            \mathbf{0} & \ldots & \mathbf{0} & \mathbf{0} & \mathbf{0}_{(K-l) \times (K-l)}   \\
            \end{bmatrix},
            \end{align}
            and for any $k > l > R$, we have $\mathbf{w}_{k}^{*} = \mathbf{h}_{k}^{*} = \mathbf{0}$. 
            
            \item \textbf{Case 2b:} There exists $j \in [R-1]$ s.t. $ \frac{b}{n_{1}} \leq \frac{b}{n_{2}} \leq \ldots \leq \frac{b}{n_{j}} \leq 1 < \frac{b}{n_{j+1}} \leq \ldots \leq \frac{b}{n_{R}} $ \\
            
            Then, the lower bound \eqref{eq:im_f_final} is minimized at:
            \begin{align}
            (s_{1}^{*}, \ldots, s_{j}^{*}, s_{j+1}^{*} \ldots, s_{K}^{*} ) =
            \left(
            \sqrt{ \sqrt{\frac{n_{1} \lambda_{H}}{\lambda_{W}} }
            - N \lambda_{H}}, \ldots,
            \sqrt{ \sqrt{\frac{n_{j} \lambda_{H}}{\lambda_{W}} }
            - N \lambda_{H}}, 0,  \ldots,
             0
            \right).
            \end{align}

            We have $(\mathcal{NC}1)$ and $(\mathcal{NC}3)$ properties are the same as \textbf{Case 2a}.\\
            
          \textbf{Case C} does not happen in this case because $b/n_{R} > 1$ and thus, $x^{*}_{R} = 0$.  Thus, we can conclude the geometry of the following:
            \begin{align}
                \mathbf{W^{*}} \mathbf{W}^{* \top} 
                &= \mathbf{U}_{W} \mathbf{S}_{W} \mathbf{S}_{W}^{\top} \mathbf{U}_{W}^{\top} \nonumber \\
                &= \operatorname{diag} \left(\sqrt{\frac{n_{1} \lambda_{H}}{\lambda_{W}} }
                    - N \lambda_{H}, \sqrt{\frac{n_{2} \lambda_{H}}{\lambda_{W}} }
                    - N \lambda_{H}, \ldots, \sqrt{\frac{n_{j} \lambda_{H}}{\lambda_{W}} }
                    - N \lambda_{H}, 0,\ldots,0 \right),
                \\
                 \mathbf{W}^{*} \mathbf{H}^{*}
                &= \mathbf{U}_{W} 
                \operatorname{diag} \left( \frac{s_{1}^{2}}{s_{1}^{2} + N \lambda_{H}} , \ldots, \frac{s_{j}^{2}}{s_{j}^{2} + N \lambda_{H}}, 0,\ldots,0 \right)
              \mathbf{U}_{W}^{\top}
              \mathbf{Y} \nonumber \\
                &= 
                \begin{bmatrix}
                \frac{s_{1}^{2}}{s_{1}^{2} + N \lambda_{H}} \mathbf{1}_{n_{1}}^{\top} & \mathbf{0} & \ldots & \mathbf{0} \\
                \mathbf{0} & \frac{s_{2}^{2}}{s_{2}^{2} + N \lambda_{H}} \mathbf{1}_{n_{2}}^{\top} &  \ldots & \mathbf{0} \\
                \vdots & \vdots & \ddots & \vdots \\
                \mathbf{0} & \mathbf{0} & \ldots & \mathbf{0}_{n_{K}}^{\top} \\
                \end{bmatrix} \in \mathbb{R}^{K \times N}, 
                \nonumber \\
                \mathbf{H}^{* \top} \mathbf{H}^{*}
                &=
                \begin{bmatrix}
                \frac{s_{1}^{2}}{(s_{1}^{2} + N \lambda_{H})^{2}} \mathbf{1}_{n_{1}}  \mathbf{1}_{n_{1}}^{\top} & \mathbf{0} & \ldots & \mathbf{0} \\
                \mathbf{0} &\frac{s_{2}^{2}}{(s_{2}^{2} + N \lambda_{H})^{2}} \mathbf{1}_{n_{2}}  \mathbf{1}_{n_{2}}^{\top}  & \ldots & \mathbf{0} \\
                \vdots & \vdots & \ddots & \vdots \\
                \mathbf{0} & \mathbf{0} & \ldots & \mathbf{0}_{n_{K} \times n_{K}} \\
                \end{bmatrix} \in \mathbb{R}^{N \times N},
            \end{align}
            where $\mathbf{1}_{n_{k}} \mathbf{1}_{n_{k}}^{\top}$ is a $n_{k} \times n_{k}$ matrix will all entries are $1$'s. And for any $k > j$, $\mathbf{w}_{k}^{*} = \mathbf{h}_{k}^{*} = \mathbf{0}$. 
            
            \item \textbf{Case 3b:} $1 < \frac{b}{n_{1}} \leq \frac{b}{n_{2}} \leq \ldots 
            \leq \frac{b}{n_{R}}$
            \\
            
            Then, the lower bound \eqref{eq:im_f_final} is minimized at:
            \begin{align}
                (s_{1}^{*}, s_{2}^{*}, \ldots, s_{K}^{*} ) = (0,0,\ldots,0).
            \end{align}
            Hence, the global minimizer of $f$ in this case is $(\mathbf{W}^{*},\mathbf{H}^{*}) = (\mathbf{0}, \mathbf{0})$.
        \end{itemize}
\end{proof}
    
\section{Proof of Theorem \ref{thm:deep_imbalance}}
\label{sec:proofs_im_deep}

First, we state the results for the case that the hidden dimension at every linear layer $d_m$ is at least the number of classes $K$. The $\mathcal{NC}2$ geometry follows GOF structure except a special case where there are some $i$'s such that $a / n_{i}$ equal exactly $(M-1)^{\frac{M-1}{M}} / M^2$.

\begin{theorem}
\label{thm:im_deep_appen}
    Let $d_{m} \geq K \: \forall \: m \in [M]$ and $(\mathbf{W}_{M}^{*}, \mathbf{W}_{M-1}^{*},\ldots, \mathbf{W}_{2}^{*}, \mathbf{W}_{1}^{*}, \mathbf{H}_{1}^{*})$ be any global minimizer of problem \eqref{eq:deep_imbalance}. We have:
    \\
    
   ($\mathcal{NC}1) \quad
        \mathbf{H}_{1}^{*} = \overline{\mathbf{H}}^{*} \mathbf{Y} \nonumber
        \Leftrightarrow \mathbf{h}_{k,i}^{*} = \mathbf{h}_{k}^{*} \: \forall \: k \in [K], i \in [n_{k}], $
    where $\overline{\mathbf{H}}^{*} = [\mathbf{h}_{1}^{*},\ldots,\mathbf{h}_{K}^{*}]  \in \mathbb{R}^{d_{1} \times K}$.
    \\
    
    $(\mathcal{NC}2)$ Let $c:= \frac{\lambda_{W_{1}}^{M-1}}
    {\lambda_{W_{M}} \lambda_{W_{M-1}} \ldots \lambda_{W_{2}} }$, $a:= N \sqrt[M]{N  \lambda_{W_{M}} \lambda_{W_{M-1}} \ldots \lambda_{W_{1}} \lambda_{H_{1}}}$ and $\forall k \in [K]$,  $x^{*}_{k}$ is the largest positive solution of the equation $\frac{a}{n_{k}} - \frac{ x^{M-1}}{ (x^{M} + 1)^{2}} = 0$, we have the following:
    \begin{align}
    \begin{aligned}
             &\mathbf{W}^{*}_{M} \mathbf{W}^{* \top}_{M}
             = \frac{\lambda_{W_{1}}}{\lambda_{W_{M}}} \operatorname{diag}
             \left\{s_{k}^{2} \right\}_{k=1}^{K},
             \\
           &\overline{\mathbf{H}}^{* \top}
            \overline{\mathbf{H}}^{*} 
            =
            \operatorname{diag}
            \left\{
            \frac{c s_{k}^{2M}}{(c s_{k}^{2M} + N \lambda_{H_{1}})^{2}}
            \right\}_{k=1}^{K}, \nonumber
            \\
             &\mathbf{W}_{M}^{*} \mathbf{W}_{M-1}^{*} \ldots  \mathbf{W}_{1}^{*}   \mathbf{H}_{1}^{*} 
             = 
             \left\{
            \frac{c s_{k}^{2M}}{c s_{k}^{2M} + N \lambda_{H_{1}}}
            \right\}_{k=1}^{K} \mathbf{Y},
    \end{aligned}    
    \end{align}
    
    ($\mathcal{NC}3$) We have, $\forall \: k \in [K]$:
    \begin{align}
        (\mathbf{W}_{M}^{*} \mathbf{W}_{M-1}^{*} \ldots \mathbf{W}_{2}^{*} \mathbf{W}_{1}^{*})_{k} = (c s_{k}^{2M} + N \lambda_{H_{1}}) \mathbf{h}_{k}^{*}, \nonumber
    \end{align}
    where:
    \begin{itemize}
        \item If $\frac{a}{n_{1}} \leq \frac{a}{n_{2}} \leq \ldots \leq \frac{a}{n_{K}} < \frac{(M-1)^{\frac{M-1}{M}}}{M^{2}}$, we have:
        \begin{align}
        \begin{aligned}
        s_{k} = 
        \sqrt[2M]{\frac{N \lambda_{H_{1}} x_{k}^{* M}}{c}} \quad \forall \: k 
        .
        \nonumber
        \end{aligned}
        \end{align}
     
        \item If there exists a $j \in [K-1]$ s.t. $\frac{a}{n_{1}} \leq \frac{a}{n_{2}} \leq \ldots \leq \frac{a}{n_{j}} < \frac{(M-1)^{\frac{M-1}{M}}}{M^{2}} < \frac{a}{n_{j+1}} \leq \ldots \leq \frac{a}{n_{K}}$, we have:
        \begin{align}
        \begin{aligned}
        s_{k} = \left\{\begin{matrix}
        \sqrt[2M]{\frac{N \lambda_{H_{1}} x_{k}^{* M}}{c}} \quad \forall \: k \leq j \\ 0 \quad \forall \: k > j
        \end{matrix}\right. .
        \nonumber
        \end{aligned}
        \end{align}
        And, for any $k$ such that $s_{k} = 0$, we have:
        \begin{align}
            (\mathbf{W}_{M}^{*})_{k} = \mathbf{h}_{k}^{*} = \mathbf{0}. \nonumber
        \end{align}  

        \item If $\frac{(M-1)^{\frac{M-1}{M}}}{M^{2}} < \frac{a}{n_{1}} \leq \frac{a}{n_{2}} \leq \ldots \leq \frac{a}{n_{K}}$, we have:
        \begin{align}
           (s_{1}, s_{2}, \ldots, s_{K} ) &= (0,0,\ldots,0), \nonumber
        \end{align}
        and $(\mathbf{W}_{M}^{*}, \ldots, \mathbf{W}_{1}^{*}, \mathbf{H}_{1}^{*}) = (\mathbf{0}, \ldots, \mathbf{0}, \mathbf{0})$ in this case.
    \end{itemize}
     The only case left is if there exists $i, j \in [K]$ ($i \leq j \leq K$) such that $\frac{a}{n_{1}} \leq \frac{a}{n_{2}} \leq \ldots \leq \frac{a}{n_{i-1}} <
        \frac{a}{n_{i}} = \frac{a}{n_{i+1}} = \ldots =
        \frac{a}{n_{j}} = \frac{(M-1)^{\frac{M-1}{M}}}{M^{2}} < \frac{a}{n_{j+1}} \leq \frac{a}{n_{j+2}} \leq \ldots \leq \frac{a}{n_{K}}$, we have:
        \begin{align}
         s_{k} = \left\{\begin{matrix}
        \sqrt[2M]{N \lambda_{H_{1}} x_{k}^{* M}/ c} \quad \forall \: k \leq i - 1 \\  \sqrt[2M]{N \lambda_{H_{1}} x_{k}^{* M}/ c} \: \text{ or } \: 0 \quad \forall \: i \leq k \leq j
        \\ 0 \quad \forall \: k \geq j+1
        \end{matrix}\right. , \nonumber
        \end{align}
        furthermore, let $r$ is the largest index that $s_{r} > 0$, we must have $s_{r+1} = s_{r+2} = \ldots = s_{K} = 0$. $(\mathcal{NC}1)$ and $(\mathcal{NC}3)$ are the same as above but for $(\mathcal{NC}2)$:
            \begin{align}
                \mathbf{W}_{M}^{*} 
                 \mathbf{W}_{M}^{* \top}
                 &= 
                 \frac{\lambda_{W_{1}}}{\lambda_{W_{M}}}
                 \begin{bmatrix}
                    s_{1}^{2} & \ldots & \mathbf{0} & \mathbf{0} & \mathbf{0}  \\
                    \vdots & \ddots & \vdots & \vdots & \vdots  \\
                    \mathbf{0} & \ldots & s_{i-1}^{2}  & \mathbf{0} & \mathbf{0}  \\
                    \mathbf{0} & \ldots & \mathbf{0} & s_{i}^{2}  \mathcal{P}_{r-i+1}(\mathbf{I}_{j-i+1}) &  \mathbf{0}  \\
                    \mathbf{0} & \ldots & \mathbf{0} & \mathbf{0} & \mathbf{0}_{(K-j) \times (K-j)}   \\
                    \end{bmatrix}, 
                    \\
                \overline{\mathbf{H}}^{* \top}
                \overline{\mathbf{H}}^{*} &= \begin{bmatrix}
                \frac{c s_{1}^{2M}}{(c s_{1}^{2M} + N \lambda_{H_{1}})^{2}} & \ldots & \mathbf{0} & \mathbf{0} & \mathbf{0}  \\
                \vdots & \ddots & \vdots & \vdots & \vdots  \\
                \mathbf{0} & \ldots & \frac{c s_{i-1}^{2M}}{(c s_{i-1}^{2M} + N \lambda_{H_{1}})^{2}}  & \mathbf{0} & \mathbf{0}  \\
                \mathbf{0} & \ldots & \mathbf{0} & \frac{c s_{i}^{2M}}{(c s_{i}^{2M} + N \lambda_{H_{1}})^{2}}  \mathcal{P}_{r-i+1}(\mathbf{I}_{j-i+1}) &  \mathbf{0}  \\
                \mathbf{0} & \ldots & \mathbf{0} & \mathbf{0} & \mathbf{0}_{(K-j) \times (K-j)}   \\
                \end{bmatrix}, 
                \\
               \mathbf{W}_{M}^{*} \mathbf{W}_{M-1}^{*} \ldots \mathbf{W}_{2}^{*} \mathbf{W}_{1}^{*}  \overline{\mathbf{H}}^{*} 
               &= 
               \begin{bmatrix}
                 \frac{c s_{1}^{2M}}{c s_{1}^{2M} + N \lambda_{H_{1}}} & \ldots & \mathbf{0} & \mathbf{0} & \mathbf{0}  \\
                \vdots & \ddots & \vdots & \vdots & \vdots  \\
                \mathbf{0} & \ldots &  \frac{c s_{i-1}^{2M}}{c s_{i-1}^{2M} + N \lambda_{H_{1}}}  & \mathbf{0} & \mathbf{0}  \\
                \mathbf{0} & \ldots & \mathbf{0} &  \frac{c s_{i}^{2M}}{c s_{i}^{2M} + N \lambda_{H_{1}}}  \mathcal{P}_{r-i+1}(\mathbf{I}_{j-i+1}) &  \mathbf{0}  \\
                \mathbf{0} & \ldots & \mathbf{0} & \mathbf{0} & \mathbf{0}_{(K-j) \times (K-j)}   \\
                \end{bmatrix},
            \end{align}
            and, for any $h > j$, $(\mathbf{W}_{M}^{*} \mathbf{W}_{M-1}^{*} \ldots \mathbf{W}_{2}^{*} \mathbf{W}_{1}^{*})_{h} = \mathbf{h}_{h}^{*} = \mathbf{0}$.
    
\end{theorem}

Next, we state the results for bottleneck case where there exists a $m$ such that $d_m < K$. 
\begin{theorem}
\label{thm:im_deep_bottleneck}
    Let $R = \min(d_{M}, \ldots, d_{1}, K) < K$ and $(\mathbf{W}_{M}^{*}, \mathbf{W}_{M-1}^{*},\ldots, \mathbf{W}_{2}^{*}, \mathbf{W}_{1}^{*}, \mathbf{H}_{1}^{*})$ be any global minimizer of problem \eqref{eq:deep_imbalance}. We have:
    \\
    
   ($\mathcal{NC}1) \quad
        \mathbf{H}_{1}^{*} = \overline{\mathbf{H}}^{*} \mathbf{Y} \nonumber
        \Leftrightarrow \mathbf{h}_{k,i}^{*} = \mathbf{h}_{k}^{*} \: \forall \: k \in [K], i \in [n_{k}], $
    where $\overline{\mathbf{H}}^{*} = [\mathbf{h}_{1}^{*},\ldots,\mathbf{h}_{K}^{*}]  \in \mathbb{R}^{d_{1} \times K}$.
    \\

     ($\mathcal{NC}3) \quad \text{We have, }  \forall \: k \in [K]$:
    \begin{align}
        (\mathbf{W}_{M}^{*} \mathbf{W}_{M-1}^{*} \ldots \mathbf{W}_{2}^{*} \mathbf{W}_{1}^{*})_{k} = (c s_{k}^{2M} + N \lambda_{H_{1}}) \mathbf{h}_{k}^{*}, \nonumber
    \end{align}
    
    $(\mathcal{NC}2) \quad \text{Let } c:= \frac{\lambda_{W_{1}}^{M-1}}
    {\lambda_{W_{M}} \lambda_{W_{M-1}} \ldots \lambda_{W_{2}} }$, $a:= N \sqrt[M]{N  \lambda_{W_{M}} \lambda_{W_{M-1}} \ldots \lambda_{W_{1}} \lambda_{H_{1}}}$ and $\forall k \in [K]$,  $x^{*}_{k}$ is the largest positive solution of the equation $\frac{a}{n_{k}} - \frac{ x^{M-1}}{ (x^{M} + 1)^{2}} = 0$, we define $\left\{ s_{k} \right\}_{k=1}^{K}$ as follows:
    \begin{itemize}
        \item If $\frac{a}{n_{1}} \leq \frac{a}{n_{2}} \leq \ldots \leq \frac{a}{n_{R}} < \frac{(M-1)^{\frac{M-1}{M}}}{M^{2}}$, we have:
        \begin{align}
        \begin{aligned}
        s_{k} = \left\{\begin{matrix}
        \sqrt[2M]{\frac{N \lambda_{H_{1}} x_{k}^{* M}}{c}} \quad \forall \: k \leq R \\ 0 \quad \forall \: k > R
        \end{matrix}\right. .
        \nonumber
        \end{aligned}
        \end{align}

        Then, if $n_{R} > n_{R+1}$, we have:
        \begin{align}
        \begin{aligned}
             &\mathbf{W}^{*}_{M} \mathbf{W}^{* \top}_{M}
             = \frac{\lambda_{W_{1}}}{\lambda_{W_{M}}} \operatorname{diag}
             \left\{s_{k}^{2} \right\}_{k=1}^{K},
             \\
           &\overline{\mathbf{H}}^{* \top}
            \overline{\mathbf{H}}^{*} 
            =
            \operatorname{diag}
            \left\{
            \frac{c s_{k}^{2M}}{(c s_{k}^{2M} + N \lambda_{H_{1}})^{2}}
            \right\}_{k=1}^{K}, \nonumber
            \\
             &\mathbf{W}_{M}^{*} \mathbf{W}_{M-1}^{*} \ldots  \mathbf{W}_{1}^{*}   \overline{\mathbf{H}_{1}}^{*} 
             = 
             \left\{
            \frac{c s_{k}^{2M}}{c s_{k}^{2M} + N \lambda_{H_{1}}}
            \right\}_{k=1}^{K},
        \end{aligned}    
        \end{align}
        and for any $k > R$, we have $(\mathbf{W}_{M}^{*} \mathbf{W}_{M-1}^{*} \ldots \mathbf{W}_{2}^{*} \mathbf{W}_{1}^{*})_{k} = \mathbf{h}_{k}^{*} = \mathbf{0}$.
        \\
        
        Otherwise, if $n_{R} = n_{R+1}$, and there exists $k \leq R$, $l > R$ such that $n_{k-1} > n_{k} = n_{k+1} = \ldots = n_{R} = \ldots = n_{l} > n_{l+1}$, we have:
            \begin{align}
                &\mathbf{W}_{M}^{*} 
                 \mathbf{W}_{M}^{* \top}
                 = 
                 \frac{\lambda_{W_{1}}}{\lambda_{W_{M}}}
                 \begin{bmatrix}
                    s_{1}^{2} & \ldots & \mathbf{0} & \mathbf{0} & \mathbf{0}  \\
                    \vdots & \ddots & \vdots & \vdots & \vdots  \\
                    \mathbf{0} & \ldots & s_{k-1}^{2}  & \mathbf{0} & \mathbf{0}  \\
                    \mathbf{0} & \ldots & \mathbf{0} & s_{k}^{2}  \mathcal{P}_{R-k+1}(\mathbf{I}_{l-k+1}) &  \mathbf{0}  \\
                    \mathbf{0} & \ldots & \mathbf{0} & \mathbf{0} & \mathbf{0}_{(K-l) \times (K-l)}   \\
                    \end{bmatrix}, 
                    \\
                &\overline{\mathbf{H}}^{* \top}
                \overline{\mathbf{H}}^{*} = \begin{bmatrix}
                \frac{c s_{1}^{2M}}{(c s_{1}^{2M} + N \lambda_{H_{1}})^{2}} & \ldots & \mathbf{0} & \mathbf{0} & \mathbf{0}  \\
                \vdots & \ddots & \vdots & \vdots & \vdots  \\
                \mathbf{0} & \ldots & \frac{c s_{k-1}^{2M}}{(c s_{k-1}^{2M} + N \lambda_{H_{1}})^{2}}  & \mathbf{0} & \mathbf{0}  \\
                \mathbf{0} & \ldots & \mathbf{0} & \frac{c s_{k}^{2M}}{(c s_{k}^{2M} + N \lambda_{H_{1}})^{2}}  \mathcal{P}_{R-k+1}(\mathbf{I}_{l-k+1}) &  \mathbf{0}  \\
                \mathbf{0} & \ldots & \mathbf{0} & \mathbf{0} & \mathbf{0}_{(K-l) \times (K-l)}   \\
                \end{bmatrix}, 
                \\
               &\mathbf{W}_{M}^{*} \mathbf{W}_{M-1}^{*} \ldots \mathbf{W}_{1}^{*}  \overline{\mathbf{H}}^{*} 
               = 
               \begin{bmatrix}
                 \frac{c s_{1}^{2M}}{c s_{1}^{2M} + N \lambda_{H_{1}}} & \ldots & \mathbf{0} & \mathbf{0} & \mathbf{0}  \\
                \vdots & \ddots & \vdots & \vdots & \vdots  \\
                \mathbf{0} & \ldots &  \frac{c s_{k-1}^{2M}}{c s_{k-1}^{2M} + N \lambda_{H_{1}}}  & \mathbf{0} & \mathbf{0}  \\
                \mathbf{0} & \ldots & \mathbf{0} &  \frac{c s_{k}^{2M}}{c s_{k}^{2M} + N \lambda_{H_{1}}}  \mathcal{P}_{R-k+1}(\mathbf{I}_{l-k+1}) &  \mathbf{0}  \\
                \mathbf{0} & \ldots & \mathbf{0} & \mathbf{0} & \mathbf{0}_{(K-l) \times (K-l)}   \\
                \end{bmatrix},
            \end{align}
            and, for any $h > l > R$, $(\mathbf{W}_{M}^{*} \mathbf{W}_{M-1}^{*} \ldots \mathbf{W}_{2}^{*} \mathbf{W}_{1}^{*})_{h} = \mathbf{h}_{h}^{*} = \mathbf{0}$.
     
        \item If there exists a $j \in [R-1]$ s.t. $\frac{a}{n_{1}} \leq \frac{a}{n_{2}} \leq \ldots \leq \frac{a}{n_{j}} < \frac{(M-1)^{\frac{M-1}{M}}}{M^{2}} < \frac{a}{n_{j+1}} \leq \ldots \leq \frac{a}{n_{R}}$, we have:
        \begin{align}
        \begin{aligned}
        s_{k} = \left\{\begin{matrix}
        \sqrt[2M]{\frac{N \lambda_{H_{1}} x_{k}^{* M}}{c}} \quad \forall \: k \leq j \\ 0 \quad \forall \: k > j
        \end{matrix}\right. .
        \nonumber
        \end{aligned}
        \end{align}
        Then, we have:
        \begin{align}
        \begin{aligned}
             &\mathbf{W}^{*}_{M} \mathbf{W}^{* \top}_{M}
             = \frac{\lambda_{W_{1}}}{\lambda_{W_{M}}} \operatorname{diag}
             \left\{s_{k}^{2} \right\}_{k=1}^{K},
             \\
           &\overline{\mathbf{H}}^{* \top}
            \overline{\mathbf{H}}^{*} 
            =
            \operatorname{diag}
            \left\{
            \frac{c s_{k}^{2M}}{(c s_{k}^{2M} + N \lambda_{H_{1}})^{2}}
            \right\}_{k=1}^{K}, \nonumber
            \\
             &\mathbf{W}_{M}^{*} \mathbf{W}_{M-1}^{*} \ldots  \mathbf{W}_{1}^{*}   \overline{\mathbf{H}_{1}}^{*} 
             = 
             \left\{
            \frac{c s_{k}^{2M}}{c s_{k}^{2M} + N \lambda_{H_{1}}}
            \right\}_{k=1}^{K},
        \end{aligned}    
        \end{align}
        and for any $k > j$, we have $(\mathbf{W}_{M}^{*} \mathbf{W}_{M-1}^{*} \ldots \mathbf{W}_{2}^{*} \mathbf{W}_{1}^{*})_{k} = \mathbf{h}_{k}^{*} = \mathbf{0}$.

        \item If $\frac{(M-1)^{\frac{M-1}{M}}}{M^{2}} < \frac{a}{n_{1}} \leq \frac{a}{n_{2}} \leq \ldots \leq \frac{a}{n_{R}}$, we have:
        \begin{align}
           (s_{1}, s_{2}, \ldots, s_{K} ) &= (0,0,\ldots,0), \nonumber
        \end{align}
        and $(\mathbf{W}_{M}^{*}, \ldots, \mathbf{W}_{1}^{*}, \mathbf{H}_{1}^{*}) = (\mathbf{0}, \ldots, \mathbf{0}, \mathbf{0})$ in this case.
    \end{itemize}
     The only case left is if there exists $i, j \in [R]$ ($i \leq j \leq R$) such that $\frac{a}{n_{1}} \leq \frac{a}{n_{2}} \leq \ldots \leq \frac{a}{n_{i-1}} <
        \frac{a}{n_{i}} = \frac{a}{n_{i+1}} = \ldots =
        \frac{a}{n_{j}} = \frac{(M-1)^{\frac{M-1}{M}}}{M^{2}} < \frac{a}{n_{j+1}} \leq \frac{a}{n_{j+2}} \leq \ldots \leq \frac{a}{n_{R}}$, we have:
        \begin{align}
         s_{k} = \left\{\begin{matrix}
        \sqrt[2M]{N \lambda_{H_{1}} x_{k}^{* M}/ c} \quad \forall \: k \leq i - 1 \\  \sqrt[2M]{N \lambda_{H_{1}} x_{k}^{* M}/ c} \: \text{ or } \: 0 \quad \forall \: i \leq k \leq j
        \\ 0 \quad \forall \: k \geq j+1
        \end{matrix}\right. , \nonumber
        \end{align}
        furthermore, let $r$ is the largest index that $s_{r} > 0$, we must have $r \leq R$ and $s_{r+1} = s_{r+2} = \ldots = s_{K} = 0$. $(\mathcal{NC}1)$ and $(\mathcal{NC}3)$ are the same as above but for $(\mathcal{NC}2)$, we have:
            \begin{align}
                \mathbf{W}_{M}^{*} 
                 \mathbf{W}_{M}^{* \top}
                 &= 
                 \frac{\lambda_{W_{1}}}{\lambda_{W_{M}}}
                 \begin{bmatrix}
                    s_{1}^{2} & \ldots & \mathbf{0} & \mathbf{0} & \mathbf{0}  \\
                    \vdots & \ddots & \vdots & \vdots & \vdots  \\
                    \mathbf{0} & \ldots & s_{i-1}^{2}  & \mathbf{0} & \mathbf{0}  \\
                    \mathbf{0} & \ldots & \mathbf{0} & s_{i}^{2}  \mathcal{P}_{r-i+1}(\mathbf{I}_{j-i+1}) &  \mathbf{0}  \\
                    \mathbf{0} & \ldots & \mathbf{0} & \mathbf{0} & \mathbf{0}_{(K-j) \times (K-j)}   \\
                    \end{bmatrix}, 
                    \\
                \overline{\mathbf{H}}^{* \top}
                \overline{\mathbf{H}}^{*} &= \begin{bmatrix}
                \frac{c s_{1}^{2M}}{(c s_{1}^{2M} + N \lambda_{H_{1}})^{2}} & \ldots & \mathbf{0} & \mathbf{0} & \mathbf{0}  \\
                \vdots & \ddots & \vdots & \vdots & \vdots  \\
                \mathbf{0} & \ldots & \frac{c s_{i-1}^{2M}}{(c s_{i-1}^{2M} + N \lambda_{H_{1}})^{2}}  & \mathbf{0} & \mathbf{0}  \\
                \mathbf{0} & \ldots & \mathbf{0} & \frac{c s_{i}^{2M}}{(c s_{i}^{2M} + N \lambda_{H_{1}})^{2}}  \mathcal{P}_{r-i+1}(\mathbf{I}_{j-i+1}) &  \mathbf{0}  \\
                \mathbf{0} & \ldots & \mathbf{0} & \mathbf{0} & \mathbf{0}_{(K-j) \times (K-j)}   \\
                \end{bmatrix}, 
                \\
               \mathbf{W}_{M}^{*} \mathbf{W}_{M-1}^{*} \ldots \mathbf{W}_{2}^{*} \mathbf{W}_{1}^{*}  \overline{\mathbf{H}}^{*} 
               &= 
               \begin{bmatrix}
                 \frac{c s_{1}^{2M}}{c s_{1}^{2M} + N \lambda_{H_{1}}} & \ldots & \mathbf{0} & \mathbf{0} & \mathbf{0}  \\
                \vdots & \ddots & \vdots & \vdots & \vdots  \\
                \mathbf{0} & \ldots &  \frac{c s_{i-1}^{2M}}{c s_{i-1}^{2M} + N \lambda_{H_{1}}}  & \mathbf{0} & \mathbf{0}  \\
                \mathbf{0} & \ldots & \mathbf{0} &  \frac{c s_{i}^{2M}}{c s_{i}^{2M} + N \lambda_{H_{1}}}  \mathcal{P}_{r-i+1}(\mathbf{I}_{j-i+1}) &  \mathbf{0}  \\
                \mathbf{0} & \ldots & \mathbf{0} & \mathbf{0} & \mathbf{0}_{(K-j) \times (K-j)}   \\
                \end{bmatrix},
            \end{align}
            and, for any $h > j$, $(\mathbf{W}_{M}^{*} \mathbf{W}_{M-1}^{*} \ldots \mathbf{W}_{2}^{*} \mathbf{W}_{1}^{*})_{h} = \mathbf{h}_{h}^{*} = \mathbf{0}$.
    
\end{theorem}

We derive the proofs of both cases as following.

\begin{proof}[Proof of Theorem \ref{thm:im_deep_appen} and \ref{thm:im_deep_bottleneck}]
First, by using lemma \ref{lm:2}, we have for any critical point  $(\mathbf{W}_{M}, \mathbf{W}_{M-1},\ldots, \mathbf{W}_{2}, \mathbf{W}_{1}, \mathbf{H}_{1})$ of $f$, we have the following:
    \begin{align}
    \begin{gathered}
        \lambda_{W_{M}} \mathbf{W}^{\top}_{M} \mathbf{W}_{M} = \lambda_{W_{M-1}} \mathbf{W}_{M-1} \mathbf{W}^{\top}_{M-1}, \\
        \lambda_{W_{M-1}} \mathbf{W}^{\top}_{M-1} \mathbf{W}_{M-1} = \lambda_{W_{M-2}} \mathbf{W}_{M-2} \mathbf{W}^{\top}_{M-2}, \\
        \ldots \nonumber \\
        \lambda_{W_{2}} \mathbf{W}_{2}^{\top} \mathbf{W}_{2} = \lambda_{W_{1}} \mathbf{W}_{1} \mathbf{W}_{1}^{\top}, \\
        \lambda_{W_{1}} \mathbf{W}_{1}^{\top} \mathbf{W}_{1} =  \lambda_{H_{1}}\mathbf{H}_{1} \mathbf{H}_{1}^{\top}. 
    \end{gathered}
    \end{align}
    
    \noindent Let $\mathbf{W}_{1} = \mathbf{U}_{W_{1}} \mathbf{S}_{W_{1}} \mathbf{V}_{W_{1}}^{\top}$ be the SVD decomposition of $\mathbf{W}_{1}$ with $\mathbf{U}_{W_{1}} \in \mathbb{R}^{d_{2} \times d_{2}}, \mathbf{V}_{W_{1}} \in \mathbb{R}^{d_{1} \times d_{1}}$ are orthonormal matrices and $\mathbf{S}_{W_{1}} \in \mathbb{R}^{d_{2} \times d_{1}}$ is a diagonal matrix with \textbf{decreasing} non-negative singular values. We denote the $r$ singular values of $\mathbf{W}_{1}$  as $\left\{s_{k}\right\}_{k=1}^{r}$ ($r \leq R:= \min(K, d_{M}, \ldots, d_{1})$). From Lemma \ref{lm:4}, we have the SVD of other weight matrices as:
\begin{align}
    \begin{gathered}
        \mathbf{W}_{M} = \mathbf{U}_{W_{M}} \mathbf{S}_{W_{M}}
        \mathbf{U}_{W_{M-1}}^{\top},
        \\
        \mathbf{W}_{M-1} = \mathbf{U}_{W_{M-1}} \mathbf{S}_{W_{M-1}} \mathbf{U}_{W_{M-2}}^{\top},     \\
        \mathbf{W}_{M-2} = \mathbf{U}_{W_{M-2}} \mathbf{S}_{W_{M-2}} \mathbf{U}_{W_{M-3}}^{\top}, \\
        \mathbf{W}_{M-3} = \mathbf{U}_{W_{M-3}} \mathbf{S}_{W_{M-3}} \mathbf{U}_{W_{M-4}}^{\top}, \\
        \ldots, \\
        \mathbf{W}_{2} = \mathbf{U}_{W_{2}} \mathbf{S}_{W_{2}} \mathbf{U}_{W_{1}}^{\top},
        \\
        \mathbf{W}_{1} = \mathbf{U}_{W_{1}} \mathbf{S}_{W_{1}} \mathbf{V}_{W_{1}}^{\top} ,
        \nonumber
    \end{gathered}
    \end{align}
    with:
    \begin{align}
       \mathbf{S}_{W_{j}} = \sqrt{\frac{\lambda_{W_{1}}}{\lambda_{W_{j}}}}
        \begin{bmatrix}
        \operatorname{diag}(s_{1},\ldots, s_{r}) & \mathbf{0}_{r \times (d_{j} - r)}  \\
        \mathbf{0}_{(d_{j+1} - r) \times r} & \mathbf{0}_{(d_{j+1} - r) \times (d_{j} - r)}  \\
        \end{bmatrix} \in \mathbb{R}^{d_{j+1} \times d_{j}} \quad \forall \: j \in [M],
        \nonumber
    \end{align} 
    and $\mathbf{U}_{W_{M}}, \mathbf{U}_{W_{M-1}}, \mathbf{U}_{W_{M-2}}, \mathbf{U}_{W_{M-3}},\ldots, \mathbf{U}_{W_{1}}, \mathbf{V}_{W_{1}} $ are all orthonormal matrices.
    \\
    
    \noindent From Lemma \ref{lm:5}, denote $c := \frac{\lambda_{W_{1}}^{M-1}}
    {\lambda_{W_{M}} \lambda_{W_{M-1}} \ldots \lambda_{W_{2}} }$, we have:
    \begin{align}
    \begin{aligned}
    \mathbf{H}_{1} &= 
    \mathbf{V}_{W_{1}}
    \underbrace{
    \begin{bmatrix}
    \operatorname{diag}\left(
        \frac{\sqrt{c} s_{1}^{M} }{c s_{1}^{2M} + N \lambda_{H_{1}}}, \ldots, \frac{\sqrt{c} s_{r}^{M}}{c s_{r}^{2M} + N \lambda_{H_{1}}}
        \right) &  \mathbf{0}\\
    \mathbf{0} &  \mathbf{0}\\
    \end{bmatrix}}_{\mathbf{C} \in \mathbb{R}^{d_{1} \times K}}
    \mathbf{U}_{W_{M}}^{\top}
    \mathbf{Y} \\
    &= \mathbf{V}_{W_{1}}
    \mathbf{C}
    \mathbf{U}_{W_{M}}^{\top}
    \mathbf{Y}.
    \label{eq:im_H1_form}
    \end{aligned}
    \end{align} 
    
    \begin{align}
    \begin{aligned}
        \mathbf{W}_{M} \mathbf{W}_{M-1} \ldots \mathbf{W}_{2} \mathbf{W}_{1}   \mathbf{H} - \mathbf{Y} 
        &= \mathbf{U}_{W_{M}} 
        \underbrace{
        \begin{bmatrix}
        \operatorname{diag}\left(
            \frac{- N \lambda_{H_{1}} }{c s_{1}^{2M} + N \lambda_{H_{1}}}, \ldots, \frac{- N \lambda_{H_{1}}}{c s_{r}^{2M} + N \lambda_{H_{1}}}
            \right ) &  \mathbf{0}\\
        \mathbf{0} &  -\mathbf{I}_{K-r}\\
        \end{bmatrix}}_{\mathbf{D} \in \mathbb{R}^{K \times K}}
        \mathbf{U}_{W_{M}}^{\top}
        \mathbf{Y} \\
        &= \mathbf{U}_{W_{M}}
        \mathbf{D}
        \mathbf{U}_{W_{M}}^{\top}
        \mathbf{Y}.
        \label{eq:im_WH1_form}
    \end{aligned}
    \end{align}
    
    \noindent Next, we will calculate the Frobenius norm of $ \mathbf{W}_{M} \mathbf{W}_{M-1} \ldots \mathbf{W}_{2} \mathbf{W}_{1} \mathbf{H}_{1} - \mathbf{Y}$:
         \begin{align}
             \| \mathbf{W}_{M} \mathbf{W}_{M-1} \ldots \mathbf{W}_{2} \mathbf{W}_{1} \mathbf{H}_{1} - \mathbf{Y}  \|_F^2
             &= \| \mathbf{U}_{W_{M}}
             \mathbf{D}
             \mathbf{U}_{W_{M}}^{\top}
             \mathbf{Y}  \|_F^2
             = \operatorname{trace}
             (\mathbf{U}_{W_{M}}
             \mathbf{D}
             \mathbf{U}_{W_{M}}^{\top}
             \mathbf{Y} (\mathbf{U}_{W_{M}}
             \mathbf{D}
             \mathbf{U}_{W_{M}}^{\top}
             \mathbf{Y})^{\top}  ) \nonumber \\
             &= \operatorname{trace}
             (\mathbf{U}_{W_{M}}
             \mathbf{D}
             \mathbf{U}_{W_{M}}^{\top}
             \mathbf{Y} \mathbf{Y}^{\top} \mathbf{U}_{W_{M}} \mathbf{D}
             \mathbf{U}_{W_{M}}^{\top}
             ) \nonumber \\
             &=  \operatorname{trace}
             ( \mathbf{D}^{2} \mathbf{U}_{W_{M}}^{\top}  \mathbf{Y} \mathbf{Y}^{\top}   \mathbf{U}_{W_{M}} ).
             \nonumber
         \end{align}
        
        \noindent We denote $\mathbf{u}^{k}$ and $\mathbf{u}_{k}$ are the $k$-th row and column of $\mathbf{U}_{W_{M}}$, respectively. Let $\mathbf{n} = (n_{1}, \ldots, n_{K})$, we have the following:
        \begin{align}
        \begin{gathered}    
            \mathbf{U}_{W_{M}} =
            \begin{bmatrix}
             -\mathbf{u}^{1}-  \\
             \ldots  \\
            -\mathbf{u}^{K}-  \\
            \end{bmatrix}
            = \begin{bmatrix}
            | & | & |  \\
            \mathbf{u}_{1} & \ldots & \mathbf{u}_{K}  \\
            | & | & | \\
            \end{bmatrix},  \\
            \mathbf{Y} \mathbf{Y}^{\top} = \operatorname{diag}(n_{1}, n_{2}, \ldots, n_{K}) \in \mathbb{R}^{K \times K}  \\
            \Rightarrow 
             \mathbf{U}_{W_{M}}^{\top}
             \mathbf{Y} \mathbf{Y}^{\top} \mathbf{U}_{W_{M}} =
             \begin{bmatrix}
            | & | & |  \\
            (\mathbf{u}^{1})^{\top} & \ldots & (\mathbf{u}^{K})^{\top}  \\
            | & | & | \\
            \end{bmatrix}
            \operatorname{diag}(n_{1}, n_{2}, \ldots, n_{K})
            \begin{bmatrix}
             -\mathbf{u}^{1}-  \\
             \ldots  \\
            -\mathbf{u}^{K}-  \\
            \end{bmatrix}  \\
            =  \begin{bmatrix}
            | & | & |  \\
            (\mathbf{u}^{1})^{\top} & \ldots & (\mathbf{u}^{K})^{\top}  \\
            | & | & | \\
            \end{bmatrix}
            \begin{bmatrix}
             - n_{1}\mathbf{u}^{1}-  \\
             \ldots  \\
            - n_{k} \mathbf{u}^{K}-  \\
            \end{bmatrix} \\
            \Rightarrow 
             (\mathbf{U}_{W_{M}}^{\top}
             \mathbf{Y} \mathbf{Y}^{\top} \mathbf{U}_{W_{M}})_{kk} 
             = n_{1} u_{1k}^{2} + n_{2} u_{2k}^{2} + \ldots + n_{k} 
             u_{K k}^{2} = (\mathbf{u}_{k} \odot \mathbf{u}_{k} )^{\top} \mathbf{n}
        \end{gathered}
        \end{align}
        
        \begin{align}
             \Rightarrow
              \| \mathbf{W}_{M} \mathbf{W}_{M-1} \ldots \mathbf{W}_{2} \mathbf{W}_{1} \mathbf{H}_{1} - \mathbf{Y} \|_F^2 
             &= 
             \operatorname{trace}
             (\mathbf{D}^{2} \mathbf{U}_{W}^{\top}
             \mathbf{Y} \mathbf{Y}^{\top} \mathbf{U}_{W}) \nonumber \\
             &=
             \sum_{k=1}^{r}
             (\mathbf{u}_{k} \odot \mathbf{u}_{k} )^{\top} \mathbf{n} \frac{(-N \lambda_{H_{1}})^{2}  }{ ( c s_{k}^{2M} + N \lambda_{H_{1}})^{2}}  
             + \sum_{h= r+ 1}^{K}
             (\mathbf{u}_{h} \odot \mathbf{u}_{h} )^{\top} \mathbf{n},
             \label{eq:im_WH1_norm}
        \end{align}
        where the last equality is from the fact that $\mathbf{D}^{2}$ is a diagonal matrix, so the diagonal of  $\mathbf{D}^{2} \mathbf{U}_{W_{M}}^{\top}
        \mathbf{Y} \mathbf{Y}^{\top} \mathbf{U}_{W_{M}}$ is the element-wise product between the diagonal of $\mathbf{D}^{2}$ and $\mathbf{U}_{W_{M}}^{\top}
        \mathbf{Y} \mathbf{Y}^{\top} \mathbf{U}_{W_{M}}$.
        \\
        
        \noindent Similarly, we calculate the Frobenius norm of $\mathbf{H}_{1}$, from equation \eqref{eq:im_H1_form}, we have:
        \begin{align}
            \| \mathbf{H}_{1} \|_F^2
            &= \operatorname{trace}
            ( \mathbf{V}_{W_{1}} 
            \mathbf{C} \mathbf{U}_{W_{M}}^{\top}
            \mathbf{Y} \mathbf{Y}^{\top}
            \mathbf{U}_{W_{M}} \mathbf{C}^{\top}
            \mathbf{V}_{W_{1}}^{\top}
            ) = 
            \operatorname{trace} 
            (  \mathbf{C}^{\top}  \mathbf{C} \mathbf{U}_{W_{M}}^{\top}
            \mathbf{Y} \mathbf{Y}^{\top}
            \mathbf{U}_{W_{M}}  ) \nonumber \\
            &= \sum_{k=1}^{r} 
            (\mathbf{u}_{k} \odot \mathbf{u}_{k} )^{\top} \mathbf{n} \frac{c s_{k}^{2M}  }{ ( cs_{k}^{2M} + N \lambda_{H_{1}})^{2}}.
            \label{eq:im_H1_norm}
        \end{align}
        
        \noindent Now, we plug the equations \eqref{eq:im_WH1_norm}, \eqref{eq:im_H1_norm} and the SVD of weight matrices into the function $f$ and note that orthonormal matrix does not change Frobenius norm, we got:
        \begin{align}
            &f
            =
            \frac{1}{2N} 
            \| \mathbf{W}_{M} \mathbf{W}_{M-1} \ldots \mathbf{W}_{1}   \mathbf{H}_{1} - \mathbf{Y} \|_F^2   
            + \frac{\lambda_{W_{M}}}{2} \| \mathbf{W}_{M} \|^2_F + \ldots + \frac{\lambda_{W_{1}}}{2} \| \mathbf{W}_{1} \|^2_F +\frac{ \lambda_{H_{1}}}{2}\left\|\mathbf{H}_{1}
            \right\|_{F}^{2} \nonumber \\
            &= 
            \frac{1}{2N} \sum_{k=1}^{r}
             (\mathbf{u}_{k} \odot \mathbf{u}_{k} )^{\top} \mathbf{n} \frac{(-N \lambda_{H_{1}})^{2}  }{ ( c s_{k}^{2M} + N \lambda_{H_{1}})^{2}} 
             + 
             \frac{1}{2N}
             \sum_{h= r+ 1}^{K}
             (\mathbf{u}_{h} \odot \mathbf{u}_{h} )^{\top} \mathbf{n} 
            + 
            \frac{\lambda_{W_{M}}}{2} \sum_{k = 1}^{r}
            \frac{\lambda_{W_{1}}}{\lambda_{W_{M}}} s_{k}^{2}
            \nonumber \\
            &+ \frac{\lambda_{W_{M-1}}}{2} \sum_{k=1}^{r}  \frac{\lambda_{W_{1}}}{\lambda_{W_{M-1}}} s_{k}^{2}
            + \ldots + \frac{\lambda_{W_{1}}}{2} \sum_{k=1}^{r} s_{k}^{2} + \frac{\lambda_{H_{1}}}{2} 
            \sum_{k=1}^{r} 
            (\mathbf{u}_{k} \odot \mathbf{u}_{k} )^{\top} \mathbf{n} \frac{c s_{k}^{2M}  }{ ( cs_{k}^{2M} + N \lambda_{H_{1}})^{2}}
            \nonumber \\
            &= \frac{\lambda_{H_{1}}}{2} \sum_{k=1}^{r}
            \frac{(\mathbf{u}_{k} \odot \mathbf{u}_{k} )^{\top} \mathbf{n}}
            {c s_{k}^{2M} + N \lambda_{H_{1}}} 
            + \frac{1}{2N} \sum_{h= r+ 1}^{K}
             (\mathbf{u}_{h} \odot \mathbf{u}_{h} )^{\top} \mathbf{n} 
            + \frac{M \lambda_{W_{1}}}{2} \sum_{k=1}^{r} s_{k}^{2} \nonumber \\
            &= \frac{1}{2N} 
            \sum_{k=1}^{r} 
            \left( 
            \frac{(\mathbf{u}_{k} \odot \mathbf{u}_{k} )^{\top} \mathbf{n}}{\frac{c s_{k}^{2M}}{N \lambda_{H_{1}}} + 1} + M N \lambda_{W_{1}} \sqrt[M]{\frac{N \lambda_{H_{1}}}{c}} \left(\sqrt[M]{\frac{c s_{k}^{2M}}{N \lambda_{H_{1}}}} \right) 
            \right)
            + \frac{1}{2N} \sum_{h= r+ 1}^{K}
             (\mathbf{u}_{h} \odot \mathbf{u}_{h} )^{\top} \mathbf{n} \nonumber \\
            &=  \frac{1}{2N} 
            \sum_{k=1}^{r} 
            \left( \frac{(\mathbf{u}_{k} \odot \mathbf{u}_{k} )^{\top} \mathbf{n}}{x^{M}_{k} +1} + bx_{k} 
            \right)
            + \frac{1}{2N} \sum_{h= r+ 1}^{K}
             (\mathbf{u}_{h} \odot \mathbf{u}_{h} )^{\top} \mathbf{n} \nonumber \\
             &= 
             \frac{1}{2N} 
            \sum_{k=1}^{r} 
            \left( \frac{a_{k}}{x^{M}_{k} +1} + bx_{k} 
            \right)
            + \frac{1}{2N} \sum_{h= r+ 1}^{K}
             a_{h}
             ,
         \label{eq:im_f_norm_deep}
        \end{align}
        with $x_{k} := \sqrt[M]{\frac{c s_{k}^{2M}}{N \lambda_{H_{1}}}} $,  $a_{k} := (\mathbf{u}_{k} \odot \mathbf{u}_{k} )^{\top} \mathbf{n}$ and $b:= M N \lambda_{W_{1}} \sqrt[M]{\frac{N \lambda_{H_{1}}}{c}} = M N \lambda_{W_{1}}
        \sqrt[M]{\frac{N \lambda_{W_{M}} \lambda_{W_{M-1}} \ldots \lambda_{W_{2}} \lambda_{H_{1}}}{ \lambda_{W_{1}}^{M-1}}}  = 
        M N \sqrt[M]{N  \lambda_{W_{M}} \lambda_{W_{M-1}} \ldots \lambda_{W_{1}} \lambda_{H_{1}}}$.
        \\
        
        \noindent From the fact that
        $\mathbf{U}_{W}$ is an orthonormal matrix, we have:
        \begin{align}
             \sum_{k=1}^{K} a_{k} = 
            \sum_{k=1}^{K} 
            (\mathbf{u}_{k} \odot \mathbf{u}_{k} )^{\top} \mathbf{n} = 
            \left(\sum_{k=1}^{K} 
            \mathbf{u}_{k} \odot \mathbf{u}_{k} \right)^{\top} \mathbf{n}
            = \mathbf{1}^{\top} \mathbf{n} 
            = \sum_{k=1}^{K} n_{k} = N,
        \end{align}
        and, for any $j \in [K]$, denote $p_{i, j} := u_{i1}^{2} + u_{i2}^{2} + \ldots + u_{ij}^{2} \: \forall \: i \in [K]$, we have:
    \begin{align}
        \sum_{k=1}^{j} a_{k}
        &= 
       \sum_{k=1}^{j} (\mathbf{u}_{k} \odot \mathbf{u}_{k} )^{\top} \mathbf{n} 
        = 
        n_{1}( u_{11}^{2} + u_{12}^{2} + \ldots + u_{1j}^{2})
        + n_{2}( u_{21}^{2} + 
        u_{22}^{2} + \ldots + u_{2j}^{2})
        + \ldots \nonumber \\
        &+ n_{K} ( u_{K1}^{2} + 
        u_{K2}^{2} + \ldots + u_{K j}^{2}) 
        \nonumber \\
        &= \sum_{k=1}^{K} p_{k,j} n_{k} \leq p_{1,j}n_{1} + p_{2,j}n_{2} + \ldots + p_{j-1,j}n_{j-1} + (p_{j,j} + p_{j+1,j} + p_{j+2,j} + \ldots + p_{K,j})n_{j} \nonumber \\
        &= p_{1,j}n_{1} + p_{2,j}n_{2} + \ldots + p_{j-1,j}n_{j-1} + (j - p_{1,j} + \ldots + p_{j-1,j})n_{j}
        \nonumber \\
        &= \sum_{k=1}^{j} n_{k} + \sum_{h=1}^{j-1} (n_{h}-n_{j})(p_{h,j} - 1) \leq \sum_{k=1}^{j} n_{k} \nonumber \\
        \Rightarrow \sum_{k = j+1}^{K} a_{k} &\geq N - \sum_{k=1}^{j} n_{k} = \sum_{k = j+1}^{K} n_{k} \quad \forall \: j \in [K],
        \label{eq:partial_a}
        \end{align}
        where we used the fact that $\sum_{k=1}^{K} p_{k,j} = j$ since it is the sum of squares of all entries of the first $j$ columns of an orthonormal matrix, and $p_{i,j} \leq 1 \: \forall \: i$ because it is the sum of squares of some entries on the $i$-th row of $\mathbf{U}_{W}$.
        \\
        
         By applying Lemma \ref{lm:weighted_sum} to the RHS of equation \eqref{eq:im_f_norm_deep} with $z_{k} = \frac{1}{x_{k}^{M} + 1} \: \forall \: k \leq r$ and $z_{k} = 1$ otherwise, we obtain:
        \begin{align}
            f(\mathbf{W}_{M}, \mathbf{W}_{M-1},\ldots, \mathbf{W}_{2}, \mathbf{W}_{1}, \mathbf{H}_{1}) &\geq 
            \frac{1}{2N} \sum_{k=1}^{r}
            \left( 
            \frac{n_{k}}{x_{k}^{M} + 1} + bx_{k}
            \right) + \frac{1}{2N} \sum_{h=r+1}^{K}  n_{h}
            \label{eq:im_ine_deep}
            \\
            &= \frac{1}{2N} \sum_{k=1}^{r}
            n_{k} \left( 
            \frac{1}{x_{k}^{M} + 1} + \frac{b}{n_{k}}  x_{k}
            \right) + \frac{1}{2N} \sum_{h=r+1}^{K} n_{h}.
            \label{eq:im_f_form_deep}
        \end{align}

         The minimizer of the function $g(x) = \frac{1}{x^{M}+1} + ax$ has been studied in Section \ref{sec:study_g}. Apply this result for the lower bound \eqref{eq:im_f_form_deep}, we finish bounding $f(\mathbf{W}_{M}, \mathbf{W}_{M-1},\ldots, \mathbf{W}_{2}, \mathbf{W}_{1}, \mathbf{H}_{1})$.
        \\

        Now, we study the equality conditions. In the lower bound \eqref{eq:im_f_form_deep},  by letting $x_{k}^{*}$ be the minimizer of $\frac{1}{x_{k}^{M} + 1} + \frac{b}{n_{k}}x_{k}$ for all $k \leq r$ and $x_{k}^{*} = 0$ for all $k > r$, there are only four possibilities as following:
    \begin{itemize}
        \item \textbf{Case A:} If $x^{*}_{1} > 0$ and $n_{1} > n_{2}$:  If $x^{*}_{2} = 0$, it is clear that $x^{*}_{1} > x^{*}_{2}$. Otherwise, we have $x^{*}_{1}$ and $x^{*}_{2}$ must satisfy (see Section \ref{sec:study_g} for details):
            \begin{align}
                \frac{M x_{1}^{* M-1}}{(x_{1}^{* M} + 1)^{2}} = \frac{b}{n_{1}}, \nonumber \\
                \frac{M x_{2}^{* M-1}}{(x_{2}^{* M} + 1)^{2}} = \frac{b}{n_{2}} .\nonumber
            \end{align}
        Because $\frac{b}{n_{1}} < \frac{b}{n_{2}}$ and the function $p(x) = \frac{Mx^{M-1}}{(x^{M}+1)^{2}}$ is a decreasing function when $x > \sqrt[M]{\frac{M-1}{M+1}}$, we got $x^{*}_{1} > x^{*}_{2}$.
        Hence, from the equality condition of Lemma \ref{lm:weighted_sum}, we have $a_{1} = n_{1}$. From the orthonormal property of $\mathbf{u}_{k}$, we have:
            \begin{align}
                a_{1} = (\mathbf{u}_{1} \odot \mathbf{u}_{1} )^{\top} \mathbf{n} = n_{1} u_{11}^{2} + 
                n_{2} u_{21}^{2} + \ldots + n_{k} u_{K1}^{2} \leq n_{1} (u_{11}^{2} + u_{21}^{2} + \ldots +  u_{K1}^{2}) = n_{1}.
                \nonumber
            \end{align}
        The equality holds when and only when $u_{11}^{2} = 1$ and $u_{21} = \ldots = u_{K1} = 0$.
        
        \item \textbf{Case B:} If $x^{*}_{1} > 0$ and there exists $1 < j \leq r$ such that $n_{1} = n_{2} = \ldots = n_{j} > n_{j+1}$, we have:
       \begin{align}
           \frac{1}{{x}^{M} + 1} + \frac{b}{n_{1}}x
           = \frac{1}{{x}^{M} + 1} + \frac{b}{n_{2}}x = \ldots
           = \frac{1}{{x}^{M} + 1} + \frac{b}{n_{j}}x, \nonumber
       \end{align}
       and thus, $x^{*}_{1} = x^{*}_{2} = \ldots = x^{*}_{j} > x^{*}_{j+1}$.
       Hence, from the equality condition of Lemma \ref{lm:weighted_sum}, we have $a_{1} + a_{2} + \ldots + a_{j} = n_{1} + \ldots + n_{j}$. 
       We have:
        \begin{align}
        \begin{aligned}
            \sum_{k=1}^{j} (\mathbf{u}_{k} \odot \mathbf{u}_{k} )^{\top} \mathbf{n} 
            &= 
            n_{1}( u_{11}^{2} + u_{12}^{2} + \ldots + u_{1j}^{2})
            + n_{2}( u_{21}^{2} + 
            u_{22}^{2} + \ldots + u_{2j}^{2})
            \nonumber \\ 
            &+ \ldots
            + n_{K} ( u_{K1}^{2} + 
            u_{K2}^{2} + \ldots + u_{K j}^{2}) 
            \leq \sum_{k=1}^{j} n_{j},
        \end{aligned}
        \end{align}
        where the inequality is from the fact that for any $k \in [K]$, $(u_{k1}^{2} + 
        u_{k2}^{2} + \ldots + u_{k j}^{2}) \leq 1$ and $\sum_{k=1}^{K} (u_{k1}^{2} +
        u_{k2}^{2} + \ldots + u_{k j}^{2}) = j$. The equality holds iff $u_{k1}^{2} + 
        u_{k2}^{2} + \ldots + u_{k j}^{2} = 1 \: \forall \: k = 1, 2, \ldots, j$ and $u_{k1} = u_{k2} = \ldots = u_{kj} = 0 \: \forall \: k = j+1, \ldots, K$, i.e. the upper left sub-matrix size $j \times j$ of $\mathbf{U}_{W_{M}}$ is an orthonormal matrix and other entries of $\mathbf{U}_{W_{M}}$ lie on the same rows or columns with this sub-matrix must all equal $0$'s.

        \item \textbf{Case C:} If $x_{1}^{*} > 0$, $r < K$
        and there exists $r < j \leq K$ such that $n_{1} = n_{2} = \ldots = n_{r} = \ldots = n_{j} > n_{j+1}$, we have $x_{1}^{*} = x_{2}^{*} = \ldots = x_{r}^{*} > 0$ and $x_{r+1}^{*} = \ldots = x_{K}^{*} = 0$. Hence, from the equality condition of Lemma \ref{lm:weighted_sum}, we have $a_{1} + a_{2} + \ldots + a_{r} = n_{1} + \ldots + n_{r}$. 
       We have:
        \begin{align}
        \begin{aligned}
            \sum_{k=1}^{r} (\mathbf{u}_{k} \odot \mathbf{u}_{k} )^{\top} \mathbf{n} 
            &= 
            n_{1}( u_{11}^{2} + u_{12}^{2} + \ldots + u_{1r}^{2})
            + n_{2}( u_{21}^{2} + 
            u_{22}^{2} + \ldots + u_{2r}^{2})
            \nonumber \\ 
            &+ \ldots
            + n_{K} ( u_{K1}^{2} + 
            u_{K2}^{2} + \ldots + u_{K r}^{2}) 
            \leq \sum_{k=1}^{r} n_{k},
        \end{aligned}
        \end{align}
        where the inequality is from the fact that for any $k \in [K]$, $(u_{k1}^{2} + 
        u_{k2}^{2} + \ldots + u_{k r}^{2}) \leq 1$ and $\sum_{k=1}^{K} (u_{k1}^{2} +
        u_{k2}^{2} + \ldots + u_{k r}^{2}) = r$. 
        The equality holds iff $u_{k1} = u_{k2} = \ldots = u_{kr} = 0 \: \forall \: k = j+1, \ldots, K$, i.e. the upper left sub-matrix size $j \times r$ of $\mathbf{U}_{W_{M}}$ includes $r$ orthonormal vectors in $\mathbb{R}^{j}$ and the bottom left sub-matrix size $(K-j) \times r$ are all zeros. The other $K-r$ columns of $\mathbf{U}_{W_{M}}$ does not matter because $\mathbf{W}_{M}^{*}$ can be written as:
        \begin{align}
            \mathbf{W}_{M}^{*} =
            \sum_{k=1}^{r} s_{k}^{*} \mathbf{u}_{k} \mathbf{v}_{k}^{\top}, \nonumber
        \end{align}
        with $\mathbf{v}_{k}$ is the right singular vector that satisfies $\mathbf{W}_{M}^{* \top} \mathbf{u}_{k} = s_{k}^{*} \mathbf{v}_{k}$. Note that since $s_{1}^{*} = s_{2}^{*} = \ldots = s_{r}^{*} := s^{*}$, thus we have compact SVD form as follows:
        \begin{align}
             \mathbf{W}_{M}^{*} = s^{*} \mathbf{U}_{W_{M}}^{'} \mathbf{V}_{W_{M}}^{' \top},
        \end{align}
        where $\mathbf{U}_{W_{M}}^{'} \in \mathbb{R}^{K \times r}$ and $\mathbf{V}_{W_{M}}^{'} \in \mathbb{R}^{d \times r}$. Especially, the last $K-j$ rows of $\mathbf{W}_{M}^{*}$ will be zeros since the last $K-j$ rows of $\mathbf{U}_{W_{M}}^{'}$ are zeros.
        Furthermore, $\mathbf{U}_{W_{M}}^{'} \mathbf{U}_{W_{M}}^{' \top}$ after removing the last $K-j$ zero rows and the last $K-j$ zero columns  is the best rank-$r$ approximation of $\mathbf{I}_{j}$.
        \\
        
        We note that if \textbf{Case C} happens, then the number of positive singular values are limited by the matrix rank $r$ (e.g., by $r \leq R = \min(d_{M},\ldots, d_{1},K) < K$), and $n_{r} = n_{r+1}$, thus $x_{r}^{*} > 0$ and $x_{r+1}^{*} = 0$ ($x_{r+1}^{*}$ should equal $x_{r}^{*} > 0$ if it is not forced to be zero).
        
        \item \textbf{Case D:} If $x_{1}^{*} = 0$, we must have $x_{2}^{*} = \ldots = x_{K}^{*} = 0$, $\sum_{k=1}^{K} (\mathbf{u}_{k} \odot \mathbf{u}_{k} )^{\top} \mathbf{n}$ always equal $N$ and thus, $\mathbf{U}_{W_{M}}$ can be an arbitrary size $K \times K$ orthonormal matrix.
    \end{itemize}

We perform similar arguments as above  for all subsequent $x^{*}_{k}$'s, after we finish reasoning for prior ones. Before going to the conclusion, we first study the matrix $\mathbf{U}_{W_{M}}$. If $\textbf{Case C}$ does not happen for any $x_{k}^{*}$'s, we have:
            \begin{align}
            \mathbf{U}_{W_{M}} = \begin{bmatrix}
            \mathbf{A}_{1} & \mathbf{0} & \mathbf{0} & \mathbf{0} \\
            \mathbf{0} & \mathbf{A}_{2} & \mathbf{0} & \mathbf{0} \\
            \vdots & \vdots &  \ddots & \vdots \\
             \mathbf{0}& \mathbf{0} & \mathbf{0} & \mathbf{A}_{l} \\
            \end{bmatrix},
            \label{eq:U_W_M_normal}
            \end{align}
            where each $\mathbf{A}_{i}$ is an orthonormal block which corresponds with one or a group of classes that have the same number of training samples and their $x^{*} > 0$ (\textbf{Case A} and \textbf{Case B}) or corresponds with all classes with $x^{*} = 0$ (\textbf{Case D}). If \textbf{Case C} happens, we have:
            \begin{align}
            \mathbf{U}_{W_{M}} = \begin{bmatrix}
            \mathbf{A}_{1} & \mathbf{0} & \mathbf{0} & \mathbf{0} \\
            \mathbf{0} & \mathbf{A}_{2} & \mathbf{0} & \mathbf{0} \\
            \vdots & \vdots &  \ddots & \vdots \\
             \mathbf{0}& \mathbf{0} & \mathbf{0} & \mathbf{A}_{l} \\
            \end{bmatrix},
            \label{eq:U_W_M_abnormal}
            \end{align}
             where each $\mathbf{A}_{i}, i \in [l-1] $ is an orthonormal block which corresponds with one or a group of classes that have the same number of training samples and their $x^{*} > 0$ (\textbf{Case A} and \textbf{Case B}). $\mathbf{A}_{l}$ is the orthonormal block has the same property as $\mathbf{U}_{W_{M}}$ in \textbf{Case C}. 
             \\
    
    We consider \textbf{the case $R = K$} from now on. By using arguments about the minimizer of $g(x)$ applied to the lower bound \eqref{eq:im_f_form_deep}, we consider four cases as following:
        
        \begin{itemize}
            \item \textbf{Case 1a:}$\frac{b}{n_{1}} \leq \frac{b}{n_{2}} \leq \ldots \leq \frac{b}{n_{K}} < \frac{(M-1)^{\frac{M-1}{M}}}{M}$.
            \\ 
            
            Then, the lower bound \eqref{eq:im_f_form_deep} is minimized at $(x_{1}^{*},x_{2}^{*},\ldots, x_{K}^{*})$ where $x^{*}_{i}$ is the largest positive solution of the equation $\frac{b}{n_{i}} - \frac{M x^{M-1}}{(x^{M}+1)^{2}} = 0$ for $i = 1,2,\ldots, K$. We conclude:
            \begin{align}
                (s_{1}^{*}, s_{2}^{*}, \ldots,s_{K}^{*} ) =
                \left( \sqrt[2M]{\frac{N \lambda_{H_{1}} x_{1}^{* M}}{c}},
                \sqrt[2M]{\frac{N \lambda_{H_{1}} x_{2}^{* M}}{c}},\ldots
                \sqrt[2M]{\frac{N \lambda_{H_{1}} x_{K}^{* M}}{c}} \right).
            \end{align}

             First, we have the property that the features in each class $\mathbf{h}_{k,i}^{*}$ collapsed to their class-mean $\mathbf{h}_{k}^{*}$ $(\mathcal{NC}1)$. Let $\overline{\mathbf{H}}^{*} = \mathbf{V}_{W_{1}} \mathbf{C} \mathbf{U}_{W_{M}}^{\top}$, we know that $\mathbf{H}^{*}_{1} = \overline{\mathbf{H}}^{*} \mathbf{Y}$ from equation \eqref{eq:im_H1_form}. Then, columns from the $(n_{k-1} + 1)$-th until $(n_{k})$-th of $\mathbf{H}_{1}^{*}$ will all equals the $k$-th column of $\overline{\mathbf{H}}^{*}$, thus the features in class $k$ collapse to their class-mean $\mathbf{h}_{k}^{*}$ (which is the $k$-th column of $\overline{\mathbf{H}}^{*}$), i.e., $ \mathbf{h}_{k,1}^{*} = \mathbf{h}_{k,2}^{*} = \ldots =  \mathbf{h}_{k,n_{k}}^{*} \: \forall \: k \in [K]$.
            \\
            
            Since $r = R = K$, 
            \textbf{Case C} never happens, and we have $\mathbf{U}_{W_{M}}$ as in equation \eqref{eq:U_W_M_normal}.
            Hence, together with equations \eqref{eq:im_H1_form} and \eqref{eq:im_WH1_form}, we can conclude the geometry of the following:
            \begin{align}
                 \mathbf{W}_{M}^{*} \mathbf{W}_{M}^{* \top}
                &= \mathbf{U}_{W_{M}} \mathbf{S}_{W_{M}} \mathbf{S}_{W_{M}}^{\top} \mathbf{U}_{W_{M}}^{\top}
                = \operatorname{diag} \left(
                \frac{\lambda_{W_{1}}}{\lambda_{W_{M}}} s_{1}^{2}, \ldots, \frac{\lambda_{W_{1}}}{\lambda_{W_{M}}} s_{K}^{2}
                \right),
                \\
                \mathbf{H}_{1}^{* \top} \mathbf{H}_{1}^{*}
                &= \mathbf{Y}^{\top} \mathbf{U}_{W_{M}} \mathbf{C}^{T} \mathbf{C}
                \mathbf{U}_{W_{M}}^{\top}
                \mathbf{Y} = 
                \begin{bmatrix}
                \frac{c s_{1}^{2M}}{(c s_{1}^{2M} + N \lambda_{H_{1}})^{2}} \mathbf{1}_{n_{1}}  \mathbf{1}_{n_{1}}^{\top} & \ldots & \mathbf{0} \\
                 \vdots & \ddots & \vdots  \\
                \mathbf{0} &  \ldots& \frac{c s_{K}^{2M}}{(c s_{K}^{2M} + N \lambda_{H_{1}})^{2}} \mathbf{1}_{n_{K}}  \mathbf{1}_{n_{K}}^{\top} 
                \end{bmatrix} ,
            \end{align}
            
            \begin{align}
                \mathbf{W}_{M}^{*} \mathbf{W}_{M-1}^{*} \ldots \mathbf{W}_{2}^{*} \mathbf{W}_{1}^{*}   \mathbf{H}_{1}^{*}
                &= 
                \mathbf{U}_{W_{M}} \mathbf{S}_{W_{M}} \mathbf{S}_{W_{M-1}} \ldots
                \mathbf{S}_{W_{1}}
                \mathbf{C} \mathbf{U}_{W_{M}}^{\top}
                \mathbf{Y} \nonumber \\
                &=
                \begin{bmatrix}
                \frac{c s_{1}^{2M}}{c s_{1}^{2M} + N \lambda_{H_{1}}} \mathbf{1}_{n_{1}}^{\top} & \ldots & \mathbf{0} \\
                 \vdots & \ddots & \vdots  \\
                \mathbf{0} &  \ldots& \frac{c s_{K}^{2M}}{c s_{K}^{2M} + N \lambda_{H_{1}}} \mathbf{1}_{n_{K}}^{\top}    \end{bmatrix}.
            \end{align}
            
            We additionally have the structure of the class-means matrix:
            \begin{align}
                \overline{\mathbf{H}}^{* \top}
                \overline{\mathbf{H}}^{*} &= \mathbf{U}_{W_{M}}^{\top} \mathbf{C}^{\top} 
                \mathbf{C}
                \mathbf{U}_{W_{M}}
                = \begin{bmatrix}
                 \frac{c s_{1}^{2M}}{(c s_{1}^{2M} + N \lambda_{H_{1}})^{2}} & \ldots & 0   \\
                 \vdots & \ddots & \vdots  \\
                0 &  \ldots&  \frac{c s_{K}^{2M}}{(c s_{K}^{2M} + N \lambda_{H_{1}})^{2}} 
                \end{bmatrix}, \\
               \mathbf{W}_{M}^{*} \mathbf{W}_{M-1}^{*} \ldots \mathbf{W}_{2}^{*} \mathbf{W}_{1}^{*}  \overline{\mathbf{H}}^{*} 
               &= \mathbf{U}_{W_{M}} \mathbf{S}_{W_{M}} \mathbf{C} \mathbf{U_{W}}^{\top}
               = \begin{bmatrix}
                 \frac{c s_{1}^{2M}}{c s_{1}^{2M} + N \lambda_{H_{1}}} & \ldots & 0  \\
                 \vdots & \ddots & \vdots  \\
                0 &  \ldots&  \frac{c s_{K}^{2M}}{c s_{K}^{2M} + N \lambda_{H_{1}}}
                \end{bmatrix}.
            \end{align}
            
            And the alignment between the weights and features are as following. For any $k \in [K]$,  denote $(\mathbf{W}_{M}^{*} \mathbf{W}_{M-1}^{*} \ldots \mathbf{W}_{2}^{*} \mathbf{W}_{1}^{*})_{k}$ the $k$-th row of $\mathbf{W}_{M}^{*} \mathbf{W}_{M-1}^{*} \ldots \mathbf{W}_{2}^{*} \mathbf{W}_{1}^{*}$:
            \begin{align}
            \begin{gathered}
               \mathbf{W}_{M}^{*} \mathbf{W}_{M-1}^{*} \ldots \mathbf{W}_{2}^{*} \mathbf{W}_{1}^{*}  = \mathbf{U}_{W_{M}} \mathbf{S}_{W_{M}}
               \mathbf{S}_{W_{M-1}}\ldots
               \mathbf{S}_{W_{1}}
               \mathbf{V}_{W_{1}}^{\top},  \\
                \overline{\mathbf{H}}^{*} = \mathbf{V}_{W_{1}} 
                \mathbf{C}
                \mathbf{U}_{W_{M}}^{\top} \\
                \Rightarrow  (\mathbf{W}_{M}^{*} \mathbf{W}_{M-1}^{*} \ldots \mathbf{W}_{2}^{*} \mathbf{W}_{1}^{*})_{k} = (cs_{k}^{2M} + N \lambda_{H_{1}})
                \mathbf{h}_{k}^{*}.
            \end{gathered}
            \end{align}
            
            \item \textbf{Case 2a:} 
            There exists $j \in [K-1]$ s.t. $ \frac{b}{n_{1}} \leq \frac{b}{n_{2}} \leq \ldots \leq \frac{b}{n_{j}} < \frac{(M-1)^{\frac{M-1}{M}}}{M} < \frac{b}{n_{j+1}} \leq \ldots \leq \frac{b}{n_{K}}$.
            \\
            
           Then, the lower bound \eqref{eq:im_f_form_deep} is minimized at $(x_{1}^{*},x_{2}^{*},\ldots, x_{K}^{*})$ where $x^{*}_{i}$ is the largest positive solution of equation $\frac{b}{n_{i}} - \frac{M x^{M-1}}{(x^{M}+1)^{2}} = 0$ for $i = 1,2,\ldots, j$ and $x^{*}_{i} = 0$ for $i = j+1,\ldots,K$. We conclude:
            \begin{align}
                (s_{1}^{*}, s_{2}^{*}, \ldots,s_{j}^{*}, s_{j+1}^{*},\ldots s_{K}^{*} ) =
                \left( \sqrt[2M]{\frac{N \lambda_{H_{1}} x_{1}^{* M}}{c}},
                \sqrt[2M]{\frac{N \lambda_{H_{1}} x_{2}^{* M}}{c}},\ldots,\sqrt[2M]{\frac{N \lambda_{H_{1}} x_{j}^{* M}}{c}}, 0,\ldots,0 \right).
            \end{align}
            
             First, we have the property that the features in each class $\mathbf{h}_{k,i}^{*}$ collapsed to their class-mean $\mathbf{h}_{k}^{*}$ $(\mathcal{NC}1)$. Let $\overline{\mathbf{H}}^{*} = \mathbf{V}_{W} \mathbf{C} \mathbf{U}_{W}^{\top}$, we know that $\mathbf{H}^{*}_{1} = \overline{\mathbf{H}}^{*} \mathbf{Y}$. Then, columns from the $(n_{k-1} + 1)$-th until $(n_{k})$-th of $\mathbf{H}^{*}_{1}$ will all equals the $k$-th column of $\overline{\mathbf{H}}^{*}$, thus the features in class $k$ are collapsed to their class-mean $\mathbf{h}_{k}^{*}$ (which is the $k$-th column of $\overline{\mathbf{H}}$), i.e $ \mathbf{h}_{k,1}^{*} = \mathbf{h}_{k,2}^{*} = \ldots =  \mathbf{h}_{k,n_{k}}^{*} \forall k \in [K]$.
            \\
            
            For any $k \in [K]$,  denote $(\mathbf{W}_{M}^{*} \mathbf{W}_{M-1}^{*} \ldots \mathbf{W}_{2}^{*} \mathbf{W}_{1}^{*})_{k}$ the $k$-th row of $\mathbf{W}_{M}^{*} \mathbf{W}_{M-1}^{*} \ldots \mathbf{W}_{2}^{*} \mathbf{W}_{1}^{*}$:
            \begin{align}
            \begin{gathered}
               \mathbf{W}_{M}^{*} \mathbf{W}_{M-1}^{*} \ldots \mathbf{W}_{2}^{*} \mathbf{W}_{1}^{*}  = \mathbf{U}_{W_{M}} \mathbf{S}_{W_{M}}
               \mathbf{S}_{W_{M-1}}\ldots
               \mathbf{S}_{W_{1}}
               \mathbf{V}_{W_{1}}^{\top},  \\
                \overline{\mathbf{H}}^{*} = \mathbf{V}_{W_{1}} 
                \mathbf{C}
                \mathbf{U}_{W_{M}}^{\top} \\
                \Rightarrow  (\mathbf{W}_{M}^{*} \mathbf{W}_{M-1}^{*} \ldots \mathbf{W}_{2}^{*} \mathbf{W}_{1}^{*})_{k} = (c s_{k}^{2M} + N \lambda_{H_{1}})
                \mathbf{h}_{k}^{*}.
            \end{gathered}
            \end{align}
            And, for $k > j$, we have $(\mathbf{W}_{M}^{*} \mathbf{W}_{M-1}^{*} \ldots \mathbf{W}_{2}^{*} \mathbf{W}_{1}^{*})_{k} = \mathbf{h}_{k}^{*} = \mathbf{0}$.
            \\
            
        Recall the form of $\mathbf{U}_{W_{M}}$ as in equation \eqref{eq:U_W_M_normal} (\textbf{Case C} cannot happen since $r = j$ and $n_{j} > n_{j+1}$). We can conclude the geometry of following objects, with the usage of equations \eqref{eq:im_H1_form} and \eqref{eq:im_WH1_form}:
            \begin{align}
                \mathbf{W}_{M}^{*} \mathbf{W}_{M}^{* \top} 
                &= \mathbf{U}_{W_{M}} \mathbf{S}_{W_{M}} \mathbf{S}_{W_{M}}^{\top} \mathbf{U}_{W}^{\top} \nonumber \\
                &= \operatorname{diag} 
                \left(
                \frac{ \lambda_{W_{1}}}{\lambda_{W_{M}}} 
                    s_{1}^{2}
                    , \frac{ \lambda_{W_{1}}}{\lambda_{W_{M}}} 
                    s_{2}^{2}, \ldots, \frac{ \lambda_{W_{1}}}{\lambda_{W_{M}}} 
                    s_{j}^{2}, 0,\ldots,0 \right),
                \\
                \mathbf{H}_{1}^{* \top} \mathbf{H}_{1}^{*}
                &=
                \begin{bmatrix}
                \frac{c s_{1}^{2M}}{(c s_{1}^{2M} + N \lambda_{H_{1}})^{2}} \mathbf{1}_{n_{1}}  \mathbf{1}_{n_{1}}^{\top} & \mathbf{0} & \ldots & \mathbf{0} \\
                \mathbf{0} &
                \frac{c s_{2}^{2M}}{(c s_{2}^{2M} + N \lambda_{H_{1}})^{2}} \mathbf{1}_{n_{2}}  \mathbf{1}_{n_{2}}^{\top}  & \ldots & \mathbf{0} \\
                \vdots & \vdots & \ddots & \vdots \\
                \mathbf{0} & \mathbf{0} & \ldots & \mathbf{0}_{n_{K} \times n_{K}} \\
                \end{bmatrix}, \\
                 \mathbf{W}_{M}^{*} \mathbf{W}_{M-1}^{*} \ldots \mathbf{W}_{2}^{*} \mathbf{W}_{1}^{*}   \mathbf{H}_{1}^{*}
                &= \mathbf{U}_{W} 
                \operatorname{diag} \left( \frac{c s_{1}^{2M}}{cs_{1}^{2M} + N \lambda_{H_{1}}} , \ldots, \frac{c s_{j}^{2M}}{c s_{j}^{2M} + N \lambda_{H_{1}}}, 0,\ldots,0 \right)
              \mathbf{U}_{W}^{\top}
              \mathbf{Y} \nonumber \\
                &= 
                \begin{bmatrix}
                \frac{c s_{1}^{2M}}{cs_{1}^{2M} + N \lambda_{H_{1}}} \mathbf{1}_{n_{1}}^{\top} & \mathbf{0} & \ldots & \mathbf{0} \\
                \mathbf{0} & \frac{c s_{2}^{2M}}{c s_{2}^{2M} + N \lambda_{H_{1}}} \mathbf{1}_{n_{2}}^{\top} &  \ldots & \mathbf{0} \\
                \vdots & \vdots & \ddots & \vdots \\
                \mathbf{0} & \mathbf{0} & \ldots & \mathbf{0}_{n_{K}}^{\top} \\
                \end{bmatrix},
                \nonumber
            \end{align}
            where $\mathbf{1}_{n_{k}} \mathbf{1}_{n_{k}}^{\top}$ is a $n_{k} \times n_{k}$ matrix will all entries are $1$'s.
            
            \item \textbf{Case 3a:} $\frac{(M-1)^{\frac{M-1}{M}}}{M} < \frac{b}{n_{1}} \leq \frac{b}{n_{2}} \leq \ldots \leq \frac{b}{n_{K}}$.
            \\
            
            In this case, the lower bound \eqref{eq:im_f_form_deep} is minimized at:
            \begin{align}
                (s_{1}^{*}, s_{2}^{*}, \ldots, s_{K}^{*} ) = (0,0,\ldots,0).
            \end{align}
            Hence, the global minimizer of $f$ is $(\mathbf{W}_{M}^{*}, \mathbf{W}_{M-1}^{*},\ldots, \mathbf{W}_{2}^{*}, \mathbf{W}_{1}^{*}, \mathbf{H}_{1}^{*}) = (\mathbf{0}, \mathbf{0},\ldots, \mathbf{0})$.
            
          \item \textbf{Case 4a:} There exists $i, j \in [K]$ ($i \leq j$) such that $\frac{b}{n_{1}} \leq \frac{b}{n_{2}} \leq \ldots \leq \frac{b}{n_{i-1}} <
            \frac{b}{n_{i}} = \frac{b}{n_{i+1}} = \ldots =
            \frac{b}{n_{j}} = \frac{(M-1)^{\frac{M-1}{M}}}{M} < \frac{b}{n_{j+1}} \leq \frac{b}{n_{j+2}} \leq \ldots \leq \frac{b}{n_{K}} $.
            
             Then, the lower bound \eqref{eq:im_f_form_deep} is minimized at $(x_{1}^{*},x_{2}^{*},\ldots, x_{K}^{*})$ where $\forall \: t \leq i-1, x^{*}_{t}$ is the largest positive solution of equation $\frac{b}{n_{t}} - \frac{M x^{M-1}}{(x^{M}+1)^{2}} = 0$. If $ i \leq t \leq j,   x_{t}^{*}$ can either be $0$ or the largest positive solution of equation $\frac{b}{n_{t}} - \frac{M x^{M-1}}{(x^{M}+1)^{2}} = 0$ as long as the sequence $\{ x^{*}_{t} \}$ is a decreasing sequence. Otherwise, $\forall \: t > j$, $x^{*}_{t} = 0$. \\
             
             In this case, we have $\mathcal{NC}1$ and  $\mathcal{NC}3$ properties similar as \textbf{Case 1a}.
             \\

            For $(\mathcal{NC}2)$, we can freely choose the number of positive singular values $r$ to be any value between $i$ and $j$. Thus, \textbf{Case C} does happen for this case. As a consequence, the diagonal block $\operatorname{diag}(s_{i}^{2}, \ldots, s_{j}^{2})$ of $\mathbf{W}_{M}^{*} \mathbf{W}_{M}^{* \top}$ in \textbf{Case 1a},  will be replace by $s_{r}^{2} \mathcal{P}_{r-i+1} (\mathbf{I}_{j-i+1}) $. Similar changes are also applied for $\mathbf{H}_{1}^{* \top} \mathbf{H}_{1}^{*}$ and $\mathbf{W}_{M}^{*} \mathbf{W}_{M-1}^{*} \ldots \mathbf{W}_{2}^{*} \mathbf{W}_{1}^{*}   \mathbf{H}_{1}^{*}$.
        \end{itemize}

     Now, we turn to consider \textbf{the bottleneck case $R < K$}. Again, we consider the following cases:

    \begin{itemize}
            \item \textbf{Case 1b:}$\frac{b}{n_{1}} \leq \frac{b}{n_{2}} \leq \ldots \leq \frac{b}{n_{R}} < \frac{(M-1)^{\frac{M-1}{M}}}{M}$.
            \\ 
            
            Then, the lower bound \eqref{eq:im_f_form_deep} is minimized at $(x_{1}^{*},x_{2}^{*},\ldots, x_{K}^{*})$ where $x^{*}_{i}$ is the largest positive solution of the equation $\frac{b}{n_{i}} - \frac{M x^{M-1}}{(x^{M}+1)^{2}} = 0$ for $i = 1,2,\ldots, R$ and $x^{*}_{i} = 0$ for $i = R+1,\ldots,K$. We conclude:
            \begin{align}
                (s_{1}^{*}, s_{2}^{*}, \ldots,s_{R}^{*},s_{R+1}^{*}, \ldots s_{K}^{*} ) =
                \left( \sqrt[2M]{\frac{N \lambda_{H_{1}} x_{1}^{* M}}{c}},
                \sqrt[2M]{\frac{N \lambda_{H_{1}} x_{2}^{* M}}{c}},\ldots
                \sqrt[2M]{\frac{N \lambda_{H_{1}} x_{R}^{* M}}{c}}, 0,\ldots,0 \right).
            \end{align}

            We have $(\mathcal{NC}1)$ and $(\mathcal{NC}3)$ properties are the same as \textbf{Case 1a}.
            \\
        
            We have \textbf{Case C} happens iff $x_{R}^{*} > 0$ (already satisfied) and $n_{R} = n_{R+1}$.
            If $n_{R} > n_{R+1}$, we can conclude the geometry of the following:
            \begin{align}
                 \mathbf{W}_{M}^{*} \mathbf{W}_{M}^{* \top}
                &= \mathbf{U}_{W_{M}} \mathbf{S}_{W_{M}} \mathbf{S}_{W_{M}}^{\top} \mathbf{U}_{W_{M}}^{\top}
                = \begin{bmatrix}
                \frac{\lambda_{W_{1}}}{\lambda_{W_{M}}} s_{1}^{2} & \ldots & 0 & \ldots & 0 &  \\
                 \vdots & \ddots & \vdots & \ddots & \vdots  \\
                0 &  \ldots& \frac{\lambda_{W_{1}}}{\lambda_{W_{M}}} s_{R}^{2} & \ldots & 0 \\
                 \vdots & \ddots & \vdots & \ddots & \vdots  \\
                0 & \ldots & 0 & \ldots & 0 \\
                \end{bmatrix} \nonumber \\
                &= \operatorname{diag} \left(
                \frac{\lambda_{W_{1}}}{\lambda_{W_{M}}} s_{1}^{2}, \ldots, \frac{\lambda_{W_{1}}}{\lambda_{W_{M}}} s_{R}^{2}, 0,\ldots,0
                \right),
                \\
                \overline{\mathbf{H}}^{* \top}
                \overline{\mathbf{H}}^{*} &= \mathbf{U}_{W_{M}}^{\top} \mathbf{C}^{\top} 
                \mathbf{C}
                \mathbf{U}_{W_{M}}
                = \begin{bmatrix}
                 \frac{c s_{1}^{2M}}{(c s_{1}^{2M} + N \lambda_{H_{1}})^{2}} & \ldots & 0 & \ldots & 0 &  \\
                 \vdots & \ddots & \vdots & \ddots & \vdots  \\
                0 &  \ldots&  \frac{c s_{R}^{2M}}{(c s_{R}^{2M} + N \lambda_{H_{1}})^{2}} & \ldots & 0 \\
                 \vdots & \ddots & \vdots & \ddots & \vdots  \\
                0 & \ldots & 0 & \ldots & 0 \\
                \end{bmatrix}, \\
               \mathbf{W}_{M}^{*} \mathbf{W}_{M-1}^{*} \ldots \mathbf{W}_{2}^{*} \mathbf{W}_{1}^{*}  \overline{\mathbf{H}}^{*} 
               &= \mathbf{U}_{W_{M}} \mathbf{S}_{W_{M}} \mathbf{C} \mathbf{U}_{W_{M}}^{\top}
               = \begin{bmatrix}
                 \frac{c s_{1}^{2M}}{c s_{1}^{2M} + N \lambda_{H_{1}}} & \ldots & 0 & \ldots & 0 &  \\
                 \vdots & \ddots & \vdots & \ddots & \vdots  \\
                0 &  \ldots&  \frac{c s_{R}^{2M}}{c s_{R}^{2M} + N \lambda_{H_{1}}} & \ldots & 0 \\
                 \vdots & \ddots & \vdots & \ddots & \vdots  \\
                0 & \ldots & 0 & \ldots & 0 \\
                \end{bmatrix}.
            \end{align}
            Furthermore, for $k > R$, we have $(\mathbf{W}_{M}^{*} \mathbf{W}_{M-1}^{*} \ldots \mathbf{W}_{2}^{*} \mathbf{W}_{1}^{*})_{k} = \mathbf{h}_{k}^{*} = \mathbf{0}$.
            \\
            
            If $n_{R} = n_{R+1}$, there exists $k \leq R$, $l > R$ such that $n_{k-1} > n_{k} = n_{k+1} = \ldots = n_{R} = \ldots = n_{l} > n_{l+1}$, then :
            \begin{align}
                &\mathbf{W}_{M}^{*} 
                 \mathbf{W}_{M}^{* \top}
                 = 
                 \frac{\lambda_{W_{1}}}{\lambda_{W_{M}}}
                 \begin{bmatrix}
                    s_{1}^{2} & \ldots & \mathbf{0} & \mathbf{0} & \mathbf{0}  \\
                    \vdots & \ddots & \vdots & \vdots & \vdots  \\
                    \mathbf{0} & \ldots & s_{k-1}^{2}  & \mathbf{0} & \mathbf{0}  \\
                    \mathbf{0} & \ldots & \mathbf{0} & s_{k}^{2}  \mathcal{P}_{R-k+1}(\mathbf{I}_{l-k+1}) &  \mathbf{0}  \\
                    \mathbf{0} & \ldots & \mathbf{0} & \mathbf{0} & \mathbf{0}_{(K-l) \times (K-l)}   \\
                    \end{bmatrix}, 
                    \\
                &\overline{\mathbf{H}}^{* \top}
                \overline{\mathbf{H}}^{*} = \begin{bmatrix}
                \frac{c s_{1}^{2M}}{(c s_{1}^{2M} + N \lambda_{H_{1}})^{2}} & \ldots & \mathbf{0} & \mathbf{0} & \mathbf{0}  \\
                \vdots & \ddots & \vdots & \vdots & \vdots  \\
                \mathbf{0} & \ldots & \frac{c s_{k-1}^{2M}}{(c s_{k-1}^{2M} + N \lambda_{H_{1}})^{2}}  & \mathbf{0} & \mathbf{0}  \\
                \mathbf{0} & \ldots & \mathbf{0} & \frac{c s_{k}^{2M}}{(c s_{k}^{2M} + N \lambda_{H_{1}})^{2}}  \mathcal{P}_{R-k+1}(\mathbf{I}_{l-k+1}) &  \mathbf{0}  \\
                \mathbf{0} & \ldots & \mathbf{0} & \mathbf{0} & \mathbf{0}_{(K-l) \times (K-l)}   \\
                \end{bmatrix}, 
                \\
               &\mathbf{W}_{M}^{*} \mathbf{W}_{M-1}^{*} \ldots \mathbf{W}_{1}^{*}  \overline{\mathbf{H}}^{*} 
               = 
               \begin{bmatrix}
                 \frac{c s_{1}^{2M}}{c s_{1}^{2M} + N \lambda_{H_{1}}} & \ldots & \mathbf{0} & \mathbf{0} & \mathbf{0}  \\
                \vdots & \ddots & \vdots & \vdots & \vdots  \\
                \mathbf{0} & \ldots &  \frac{c s_{k-1}^{2M}}{c s_{k-1}^{2M} + N \lambda_{H_{1}}}  & \mathbf{0} & \mathbf{0}  \\
                \mathbf{0} & \ldots & \mathbf{0} &  \frac{c s_{k}^{2M}}{c s_{k}^{2M} + N \lambda_{H_{1}}}  \mathcal{P}_{R-k+1}(\mathbf{I}_{l-k+1}) &  \mathbf{0}  \\
                \mathbf{0} & \ldots & \mathbf{0} & \mathbf{0} & \mathbf{0}_{(K-l) \times (K-l)}   \\
                \end{bmatrix},
            \end{align}
            and, for any $h > l > R$, $(\mathbf{W}_{M}^{*} \mathbf{W}_{M-1}^{*} \ldots \mathbf{W}_{2}^{*} \mathbf{W}_{1}^{*})_{h} = \mathbf{h}_{h}^{*} = \mathbf{0}$.
            
            \item \textbf{Case 2b:} 
            There exists $j \in [R-1]$ s.t. $ \frac{b}{n_{1}} \leq \frac{b}{n_{2}} \leq \ldots \leq \frac{b}{n_{j}} < \frac{(M-1)^{\frac{M-1}{M}}}{M} < \frac{b}{n_{j+1}} \leq \ldots \leq \frac{b}{n_{R}}$.
            \\
            
           Then, the lower bound \eqref{eq:im_f_form_deep} is minimized at $(x_{1}^{*},x_{2}^{*},\ldots, x_{K}^{*})$ where $x^{*}_{i}$ is the largest positive solution of equation $\frac{b}{n_{i}} - \frac{M x^{M-1}}{(x^{M}+1)^{2}} = 0$ for $i = 1,2,\ldots, j$ and $x^{*}_{i} = 0$ for $i = j+1,\ldots,K$. We conclude:
            \begin{align}
                (s_{1}^{*}, s_{2}^{*}, \ldots,s_{j}^{*}, s_{j+1}^{*},\ldots s_{K}^{*} ) =
                \left( \sqrt[2M]{\frac{N \lambda_{H_{1}} x_{1}^{* M}}{c}},
                \sqrt[2M]{\frac{N \lambda_{H_{1}} x_{2}^{* M}}{c}},\ldots,\sqrt[2M]{\frac{N \lambda_{H_{1}} x_{j}^{* M}}{c}}, 0,\ldots,0 \right).
            \end{align}

        We have $(\mathcal{NC}1)$ and $(\mathcal{NC}3)$ properties are the same as \textbf{Case 2a}. 
        \\
        
        We can conclude the geometry of following objects, with the usage of equations \eqref{eq:im_H1_form} and \eqref{eq:im_WH1_form}:
            \begin{align}
                \mathbf{W}_{M}^{*} \mathbf{W}_{M}^{* \top} 
                &= \mathbf{U}_{W_{M}} \mathbf{S}_{W_{M}} \mathbf{S}_{W_{M}}^{\top} \mathbf{U}_{W}^{\top} \nonumber \\
                &= \operatorname{diag} 
                \left(
                \frac{ \lambda_{W_{1}}}{\lambda_{W_{M}}} 
                    s_{1}^{2}
                    , \frac{ \lambda_{W_{1}}}{\lambda_{W_{M}}} 
                    s_{2}^{2}, \ldots, \frac{ \lambda_{W_{1}}}{\lambda_{W_{M}}} 
                    s_{j}^{2}, 0,\ldots,0 \right),
                \\
                \mathbf{H}_{1}^{* \top} \mathbf{H}_{1}^{*}
                &=
                \begin{bmatrix}
                \frac{c s_{1}^{2M}}{(c s_{1}^{2M} + N \lambda_{H_{1}})^{2}} \mathbf{1}_{n_{1}}  \mathbf{1}_{n_{1}}^{\top} & \mathbf{0} & \ldots & \mathbf{0} \\
                \mathbf{0} &
                \frac{c s_{2}^{2M}}{(c s_{2}^{2M} + N \lambda_{H_{1}})^{2}} \mathbf{1}_{n_{2}}  \mathbf{1}_{n_{2}}^{\top}  & \ldots & \mathbf{0} \\
                \vdots & \vdots & \ddots & \vdots \\
                \mathbf{0} & \mathbf{0} & \ldots & \mathbf{0}_{n_{K} \times n_{K}} \\
                \end{bmatrix}, \\
                 \mathbf{W}_{M}^{*} \mathbf{W}_{M-1}^{*} \ldots \mathbf{W}_{2}^{*} \mathbf{W}_{1}^{*}   \mathbf{H}_{1}^{*}
                &= \mathbf{U}_{W} 
                \operatorname{diag} \left( \frac{c s_{1}^{2M}}{cs_{1}^{2M} + N \lambda_{H_{1}}} , \ldots, \frac{c s_{j}^{2M}}{c s_{j}^{2M} + N \lambda_{H_{1}}}, 0,\ldots,0 \right)
              \mathbf{U}_{W}^{\top}
              \mathbf{Y} \nonumber \\
                &= 
                \begin{bmatrix}
                \frac{c s_{1}^{2M}}{cs_{1}^{2M} + N \lambda_{H_{1}}} \mathbf{1}_{n_{1}}^{\top} & \mathbf{0} & \ldots & \mathbf{0} \\
                \mathbf{0} & \frac{c s_{2}^{2M}}{c s_{2}^{2M} + N \lambda_{H_{1}}} \mathbf{1}_{n_{2}}^{\top} &  \ldots & \mathbf{0} \\
                \vdots & \vdots & \ddots & \vdots \\
                \mathbf{0} & \mathbf{0} & \ldots & \mathbf{0}_{n_{K}}^{\top} \\
                \end{bmatrix},
                \nonumber
            \end{align}
            where $\mathbf{1}_{n_{k}} \mathbf{1}_{n_{k}}^{\top}$ is a $n_{k} \times n_{k}$ matrix will all entries are $1$'s. \textbf{Case C} cannot happen in this case because $r = j < R$ and $n_{j} > n_{j+1}$.
            \\
            
            And, for $k > j$, we have $(\mathbf{W}_{M}^{*} \mathbf{W}_{M-1}^{*} \ldots \mathbf{W}_{2}^{*} \mathbf{W}_{1}^{*})_{k} = \mathbf{h}_{k}^{*} = \mathbf{0}$.
            
            \item \textbf{Case 3b:} $\frac{(M-1)^{\frac{M-1}{M}}}{M} < \frac{b}{n_{1}} \leq \frac{b}{n_{2}} \leq \ldots \leq \frac{b}{n_{R}}$.
            \\
            
            In this case, the lower bound \eqref{eq:im_f_form_deep} is minimized at:
            \begin{align}
                (s_{1}^{*}, s_{2}^{*}, \ldots, s_{K}^{*} ) = (0,0,\ldots,0).
            \end{align}
            Hence, the global minimizer of $f$ is $(\mathbf{W}_{M}^{*}, \mathbf{W}_{M-1}^{*},\ldots, \mathbf{W}_{2}^{*}, \mathbf{W}_{1}^{*}, \mathbf{H}_{1}^{*}) = (\mathbf{0}, \mathbf{0},\ldots, \mathbf{0})$.
            
          \item \textbf{Case 4b:} There exists $i, j \in [R]$ ($i \leq j \leq R$) such that $\frac{b}{n_{1}} \leq \frac{b}{n_{2}} \leq \ldots \leq \frac{b}{n_{i-1}} <
            \frac{b}{n_{i}} = \frac{b}{n_{i+1}} = \ldots =
            \frac{b}{n_{j}} = \frac{(M-1)^{\frac{M-1}{M}}}{M} < \frac{b}{n_{j+1}} \leq \frac{b}{n_{j+2}} \leq \ldots \leq \frac{b}{n_{R}} 
            $.
            \\
            
             Then, the lower bound \eqref{eq:im_f_form_deep} is minimized at $(x_{1}^{*},x_{2}^{*},\ldots, x_{K}^{*})$ where $\forall \: t \leq i-1, x^{*}_{t}$ is the largest positive solution of equation $\frac{b}{n_{t}} - \frac{M x^{M-1}}{(x^{M}+1)^{2}} = 0$. If $ i \leq t \leq j,   x_{t}^{*}$ can either be $0$ or the largest positive solution of equation $\frac{b}{n_{t}} - \frac{M x^{M-1}}{(x^{M}+1)^{2}} = 0$ as long as the sequence $\{ x^{*}_{t} \}$ is a decreasing sequence and there is no more than $R$ positive singular values. Otherwise, $\forall \: t > j$, $x^{*}_{t} = 0$. 
             \\
             
             In this case, we have $(\mathcal{NC}1)$ and  $(\mathcal{NC}3)$ properties similar as \textbf{Case 1b}. 
             \\

             For $(\mathcal{NC}2)$, if $b/ n_{R} > \frac{(M-1)^{\frac{M-1}{M}}}{M}$,  we can freely choose the number of positive singular values $r$ between $i$ and $j$, thus we have similar results as in \textbf{Case 4a}.

             Otherwise, if $b/ n_{R} = \frac{(M-1)^{\frac{M-1}{M}}}{M}$, we can freely choose the number of positive singular values $r$ between $i$ and $R$, thus we still have similar geometries as in \textbf{Case 4a}.
        \end{itemize}
    We finish the proof.
\end{proof}

\section{Proof of Theorem \ref{thm:CE}}
\label{sec:CE_proof}
\begin{proof}[Proof of Theorem \ref{thm:CE}]
 Let $\mathbf{Z} =  \mathbf{W}_{M} \mathbf{W}_{M-1} \ldots \mathbf{W}_{2} \mathbf{W}_{1}   \mathbf{H}_{1}$. We begin by noting that any critical point 
 $( \mathbf{W}_{M}, \mathbf{W}_{M-1},\ldots, \mathbf{W}_{2}, \mathbf{W}_{1}, \mathbf{H}_{1}, \mathbf{b})$ of $f$ satisfies the following:
  \begin{align}
  &\frac{\partial f}{\partial \mathbf{W}_{M}} =
  \frac{2}{N} \frac{\partial g}{\partial \mathbf{Z}} \mathbf{H}_{1}^{\top} \mathbf{W}_{1}^{\top} \ldots
  \mathbf{W}_{M-1}^{\top} + \lambda_{W_{M}} \mathbf{W}_{M} = \mathbf{0}, \\
  &\frac{\partial f}{\partial \mathbf{W}_{M-1}} =
  \frac{2}{N} \mathbf{W}_{M}^{\top} \frac{\partial g}{\partial \mathbf{Z}} \mathbf{H}^{\top}_{1} \mathbf{W}_{1}^{\top} \ldots \mathbf{W}_{M-2}^{\top}  + \lambda_{W_{M-1}} \mathbf{W}_{M-1} = \mathbf{0}, \\
  &\ldots, \nonumber \\
  &\frac{\partial f}{\partial \mathbf{W}_{1}} =
  \frac{2}{N} \mathbf{W}_{2}^{\top} \mathbf{W}_{3}^{\top} \ldots \mathbf{W}_{M}^{\top} \frac{\partial g}{\partial \mathbf{Z}} \mathbf{H}_{1}^{\top}  + \lambda_{W_{1}} \mathbf{W}_{1} = \mathbf{0}, \\
  &\frac{\partial f}{\partial \mathbf{H}_{1}} =
  \frac{2}{N} \mathbf{W}_{1}^{\top} \mathbf{W}_{2}^{\top} \ldots \mathbf{W}_{M}^{\top} \frac{\partial g}{\partial \mathbf{Z}} \mathbf{H}^{\top}  + \lambda_{H_{1}} \mathbf{H}_{1} = \mathbf{0}.
\end{align}
    
    \noindent Next, we have:
\begin{align}
    &\mathbf{0} = \mathbf{W}^{\top}_{M} \frac{\partial f}{\partial \mathbf{W}_{M}} - \frac{\partial f}{\partial \mathbf{W}_{M-1}} \mathbf{W}_{M-1}^{\top} = \lambda_{W_{M}} \mathbf{W}^{\top}_{M} \mathbf{W}_{M} - \lambda_{W_{M-1}} \mathbf{W}_{M-1} \mathbf{W}^{\top}_{M-1} \nonumber \\ 
    &\Rightarrow \lambda_{W_{M}} \mathbf{W}^{\top}_{M} \mathbf{W}_{M} = \lambda_{W_{M-1}} \mathbf{W}_{M-1} \mathbf{W}^{\top}_{M-1}. \nonumber \\
    &\mathbf{0} = \mathbf{W}^{\top}_{M-1} \frac{\partial f}{\partial \mathbf{W}_{M-1}} - \frac{\partial f}{\partial \mathbf{W}_{M-2}} \mathbf{W}_{M-2}^{\top} = \lambda_{W_{M-1}} \mathbf{W}^{\top}_{M-1} \mathbf{W}_{M-1} - \lambda_{W_{M-2}} \mathbf{W}_{M-2} \mathbf{W}^{\top}_{M-2} \nonumber \\
    &\Rightarrow \lambda_{W_{M-1}} \mathbf{W}^{\top}_{M-1} \mathbf{W}_{M-1} = \lambda_{W_{M-2}} \mathbf{W}_{M-2} \mathbf{W}^{\top}_{M-2}. \nonumber 
\end{align}
Making similar argument for the other derivatives, we also have:
\begin{align}
\begin{gathered}
    \lambda_{W_{M}} \mathbf{W}^{\top}_{M} \mathbf{W}_{M} = \lambda_{W_{M-1}} \mathbf{W}_{M-1} \mathbf{W}^{\top}_{M-1},  \\
    \lambda_{W_{M-1}} \mathbf{W}^{\top}_{M-1} \mathbf{W}_{M-1} = \lambda_{W_{M-2}} \mathbf{W}_{M-2} \mathbf{W}^{\top}_{M-2},  \\
    \ldots,  \\
    \lambda_{W_{2}} \mathbf{W}_{2}^{\top} \mathbf{W}_{2} = \lambda_{W_{1}} \mathbf{W}_{1} \mathbf{W}_{1}^{\top},  \\
    \lambda_{W_{1}} \mathbf{W}_{1}^{\top} \mathbf{W}_{1} =  \lambda_{H_{1}}\mathbf{H}_{1} \mathbf{H}_{1}^{\top}.
    \label{eq:W_H}    
\end{gathered}
\end{align}

\noindent Now, let $\mathbf{H}_{1} = \mathbf{U}_{H} \mathbf{S}_{H} \mathbf{V}_{H}^{\top}$ be the SVD decomposition of $\mathbf{H}_{1}$ with orthonormal matrices $\mathbf{U} \in \mathbb{R}^{d_{1} \times d_{1}}, \mathbf{V} \in \mathbb{R}^{N \times N}$ and $\mathbf{S} \in \mathbb{R}^{d_{1} \times N}$ is a diagonal matrix with decreasing singular values. We note that from equations \eqref{eq:W_H}, $r:= \operatorname{rank} (\mathbf{W}_{M}) = \ldots = \operatorname{rank} (\mathbf{W}_{1}) = \operatorname{rank} (\mathbf{H}_{1})$ is at most $R := \min(d_{M}, d_{M-1},\ldots, d_{1}, K)$. We denote $r$ singular values  of $\mathbf{H}_{1}$ as $\left\{s_{k}\right\}_{k=1}^{r}$.

    \noindent Next, we start to bound $g(\mathbf{W}_{M} \mathbf{W}_{M-1} \ldots \mathbf{W}_{2} \mathbf{W}_{1} \mathbf{H}_{1} + \mathbf{b} \mathbf{1}^{\top})$ with techniques extended from Lemma D.3 in \cite{Zhu21}. By using Lemma \ref{lm:7} for $\mathbf{z}_{k,i} = \mathbf{W}_{M} \mathbf{W}_{M-1} \ldots \mathbf{W}_{2} \mathbf{W}_{1} \mathbf{h}_{k,i} + \mathbf{b}$ with the same scalar $c_{1}, c_{2}$ ($c_{1}$ can be chosen arbitrarily) for all $k$ and $i$, we have:
    \begin{align}
    \begin{aligned}
        &(1 + c_{1})(K-1)[ g(\mathbf{W}_{M} \mathbf{W}_{M-1} \ldots \mathbf{W}_{2} \mathbf{W}_{1}   \mathbf{H}_{1} + \mathbf{b} \mathbf{1}^{\top}) - c_{2}] \\
        = \: &(1 + c_{1})(K-1) \left[ \frac{1}{N} \sum_{k=1}^{K} \sum_{i=1}^{n} 
        \mathcal{L}_{CE} (\mathbf{W}_{M} \mathbf{W}_{M-1} \ldots \mathbf{W}_{2} \mathbf{W}_{1}   \mathbf{h}_{k,i} + \mathbf{b}, \mathbf{y}_{k}  )
        - c_{2} \right] \\
        \geq \: &\frac{1}{N} \sum_{k=1}^{K} \sum_{i=1}^{n} \left[ 
        \sum_{j=1}^{K} ((\mathbf{W}_{M} \mathbf{W}_{M-1} \ldots \mathbf{W}_{2} \mathbf{W}_{1} )_{j} \mathbf{h}_{k,i} + 
        b_{j}) - K ((\mathbf{W}_{M} \mathbf{W}_{M-1} \ldots \mathbf{W}_{2} \mathbf{W}_{1} )_{k} \mathbf{h}_{k,i} + b_{k} ) \right] \\
        = \: &\frac{1}{N} \sum_{i=1}^n \left[ \left(\sum_{k=1}^K \sum_{j=1}^K
        (\mathbf{W}_{M} \mathbf{W}_{M-1} \ldots \mathbf{W}_{1}   )_{j}
        \mathbf{h}_{k, i}
        -K \sum_{k=1}^K 
        (\mathbf{W}_{M} \mathbf{W}_{M-1} \ldots \mathbf{W}_{1})_{k}
        \mathbf{h}_{k, i} \right) + \underbrace{\sum_{k=1}^K \sum_{j=1}^K\left(b_j-b_k\right)}_{=0} \right] \\
        = \: &\frac{1}{N} \sum_{i=1}^n \left( \sum_{k=1}^K \sum_{j=1}^K
        (\mathbf{W}_{M} \mathbf{W}_{M-1} \ldots \mathbf{W}_{2} \mathbf{W}_{1} )_{j}
        \mathbf{h}_{k, i}
        -K \sum_{k=1}^K 
        (\mathbf{W}_{M} \mathbf{W}_{M-1} \ldots \mathbf{W}_{2} \mathbf{W}_{1} )_{k}
        \mathbf{h}_{k, i} 
        \right) \\
        = \: &\frac{K}{N}
        \sum_{i=1}^n \sum_{k=1}^K 
        \left[ 
        (\mathbf{W}_{M} \mathbf{W}_{M-1} \ldots \mathbf{W}_{2} \mathbf{W}_{1} )_{k} \left(
        \frac{1}{K} \sum_{j=1}^{K} (
        \mathbf{h}_{j,i} - \mathbf{h}_{k,i}) 
        \right)
        \right] \\
        = \: &\frac{1}{n}
        \sum_{i=1}^n \sum_{k=1}^K 
        (\mathbf{W}_{M} \mathbf{W}_{M-1} \ldots \mathbf{W}_{2} \mathbf{W}_{1} )_{k} ( \overline{\mathbf{h}}_{i} - \mathbf{h}_{k,i}) \\
        = \: &\frac{-1}{n}
        \sum_{i=1}^n \sum_{k=1}^K 
        (\mathbf{W}_{M} \mathbf{W}_{M-1} \ldots \mathbf{W}_{2} \mathbf{W}_{1} )_{k} (\mathbf{h}_{k, i} -  \overline{\mathbf{h}}_{i} ), 
 \label{eq:CE_form}
    \end{aligned}
    \end{align}
    where $\overline{\mathbf{h}}_{i} = \frac{1}{K} \sum_{j=1}^{K} \mathbf{h}_{j,i}$. Now, from the AM-GM inequality, we know that for any $\mathbf{u}, \mathbf{v} \in \mathbb{R}^{K}$ and any $c_{3} > 0$,
    \begin{align}
        \mathbf{u}^{\top} \mathbf{v} \leq \frac{c_{3}}{2} \| \mathbf{u} \|_2^2 + \frac{1}{2 c_{3}} \| \mathbf{v} \|_2^2. \nonumber
    \end{align}

    \noindent The equality holds when $c_{3} \mathbf{u} = \mathbf{v}$. Therefore, by applying AM-GM for each term  $ (\mathbf{W}_{M} \mathbf{W}_{M-1} \ldots \mathbf{W}_{2} \mathbf{W}_{1} )_{k} (\mathbf{h}_{k, i} -  \overline{\mathbf{h}}_{i})$, we further have:
    \begin{align}
    \begin{aligned}
        & (1 + c_{1})(K-1)[ g(\mathbf{W}_{M} \mathbf{W}_{M-1} \ldots \mathbf{W}_{2} \mathbf{W}_{1}  + \mathbf{b} \mathbf{1}^{\top}) - c_{2}] \\
        \geq & -\frac{c_3}{2} \sum_{k=1}^K\left
        \| (\mathbf{W}_{M} \mathbf{W}_{M-1} \ldots \mathbf{W}_{2} \mathbf{W}_{1} )_{k}\right\|_2^2
        -\frac{1}{2 c_3 n} \sum_{i=1}^n \sum_{k=1}^K
        \left\| \mathbf{h}_{k, i} - \overline{\mathbf{h}}_i
        \right\|_2^2 \\
        = &  -\frac{c_3}{2} \sum_{k=1}^K\left
        \| (\mathbf{W}_{M} \mathbf{W}_{M-1} \ldots \mathbf{W}_{2} \mathbf{W}_{1} )_{k}\right\|_2^2
        -\frac{1}{2 c_3 n} \sum_{i=1}^n\left[\left(\sum_{k=1}^K
        \left\|
        \mathbf{h}_{k,i} \right\|_2^2\right)
        - K\left\| \overline{\mathbf{h}}_i \right\|_2^2\right] \\
        = & -\frac{c_3}{2}\|\mathbf{W}_{M} \mathbf{W}_{M-1} \ldots \mathbf{W}_{2} \mathbf{W}_{1} \|_F^2
        - \frac{1}{2c_3n}
        \left(\|\mathbf{H}_{1}\|_F^2-K \sum_{i=1}^n\left\|\overline{\mathbf{h}}_i\right\|_2^2\right) \\
        \geq & -\frac{c_3}{2}\|\mathbf{W}_{M} \mathbf{W}_{M-1} \ldots \mathbf{W}_{2} \mathbf{W}_{1} \|_F^2 - \frac{1}{2c_3n} \|\mathbf{H}_{1}\|_F^2,        \label{eq:CE_form1}
    \end{aligned}
    \end{align}
    where the first inequality becomes an equality if and only if
    \begin{align}
        c_{3} ( \mathbf{W}_{M} \mathbf{W}_{M-1} \ldots \mathbf{W}_{2} \mathbf{W}_{1}   )_{k} = \mathbf{h}_{k,i} - \overline{\mathbf{h}}_{i} \: \forall k, i,
    \end{align}
    and we ignore the term $\sum_{i=1}^{n} \left\|\overline{\mathbf{h}}_i\right\|_2^2$ in the last inequality (equality holds iff $\overline{\mathbf{h}}_i = \mathbf{0} \: \forall i$).
    \\
    
    \noindent Now, by using equation \eqref{eq:W_H}, we have:
    \begin{align}
        \|\mathbf{W}_{M} \mathbf{W}_{M-1} \ldots \mathbf{W}_{2} \mathbf{W}_{1} \|_F^2 &= \operatorname{trace}
        (\mathbf{W}_{1}^{\top} \mathbf{W}_{2}^{\top}
        \ldots \mathbf{W}_{M-1}^{\top}
        \mathbf{W}_{M}^{\top}
        \mathbf{W}_{M} \mathbf{W}_{M-1} \ldots \mathbf{W}_{2} \mathbf{W}_{1} ) \nonumber \\
        &= \underbrace{\frac{\lambda_{H_{1}}^{M}}{ \lambda_{W_{M}} \lambda_{W_{M-1}} \ldots \lambda_{W_{1}}}}_{c}
        \operatorname{trace}
        [ (\mathbf{H}_{1} \mathbf{H}_{1}^{\top})^{M} ] = c \sum_{k=1}^{K} s_{k}^{2M}.
    \end{align}
    
    \noindent We will choose $c_{3}$ to let all the inequalities at \eqref{eq:CE_form1} become equalities, which is as following:
    \begin{align}
       & c_{3} ( \mathbf{W}_{M} \mathbf{W}_{M-1} \ldots \mathbf{W}_{2} \mathbf{W}_{1} )_{k} = \mathbf{h}_{k, i} \quad \forall k, i \nonumber \\
       \Rightarrow \: & c_{3}^{2} = 
       \frac{\sum_{k=1}^{K} \sum_{i=1}^{n} \|  \mathbf{h}_{k,i} \|_2^2}
       {n \sum_{k=1}^{K} \| ( \mathbf{W}_{M} \mathbf{W}_{M-1} \ldots \mathbf{W}_{2} \mathbf{W}_{1} )_{k}  \|_2^2   } = 
       \frac{\|\mathbf{H_{1}}\|_F^2   }{ n   \|\mathbf{W}_{M} \mathbf{W}_{M-1} \ldots \mathbf{W}_{2} \mathbf{W}_{1} \|_F^2}
       = \frac{ \sum_{k=1}^{r} s_{k}^{2} }{ c n \sum_{k=1}^{r} s_{k}^{2M} }.
       \label{eq:c3}
    \end{align}
    
    \noindent With $c_{3}$ chosen as above, continue from the lower bound at \eqref{eq:CE_form1}, we have:
    \begin{align}
        g(\mathbf{W}_{M} \mathbf{W}_{M-1} \ldots \mathbf{W}_{2} \mathbf{W}_{1}   \mathbf{H}_{1} + \mathbf{b} \mathbf{1}^{\top}) \geq \frac{1}{(1+c_{1})(K-1)} \left( 
        - \sqrt{\frac{c}{n}} \sqrt{ \left( \sum_{k=1}^{r} s_{k}^{2} \right ) \left( \sum_{k=1}^{r} s_{k}^{2M} \right) } 
        \right) + c_{2}.
        \label{eq:lower_g}
    \end{align}
    
    \noindent Using this lower bound of $f$, we have for any critical point $(\mathbf{W}_{M} \mathbf{W}_{M-1} \ldots \mathbf{W}_{2} \mathbf{W}_{1} , \mathbf{H}_{1}, \mathbf{b})$ of function $f$ and $c_{1} > 0$:
    \begin{align}
    \begin{aligned}
        &f(\mathbf{W}_{M}, \mathbf{W}_{M-1},\ldots, \mathbf{W}_{2}, \mathbf{W}_{1}, \mathbf{H}_{1}, \mathbf{b}) 
        =
        g(\mathbf{W}_{M} \mathbf{W}_{M-1} \ldots \mathbf{W}_{2} \mathbf{W}_{1}   \mathbf{H}_{1} + \mathbf{b} \mathbf{1}^{\top}) +
        \frac{\lambda_{W_{M}}}{2} \| \mathbf{W}_{M} \|^2_F 
        \\
        & +  \ldots +   \frac{\lambda_{W_{2}}}{2} \| \mathbf{W}_{2} \|^2_F + \frac{\lambda_{W_{1}}}{2} \| \mathbf{W}_{1} \|^2_F +  \frac{\lambda_{H_{1}}}{2} \| \mathbf{H}_{1} \|^2_F \\
        &\geq 
        \frac{1}{(1+c_{1})(K-1)} \left( 
        - \sqrt{\frac{c}{n}} \sqrt{ \left( \sum_{k=1}^{r} s_{k}^{2} \right ) \left( \sum_{k=1}^{r} s_{k}^{2M} \right) } 
        \right) + c_{2} 
        + \frac{\lambda_{W_{M}}}{2} \frac{\lambda_{H_{1}}}{\lambda_{W_{M}}} \sum_{k=1}^{r} s_{k}^{2}
        \\
        &+ \ldots 
        +
        \frac{\lambda_{W_{1}}}{2} \frac{\lambda_{H_{1}}}{\lambda_{W_{1}}} \sum_{k=1}^{r} s_{k}^{2}
        + \frac{\lambda_{H_{1}}}{2}  \sum_{k=1}^{r} s_{k}^{2}
        + \frac{\lambda_{b}}{2} \|  \mathbf{b} \|_2^2 \\
        &= 
        \underbrace{\frac{1}{(1+c_{1})(K-1)} \left( 
        - \sqrt{\frac{c}{n}} \sqrt{ \left( \sum_{k=1}^{r} s_{k}^{2} \right ) \left( \sum_{k=1}^{r} s_{k}^{2M} \right) } 
        \right) + c_{2} + \frac{M+1}{2} \lambda_{H_{1}} \sum_{k=1}^{r} s_{k}^{2}}
        _{\xi (s_{1}, s_{2},\ldots, s_{r}, \lambda_{W_{2}}, \lambda_{W_{1}}, \lambda_{H_{1}} )} +  \frac{\lambda_{b}}{2} \|  \mathbf{b} \|_2^2 \\
        &\geq \xi(s_{1}, s_{2},\ldots, s_{r}, \lambda_{W_{M}},\ldots, \lambda_{W_{1}}, \lambda_{H_{1}}),
    \end{aligned}
    \end{align}
    where the last inequality becomes an equality when either $\mathbf{b} = \mathbf{0}$ or $\lambda_{b} = 0$.
    \\
    
    \noindent From Lemma \ref{lm:8}, we know that the inequality $f(\mathbf{W}_{M}, \mathbf{W}_{M-1},\ldots, \mathbf{W}_{2}, \mathbf{W}_{1}, \mathbf{H}_{1}, \mathbf{b}) \geq  \xi(s_{1}, s_{2},\ldots, s_{r}, \lambda_{W_{M}},\ldots, \lambda_{W_{1}}, \lambda_{H_{1}})$ becomes equality if and only if:
    \begin{align}
        \begin{aligned}
        & \left\| (\mathbf{W}_{M} \mathbf{W}_{M-1} \ldots  \mathbf{W}_{1} )_{1} \right\|_2
        =
        \left\|
        (\mathbf{W}_{M} \mathbf{W}_{M-1} \ldots  \mathbf{W}_{1} )_{2}
        \right\|_2
        =
        \cdots
        =
        \left\|(\mathbf{W}_{M} \mathbf{W}_{M-1} \ldots  \mathbf{W}_{1} )_{K}
        \right\|_2, \\ 
        & \mathbf{b} = \mathbf{0}  \text { or } \lambda_{b} = 0 , \\
        & \overline{\mathbf{h}}_i:=\frac{1}{K} \sum_{j=1}^K \mathbf{h}_{j, i}=\mathbf{0}, \quad \forall i \in[n], \quad \text { and } 
        \quad 
        c_{3} (\mathbf{W}_{M} \mathbf{W}_{M-1} \ldots  \mathbf{W}_{1} )_{K} = \mathbf{h}_{k, i}, \quad \forall k \in[K], i \in[n], \\
        & \mathbf{W}_{M} \mathbf{W}_{M-1} \ldots  \mathbf{W}_{1} (\mathbf{W}_{M} \mathbf{W}_{M-1} \ldots  \mathbf{W}_{1})^{\top}=\frac{ c \sum_{k=1}^{r} s_{k}^{2M}}{K-1}\left(\boldsymbol{I}_K-\frac{1}{K} \mathbf{1}_K \mathbf{1}_K^{\top}\right), \\
        & c_1 = \left[(K-1) \exp \left(
        -\frac{\sqrt{c}}{(K-1) \sqrt{n}} 
        \sqrt{ \left( \sum_{k=1}^{r} s_{k}^{2} \right ) \left( \sum_{k=1}^{r} s_{k}^{2M} \right) }
        \right)
        \right]^{-1},
        \label{eq:econ}
        \end{aligned} 
    \end{align}
    with $c_{3}$ as in equation \eqref{eq:c3}. Furthermore, $\mathbf{H}_{1}$ includes repeated columns with $K$ non-repeated columns, and the sum of these non-repeated columns is $\mathbf{0}$. Hence, $\operatorname{rank} (\mathbf{H}_{1}) \leq  \min(d_{M}, d_{M-1}, \ldots, d_{1},  K-1) = K - 1$.
    \\
    
    \noindent Now, the only work left is to prove $\xi(s_{1}, s_{2},\ldots, s_{r}, \lambda_{W_{M}},\ldots, \lambda_{W_{1}}, \lambda_{H_{1}})$ achieve its minimum at finite $s_{1},\ldots, s_{r}$ for any fixed $\lambda_{W_{M}},\ldots \lambda_{W_{1}}, \lambda_{H_{1}}$. From equation \eqref{eq:econ}, we know that $c_1 =  \left[(K-1) \exp \left(
    -\frac{\sqrt{c}}{(K-1) \sqrt{n}} 
    \sqrt{ \left( \sum_{k=1}^{r} s_{k}^{2} \right ) \left( \sum_{k=1}^{r} s_{k}^{2M} \right) }
    \right)
    \right]^{-1}$ is an increasing function in terms of $s_{1}, s_{2},   \ldots, s_{r}$, and $c_2=\frac{1}{1+c_1} \log \left(\left(1+c_1\right)(K-1)\right)+\frac{c_1}{1+c_1} \log \left(\frac{1+c_1}{c_1}\right)$ is a decreasing function in terms of $c_1$. Therefore, we observe the following:
    When any $s_{k} \rightarrow+\infty, c_1 \rightarrow+\infty$
    and  $\frac{1}{(1+c_{1})(K-1)} \left( 
    - \sqrt{\frac{c}{n}} \sqrt{ \left( \sum_{k=1}^{r} s_{k}^{2} \right ) \left( \sum_{k=1}^{r} s_{k}^{2M} \right) } 
    \right) \rightarrow 0$
    , $c_2 \rightarrow 0$, so that $\xi(s_{1},\ldots, s_{K}, \lambda_{W_{M}},\ldots \lambda_{W_{1}}, \lambda_{H_{1}})  \rightarrow+\infty$ as $s_{k} \rightarrow+\infty$. 
    \\
    
    \noindent Since $\xi(s_{1}, s_{2},\ldots, s_{r}, \lambda_{W_{M}},\ldots, \lambda_{W_{1}}, \lambda_{H_{1}})$ is a continuous function of $(s_{1}, s_{2},\ldots, s_{r})$ and  $\xi(s_{1}, s_{2},\ldots, s_{r}, \lambda_{W_{M}},\ldots, \lambda_{W_{1}}, \lambda_{H_{1}}) \rightarrow+\infty$ when any $s_{k} \rightarrow+\infty$, $\xi$ must achieves its minimum at finite $(s_{1}, s_{2},\ldots, s_{r})$. This finishes the proof.

\end{proof}

\subsection{Supporting lemmas}
\begin{lemma}[Lemma D.5 in \cite{Zhu21}]
\label{lm:7}
    Let $\boldsymbol{y}_k \in \mathbb{R}^K$ be an one-hot vector with the $k$-th entry equalling 1 for some $k \in[K]$. For any vector $\boldsymbol{z} \in \mathbb{R}^K$ and $c_1>0$, the cross-entropy loss $\mathcal{L}_{\mathrm{CE}}\left(\boldsymbol{z}, \boldsymbol{y}_k\right)$ with $\boldsymbol{y}_k$ can be lower bounded by
$$
\mathcal{L}_{\mathrm{CE}}\left(\boldsymbol{z}, \boldsymbol{y}_k\right) \geq \frac{1}{1+c_1} \frac{\left(\sum_{i=1}^K z_i\right)-K z_k}{K-1}+c_2,
$$
where $c_2=\frac{1}{1+c_1} \log \left(\left(1+c_1\right)(K-1)\right)+\frac{c_1}{1+c_1} \log \left(\frac{1+c_1}{c_1}\right)$. The inequality becomes an equality when
$$
z_i=z_j, \quad \forall i, j \neq k, \quad \text { and } \quad c_1=\left[(K-1) \exp \left(\frac{\left(\sum_{i=1}^K z_i\right)-K z_k}{K-1}\right)\right]^{-1}.
$$
\end{lemma}

\begin{lemma}[Extended from Lemma D.4 in \cite{Zhu21}]
\label{lm:8}
    Let $(\mathbf{W}_{M}, \mathbf{W}_{M-1},\ldots, \mathbf{W}_{2}, \mathbf{W}_{1}, \mathbf{H}_{1}, \mathbf{b})$ be a critical point of $f$ with $\left\{ s_{k} \right\}_{k=1}^{r}$ be the singular values of $\mathbf{H}_{1}$.
    The lower bound \eqref{eq:lower_g} of $g$  is attained for  $(\mathbf{W}_{M}, \mathbf{W}_{M-1},\ldots, \mathbf{W}_{2}, \mathbf{W}_{1}, \mathbf{H}_{1}, \mathbf{b})$ if and only if:
    \begin{align}
        \begin{aligned}
        & \left\| (\mathbf{W}_{M} \mathbf{W}_{M-1} \ldots \mathbf{W}_{2} \mathbf{W}_{1} )_{1} \right\|_2
        =
        \left\|
        (\mathbf{W}_{M} \mathbf{W}_{M-1} \ldots \mathbf{W}_{2} \mathbf{W}_{1} )_{2}
        \right\|_2
        =
        \cdots
        =
        \left\|(\mathbf{W}_{M} \mathbf{W}_{M-1} \ldots \mathbf{W}_{2} \mathbf{W}_{1} )_{K}
        \right\|_2, \\
        & \mathbf{b} = b \mathbf{1}, \\
        & \bar{\mathbf{h}}_i:=\frac{1}{K} \sum_{j=1}^K \mathbf{h}_{j, i}=\mathbf{0}, \quad \forall i \in[n], \quad \text { and } 
        \quad 
        c_{3} (\mathbf{W}_{M} \mathbf{W}_{M-1} \ldots \mathbf{W}_{2} \mathbf{W}_{1} )_{k} = \mathbf{h}_{k, i}, \quad \forall k \in[K], i \in[n], \\
        & \mathbf{W}_{M} \mathbf{W}_{M-1} \ldots \mathbf{W}_{2} \mathbf{W}_{1} (\mathbf{W}_{M} \mathbf{W}_{M-1} \ldots \mathbf{W}_{2} \mathbf{W}_{1})^{\top} = \frac{c \sum_{k=1}^{K} s_{k}^{2M}}{K-1}\left(\boldsymbol{I}_K-\frac{1}{K} \mathbf{1}_K \mathbf{1}_K^{\top}\right), \\
        & c_1 = \left[(K-1) \exp \left(
        -\frac{\sqrt{c}}{(K-1) \sqrt{n}} 
        \sqrt{ \left( \sum_{k=1}^{K} s_{k}^{2} \right ) \left( \sum_{k=1}^{K} s_{k}^{2M} \right) }
        \right)
        \right]^{-1},
        \end{aligned}
    \end{align}
    with $c_{3}$ as in equation \eqref{eq:c3}.
\end{lemma}

\begin{proof}[Proof of Lemma \ref{lm:8}]
    For the inequality \eqref{eq:lower_g}, to become an equality, first we will need two inequalities at \eqref{eq:CE_form1} to become equalities, this leads to:
    \begin{align}
        \overline{\mathbf{h}}_{i} &= 0 \quad \forall i \in  [n], \nonumber \\
        c_{3} (\mathbf{W}_{M} \mathbf{W}_{M-1} \ldots \mathbf{W}_{2} \mathbf{W}_{1} )_{k} &= \mathbf{h}_{k,i} \quad \forall k \in  [K], i \in [n], \nonumber
    \end{align}
    with $c_{3} = \sqrt{ \frac{ \sum_{k=1}^{r} s_{k}^{2} }{ c n \sum_{k=1}^{r} s_{k}^{2M} }}$ and $c = \frac{\lambda_{H_{1}}^{M}}{ \lambda_{W_{M}} \lambda_{W_{M-1}} \ldots \lambda_{W_{1}}}$.
    \\
    
    \noindent Next, we will need the inequality at \eqref{eq:CE_form} to become an equality, which is true if and only if (from the equality conditions of Lemma \ref{lm:7}):
    \begin{align}
    \begin{gathered}
        (\mathbf{W}_{M} \mathbf{W}_{M-1} \ldots \mathbf{W}_{2} \mathbf{W}_{1})_{j} \mathbf{h}_{k,i} + 
        b_{j} =  (\mathbf{W}_{M} \mathbf{W}_{M-1} \ldots \mathbf{W}_{2} \mathbf{W}_{1})_{l} \mathbf{h}_{k,i} + 
        b_{l}, \quad \forall j, l \neq k, \nonumber \\
        c_1 = \left[(K-1) \exp \left(\frac{\left(\sum_{j=1}^K [z_{k,i}]_{j}\right)-K [z_{k,i}]_{k}}{K-1}\right)\right]^{-1} \quad \forall  i \in [n]; k \in [K], \nonumber       
    \end{gathered}
    \end{align}
    with $z_{k,i} = \mathbf{W}_{M} \mathbf{W}_{M-1} \ldots \mathbf{W}_{2} \mathbf{W}_{1} \mathbf{h}_{k,i}$, and we have:
    \begin{align}
        \begin{aligned}
        \sum_{j=1}^K\left[\boldsymbol{z}_{k, i}\right]_j
        &= 
        \sum_{j=1}^K (\mathbf{W}_{M} \mathbf{W}_{M-1} \ldots \mathbf{W}_{2} \mathbf{W}_{1})_{j}
        \mathbf{h}_{k, i}
        +\sum_{j=1}^K b_j 
        =
        \sum_{j=1}^K
        \frac{1}{c_{3}} \mathbf{h}_{j,i}^{\top}
        \mathbf{h}_{k, i} +  \sum_{j=1}^K b_j \\
        & =
         K \overline{\mathbf{h}}_{i}
         \mathbf{h}_{k, i}^{\top} 
        +
         \sum_{j=1}^K b_j=K \bar{b},
        \end{aligned} \nonumber
    \end{align}
    with $\bar{b}=\frac{1}{K} \sum_{i=1}^K b_i$, and:
    \begin{align}
        K\left[\boldsymbol{z}_{k, i}\right]_k
        = K
        (\mathbf{W}_{M} \mathbf{W}_{M-1} \ldots \mathbf{W}_{2} \mathbf{W}_{1})_{k}
        \mathbf{h}_{k, i}
        + K b_k
        =
        K c_{3} \| (\mathbf{W}_{M} \mathbf{W}_{M-1} \ldots \mathbf{W}_{2} \mathbf{W}_{1})_{k} \|_2^2 +
        K b_k .
        \nonumber
    \end{align}
    With these calculations, we can calculate $c_{1}$ as following:
    \begin{align}
        \begin{aligned}
        c_1 & =\left[(K-1) \exp \left(\frac{\left(\sum_{j=1}^K\left[\boldsymbol{z}_{k, i}\right]_j\right)-K\left[\boldsymbol{z}_{k, i}\right]_k}{K-1}\right)\right]^{-1} \\
        & =
        \left[(K-1) \exp \left(\frac{K}{K-1}\left(\bar{b}
        -
        c_{3} \| (\mathbf{W}_{M} \mathbf{W}_{M-1} \ldots \mathbf{W}_{2} \mathbf{W}_{1})_{k} \|_2^2
        - b_k\right)\right)\right]^{-1} .
        \end{aligned}
        \label{eq:125}
    \end{align}
    
    \noindent Since $c_{1}$ is chosen to be the same for all $k \in [K]$, we have:
    \begin{align}
         c_{3} \| (\mathbf{W}_{M} \mathbf{W}_{M-1} \ldots \mathbf{W}_{2} \mathbf{W}_{1})_{k} \|_2^2
        + b_k = c_{3} \| (\mathbf{W}_{M} \mathbf{W}_{M-1} \ldots \mathbf{W}_{2} \mathbf{W}_{1})_{l} \|_2^2
        + b_l \quad \forall l \neq k,
        \label{eq:126}
    \end{align}
    
    \noindent Second, since $\left[z_{k, i}\right]_j=\left[z_{k, i}\right]_{\ell}$ for all $\forall j, \ell \neq k, k \in[K]$, we have:
    \begin{align}
    \begin{aligned}
    & (\mathbf{W}_{M} \mathbf{W}_{M-1} \ldots  \mathbf{W}_{1})_{j} \mathbf{h}_{k,i} + 
        b_{j} =  (\mathbf{W}_{M} \mathbf{W}_{M-1} \ldots  \mathbf{W}_{1})_{l} \mathbf{h}_{k,i} + 
        b_{l}, \quad \forall  j, l \neq k \\
    \Leftrightarrow \: & 
    c_{3} (\mathbf{W}_{M}  \ldots  \mathbf{W}_{1})_{j} (\mathbf{W}_{M}  \ldots \mathbf{W}_{1})_{k}
    + b_j
    =
    c_{3} (\mathbf{W}_{M}  \ldots  \mathbf{W}_{1})_{l} (\mathbf{W}_{M}  \ldots \mathbf{W}_{1})_{k}
    + b_{l}, \quad \forall  j, l \neq k.
    \label{eq:127}
    \end{aligned}        
    \end{align}

    \noindent Based on this and $\sum_{k=1}^{K} (\mathbf{W}_{M} \mathbf{W}_{M-1} \ldots \mathbf{W}_{2} \mathbf{W}_{1})_{k} = \frac{1}{c_{3}} \sum_{k=1}^{K} \mathbf{h}_{k,i} = \frac{1}{c_{3}} K \overline{\mathbf{h}_{i}} = \mathbf{0} $, we have:
    \begin{align}
    \begin{aligned}
    & c_{3} \left\|
    (\mathbf{W}_{M} \mathbf{W}_{M-1} \ldots \mathbf{W}_{2} \mathbf{W}_{1})_{k}
    \right\|_2^2 + b_k 
    = - c_{3} \sum_{j \neq k}  (\mathbf{W}_{M} \mathbf{W}_{M-1} \ldots \mathbf{W}_{1})_{l}
    (\mathbf{W}_{M} \mathbf{W}_{M-1} \ldots  \mathbf{W}_{1})_{k}
    + b_k \\
    & = - (K-1) c_{3}  \underbrace{ (\mathbf{W}_{M} \mathbf{W}_{M-1} \ldots \mathbf{W}_{2} \mathbf{W}_{1})_{l}
    (\mathbf{W}_{M} \mathbf{W}_{M-1} \ldots \mathbf{W}_{2} \mathbf{W}_{1})_{k}}_{l \neq k} 
    + \left(b_k+\sum_{j \neq l, k}\left( b_{l}-b_j \right)\right) \\
    & = - (K-1) c_{3}  (\mathbf{W}_{M} \mathbf{W}_{M-1} \ldots \mathbf{W}_{2} \mathbf{W}_{1})_{l}
    (\mathbf{W}_{M} \mathbf{W}_{M-1} \ldots \mathbf{W}_{2} \mathbf{W}_{1})_{k}
    + \left[2 b_k+(K-1) b_{l}-K \bar{b}\right],
    \label{eq:128}
    \end{aligned}        
    \end{align}
    for all $l \neq k$. Combining equations \eqref{eq:126} and \eqref{eq:128}, for all $k, l \in[K]$ with $k \neq l$ we have:
    $$
    2 b_k+(K-1) b_{\ell}-K \bar{b}=2 b_{l}+(K-1) b_k-K \bar{b} \quad \Longleftrightarrow \quad b_k=b_{l}, \forall k \neq l .
    $$
    Hence, we have $\mathbf{b} = b \mathbf{1}$ for some $b > 0$. Therefore, from equations \eqref{eq:126},  \eqref{eq:127} and \eqref{eq:128}:
    \begin{align}
        &\| (\mathbf{W}_{M}  \ldots  \mathbf{W}_{1})_{1}  \|_2^2 = \ldots 
        = \| (\mathbf{W}_{M}  \ldots  \mathbf{W}_{1})_{K}  \|_2^2
        = \frac{1}{K} \| (\mathbf{W}_{M}  \ldots  \mathbf{W}_{1})  \|_F^2 = \frac{c}{K} \sum_{k=1}^{r} s_{k}^{2M},
        \\
        &(\mathbf{W}_{M} \mathbf{W}_{M-1} \ldots  \mathbf{W}_{1})_{j} 
        (\mathbf{W}_{M} \mathbf{W}_{M-1} \ldots  \mathbf{W}_{1})_{k} =
        (\mathbf{W}_{M} \mathbf{W}_{M-1} \ldots  \mathbf{W}_{1})_{l} 
        (\mathbf{W}_{M} \mathbf{W}_{M-1} \ldots  \mathbf{W}_{1})_{k} 
        \nonumber \\
        &=
        - \frac{1}{K-1} \| (\mathbf{W}_{M} \mathbf{W}_{M-1} \ldots  \mathbf{W}_{1})_{k}  \|_2^2 = - \frac{c}{K(K-1)} \sum_{k=1}^{r} s_{k}^{2M} \quad \forall j, l \neq k,
    \end{align}
    and this is equivalent to:
    \begin{align}
        (\mathbf{W}_{M} \mathbf{W}_{M-1} \ldots  \mathbf{W}_{1}) (\mathbf{W}_{M} \mathbf{W}_{M-1} \ldots  \mathbf{W}_{1})^{\top} =  \frac{c \sum_{k=1}^{r} s_{k}^{2M}}{K-1} \left( 
        \mathbf{I}_{K} - \frac{1}{K} \mathbf{1}_{K} \mathbf{1}_{K}^{\top}
        \right).
    \end{align}
    
    \noindent Continue with $c_{1}$ in equation \eqref{eq:125}, we have:
    \begin{align}
        c_1 &= \left[ (K-1) \exp \left( \frac{-K}{K-1}
        c_{3} \| (\mathbf{W}_{M} \mathbf{W}_{M-1} \ldots  \mathbf{W}_{1})_{k} \|_2^2 \right)
       \right]^{-1} \nonumber \\
       &= \left[(K-1) \exp \left(
        -\frac{\sqrt{c}}{(K-1) \sqrt{n}} 
        \sqrt{ \left( \sum_{k=1}^{r} s_{k}^{2} \right ) \left( \sum_{k=1}^{r} s_{k}^{2M} \right) }
        \right)
        \right]^{-1}. \nonumber
    \end{align}
\end{proof}

\end{document}